\documentclass[10pt,a4 paper]{article}
\usepackage[margin=3.3cm]{geometry}
\usepackage[pdftex]{graphicx}

\usepackage{times}
\usepackage[utf8]{inputenc} 
\usepackage[T1]{fontenc}

\usepackage{url}            
\usepackage{booktabs}       
\usepackage{amsfonts}       
\usepackage{nicefrac}       
\usepackage{microtype}      
\usepackage[ruled]{algorithm2e}

\usepackage{amsmath}
\usepackage{amsthm}
\usepackage{amssymb}
\usepackage{bbm}
\usepackage{paralist}
\usepackage{enumitem}
\usepackage{xcolor}
\usepackage{mathtools}
\usepackage{dsfont}
\usepackage{subcaption}
\usepackage{yfonts}

\definecolor{yale}{RGB}{14,77,146}
\usepackage[pagebackref=false,breaklinks=true,colorlinks,bookmarks=false,citecolor=yale]{hyperref}

\usepackage{color, colortbl}
\definecolor{Gray}{gray}{0.95}
\definecolor{LightCyan}{rgb}{0.88,1,1}

\DeclareMathOperator{\tr}{tr}
\newcommand{\E}{\mathbb{E}}

\DeclareMathOperator*{\argmin}{arg\,min}

\newcommand{\Alg}{\textnormal{CEC}({\cal T})}

\newtheorem{theorem}{Theorem}
\newtheorem{definition}{Definition}
\newtheorem{lemma}{Lemma}
\newtheorem{remark}{Remark}

\newtheorem{proposition}{Proposition}
\newtheorem{corollary}{Corollary}

\renewcommand{\Pr}{\mathbb{P}}

\newcommand{\poly}{\mathrm{poly}}

\title{Minimal Expected Regret\\
in Linear Quadratic Control}

\author{%
  Yassir Jedra \\
  KTH Royal Institute of Technology\\
  Stockholm, Sweden \\
  \texttt{jedra@kth.se}
  \and
  Alexandre Proutiere\\
  KTH Royal Institute of Technology\\
  Stockholm, Sweden \\
  \texttt{alepro@kth.se}
}

\date{}

\begin{document}

\maketitle


\begin{abstract}
We consider the problem of online learning in Linear Quadratic Control systems whose state transition and state-action transition matrices $A$ and $B$ may be initially unknown. We devise an online learning algorithm and provide guarantees on its expected regret. This regret at time $T$ is upper bounded (i) by $\widetilde{O}((d_u+d_x)\sqrt{d_xT})$ when $A$ and $B$ are unknown, (ii) by $\widetilde{O}(d_x^2\log(T))$ if only $A$ is unknown, and (iii) by $\widetilde{O}(d_x(d_u+d_x)\log(T))$ if only $B$ is unknown and under some mild non-degeneracy condition ($d_x$ and $d_u$ denote the dimensions of the state and of the control input, respectively). These regret scalings are minimal in $T$, $d_x$ and $d_u$ as they match existing lower bounds in scenario (i) when $d_x\le d_u$ \cite{simchowitz2020naive}, and in scenario (ii) \cite{lai1986}. We conjecture that our upper bounds are also optimal in scenario (iii) (there is no known lower bound in this setting).

Existing online algorithms proceed in epochs of (typically exponentially) growing durations. The control policy is fixed within each epoch, which considerably simplifies the analysis of the estimation error on $A$ and $B$ and hence of the regret. Our algorithm departs from this design choice: it is a simple variant of certainty-equivalence regulators, where the estimates of $A$ and $B$ and the resulting control policy can be updated as frequently as we wish, possibly at every step. Quantifying the impact of such a constantly-varying control policy on the performance of these estimates and on the regret constitutes one of the technical challenges tackled in this paper.



\end{abstract}



\section{Introduction}\label{sec:intro}

The Linear Quadratic Regulator (LQR) problem arguably constitutes the most iconic, studied, and applied problem in control theory. In this problem, the system dynamics are approximated by those of a linear system, which, in discrete time, are $x_{t+1}=Ax_t+Bu_t+\eta_t$. $x_t$ and $u_t$ represent the state and action vectors at time $t$, respectively, and typically, the noise sequence $(\eta_t)_{t\ge 0}$ is i.i.d. Gaussian. The decision maker experiences an instantaneous cost quadratic in both the state and the control input $x_t^\top Qx_t+u_t^\top Ru_t$, where $Q$ and $R$ are positive semidefinite. Her objective is to devise a control policy minimizing her accumulated long term expected cost. In the perfect information setting (as investigated in this paper), the state is observed without noise, and is plugged in as an input in the control policy. When the state transition and state-action transition matrices $A$ and $B$ are known, the optimal control policy is a simple feedback control $u_t = K_\star x_t$ where $K_\star = - \left(R + B^\top P_\star B\right)^{-1}B^\top P_\star A$ and $P_\star$ solves the discrete algebraic Riccati equation.

This paper considers the online learning version of the LQR problem, where the matrices $A$ and $B$ may be initially unknown. In such a scenario, the decision maker must control the system while learning these matrices. Early efforts in the control community \cite{astrom1973, kumar1985} were devoted to establish the convergence and asymptotic properties of adaptive control algorithms. The regret analysis of these algorithms, initiated by Lai \cite{lai1986, lai1987} (in very specific cases) and by Abbasi-Yadkori and Szepesv\'ari \cite{abbasi2011regret}, has attracted a lot of attention recently, see e.g., \cite{mania2019certainty,cohen2019learning,faradonbeh2020input,simchowitz2020naive, abeille2020efficient, lale2020explore, cassel2020logarithmic}. Refer to \textsection \ref{sec:related} for details, and to Appendix \ref{app:related} for an even longer discussion. Despite these efforts, the picture remains blurry. We are unable from the aforementioned literature (see \cite{ziemann2021uninformative}) to determine the conditions on $(A,B,R,Q)$ under which one can achieve a regret scaling as $O(\log(T))$ or $O(\sqrt{T})$.

The objective of this work is to devise online algorithms for the LQR problem whose expected regret exhibits the best scaling in both the time horizon $T$ and the dimensions $d_x$ and $d_u$ of the state and control input, in the three envisioned scenarios: (i) when $(A,B)$ are unknown, (ii) when $A$ only is unknown, (iii) when $B$ only is unknown. In addition to guaranteeing minimal regret, we wish our algorithms to enjoy a simple and natural design and require as little inputs as possible. Existing algorithms proceed in {\it epochs} of exponentially growing durations. The control policy is fixed within each epoch. This doubling trick considerably simplifies the analysis of the estimation error on $A$ and $B$ and hence of the regret, but seems rather impractical. It is also worth noting that most algorithms take as input a level of confidence $\delta$, the time horizon $T$, a stabilizer $K_\circ$, but also sometimes known upper bounds on the norms of $A$ and $B$.

{\bf Contributions.} We propose CEC(${\cal T}$) a family of algorithms based on the certainty equivalence principle and with the following desirable properties: the control policy is not fixed within epochs, and may be updated continuously as the estimates of $A$ and $B$ and hence of the optimal control improve; it is {\it anytime}, and does not leverage any a priori information about the system (except for the knowledge of a stabilizing controller). We provide upper bounds of the expected regret of our algorithms in the three envisioned scenarios:
\begin{itemize}
\item[(i)] When $A$ and $B$ are unknown, the expected regret of CEC(${\cal T}$) is\footnote{The notation $\widetilde{O}$ hides logarithmic factors.} $\widetilde{O}((d_u+d_x)\sqrt{d_xT})$, which matches existing the lower bound derived in \cite{simchowitz2020naive} when $d_x\le d_u$;
\item[(ii)] when $A$ only is unknown, the expected regret of CEC(${\cal T}$) is $\widetilde{O}(d_x(d_u+d_x)\log(T))$, matching the lower bound proved in \cite{lai1986};
\item[(iii)] when $B$ only is unknown, the expected regret of CEC(${\cal T}$) is $\widetilde{O}(d_x(d_u+d_x)\log(T))$ under the assumption that $K_\star$ is full-rank; there is no lower bound in this setting, but we conjecture that our upper bound exhibits the optimal scalings in $T$, $d_x$, and $d_u$.
\end{itemize}
It is worth comparing the design and performance guarantees of $\Alg$ to those of existing algorithms. For scenario (i), \cite{simchowitz2020naive} presents an algorithm achieving the same regret upper bound as ours but with probability $1-\delta$ where $\delta$ is an input of the algorithm (see the next section for a detailed discussion on the difference between regret in expectation and with high probability), and using a doubling trick. For scenarios (ii) and (iii), the best regret guarantees were $O(\textrm{poly}(d_x,d_u) \log^2(T))$ \cite{cassel2020logarithmic} with unspecified polynomial dependence in $(d_x,d_u)$ (with degree at least 8 as far as we can infer from the analysis of \cite{cassel2020logarithmic}). These guarantees were achieved by an algorithm with several inputs, including the time horizon, a stabilizer, upper bounds on $\| A\|$, $\| B\|$, on the minimal ergodic cost, and that achieved under the stabilizer. Refer to Appendix \ref{app:related} for a detailed discussion.

Further note that our algorithm, $\Alg$, is anytime and does not apply any kind of doubling trick: it is a simple variant of certainty-equivalence regulators, where the estimates of $A$ and $B$ and the resulting control policy can be updated as frequently as we wish, possibly at every step. Quantifying the impact of such a constantly-varying control policy on the performance of these estimates and on the regret has eluded researchers and constitutes one of the technical challenges tackled in this paper. We address this challenge by developing a novel decomposition of the cumulative covariates matrix (sometimes referred to as Gram matrix), and by deriving concentration results on its spectrum.


\section{Related Work}\label{sec:related}

The online LQR problem is a theoretically intriguing and practically important problem that has received a lot of attention in the control and learning communities. The research efforts towards the development of algorithms with optimal regret guarantees have recently intensified. We may categorize these algorithms into two classes.

In the first class, we find algorithms based on the so-called self-tuning regulators \cite{astrom1973}, as those developed in second half of the 20th century in the control community, see \cite{kumar1985,matni2019}. Self-tuning regulators work as follows. At any given step, they estimate the unknown matrices $A$ and $B$, and apply a control policy corresponding to the optimal control obtained replacing $A$ and $B$ by their estimators. To ensure an appropriate level of excitation of the system, and the ability to learn $A$ and $B$, the control inputs are typically perturbed using white noise. The algorithms developed in \cite{lai1986,lai1987,rantzer2018, faradonbeh2020input,mania2019certainty,simchowitz2020naive,cassel2020logarithmic} obey these principles. In the second class, we find algorithms applying the Optimism in Front of Uncertainty (OFU) principle, extensively used to devise regret optimal algorithms in stochastic bandit problems \cite{lai1985,lattimore2020bandit}. These algorithms maintain confidence ellipsoids where the system parameters lie with high probaility, and select optimistically a system in this ellipsoid to compute the control policy, see \cite{abbasi2011regret,faradonbeh2017,cohen2019learning,abeille2020efficient, lale2020explore}. A description of these algorithms and of their regret guarantees can be found in Appendix \ref{app:related}.

All these algorithms use a doubling trick which considerably simplifies their analysis, and most of them are not anytime, as they use the time-horizon as input. For Scenario I where $A$ and $B$ are unknown, all are designed in the fixed confidence setting, and one can establish regret guarantees with a fixed confidence level $\delta$. The best regret upper bound so far is $\widetilde{O}(d_u\sqrt{d_xT\log(1/\delta)})$ with probability $1-\delta$ \cite{simchowitz2020naive}. For Scenario II ($B$ known) and ($A$ known), \cite{cassel2020logarithmic} presents an algorithm with an expected regret scaling as $O(\textrm{poly}(d_x,d_u) \log^2(T))$. $\Alg$ offers much better guarantees with a simplified design, and is anytime. A very detailed account of the related work is presented in Appendix \ref{app:related}.

We conclude this section with a brief discussion on the differences between regret guarantees in expectation or with high probability. Most existing online algorithms for the LQR problem have regret guarantees holding with high probability, i.e., with a level of confidence $1-\delta$ where $\delta$ is an input of the algorithms. Their regret analysis consists in identifying a "good" event under which the algorithm behaves well and holding with probability at least $1-\delta$. Devising algorithms with expected regret upper bounds is more involved since one needs to also analyze the behavior of the algorithm under the complementary event (the "bad" event). In turn, analyzing the expected regret requires a deeper understanding of the problem. There is however a method to transform an algorithm with regret guarantees with probability $1-\delta$ to an algorithm with expected regret guarantees: it consists in tuning $\delta$ as a function of the time horizon $T$, and in controlling the regret under the bad event. For example, consider Scenario I where $A$ and $B$ are unknown, and consider the algorithm of \cite{simchowitz2020naive} with regret upper bounded by $C\sqrt{T}\log(1/\delta)$ with probability $1-\delta$. Now choosing $\delta = 1/T^2$ and by applying the stabilizing controller when the state norm exceeds some threshold, it can be shown that the expected regret of the modified algorithm scales at most as $C\sqrt{T}\log(T)$. Note that this method induces a multiplicative regret cost proportional to $\log(T)$ and leads to an algorithm that requires the time horizon as input. We believe that, because of the additional $\log(T)$ multiplicative cost, this method would lead to sub-optimal expected regret guarantees in Scenarios II and III (anyway, there is no algorithms in these settings with high probability regret upper bounds).


\section{Preliminaries and Assumptions}\label{sec:prelim}

{\bf The LQR problem.} We consider a linear system $x_{t+1}=Ax_t+Bu_t+\eta_t$ as described in the introduction, and initial state $x_0=0$. $(\eta_t)_{t\ge 0}$ is a sequence of i.i.d. zero-mean, isotropic\footnote{We say that a random vector $\eta$ is isotropic if $\E[\eta \eta^\top] = I_d$. If $\eta$ is zero-mean, isotropic and $\sigma^2$-subgaussian, then we also have $1 \le 4\sigma^2$. The isotropy assumption is without loss of generality because if $\E[\eta \eta^\top] = \Sigma \succ 0$ then we can rescale the dynamics by $\Sigma^{-1/2}$.}, $\sigma^2$-sub-gaussian random vectors. The objective of the decision maker is to identify a control policy $(u_t)_{t\ge 0}$ minimize the following ergodic cost $\limsup_{T \to \infty}\frac{1}{T}\E[\sum_{t=1}^T x_t^\top  Q x_t + u_t^\top R u_t]$, where $Q$ and $R$ are positive semidefinite matrices. Under the assumption that $(A,B)$ is stabilizable, the Discrete Algebraic Riccati Equation (DARE)
$P = A^\top P A  - A^\top P B (R + B^\top P^\top B)^{-1} B^\top P A + Q$
admits a unique positive definite solution $P_\star$ \cite{kuvcera1972discrete}, and that the optimal control $(u_t)_{t\ge 0}$ that minimizes the above objective is defined as
\begin{equation}
\forall t \ge 0,\quad u_t = K_\star x_t, \ \ \hbox{ with } \ \ K_\star = - \left(R + B^\top P_\star B\right)^{-1}B^\top P_\star A.
\end{equation}
The minimum ergodic cost achieved under the feedback controller $K_\star$ is denoted by $\mathcal{J}^\star_{(A,B)}=\limsup_{T \to \infty}\frac{1}{T}\E[\sum_{t=1}^T x_t^\top  (Q+ K_\star^\top R K_\star)x_t]$.

{\bf Regret in the online LQR problem.} We investigate scenarios where $A$ and / or $B$ are initially unknown. We call Scenario I, the case when $A$ and $B$ are unknown; Scenario II -- $A$ known, the case when $A$ is known but $B$ is unknown; and Scenario II -- $B$ known, the case when $B$ is known but $A$ is unknown. An adaptive control algorithm $\pi$ is defined as a sequence of measurable functions $u_t$ from the past observations to a control input: for any $t\ge 0$, $u_t$ is ${\cal F}_t$-measurable where $\mathcal{F}_t = \sigma(x_0, u_0, \dots, x_{t-1}, u_{t-1}, x_t)$. The performance of the algorithm $\pi$ is assessed through its regret defined as:
$$
R_T(\pi) = \sum_{t=1}^T (x_t^\top Q x_t + u_t^\top R u_t) - T \mathcal{J}^\star_{(A,B)},
$$
where $(x_t,u_t)$ are the state and control input at time $t$ under $\pi$. The above regret definition is used in most related papers, and somehow assumes that an Oracle algorithm (aware of $A$ and $B$ initially) would pay a cost of $\mathcal{J}^\star_{(A,B)}$ in each round. We discuss an alternative definition in Appendix \ref{app:regretproof}, and justify this definition when the expected regret is the quantity of interest.

{\bf Assumptions.} Throughout the paper, we make the following assumptions. (i) $(\eta_t)_{t\ge 0}$ is a sequence of i.i.d. zero-mean, isotropic, $\sigma^2$-sub-gaussian random vectors. (ii) We assume w.l.o.g.\footnote{This is achieved by a change of basis in the state and input spaces, and by rescaling the dynamics. See \cite{simchowitz2020naive}.} that $Q \succ I_{d_x}$ and $R = I_{d_u}$. (iii) We assume as in most existing papers that the system $(A, B)$ is stabilizable, and that the learner has access to a stabilizing controller $K_\circ$ (i.e, $\rho(A + B  K_\circ)<1$).



\section{The $\Alg$ Algorithm}\label{sec:algo}

The pseudo-code of our algorithm, $\Alg$, is presented in Algorithm \ref{algo:cce}. It essentially based on the Certainty Equivalence principle: the control policy applied at time $t$ is the optimal control policy obtained by replacing $(A,B)$ by their Least Squares Estimators (LSEs). However $\Alg$ includes three additional components described in more details below. First, the control inputs are perturbed to ensure a sufficient excitation of the system (so that the LSE is consistent). Then, $\Alg$ exploits the stabilizing controller $K_\circ$ to avoid pathological cases where the system state could become unstable. The use of $K_\circ$ is driven by an hysteresis switching mechanism. Finally, the LSE of $(A,B)$ and the corresponding optimal policy can be updated at will, as frequently as we wish. Hence, $\Alg$ allows for lazy updates, which can be interesting in case of low computational budget.

{\bf Certainty Equivalence and lazy updates.} $\Alg$ takes as input an infinite set ${\cal T}\subset \mathbb{N}$ corresponding to the times when the control policy is updated. At such times, we compute the LSE $(A_t,B_t)$ of $(A,B)$ (see Appendix \ref{app:lse} for a pseudo-code). For example in Scenario I, we have: for $t\ge 2$,
\begin{equation}\label{eq:lse}
[A_t \ \  B_t ] = \left(\sum_{s=0}^{t-2}   x_{s+1} \begin{bmatrix} x_{s} \\ u_s\end{bmatrix}^{\top}\right) \left(   \sum_{s=0}^{t-2} \begin{bmatrix} x_{s} \\ u_s\end{bmatrix} \begin{bmatrix} x_{s} \\ u_s\end{bmatrix}^{\top} \right)^{-1}.
\end{equation}
Refer to Appendix \ref{app:lse} for the expressions of the LSE in the other scenarios. You will note that, when $B$ is unknown (Scenario II -- $A$ known), then at time $t$, we only use the sample path $(x_0,u_0, \dots, x_{t-2}, u_{t-2}, x_{t-1})$ to compute the LSE of $B$. This ensures that $\eta_t$ and ${K_{t+1}}$ are independent, which will turn to be crucial in our analysis. From the LSE $(A_t,B_t)$, we compute the updated control policy by solving Ricatti equations:
\begin{align}
P_t & = A_t^\top P_t A_t - A_t^\top P_t B_t(R + B_t^\top P_t B_t)^{-1} B_t^\top P_t A_t  + Q,\\
K_t & = -(R + B_t^\top P_t B_t)^{-1} B_t^\top P_t A_t.
\end{align}
We conclude by describing the conditions imposed on the set ${\cal T}$ of update times. We just assume that: for some constant $C>1$,
\begin{equation}\label{eq:lazy}
{\cal T}=(t_k)_{k\ge 1} \ \ \ \hbox{with} \ \ \ \forall k\ge 1, \ \ t_k< t_{k+1}\le Ct_k.
\end{equation}
These conditions are very general and are compatible with ${\cal T}=\mathbb{N}$ (update every round) and ${\cal T}=\{2^k, k\in \mathbb{N}\}$ (doubling trick).

{\bf Hysteresis switching and stability.} In $\Alg$, the calls of the stabilizer $K_\circ$ is driven by an hysteresis switching mechanism defined by two sequences of stopping times $(\tau_k, \upsilon_k)_{k\ge 1}$ where
\begin{equation*}
  \tau_k = \inf \bigg\{ t > \upsilon_k: \sum_{s=0}^t \Vert x_s \Vert^2 > \sigma^2 d_x \;  g(t) \bigg\} \ \ \text{and} \ \ \upsilon_k = \inf \bigg\{ t > \tau_{k -1}: \sum_{s=0}^t \Vert x_s\Vert^2 < \sigma^2 d_x \; f(t) \bigg\}
\end{equation*}
with $\tau_ 0 = 0$ and $g(t) \ge f(t)$ for all $t\ge 1$. By construction, the sequences are interlaced: for all $k\ge 1$, $\tau_{k-1}<\upsilon_k< \tau_k$. The use of $K_\circ$ is done as follows. For $\tau_k < t < \upsilon_{k+1}$, we use $K_\circ$ until the growth rate of $\sum_{s=0}^t \Vert x_s \Vert^2$ decreases from that of $g(t)$ to that of $f(t)$. For $\upsilon_{k+1} \le t < \tau_{k+1}$, we use the adaptive controller $K_t$. We choose $f(t) = t^{1+\gamma/2}$, $g(t)= t^{1+\gamma}$ and $h(t)= t^{\gamma}$, where $\gamma > 0$ (for the analysis, we need to have $g(t) > f(t) > t$). With these choices, we will establish that the expected number of times when $K_\circ$ is used is finite. This means that after some time, $\Alg$ only uses the certainty equivalence controller.

\noindent
{\bf Input perturbations.} A sufficient excitation of the system is achieved by sometimes adding noise to the control inputs -- mainly in Scenario I. In $\Alg$, $(\nu_{t})_{t\ge 0}$ and $(\zeta_{t})_{t\ge 0}$ are sequences of independent random vectors where for all $t \ge 1$, $\nu_{t} \sim \mathcal{N}(0, \sigma_t^2 I_{d_u})$ and $\zeta_{t} \sim \mathcal{N}(0,I_{d_u})$. We choose $\sigma_t^2 = \sqrt{d_x}   \sigma^2 / \sqrt{t}$.

\begin{algorithm}[H]\label{algo:cce}
\SetAlgoLined
\LinesNumbered
\SetKwInOut{Input}{input}\SetKwInOut{Output}{output}
 \Input{\ \ Cost matrices $Q$ and $R$, a stabilizing controller $K_\circ$, variance proxy $\sigma$ of the noise, set of rounds ${\cal T}$ for the controller updates.}
 \BlankLine
 $\ell_{-1} \gets 0$, $K_{-1}\gets 0$\;
 \For{$t\ge 0$}{
  $\ell_t \gets \begin{cases}
    0  & \text{if} \quad \sum_{s=0}^t \Vert x_s \Vert^2 > \sigma^2 d_x g(t), \\
    1  & \text{if} \quad \sum_{s=0}^t \Vert x_s \Vert^2 < \sigma^2 d_x f(t), \\
    \ell_{t-1} & \text{otherwise.}
  \end{cases}$ \\
  {\bf if}$\ (t\in {\cal T})$ \ \ compute $(A_t, B_t)$ (applying the LSE algorithm, see Appendix)\;
  \qquad\qquad \ \ \ \  compute $(P_t,K_t)$ (solving Riccati equations using $(A_t, B_t)$)\;
  {\bf else} \ \ $K_t\leftarrow K_{t-1}$\;
  \emph{Scenario II - $B$ known:}\\
  $u_t \gets \begin{cases}
     K_t  x_t \phantom{+ \nu_{1,t} }  & \text{if } \ell_t = 1 \text{ and } \Vert K_t \Vert^2 \le h(t), \\
     K_\circ  x_t  & \text{otherwise.}
  \end{cases}$ \\
  \emph{Scenario II - $A$ known:}\\
  $u_t \gets \begin{cases}
     K_t x_t  & \text{if } \ell_t = 1 \text{ and } \Vert K_t \Vert^2 \le h(t), \text{ and } \lambda_{\min} \left(\sum_{s=0}^{t-1} u_s u_s^{\top} \right) \ge \sqrt{t} \\
     K_\circ x_t + \zeta_{t}  & \text{otherwise.}
  \end{cases}$ \\
  \emph{Scenario I:}\\
  $u_t \gets \begin{cases}
     K_t  x_t + \nu_{t} & \text{if } \ell_t = 1 \text{ and } \Vert K_t \Vert^2 \le h(t), \text{ and } \lambda_{\min} \left(\sum_{s=0}^{t-1} \begin{bmatrix}
       x_s \\ u_s
     \end{bmatrix} \begin{bmatrix}
       x_s \\ u_s
     \end{bmatrix}^\top \right) \ge t^{1/4} \\
     K_\circ  x_t + \nu_{t} & \text{otherwise.}
  \end{cases}$ \\
 }
 \caption{Certainty Equivalence Control (${\cal T}$) ($\Alg$)}
\end{algorithm}


\section{Regret Guarantees}\label{sec:analysis}

In this section, we present our main results. We provide finite-time expected regret upper bounds for the $\Alg$ algorithm in the three envisioned scenarios, and give a sketch of the way they are derived. In the statement of the results, for simplicity, we use the following notations: the relationship $\lesssim$ corresponds to $\le$ up to a universal multiplicative constant. For any matrix $M$ with $\rho(M)<1$, we define $\mathcal{G}_M = \sum_{s=0}^\infty \Vert M^s\Vert$, and for any matrix $K$ such that $\rho(A+BK) < 1$, we denote $P_\star(K) = \sum_{t=0}^\infty ((A+BK)^t)^\top (Q + K^\top RK)(A+BK)^t$ . We further introduce $\mathcal{G}_\circ = \mathcal{G}_{A + B K_\circ}$, $C_B = \max(\Vert B\Vert, 1)$, $C_\circ = \max(\Vert A \Vert, \Vert B  \Vert, \Vert B  K_\circ\Vert, \Vert K_\circ \Vert, 1)$, and $C_K = \max(\Vert K_\circ\Vert, \Vert K_\star \Vert, 1)$.

\subsection{Expected regret upper bounds}

\begin{theorem}\label{th:scenarioI}
(Scenario I - $A$ and $B$ unknown) Let the set of update times ${\cal T}$ satisfy (\ref{eq:lazy}). The regret of $\pi=\Alg$ with input ${\cal T}$ satisfies in Scenario I: for all $T \ge 1$,
$$
\mathbb{E}[R_T(\pi)]\lesssim  C_1  \sqrt{d_x}(d_x +d_u)\sqrt{T} \log(T)  +  C_2.
$$
with the problem dependent constants $C_1 \lesssim  \sigma^2 C_K^2 C_B^2 \Vert P_\star \Vert^{9.5} \log(e\sigma d_x d_u \mathcal{G}_\circ C_\circ \Vert P_\star \Vert C_B )^2$ and  $C_2 \lesssim \mathrm{poly}( \sigma, d_x, d_u, \mathcal{G}_\circ ,C_\circ,\Vert P_\star \Vert, C_B )$.
\end{theorem}

The above theorem provides the first regret guarantees valid in expectation for Scenario I: the expected regret is $\widetilde{O}((d_x+d_u)\sqrt{d_xT})$. It is worth noting that the regret upper bounds match the lower bound derived in \cite{simchowitz2020naive} when $d_x\le d_u$. When ${\cal T}=\{2^k,k\in \mathbb{N}\}$, we can improve our upper bound and show:
$$
\E[R_T(\pi)] \lesssim \sigma^2 C_B^2 \Vert P_\star \Vert^{5.25} \sqrt{d_x} d_u  \sqrt{T}\log(T) + \sigma^2 \Vert P\Vert^{9.5} d_x^2 \log(T) + \mathrm{poly}( \sigma, d_x, \mathcal{G}_\circ ,C_\circ,\Vert P_\star \Vert).
$$
In \cite{simchowitz2020naive}, the authors prove a similar regret upper bound, but in the fixed confidence setting, i.e., with probability $1-\delta$, for an algorithm taking $\delta$ and $T$ as inputs. Our algorithm, $\Alg$, is anytime and enjoys regret guarantees in expectation. It is worth noting that the second term in the regret upper bound of Theorem \ref{th:scenarioI}, $\mathrm{poly}( \sigma, d_x, d_u, \mathcal{G}_\circ ,C_\circ,\Vert P_\star \Vert )$, corresponds to the regret generated in rounds where the stabilizer is used.

The next two theorems provide regret upper bounds for $\Alg$ in the remaining scenarios.

\begin{theorem}\label{th:scenarioIIB}
(Scenario II - $B$ known) Let the set of update times ${\cal T}$ satisfy (\ref{eq:lazy}). The regret of $\pi=\Alg$ with input ${\cal T}$ satisfies in Scenario II -- $B$ known: for all $T \ge 1$,
$$
\E[R_T(\pi)] \lesssim \sigma^2 \Vert P_\star \Vert^{9.5} \log( e \sigma   \mathcal{G}_\circ C_\circ \Vert P_\star \Vert d_x)^2  d_x^2   \log(T) + \mathrm{poly}( \sigma, d_x, \mathcal{G}_\circ ,C_\circ,\Vert P_\star \Vert).
$$
\end{theorem}

\begin{theorem}\label{th:scenarioIIA}
(Scenario II - $A$ known) Let the set of update times ${\cal T}$ satisfy (\ref{eq:lazy}). Assume that $K_\star K_\star^\top \succ 0 $. The regret of $\pi=\Alg$ with input ${\cal T}$ satisfies in Scenario II -- $A$ known: for all $T \ge 1$,
$$
\E[R_T(\pi)] \lesssim C_1  d_x (d_x + d_u) \log(T) + C_2,
$$
with the problem dependent constants $C_1 \!\! \lesssim \!\!\sigma^2 \Vert P_\star \Vert^{9.5} \mu_\star^{-2} \log( e \sigma C_\circ \mathcal{G}_\circ \Vert P_\star \Vert \mu_\star^{-1} C_B d_x d_u  )^2$, and $C_2 \lesssim \mathrm{poly}\big( \sigma, d_x, \mathcal{G}_\circ ,C_\circ,\Vert P_\star \Vert, \mu_\star^{-1} \big)$, where we denote $\mu_\star^2 = \min(\lambda_{\min}(K_\star K_\star^\top.1))$.
\end{theorem}

The results presented in the two above theorems significantly improve those derived in \cite{cassel2020logarithmic}. There, the authors devise an algorithm whose inputs include upper bounds on $\|A\|$ and $\|B\|$, and on the minimal ergodic cost. The expected regret of this algorithm is upper bounded by $O(\mathrm{poly}(d_x,d_u) \log^2(T))$. In contrast, $\Alg$ has an expected regret $O(d_x^2\log(T))$ when $B$ is known and  $O(d_x(d_u + d_x)\log(T))$ when $A$ is known. Note that these scalings are natural, similar to the optimal regret scalings one would typically get in stochastic bandit problems. In fact, Lai in \cite{lai1986} (see Section 3) establishes that the expected regret cannot be smaller than $d_x^2\log(T)$ in Scenario II -- $B$ known. Note that \cite{lai1986} does not contain any algorithm with expected regret guarantees.

\subsection{Sketch of the regret analysis}

We provide below a brief description of the strategy used to establish Theorems \ref{th:scenarioI}, \ref{th:scenarioIIB}, and \ref{th:scenarioIIA}. In this subsection, and only for simplicity, the notations $\gtrsim$ and  $\lesssim$ will sometimes hide the problem dependent constants that appear in the analysis. Let $\pi$ denote $\Alg$.

{\bf Step 1. Regret decomposition and integration.} Our strategy is to establish that
\begin{equation}\label{eq:regret:w.h.p.}
\forall \delta \in (0,1): \ \ \mathbb{P}\bigg(R_T(\pi) \gtrsim c_1 \psi(T) \log(e/\delta) + c_2 \poly(\log(e/\delta))\bigg) \le c_3 \poly(\log(e/\delta)) \delta,
\end{equation}
where $c_1,c_2,c_3$ are positive problem dependent constants and $\psi(T)$ is the targeted regret rate (e.g., $\psi(T)= \sqrt{T}$ in Scenario I). Integrating over $\delta$, we obtain the desired upper bound in expectation $\E[R_T(\pi)] \lesssim c_1 \log(c_3) \psi(T) + c_2 \poly(\log(c_3))$. In order to show \eqref{eq:regret:w.h.p.}, we define for each $\delta \in (0,1)$ an event $\mathcal{E}_\delta$ such that
\begin{equation}\label{eq:nice}
  \mathbb{P}(\mathcal{E}_\delta) \ge 1 - c_3 \poly(\log(e/\delta))\delta
\end{equation}
and then prove that $ \mathcal{E}_\delta \subseteq \lbrace R_T(\pi) \lesssim c_1 \psi(t) \log(e/\delta) + c_2 \poly(\log(e/\delta)) \rbrace
$ which in turn gives \eqref{eq:regret:w.h.p.}. To define $\mathcal{E}_\delta$, we have to look at how regret decomposes. In fact, provided that, the following conditions \emph{\textbf{(i)}}-\emph{\textbf{(ii)}} hold at times between $ \tau \le t \le T$, then the regret can be nicely decomposed between $\tau  \le t \le T$ as follows
$$
R_T(\pi) - R_\tau(\pi) \lesssim \sum_{t = \tau + 1}^T \Vert x_t \Vert^2_{P_\star(K_t) - P_\star(K_{t-1})} + \tr(P_\star(K_{t}) - P_\star) + \sigma_t^2 \Vert P_\star(K_t)\Vert.
$$
The conditions are:
\begin{itemize}
  \itemsep0em
\item[\emph{\textbf{(i)}}] the algorithm only uses the certainty equivalence controller $K_t$;
\item[\emph{\textbf{(ii)}}] the controller $K_t$ is sufficiently close to $K_\star$, so that that $\rho(A+BK_t) < 1$.
\end{itemize}
Clearly if at these times $\tau \le t < T$, we additionally ensure that:
\begin{itemize}
  \itemsep0em
\item[\emph{\textbf{(iii)}}] $\Vert P_\star(K_t) - P_\star\Vert \lesssim \varepsilon_t^2 $ where $(\varepsilon_t)_{t\ge1}$ is non-increasing and satisifies $\sum_{s=1}^T \varepsilon_s^2 \lesssim \psi(T)$,
\item[\emph{\textbf{(iv)}}] $\Vert P_\star(K_t)\Vert = O(1)$,
\item[\emph{\textbf{(v)}}] $\Vert x_t \Vert^2 = O(1)$,
\end{itemize}
then
$$
R_T(\pi) - R_\tau(\pi) \lesssim c_1 \sum_{t= \tau+1}^T \varepsilon_t^2  + \sigma_t^2 \lesssim c_1 (\psi(T) - \psi(\tau)).
$$
We still need to bound the regret up to time $\tau$. Again if at time $\tau$, from \emph{\textbf{(i)}} we are using the certainty equivalence controller, then, in view of the design of our algorithm, we must have $\sum_{s=0}^\tau \Vert x_\tau \Vert^2 \le \sigma^2 d_x f(s)$, and that $\max_{0 \le s \le \tau}\Vert K_s \Vert^2 \le h(\tau)$, which in turn implies that $R_\tau(\pi) \lesssim c_2'\poly(\tau)$ for some problem dependent constant $c_2'>0$. Therefore, we can see that $ \lbrace \forall t \ge \tau:  \textit{\textbf{(i) - (v)}} \ \mathrm{hold}  \rbrace \subseteq \lbrace R_T(\pi) \lesssim c_1 \psi(T) + \poly(\tau)\rbrace$. Now, if we choose $\tau \ge c_2'' \poly(\log(e/\delta))$ for some constant $c_2''>0$, if we redefine the condition \emph{\textbf{(iii)}} as $\Vert P_\star(K_t) - P_\star \Vert \lesssim c_1 \varepsilon_t^2 \log(e/\delta)$, and define
\begin{equation}\label{eq:nice2}
\mathcal{E}_\delta = \lbrace \forall t \ge c_2'' \poly(\log(e/\delta)): \textit{\textbf{(i) - (v)}} \ \mathrm{hold} \rbrace,
\end{equation}
then we have the desired set inclusion $\mathcal{E}_\delta \subseteq \lbrace R_T(\pi) \lesssim c_1 \psi(t) \log(e/\delta) + c_2 \poly(\log(e/\delta))  \rbrace$. These arguments are made precise in Appendix \ref{app:regretproof} for each of the scenarios I, II--$A$ known and II--$B$ known. The regret decomposition is stated in Lemma \ref{lem:rd}, as for the integration of the high probability bound, we refer the reader to lemma \ref{lem:integration}.


{\bf Step 2. High probability regret upper bounds.} It remains to establish \eqref{eq:nice}. The proof relies on several important ingredients. Note that the conditions \emph{\textbf{(i)}}, \emph{\textbf{(ii)}} and \emph{\textbf{(iii)}} that the event ${\cal E}_\delta$ must satisfy concern the fact that $K_t$ is played after $\tau$ and our ability to control $\Vert B (K_t - K_\star) \Vert$ and $\Vert P_\star(K_t) - P_\star \Vert$. We show that these three conditions will be satisfied if the error of our LSE $(A_t,B_t)$ of $(A,B)$ is small enough. To this aim, we use the perturbation bounds derived in Proposition \ref{prop:DARE:bounds}.
These bounds allow us to bound $\Vert B (K_t - K_\star) \Vert$ and $\Vert P_\star(K_t) - P_\star \Vert$ as a function of $\max (\Vert A_t - A \Vert^2, \Vert B_t - B\Vert^2)$. Observe that if $\Vert B (K_t - K_\star) \Vert$ is small, then playing $K_t$ will stabilize the system, and we will keep using $K_t$, which leads to the condition \emph{\textbf{(i)}} of ${\cal E}_\delta$. The perturbation bounds directly control $\Vert B (K_t - K_\star) \Vert$ and $\Vert P_\star(K_t) - P_\star \Vert$, and yield the conditions \emph{\textbf{(ii)}} and \emph{\textbf{(iii)}} of ${\cal E}_\delta$. In summary, we can establish (\ref{eq:nice}) provided that we are able to control $\max (\Vert A_t - A \Vert^2, \Vert B_t - B\Vert^2)$. More precisely, we just need to prove the following probabilistic statement:
\begin{equation}\label{eq:lsep}
\forall t \gtrsim \log(e/\delta),\ \  \max (\Vert A_t - A \Vert^2, \Vert B_t - B\Vert^2) \lesssim \frac{\log(t)}{t^{1/2}}\log(e/\delta),\ \ \hbox{w. p. }\ge 1-\delta.
\end{equation}

The results regarding the event $\mathcal{E}_\delta$ are established in Appendix \ref{app:nice}. Appendix \ref{app:hysteresis} is devoted to proving that the algorithm eventually commits to the certainty equivalence controller $K_t$. The analysis of LSE is presented in Appendix \ref{app:lse}. Results about the perturbations bounds for Riccati equations are stated in Appendix \ref{app:control}.

{\bf Step 3. Performance of the LSE under varying control.} In this last and most interesting step, we prove (\ref{eq:lsep}). It is well established, see e.g. \cite{mania2019certainty} that the error of the LSE $(A_t,B_t)$ heavily depends on the spectral properties of what we refer to as the cumulative covariates matrix. This matrix is defined as $\sum_{s=0}^{t-1} \begin{bmatrix} x_s \\ u_s \end{bmatrix} \begin{bmatrix} x_s \\ u_s \end{bmatrix}^\top$ for Scenario I, $\sum_{s=0}^{t-1} u_s^\top u_s^\top$, for Scenario II - $A$ known, and $\sum_{s=0}^{t-1} x_s x_s^\top$  for Scenario II - $B$ known. We show that for example in Scenario I, the critical condition for (\ref{eq:lsep}) to hold is that:
\begin{equation}\label{eq:lmin}
\forall t\gtrsim \log(e/\delta),\ \ \lambda_{\min}\left(\sum_{s=1}^{t-1} \begin{bmatrix} x_s \\ u_s \end{bmatrix} \begin{bmatrix} x_s \\ u_s \end{bmatrix}^\top \right) \gtrsim t^{1/2} \ \ \hbox{w. p. }\gtrsim 1-\delta.
\end{equation}
Establishing (\ref{eq:lmin}) is one of the main technical contribution of this paper, and is detailed in the next section.


\section{Spectrum of the Cumulative Covariates Matrix}\label{sec:spectral}

As explained in the previous section, a critical step in the analysis of the regret of $\Alg$ is to characterize the performance of the LSE given that the underlying controller evolves over time. In turn, this requires us to be able to control the spectrum of the cumulative covariates matrix. To this aim, we present a new decomposition of this matrix, and show how the decomposition leads to concentration results on its smallest eigenvalue. We believe that our  method is of independent general interest. We apply it to the analysis of the cumulative covariates matrix in Scenario I. Refer to Appendix \ref{app:se} for details, and for the treatment of the two other scenarios.

In this section, we denote by $\widetilde{K}_t$ the controller used by $\Alg$ at time $t$, i.e., either $K_{\circ}$ or $K_t$. Again, we hide problem dependent constants in the $\lesssim$ and $\gtrsim$ in this subsection.

\subsection{A generic recipe: decomposition and concentration}

We sketch here a method to study the smallest eigenvalue of a random random matrix of the form $\sum_{s=1}^t y_s y_s^\top$ where $y_s = z_s + M_s \xi_s$ where $(z_s, M_s, \xi_s)_{s\ge 1}$ is a stochastic process such that $\xi_s$ is independent of $(z_{1}, \dots, z_s)$ and $(M_{1}, \dots, M_s)$ for all $s\ge 1$. This model covers the cumulative covariates matrices obtained in the various scenarios. Indeed,\\
for Scenario I, we have:
$$
y_s=\begin{bmatrix}
  x_{s} \\
  u_{s}
\end{bmatrix} = z_s  + M_s \xi_s,   \quad
z_s = \begin{bmatrix} A x_{s-1} + B u_{s-1} \\
\widetilde{K}_s (Ax_{s-1} + Bu_{s-1})
\end{bmatrix}, \quad M_s =
\begin{bmatrix}
  I_{d_x} & O \\
  \widetilde{K}_{s} & I_{d_u}
\end{bmatrix},
\quad \xi_s = \begin{bmatrix}
  \eta_{s-1} \\
  \nu_{s}
\end{bmatrix}.
$$
For Scenario II -- $B$ known, we have:
$$
y_s=x_{s+1}=z_s +M_s\xi_s, \quad z_s=(A+\widetilde{K}_s) x_{s}, \quad M_s = I_{d_x},\quad \xi_s = \eta_s.
$$
For Scenario II -- $A$ known, we have:
$$
y_s=u_s = z_s + M_s \xi_s, \quad z_s = \widetilde{K}_s ( A x_{s-1} + B u_{s-1} ) \quad M_s = \widetilde{K}_s, \! \quad \xi_s = \widetilde{K}_s \eta_{s-1} + 1_{\lbrace \widetilde{K}_s = K_\circ \rbrace}\zeta_s.
$$
Next, we claim that for some $\alpha>0$, the smallest eigenvalue of $\sum_{s=1}^t y_s y_s^\top$ grows at least as $t^\alpha$ as $t$ grows large when {\bf (C1)} $\lambda_{\min}(\sum_{s=1}^t M_s M_s^\top)$ is growing at least as $t^\alpha$ and {\bf (C2)} $\lambda_{\max}\left(\sum_{s=1}^t z_s z_s^\top\right)$ grows at most polynomially in $t$. As a consequence of this claim, in all scenarios, to complete Step 3 of the regret analysis, we just need to verify that the conditions {\bf (C1)} and {\bf (C2)} hold. This verification is explained in Scenario I in the next subsection. For complete statements and proofs, refer to Appendix \ref{app:se}.

The first step towards our claim is the easiest but perhaps the most insightful: it consists in applying Lemma \ref{lem:selb} in Appendix \ref{app:se} to show that\footnote{Here we mean lower bound in the sense of the Löwner partial order over symmetric matrices.} for all positive definite matrix $V$,
\begin{equation*}
\sum_{s=1}^t y_s y_s^\top \! \succeq \! \sum_{s=1}^t (M_s\xi_s) (M_s\xi_s)^\top \! - \underbrace{\bigg(\sum_{s=1}^tz_s (M_s\xi_s)^\top \bigg)^\top \! \bigg(\sum_{s=1}^t z_s z_s^\top + V \bigg)^{-1} \! \bigg(\sum_{s=1}^t z_s (M_s\xi_s)^\top \bigg)}_{(\star)} \! - V.
\end{equation*}
The second step consists in observing that the second term $(\star)$ in the above inequality is a self-normalized matrix valued process (see for example \cite{yadkori2011}, or Proposition \ref{prop:SNP++}). Concentration results for such processes lead to
$$
\left\Vert \left(\sum_{s=1}^t z_s z_s^\top + V \right)^{-1/2} \left(\sum_{s=1}^t z_s (M_s\xi_s)^\top \right) \right\Vert^2 \lesssim \log\left(\lambda_{\max}\left(\sum_{s=1}^t z_s z_s^\top\right)\right) \quad \mathrm{w.h.p.}
$$
provided the sequence $(M_s)_{s\ge 1}$ are bounded. \\
In the last step, we derive a concentration inequality for the matrix $\sum_{s=1}^t (M_s \xi_s )(M_s \xi_s)^\top$ (see Proposition \ref{prop:RM++}). It concentrates around $\sum_{s=1}^t M_s M_s^\top$. \\
In summary, we have proved that under condition {\bf (C2)}, $\lambda_{\min}(\sum_{s=1}^t y_s y_s^\top)$ scales at least as $\lambda_{\min}(\sum_{s=1}^t M_s M_s^\top)$, which combined with {\bf (C1)} provides the desired claim.


\subsection{The recipe at work in Scenario I}

We first establish a weak growth rate for $\lambda_{\min}(\sum_{s=1}^t y_s y_s^\top)$  of order $t^{1/4}$ where we have a priori no information on the boundedness of the matrix sequence  $(M_s)_{s\ge1}$ (see statement \eqref{eq:weak} Theorem \ref{thm:I:se1}). A consequence of this first result is that LSE is consistent and therefore $(A_t, B_t)$ will eventually be sufficiently close to $(A,B)$ (See Appendix \ref{app:lse}). Using the perturbation bounds of Proposition  $\ref{prop:DARE:bounds}$ (See Appendix \ref{app:control}), we can then guarantee that eventually the sequence of $(M_s)_{s\ge1}$ will become uniformly bounded over time w.h.p..  Provided this guarantee holds, we show that the growth rate of $\lambda_{\min}(\sum_{s=1}^t y_s y_s^\top)$  may be refined to an order of $t^{1/2}$ (see the statement \eqref{eq:strong} of Theorem \ref{thm:I:se1}).

\begin{theorem}[\emph{Informal}]\label{thm:I:se1}
  Under Algorithm \ref{algo:cce}, for all $\delta \in (0,1)$ we have
  \begin{align}\label{eq:weak}
      \forall t\gtrsim \log(e/\delta):\quad  & \mathbb{P}\left( \lambda_{\min}\left(\sum_{s=0}^{t-1} \begin{bmatrix} x_s  \\ u_s \end{bmatrix} \begin{bmatrix} x_s  \\ u_s \end{bmatrix}^\top\right) \gtrsim t^{1/4}  \right) \ge 1 - \delta. &   \quad \textit{(weak rate)}
  \end{align}
  Furthermore, provided we can guarantee that $\forall t \gtrsim \log(e/\delta)$ we have $\mathbb{P}(\Vert \widetilde{K}_t - K_\star \Vert \le C_K) \ge 1-\delta$, then for all $\delta \in (0,1)$, we have
  \begin{align}\label{eq:strong}
    \forall t\gtrsim \log(e/\delta):\quad & \mathbb{P}\left( \lambda_{\min}\left( \sum_{s=0}^{t-1} \begin{bmatrix}
    x_s \\
    u_s
  \end{bmatrix} \begin{bmatrix}
    x_s \\
    u_s
  \end{bmatrix}^\top \right) \gtrsim t^{1/2}   \right) \ge 1 - \delta. & \textit{(refined rate)}
\end{align}
\end{theorem}
The proof of Theorem \ref{thm:I:se1} relies on showing that the conditions {\bf (C1)} and {\bf (C2)} hold. We can start by establishing {\bf (C2)} since its proof is common to both, statement \eqref{eq:weak} and \eqref{eq:strong}. We can observe that
$$
\lambda_{\max}\left( \sum_{s=1}^t z_s z_s^\top \right) \lesssim h(t) \left(\sum_{s=1}^t \Vert x_s \Vert^2 + \sum_{s=1}^t \Vert \eta_{s-1} \Vert^2 \right),
$$
where we use that $A x_{s-1} + B u_{s-1} = x_{s} - \eta_{s-1}$, and $\Vert \widetilde{K}_s\Vert^2 \le h(s)$ for all $s \ge 1$. We can have via a matrix concentration argument (See Proposition \ref{prop:RM++} in Appendix \ref{app:se}) that  $\sum_{s=1}^t \Vert \eta_s \Vert^2 \lesssim t$ w.h.p.. We can also establish that $\sum_{s=1}^t \Vert x_s \Vert^2$ does not grow more than a polynomial of order $g(t)h(t)$ w.h.p. under $\Alg$ (see $\ref{prop:caution}$ in Appendix \ref{app:polygrowth}). Thus, condition {\bf (C2)} is satisfied.

Establishing {\bf (C1)} is slightly more involved especially for the statement \eqref{eq:strong}, but essentially we can prove that $\lambda_{\min}( \sum_{s=1}^t (M_s\xi_s) (M_s \xi_s)^\top) \gtrsim t^{1/4}$ w.h.p., and provided we can guarantee that for all $t\ge \log(e/\delta)$, it holds that $\mathbb{P}(\Vert \widetilde{K}_t - K_\star \Vert \le C_K ) \ge 1 -\delta$, then $\lambda_{\min}( \sum_{s=1}^t (M_s\xi_s) (M_s \xi_s)^\top) \gtrsim t^{1/2}$  w.h.p. for $t \gtrsim \log(e/\delta)$. We skip the details behind these claims here due to space constraints and refer the reader to Appendix \ref{app:se}. At a high level these results follow because of the special structure of the sequence of matrices $(M_s)_{s \ge 1}$ and the independence between the sequences $(\eta_s)_{s\ge0}$ and $(\nu_s)_{s \ge 0}$.

The precise statements of Theorem \ref{thm:I:se1} and its proof are deferred to Appendix \ref{app:se}.

\section{Conclusion}\label{sec:conclusion}

In this paper, we have designed $\Alg$ a simple certainty equivalence-based algorithm for the online LQR problem. This is the first algorithm enjoying regret guarantees in expectation when the state transition and state-action transition matrices $A$ and $B$ are both unknown. The upper bounds of the expected regret of $\Alg$ have an optimal scaling in the time horizon $T$ and in most cases in the dimensions of the state and control input vectors. Yet, many interesting questions remain open. $\Alg$ exploits a stabilizer when needed, and we proved that the expected regret generated in  rounds where the stabilizer was used was finite. Does it mean that we can get rid of the stabilizer? Another interesting research direction is to investigate whether our approach and results extend to LQG systems where the decision maker receives noisy measurements of the state.

\newpage

\addcontentsline{toc}{section}{References}
\bibliographystyle{alpha}
\bibliography{ref,vr2017}

\newpage

\addcontentsline{toc}{section}{Checklist}

%
\appendix
\addcontentsline{toc}{section}{Table of Notations and Assumptions}

\newpage

\setcounter{tocdepth}{2}
\tableofcontents

\newpage
\section*{Notations and Assumptions}

\begin{itemize}
  \setlength\itemsep{0em}
  \item $f(x) \gtrsim g(x)$ means there exists an universal constant $c>0$ such that $f(x) \ge c g(x)$.
  \item $f(x) \lesssim g(x)$ means there exists an universal constant $c>0$ such that $f(x) \le c g(x)$.
  \item $\lambda_{\min}(\cdot)$ denotes the minimum eigenvalue.
  \item $\lambda_{\max}(\cdot)$ denotes the maximum eigenvalue.
  \item $\Vert \cdot \Vert$ denotes operator norm for matrices or $\ell_2$-norm for vectors.
  \item $\Vert \cdot \Vert_F$ denotes Frobeinus norm.
  \item $\Vert x \Vert_{M} = \sqrt{x^\top M x}$ for any vectors $x$.
  \item $\Vert a_{1:t}\Vert_\infty = \max_{ 1\le s \le t }\vert a_s \vert$ where $(a_s)_{s\ge 1}$ is a scalar valued sequence.
  \item $\Vert a_{1:t}\Vert_2 = \sqrt{\sum_{s=1}^t \vert a_s \vert^2}$ where $(a_s)_{s\ge1}$ is a scalar valued sequence.
  \item $d_x$ and $d_u$ denote respectively the dimension of the state/input space.
  \item $d = d_x + d_u$.
  \item $\gamma$ is a postive constant used to define $f(t)= t^{1+\gamma/2}, g(t)=t^{1+\gamma}$ and $h(t) = t^\gamma$.
  \item $\gamma_\star = \max \lbrace 1, \gamma\rbrace$.
  \item $C_\circ = \max(\Vert A \Vert, \Vert B  \Vert, \Vert B  K_\circ\Vert, \Vert K_\circ \Vert, 1)$.
  \item $\mathcal{G}_M = \sum_{s=0}^\infty \Vert M^s\Vert$.
  \item $\mathcal{G}_M(\varepsilon) = \sup \left\lbrace  \sum_{s=0}^\infty \left\Vert \prod_{k=0}^s (M + \Delta_k)\right\Vert: \Delta: \; (\Delta_t)_{t\ge1}, \sup_{t\ge 0}\Vert  \Delta_t\Vert \le \varepsilon  \right \rbrace$.
  \item $P(A,B)$ solution to the DARE corresponding to the LQR problem $(A, B, Q ,R)$.
  \item $K(A,B)$ optimal gain matrix corresponding to the LQR problem $(A, B, Q, R)$.
  \item $\mathcal{L}(M,N)$ is the solution to the Discrete Lyapunov equation $X = M^\top X M + N$.
  \item $P(A,B,K) = \mathcal{L}(A+BK, Q +K^\top R K)$.
\end{itemize}
\paragraph{Shorthands for the true parameters $(A,B)$.}
\begin{itemize}
  \setlength\itemsep{0em}
  \item $C_B = \max(\Vert B \Vert, 1)$.
  \item $C_\circ = \max(\Vert K_\circ \Vert , 1)$.
  \item $\mathcal{G}_\circ = \mathcal{G}_{A + B K_\circ}$.
  \item $\mathcal{G}_\star = \mathcal{G}_{A + B K_\star}$. 
  \item $\mathcal{G}_{\star}(\varepsilon) = \mathcal{G}_{A+BK_\star}(\varepsilon)$. 
  \item $P_\star = P(A, B)$.
  \item $K_\star = K(A, B)$.
  \item $P_\star(K) = P(A, B, K)$.
  \item $\mu_\star = \min(\sqrt{\lambda_{\min}(K_\star K_\star^\top)}, 1)$.
\end{itemize}

\paragraph{Assumptions.}
\begin{itemize}
  \setlength\itemsep{0em}
  \item We assume without loss of generality that $Q \succ I_{d_x}$ and $R = I_{d_u}$. These can be enforced by a change of basis of the state and input spaces, and rescaling the dynamics (see e.g., \cite{simchowitz2020naive}).
  \item The noise sequence $(\eta_t)_{t\ge1}$ is assumed to be i.i.d. zero-mean, isotropic and $\sigma^2$-sub-gaussian random vectors. Isotropy here is assumed for simplicity and is without loss of generality. Observe that isotropy implies $4\sigma^2 \ge 1$.
  \item In all the envisioned scenarios, we assume access to a stabilizing controller. That is we know $K_\circ \in \mathbb{R}^{d_u \times d_x}$ such that $\rho(A+BK_\circ) < 1$.
  \item In scenario II -- ($A$ known), we assume that $\mu_\star > 0$.
\end{itemize}


\section{Related Work}\label{app:related}

In this section, we describe existing learning algorithms for the online LQR problem. These algorithms may be roughly categorized into two classes. In the first class, we find algorithms based on slightly perturbing the so-called self-tuning regulators \cite{astrom1973}, as those developed in second half of the 20th century in the control community, see \cite{kumar1985} for a survey and \cite{matni2019} for a more recent discussion. The second class of algorithms applies the Optimism in Front of Uncertainty (OFU) principle, extensively used to devise regret optimal algorithms in stochastic bandit problems \cite{lai1985,lattimore2020bandit}. Before describing these two classes of algorithm in more detail, we start by discussing how algorithms may differ in terms of regret guarantees, design principles, and the assumptions made towards their analysis.

\subsection{Types of guarantees, algorithm design, and assumptions}

{\bf Regret guarantees.} We may assess the performance of an algorithm by establishing various kinds of regret guarantees. Most often, the regret guarantees are in the {\it fixed confidence setting} only, in the following sense. The regret $R_T^\pi$ of an algorithm $\pi$ up to time $T$ satisfies a probabilistic guarantee of the form: $\mathbb{P}\left( R_T^\pi \le \psi(T) \left(\log\left(1/\delta \right) \right)^{1/\gamma} \right) \ge 1 - \delta$, for some $\gamma \le 2$ and some increasing function $\psi$. Typically such as guarantee is shown for a fixed confidence level (a fixed $\delta$), since the algorithm $\pi$ is most often actually parametrized by $\delta$. As a consequence, this probabilistic guarantee cannot be integrated over $\delta$ to obtain an upper bound of the expected regret. In turn, by only deriving guarantees with a fixed level of confidence, one avoids the difficult analysis of the algorithm behavior under the failure event (this event occurs with probability at most $\delta$, but potentially generates a very high regret). As far as we are aware, upper bounds on the {\it expected} regret have been investigated in \cite{rantzer2018, cassel2020logarithmic} in Scenario II where $A$ or $B$ is known only. Finally, it is worth mentioning that early work on adaptive control have focussed on deriving {\it asymptotic} regret guarantees. For example in \cite{lai1986, lai1987}, Lai devised, for systems where control inputs induce no cost $R=0$, algorithms whose regret satisfies $\lim\sup_{T\to \infty}  R_T^\pi/\log(T) \le C$ almost surely. This type of guarantee does not imply either guarantees w.h.p. as described above or guarantees in expectation.

\medskip
\noindent
{\bf Algorithm design.} Over the last few years, we have witnessed a significant research effort towards the design of learning algorithms with regret guarantees. Most choices in this design have been made to simplify the algorithm analysis rather than to improve their performance in practice. We discuss these choices below.\\
{\it (i) Doubling trick.} This trick is usually applied in online optimization problems (including bandits) \cite{besson2018, lattimore2020bandit} to come up with algorithms with {\it anytime} regret guarantees. The doubling trick generally comes with a cost in terms of regret \cite{besson2018}. In linear quadratic control, the doubling trick is used in all recent papers to simplify the analysis, but also to reduce the computational complexity of the algorithms. It consists in splitting time into successive {\it phases} whose durations grow exponentially. The control policy is computed at the beginning of each phase, and is applied throughout the epoch. The analysis may then leverage the fact that the control policy is fixed within each phase. Even if recent algorithms include a doubling trick, most of them still take the time horizon $T$ as an input, and hence are not anytime\footnote{An algorithm is anytime, if its regret up to any time can be upper bounded.}. \\
{\it (ii) Explore-Then-Commit and the known time horizon and confidence level.} As already mentioned, most of the recent algorithms target regret guarantees with a fixed confidence level, parametrized by $\delta$. To this aim, they adopt an Explore-Then-Commit (ETC) strategy, namely they rely on statistical tests to decide to switch from an exploration phase to an exploitation phase. These tests requires of course the knowledge of $\delta$. Applying an ETC strategy with confidence level $(1-\delta)$ imposes some constraints on the time horizon $T$ (it has to be greater than some decreasing function of $\delta$ so that the statistical tests end). Observe that to further simplify the design and analysis of algorithms, the time horizon is often assumed to be known in advance. It is finally interesting to note that ETC strategies are known to be sub-optimal, even in the simplest of the stochastic bandit problems \cite{garivier2016onexplore}.

\medskip
\noindent
{\bf Assumptions.} The set of assumptions made to design and analyze algorithms varies in the literature, which makes it hard to report and compare existing results precisely. Most existing work assume that we have access to a stabilizer. Recent attempts to remove this assumption include \cite{lale2020explore}. There, for example, the authors assume that the algorithm knows that the system $(A,B)$ belongs to a set of systems $(A',B')$ such that $\| A'+B'K_{(A',B')}\|\le \Upsilon<1$ and $\| [A', B']\|_F\le S$, which in particular implies that $\|P_{(A',B')}\| \le L$, for some constants $\Upsilon$, $S$, and $L$. We will provide a description as precise as possible of the set of assumptions made in each paper reported below.

\medskip
We propose an algorithm that does not take as input the time horizon $T$, or a certain level of confidence $\delta$. The control policy used in the algorithm can be updated every step, but also as frequently as we wish (we may decide reduce the computational complexity of the algorithm). Our analysis provides regret guarantees in expectation in all scenarios.

\subsection{Existing algorithms}

Next we describe selected recent learning algorithms. The first set of algorithms consist in slightly perturbing the control policies obtained when applying the certainty equivalence principle. The second set consists of algorithms applying the OFU principle.

\subsubsection*{Perturbed self-tuning regulators}

Self-tuning regulators work as follows. At any given step, they estimate the unknown matrices $A$ and $B$, and apply a control policy corresponding to the optimal control obtained replacing $A$ and $B$ by their estimators. Unfortunately as proved in \cite{lai1982}, self-tuning regulators may fail at converging -- the certainty equivalence principle does not always hold. This is due to the fact that under these regulators, the system may not be as excited as needed to obtain precise estimators. To circumvent this difficulty, the natural idea is to introduce some noise in the control inputs, leading to what we refer to as perturbed self-tuning regulators. We list below papers applying this idea, and analyzing the resulting regret.

As far as we know, the first regret analysis of perturbed self-tuning regulators is due to Lai and co-authors in the 80's, see e.g., \cite{lai1986,lai1987}. The focus is on a scenario where $A$ and $B$ are unknown, but where the control inputs do not contribute to the costs ($R=0$). Lai first establishes, using the techniques developed in \cite{lai1979}, that even if $B$ is known, asymptotically the regret cannot be smaller than $d_x^2\log(T)$ when the noise process $(\eta_t)_{t\ge 0}$ is i.i.d. with distribution ${\cal N}(0,I_{d_x})$. More precisely, it is shown that for the best learning algorithm $\pi$, $\lim\inf_{T\to\infty}{R_T^\pi\over \log(T)} \ge d_x^2$ almost surely. Lai then devises an algorithm adding white noise to the inputs when needed, and proves that the resulting perturbed self-tuning regulator has a regret asymptotically no larger than that predicted by the aforementioned lower bound.

In \cite{faradonbeh2020input}, the authors study Scenario I. They propose a perturbed self-tuning regulator, referred to as Perturbed Greedy Regulator, where the variance of the noise added to the inputs decreases over time. The regulator uses a doubling trick so that the estimated optimal controller can be updated rarely, and to simplify the analysis. The authors establish that with probability $1-\delta$, a regret is bounded by $\widetilde{\mathcal{O}}(\sqrt{T} \log(1/\delta)^{4})$ for $T\ge g(\delta)$. Even if the algorithm does not seem to use $\delta$ as an input, we cannot easily derive a meaningful upper bound on the expected regret by integrating over $\delta$ the probabilistic upper bound. The dependence of the upper bound in the system and its dimensions is not explicit. It is worth noting that the authors assume that the algorithm has access to a stabilizer, which according to their companion paper \cite{faradonbeh2019} can be learnt in finite time. The authors further assume, without any formal justification, that the system remains stable during the execution of their algorithm.

In \cite{mania2019certainty}, the authors do not explicitly propose a perturbed self-tuning regulator with regret guarantees. However, they show that if the estimation error is sufficiently small and if one guarantees that then the resulting algorithm achieves a $\widetilde{O}(\sqrt{T})$ regret with explicit dependence on the problem dimensions $d_x, d_u$. Their main contribution is a perturbation bound on the solution to the Discrete Algebraic Riccati equations, an important piece of the regret analysis. They propose an alternative proof to that of \cite{konstantinov1993perturbation} and compute explicitly the problem dependent constants.

The authors of \cite{simchowitz2020naive} propose, for Scenario I, a perturbed self-tuning regulator, that takes as input $\delta$ and $T$, as well as a stabilizer. Again to simplify the analysis, a doubling trick is used. The algorithm achieves a regret of $\widetilde{O}(d_u\sqrt{d_x T \log(1/\delta)})$ with probability $1-\delta$. The authors further derive what they refer to as a {\it local minimax} lower bound on the expected regret. This lower bound is obtained by varying the potential system matrices $(A,B)$ around those of the true system, and is in a sense close to a problem-specific lower bound. The lower bound is scaling as $\Omega(d_u\sqrt{d_x T \log(1/\delta)})$. Our results for Scenario I matches this lower bound in expectation.

 \cite{cassel2020logarithmic} presents perturbed self-tuning regulators for Scenario II (when $A$ or $B$ is known). The regulators have numerous inputs, including a stabilizer $K_0$ (actually a strongly stable control) and upper bounds on $\| A\|$, $\| B\|$, on the minimal ergodic cost, and that achieved under $K_0$. When $A$ is known, the proposed regulator is shown to have an expected regret of order $O({\rm{poly}}(d_x,d_u)\log^2(T))$ (where the degree of the polynomial scale in the dimensions is not precised). The regulator presented for the case when $B$ is unknown achieves similar expected regret guarantees provided that the optimal controller $K_\star$ satisfies $K_\star K_\star^\top\succeq \mu>0$.


\subsubsection*{OFU-based algorithms}

A typical OFU-based algorithm proceeds as follows. It maintains a confidence set ${\cal C}$ (an ellipsoid) where the system parameters $(A,B)$ lie with high probability. At a given step, the algorithm selects $(A',B')\in {\cal C}$ that minimizes the optimal cost $J_{(A',B')}$ (with sometimes an additional margin). The controller $K_{(A',B')}$ is then applied. Since updating the controller requires solving a complex optimization problem, referred to as the optimistic LQR below, this update should be done rarely (using a doubling trick).

\cite{abbasi2011regret} presents OFU-LQ, an algorithm taking as inputs a confidence level $\delta$ and $T$, as well as a bounded set $\mathcal{S}$ where $(A, B)$ lies and an upper bound on $\left\Vert \begin{bmatrix} A & B \end{bmatrix}\right\Vert_F$. OFU-LQ leverages a {\it random} doubling trick: the controller is updated each time the determinant of the covariates matrix is doubled. This determinant roughly grows as $t^{d_x+d_u}$, and hence the successive phases have durations multiplied by $2^{1/(d_x+d_u)}>1$. OFU-LQ has a regret of order $\widetilde{\mathcal{O}}\sqrt{T\log(1/\delta)}$ with probability $1-\delta$. Here, the notation $\widetilde{\mathcal{O}}(\cdot)$ hides polynomial factors in $\log(T)$ and multiplicative constants exponentially growing in $d_x + d_u$. \cite{abbasi2011regret} does not indicate how to solve the optimistic LQR, and cannot be implemented directly. The same conclusion holds for the algorithm proposed in \cite{faradonbeh2017} (there, the authors were able to relax some assumptions on the noise and stability, but an efficient and practical implementation of the algorithm is not investigated).

A first practical implementation of OFU-based algorithms was presented in \cite{cohen2019learning}. The algorithm, OSLO, uses an SDP formulation to solve the optimistic LQR. It uses a somewhat random doubling trick similar to \cite{abbasi2011regret} and requires the knowledge of $\delta$, $T$, upper bounds on the norms of $A, B$ and $P_{(A,B)}$, as well as a stabilizing controller. OSLO achieves a regret of order $\widetilde{O}((d_x + d_u)^3\sqrt{T \log(1/\delta)^4})$.

\cite{abeille2020efficient} presents an efficient algorithm to implement OFU-based algorithms proposed in \cite{abbasi2011regret}. The idea is to perform a relaxation of the optimistic LQR. The analysis requires the knowledge of $\delta$ and the horizon $T$. And a similar random doubling trick to that of \cite{abbasi2011regret} is used. They obtain a regret of order $\widetilde{O}((d_u + d_x)\sqrt{d_x T})$ with probability $1-\delta$. Note that the learner is assumed to know an initial state that is sufficiently good so that stability is maintained throughout the learning process.

In \cite{lale2020explore}, the authors provide another improvement on the OFU based algorithm of \cite{abbasi2011regret}. Their goal is to improve the dependency of the regret upper bound on the dimension and to remove the assumption of having access to a stabilizing controller. Doing so, the authors need to introduce other assumptions about the set of systems to which the algorithm applies: $A$, $B$ and $A+BK_{(A,B)}$ have bounded norms. The regret of the proposed algorithm is with probability $1-\delta$ of order $\widetilde{O}({\rm{poly}}(d_x+d_u)\sqrt{T \log(1/\delta)^\gamma})$ with $\gamma \ge 1$ but unspecified, and the degree $\mathrm{poly}(d)$ is also unspecified.

\subsubsection*{Summary of existing results and comparison with $\Alg$}

In summary, all the aforementioned algorithms use a {doubling trick} so that the algorithm becomes amenable to theoretical analysis (using the independence between epochs etc). Most of the algorithms are designed in the fixed confidence setting. They are variants of ETC strategies and their construction rely heavily on knowledge of the confidence level $\delta$. Most of them also take as input the time horizon $T$, i.e., they are not anytime.

The assumptions made towards the regret analysis of these algorithms are not unified. Therefore it is very hard to obtain fair comparisons between these algorithms. Except for \cite{cassel2020logarithmic}, the algorithms have regret guarantees in the fixed confidence setting only, i.e., with probability $1-\delta$. The best regret dependence in $\delta$ is $\sqrt{\log(1/\delta)}$, but it is not achieved by all algorithms ($\log(1/\delta)$ \cite{abeille2020efficient} and $\log(1/\delta)^2$ (\cite{faradonbeh2020input},\cite{cohen2019learning}).

The table below summarizes the regret guarantees achieved by the various algorithms.

\begin{table}[!h]
  \centering
  \begin{tabular}{|c | l | l | c |}
    \hline
    \rowcolor{Gray}
    \multicolumn{4}{|c|}{\textbf{Scenario I - $A$ and $B$ unknown} } \\
    \hline
    {paper}  & {regret upper bound w.p. $1-\delta$} & {expected regret upper bound}& {required inputs} \\
    \hline
    \cite{faradonbeh2020input}  & $\widetilde{O}(g(d_x+d_u)\sqrt{T \log(1/\delta)^4})$ & - &  $\delta$ (unclear)  \\
    \cite{simchowitz2020naive}  & $\widetilde{O}(d_u\sqrt{d_x T \log(1/\delta)} )$ & - & $\delta$ and $T$ \\
    \cite{mania2019certainty}   & $\widetilde{O}(g(d_x+d_u)\sqrt{T}f(\log(1/\delta)))$ & - & unkown \\
    \cite{lale2020explore}      & $\widetilde{O}({\rm{poly}}(d_x+d_u)\sqrt{T \log(1/\delta)})$& - & $\delta$ \\
    \cite{abbasi2011regret}     & $\widetilde{O}(\sqrt{c^{d_x+d_u}T \log(1/\delta)})$& - & $\delta$ and $T$ \\
    \cite{abeille2020efficient} & $\widetilde{O}((d_x+d_u)\sqrt{d_x T \log(1/\delta)^2})$ & - & $\delta$ and $T$\\
    \cite{cohen2019learning}    & $\widetilde{O}((d_x + d_u)^3\sqrt{T \log(1/\delta)^4})$ & - & $\delta$ and $T$\\
   this paper & - & $\widetilde{O}((d_u+d_x)\sqrt{d_x T} )$ &  - \\
   \hline
  \end{tabular}%

 \medskip
 \medskip

  \begin{tabular}{|c | l | l | c |}
    \hline
    \rowcolor{Gray}
    \multicolumn{4}{|c|}{\textbf{Scenario II - $A$ or $B$ known} } \\
    \hline
    {paper}  & assumptions & {expected regret upper bound}& {required inputs} \\
    \hline
      \cite{cassel2020logarithmic} & $A$ known, $\| B\|\le M$ & $O({\rm{poly}}(d_x,d_u)\log^2(T))$ & $T$, $M$ (among others)\\
\hline
 this paper & $A$ known &  $O(d_x(d_x+ d_u)\log(T))$ &  - \\
 \hline
 \cite{cassel2020logarithmic} & $B$ known, $\| A\|\le M$ & $O({\rm{poly}}(d_x,d_u)\log^2(T))$ & $T$, $M$ (among others)\\
   \hline
    this paper & $B$ known &  $O(d_x^2\log(T))$ &  - \\
 \hline
  \end{tabular}

  \medskip

  \caption{Regret guarantees of existing algorithms. The notation $\widetilde{O}(\cdot)$ hides polynomial factors in $\log(T)$ and additive low order terms in $T$ with potentially worse dependencies in $\log(1/\delta)$, and the constants depends on problem parameters.}
\end{table}


\newpage
\section{Regret Definitions and Analysis} \label{app:regretproof}

This appendix includes in \ref{app:regretdef} a discussion about the definition of the regret of an adaptive control algorithm. In \ref{app:regretdec}, we provide a useful decomposition of the expected regret that will serve as the starting point of our analysis in the three scenarios. We present in \ref{app:integration} the so-called integration lemma, that will helps us to derive expected regret upper bounds based on high probability bounds. The three last subsections give the proofs of our main theorems: the regret upper bound of $\Alg$ in Scenario I (Theorem \ref{th:scenarioI})  is proved in \ref{sec:proof:th1}. The proof of Theorem \ref{th:scenarioIIA} for Scenario II -- $A$ known is given in \ref{sec:proof:th3}. That of Theorem \ref{th:scenarioIIB} for Scenario II -- $B$ known is finally presented in \ref{app:thlast}.

\subsection{Regret definitions}\label{app:regretdef}

Let us denote by $\Pi$ the set of all possible adaptive control policies. For a poolicy $\pi\in \Pi$, $(x_1^\pi,u_1^\pi, \dots, x_t^\pi,u_t^\pi)$ is the sequence of states and control inputs generated under $\pi$. Remember that $x_0^\pi=0=u_0^\pi$.

We define the ergodic cost of a policy $\pi\in\Pi$ as
\begin{align*}
  \mathcal{J}(\pi) & =  \limsup_{T \to \infty} \frac{1}{T} \E\left[\sum_{t=1}^T (x_t^{\pi})^\top Q (x_t^\pi) + (u_t^{\pi})^\top R (u_t^{\pi})\right] & \textit{(ergodic cost / objective)}
\end{align*}
It can be shown under suitable assumptions on $(A,B,Q,R)$ that there exists a policy $\pi_\star \in \argmin_{\pi \in \Pi} \mathcal{J}(\pi)$. Let $\mathcal{J}_\star = \mathcal{J}(\pi^{\star})$. The optimal policy $\pi_\star$ can be found explicitly: for all $t \ge 1$, $\pi^{\star}$ defines the feedback control $u_t^{\pi_\star} = K_\star x_t^{\pi_\star}$. The matrix $K_\star$ can be computed by solving the Riccati equations. We have the useful identity that $P_\star = (A + BK_\star)^{\top} P_\star (A+BK_\star) + Q + K_\star^\top R K_\star$ where $P_\star \succ 0$ is the solution to the Ricatti equations. Furthermore, we have $\mathcal{J}_\star= \tr(P_\star)$.

Now, we define the regret of a policy $\pi \in \Pi$ as
\begin{align}
    \sum_{t=1}^T  (x_t^\pi)^\top Q(x_t^\pi)  + (u_t^\pi)^\top R (u_t^\pi) - \E\left[\sum_{t=1}^T (x_t^{\pi_{\star}})^\top Q(x_t^{\pi_{\star}})  + (u_t^{\pi_{\star}})^\top R (u_t^{\pi_{\star}}) \right].
\end{align}
This definition is natural as we compare the cost under $\pi$ to that under $\pi^{\star}$, cumulated over $T$ steps. During these $T$ steps to compute the costs, we follow the trajectory of the system. This contrasts with the definition of regret often used in the literature:
\begin{align}\label{eq:regretdefinition}
    R_T (\pi) = \sum_{t=1}^T  (x_t^\pi)^\top Q(x_t^\pi)  + (u_t^\pi)^\top R (u_t^\pi) - T \mathcal{J}_\star.
\end{align}

In fact when considering the expected regret, the above two definitions coincide up to a constant. Indeed, note that for all $t\ge 1$,
$
\Vert x_t^{\pi_{\star}} \Vert^2_{P_\star} = \Vert x_{t+1}^{\pi_{\star}}- \eta_{t} \Vert_{P_\star}^2 + \Vert x_t^{\pi_{\star}}  \Vert^2_Q +  \Vert u_t^{\pi_{\star}}  \Vert^2_R.
$
Taking expectation (note $\E[\eta_t^\top X] = 0$ provided $\eta_t$ is independent of $X$) gives
$$
\E\left[ \Vert x_t^{\pi_{\star}} \Vert^2_{P_\star} \right] = \E\left[\Vert x_{t+1}^{\pi_{\star}} \Vert^2_{P_\star} - \Vert \eta_t \Vert_{P_\star}^2  + \Vert x_t^{\pi_{\star}}  \Vert^2_Q +  \Vert u_t^{\pi_{\star}}  \Vert^2_R \right].
$$
Summing over time after rearranging  gives:
\begin{align*}
   \E\left[\sum_{t=1}^T (x_t^{\pi_{\star}})^\top Q(x_t^{\pi_{\star}})  + (u_t^{\pi_{\star}})^\top R (u_t^{\pi_{\star}}) \right] & = \E\left[ \sum_{t=1}^T \Vert x_t^{\pi_{\star}} \Vert^2_{P_\star} -  \Vert x_{t+1}^{\pi_{\star}} \Vert^2_{P_\star}  + \Vert \eta_t \Vert_{P_\star}^2   \right] \\
   & = \E\left[ \Vert x_1 \Vert^2_{P_\star} - \Vert x_{T+1}^{\pi_{\star}} \Vert^2_{P_\star} \right] + T J_\star.
\end{align*}
Note that $\E[ \Vert x_{T+1}^{\pi_{\star}} \Vert^2_{P_\star} ] < \infty$. Indeed it can be verified that
\begin{align*}
\E[\Vert x_T^{\pi_\star} \Vert^2] & = \sum_{t=0}^{T-1} ((A+BK_\star)^t) P_\star (A+BK_\star)^t  + \E[ x_1^\top ((A+BK_\star)^T)^\top P_\star (A+BK_\star)^Tx_1].
\end{align*}
Therefore, when considering the expected regret, we can take (\ref{eq:regretdefinition}) as the regret definition.
\newpage

\subsection{Regret decomposition}\label{app:regretdec}

The regret decomposition for the three scenarios can be unified. To that end, let us denote
$$
\forall t\ge 1, \ \    \xi_t = \begin{cases}
  \nu_t & \text{ in Scenario I} \\
  \zeta_t & \text{ in Scenario II -- (}$A$ \text{ known)} \\
  0 & \text{ in Scenario II -- (}$B$ \text{ known)} \\
\end{cases}
$$
and note that for all $t\ge 1$, $\xi_t$ is a zero mean gaussian random vector with variance proxy $\tilde{\sigma}_t = \sigma_t$ in Scenario I, $\tilde{\sigma}_t = 1$ in Scenario II -- ($A$ known), and $\tilde{\sigma}_t = 0$ in Scenario II -- ($B$ known). Let $(\mathcal{F}_{t})_{t\ge 0}$ be a filtration such that $\mathcal{F}_{t}$ is the $\sigma$-algebra generated by $(\eta_1, \dots, \eta_t)$ and $(\xi_1,\dots, \xi_t)$ for all $t\ge 0$. Now, we may observe that the controller used by our algorithm is of the form
\begin{equation}\label{eq:controllerform}
    \forall t \ge 1, \ \ u_t = \widetilde{K}_t x_t + \alpha_t \xi_t
\end{equation}
where $(\widetilde{K}_t)_{t\ge0}$ is a sequence of random matrices taking values in $\mathbb{R}^{d_u \times d_x}$, such that $\widetilde{K}_t$ is $\mathcal{F}_{t-1}$-measureable $\forall t \ge 1$ and $\widetilde{K}_0  = 0$, and where $(\alpha_t)_{t\ge 0}$ is a sequence of random variables taking values in $\lbrace 0, 1\rbrace$ such that $\alpha_t$ is $\mathcal{F}_{t-1}$-measurable and $\alpha_0 = 0$.

Now we are ready to establish a regret decomposition that is valid for any controller of the form \eqref{eq:controllerform}. We state this decomposition in the following result.

\begin{lemma}[Exact Regret Decomposition]\label{lem:rd}
  Let $(u_t)_{t\ge 0}$ be a sequence of control inputs that can be expressed as in \eqref{eq:controllerform}. Define, for all $t\ge 0$,
  \begin{align}
      \widetilde{P}_t & = \begin{cases}
      P_{\star}(\widetilde{K}_t) &  \text{if} \quad \Vert B (\widetilde{K}_t - K_\star)\Vert < \frac{1}{4\Vert P_\star \Vert^{3/2}} \text{ and } \Vert P_\star(\widetilde{K}_t) \Vert \le 2 \Vert P_\star \Vert  \\
      P_\star & \text{ otherwise}
      \end{cases}  \label{eq:rd:1}   \\
     P_{\star,t} & = (A+B\widetilde{K}_t)^\top \widetilde{P}_t (A+B\widetilde{K}_t) + Q + \widetilde{K}_t^\top R \widetilde{K}_t. \label{eq:rd:2}
  \end{align}
  Then, for all $T \ge 1$
  \begin{align*}
    \E[R_{T}(\pi)] & = \E\left[ \sum_{t=1}^T \Vert x_t \Vert_{P_{\star,t} - \widetilde{P}_{t-1}}^2  +   \Vert \eta_t \Vert^2_{\widetilde{P}_t - P_\star} + \alpha_t \Vert \xi_t \Vert^2_{ B^\top \widetilde{P}_t B + R} \right] +  \E\left[ \Vert x_{1}\Vert^2_{\widetilde{P}_0} - \Vert x_{T+1} \Vert^2_{\widetilde{P}_T} \right].
  \end{align*}
\end{lemma}


\begin{proof}[Proof of Lemma \ref{lem:rd}]
  First, we note that the sequence $(\widetilde{P}_t)_{t\ge1}$ is well defined. Indeed, when $\Vert B(\widetilde{K}_t - K_\star) \Vert< 1/(4\Vert P_\star\Vert^{3/2})$, $P_\star(K_t)$ exists and is the solution to the Lyapunov equation $ P = (A+B\widetilde{K}_t)^\top P (A+B\widetilde{K}_t) + Q + \widetilde{K}_t^\top R \widetilde{K}_t$ (see Lemma \ref{lem:perturbed}).

  Next, in view of our choice of control imputs $(u_t)_{t \ge 1}$, we can express the dynamics of the problem as $x_{t+1} = (A + B \widetilde{K}_t) x_t + B \alpha_t \xi_t + \eta_t$ for all $t\ge0$. Thus, multiplying both sides of \eqref{eq:rd:2} by $x_t$, we obtain the identity
  $$
  \Vert x_t \Vert^2_{P_{\star,t}} = \Vert x_{t+1} - B\alpha_t\xi_t - \eta_t \Vert_{\widetilde{P}_t}^2 + \Vert x_t\Vert^2_Q + \Vert u_t - \alpha_t\xi_t\Vert^2_R.
  $$
  Then, carefully expanding the above identity, leads to
  \begin{align*}
  \Vert x_t\Vert^2_{Q} + \Vert u_t \Vert^2_R &  = \Vert x_{t} \Vert^2_{P_{\star,t}} - \Vert x_{t+1} \Vert^2_{\widetilde{P}_t} + \Vert \eta_t \Vert^2_{\widetilde{P}_t} + \alpha_t \Vert \xi_t \Vert^2_{ B^\top \widetilde{P}_t B + R} \\
  & + 2 (B\alpha_t\xi_t + \eta_t)^\top \widetilde{P}_t (A+B\widetilde{K}_t)x_{t} + 2 \eta_t^\top \widetilde{P}_tB\xi_t  + 2\alpha_t\xi_t^\top R \widetilde{K}_t x_t.
  \end{align*}

  We note that $\E[\eta_t \vert \mathcal{F}_{t-1}] = \E[\xi_t \vert \mathcal{F}_{t-1}] = 0$, and that $\widetilde{P_t}$, $x_t$, $\alpha_t$, and $\widetilde{K}_t$ are all $\mathcal{F}_{t-1}$-measurable. Thus, using the tower rule, and substracting $\Vert \eta_t \Vert^2_{P_\star}$ from both sides, we obtain for all $t\ge 1$,
  $$
  \E\left[ \Vert x_t \Vert_Q^2 + \Vert u_t \Vert_R^2 - \Vert \eta_t \Vert^2_{P_\star} \right] = \E\left[  \Vert x_t \Vert_{P_{\star,t}}^2  - \Vert x_{t+1}\Vert_{\widetilde{P}_{t}}^2 + \Vert \eta_t \Vert^2_{\widetilde{P}_t - P_\star} + \alpha_t \Vert \xi_t \Vert^2_{ B^\top \widetilde{P}_t B + R} \right].
  $$
  Summing over $t \in \lbrace 1, \dots, T \rbrace$ (we note that $x_0 = 0$ and $u_0 = 0$) we obtain
  \begin{align*}
      \E[R_T(\pi)]  & = \E\left[ \sum_{t=1}^T  \Vert x_t \Vert_Q^2 + \Vert u_t \Vert_R^2 - \Vert \eta_t \Vert^2_{P_\star} \right] \\
      & = \E\left[\sum_{t=1}^T \Vert x_t \Vert_{P_{\star,t}}^2  - \Vert x_{t+1}\Vert_{\widetilde{P}_{t}}^2 + \Vert \eta_t \Vert^2_{\widetilde{P}_t - P_\star} + \alpha_t \Vert \xi_t \Vert^2_{ B^\top \widetilde{P}_t B + R} \right] \\
      & = \E\left[ \sum_{t=1}^T \Vert x_t \Vert_{P_{\star,t} - \widetilde{P}_{t-1}}^2  +   \Vert \eta_t \Vert^2_{\widetilde{P}_t - P_\star} + \alpha_t \Vert \xi_t \Vert^2_{ B^\top \widetilde{P}_t B + R} \right] + \E\left[ \Vert x_{1}\Vert^2_{\widetilde{P}_0} - \Vert x_{T+1} \Vert^2_{\widetilde{P}_T} \right],
  \end{align*}
  where we recall that $\tr(P_\star) = \mathcal{J}_\star$. This concludes the proof.
\end{proof}

\subsection{Integration Lemma}\label{app:integration}

We use Lemma \ref{lem:integration} to integrate the high probability bounds on regret.

\begin{lemma}\label{lem:integration} Let $X$ be a postive random variable such that for all $\delta \in (0,1)$
  $$
  \mathbb{P}\left( X > C_1 \log(e/\delta) + C2 \log(e/\delta)^{\beta_1} \right) \le C_3 \log(e/\delta)^{\beta_2} \delta
  $$
  where $C_1,C_2,C_3, \beta_1,\beta_2 >0$. Then
  $$
  \E[X] \le ( 2\log(eC_3)  +   2^{\beta_2 + 1} \Gamma(\beta_2 + 1)) C_1  + ((2\log(eC_3))^{\beta_1} +  \beta_1 2^{\beta_1 + \beta_2} \Gamma(\beta_1 + \beta_2) ) C_2
  $$
  where, here, $\Gamma(\cdot)$ refers to the gamma function.
\end{lemma}
\begin{proof}
  First, for convenience, we start by reparmetrizing $\rho = \log(e/\delta)$, so that we have $
  \mathbb{P}(X > C_1 \rho + C_2 \rho^{\beta_1} ) \le C_3 \rho^{\beta_2}e^{-\rho}
  $ for all $\rho >1$. Additionally, we note that for $\rho > 2\log(eC_3)$, we have $C_3e^{-\rho /2} < 1 $. Thus for $\rho > 2\log(eC_3)$, we have
  $$
  \mathbb{P}( X > C_1 \rho + C_2 \rho^{\beta_1} ) \le \rho^{\beta_2} e^{-\rho/2}.
  $$
  Now, we integrate and perform the change of variable $u = C_1\rho + C_2\rho^{\beta_1}$ for $u > C_1 2\log(eC_3) + C_2 (2\log(eC_3))^{\beta_1}$, which yields
  \begin{align*}
    \E[X] &  = \int_{0}^\infty \mathbb{P}(X > u) du \\
        & \le  2\log(eC_3) C_1  +  (2\log(eC_3))^{\beta_1} C_2   + \int_{0}^\infty \mathbb{P}( X > C_1 \rho + C_2 \rho^{\beta_{1}}) (C_1 + C_2 \beta_{1} \rho^{\beta_1 -1} ) d\rho.
  \end{align*}
Then observe that
  \begin{align*}
  & \int_{0}^\infty \mathbb{P}( X > C_1 \rho + C_2 \rho^{\beta_{1}}) (C_1 + C_2 \beta_{1} \rho^{\beta_1 -1} ) d\rho \le \int_{0}^\infty (C_1 + C_2 \beta_1 \rho^{\beta_1 - 1}) \rho^{\beta_2}e^{-\rho /2} d\rho \\
  & \qquad \qquad \qquad  \le C_1 \int_{0}^\infty\rho^{\beta_2}e^{-\rho/2} d\rho+ \beta_1 C_2  \int_{0}^\infty \rho^{\beta_1+\beta_2 - 1} e^{-\rho/2} d\rho \\
  & \qquad \qquad \qquad  \le 2^{\beta_2 + 1 } C_1 \int_{0}^\infty \rho^{\beta_2 }e^{-\rho} d\rho + \beta_1 2^{\beta_1 + \beta_2} C_2 \int_{0}^\infty \rho^{\beta_1 + \beta_2 - 1} e^{-\rho} d\rho \\
  & \qquad \qquad \qquad  \le 2^{\beta_2 + 1 } \Gamma(\beta_2 + 1)  C_1 +   \beta_1 2^{\beta_1 + \beta_2} \Gamma(\beta_1 + \beta_2) C_2
  \end{align*}
  where $\Gamma(x)$ refers to the gamma function evaluated at $x$. To conclude, we have shown that
  \begin{align*}
      \E[X] \le ( 2\log(eC_3)  +   2^{\beta_2 + 1} \Gamma(\beta_2 + 1)) C_1  + ((2\log(eC_3))^{\beta_1} +  \beta_1 2^{\beta_1 + \beta_2} \Gamma(\beta_1 + \beta_2) ) C_2.
  \end{align*}

\end{proof}

%

\medskip

\subsection{Proof of Theorem \ref{th:scenarioI} - Regret analysis in Scenario I}\label{sec:proof:th1}

Motivated by the regret decomposition established in Lemma \ref{lem:rd}, we define what we shall refer to from now on as \emph{proxy regret} as follows
\begin{equation}
\widetilde{R}_T(\pi) =   \sum_{t=1}^T \Vert x_t \Vert_{P_{\star,t} - \widetilde{P}_{t-1}}^2  +   \tr\left(\widetilde{P}_t - P_\star\right) +  \sigma_t^2\tr\left(B^\top \widetilde{P}_t B + R\right) + \Vert x_{1}\Vert^2_{\widetilde{P}_0}
\end{equation}
where $(\widetilde{P}_t)_{t\ge 0}$ is definied as in Lemma \ref{lem:rd} with $\alpha_t = 1$, and $\xi_t = \nu_t$ for all $t\ge 1$. We note by the same lemma that $\E[R_T(\pi)] \le \E[\widetilde{R}_T(\pi)]$, thus we may restrict our attention to analysing the \emph{proxy regret} instead of the true regret. Now, we provide a high probability bound on this \emph{proxy regret} which holds for all confidence levels $\delta>0$.

\medskip

\underline{\textbf{Step 1:}} {(Defining the nice event)} We start by defining the following event.
\begin{align}
\mathcal{E}_{\delta} = \left\lbrace \forall t \ge t(\delta), \begin{array}{rl} \textit{(i)} & \widetilde{K}_t  = K_t, \\
\textit{(ii)} & \Vert B (K_t - K_\star) \Vert \le (4\Vert P_\star\Vert^{3/2})^{-1} \\
\textit{(iii)} & \Vert P_\star(K_t) - P_\star \Vert  \le C_1(\delta) r_t^2  \\
\textit{(iv)} & \Vert P_\star(K_t) \Vert  \le 2 \Vert P_\star \Vert \\
\textit{(v)} & \sum_{t=t(\delta)}^T r_t^2 \Vert x_t \Vert^2 \le C_2(\delta)  + C_3 \Vert r_{1:T}\Vert^2_2
\end{array}
\right\rbrace
\end{align}
where
\begin{align*}
  r_t^2 &= \log(ei_t) i_t^{-1/2},\\
  i_t &=\max\{t_k\in {\cal T}: t_k\le t\},\\
  t(\delta) & =  \poly( \sigma, C_\circ, \mathcal{G}_\circ, \Vert P_\star\Vert,d_x, d_u,\gamma_\star) \log(e/\delta)^{12\gamma_\star} \\
  C_1(\delta) & = c_1 \sigma^2 C_K^2 \Vert P_\star \Vert^8 \gamma_\star d_x^{-1/2} (d_x + d_u) \log(e C_\circ  \mathcal{G}_\circ d_x ) \log(e/\delta) \\
  C_2(\delta) & = \poly(\sigma, C_\circ, \mathcal{G}_\circ, \Vert P_\star \Vert, C_B, d_x, d_u, \gamma_\star ) \log(e/\delta)^{19\gamma_\star^2} \\
  C_3 & = c_3\sigma^2 C_B^2 \Vert P_\star \Vert^{3/2} \gamma_\star d_x
\end{align*}
for some universal postive constants $c_1, c_3>0$. Furthremore, applying Theorem \ref{thm:nice:I}, we have
$$
\mathbb{P}(\mathcal{E}_\delta) \ge 1-t(\delta) \delta
$$
for some proper choice of the universal constants defining $t(\delta), C_1(\delta), C_2(\delta), C_3$.

We can interpret the nice event $\mathcal{E}_{\delta}$ as follows. The first point \textit{(i)} means that $\Alg$ is playing certainty equivalence for all $t\ge t(\delta)$. The second \textit{(ii)} means that the sequence of $(K_t)_{t\ge t(\delta)}$ is such that the resulting behaviour of the system is that of a stable system, and naturally $\rho(A+BK_t)<1$. To see that, we refer the reader Proposition \ref{prop:stable:K_t} (see also Lemma \ref{lem:perturbed}). The third point \textit{(iii)} indicates that the error rate is decreasing as $r_t^2$. The final points \textit{(iv)- (v)} are perhaps redundent since they can be deduced from points \textit{(ii)-(iii)}, but we include them here for convenience.
 \medskip

 \underline{\textbf{Step 2:}} {(Regret from $t (\delta)$ onwards)} We bound $\widetilde{R}_T(\pi) - \widetilde{R}_{t(\delta)-1}(\pi)$ under the event $\mathcal{E}_{\delta}$. Note that under this event, we have
 \begin{align*}
   \widetilde{R}_T(\pi) - \widetilde{R}_{t(\delta) -1}(\pi) & \le   C_1(\delta) \sum_{t=t(\delta)}^T r_t^2  (\Vert x_t \Vert^2 + d_x) + d_x (2\Vert P_\star \Vert \Vert B\Vert^2 + 1)  \sum_{t=t_{\mathcal{E}}(\delta)}^T  \sigma_t^2 \\
   & \le  C_1(\delta) \left( d_x \Vert r_{1:T} \Vert^2_2 + \sum_{t=t(\delta)}^T r_t^2  \Vert x_t \Vert^2  \right)  + 3d_x C_B^2\Vert P_\star \Vert  \Vert \sigma_{1:T}\Vert^2_2 \\
   & \le C_1(\delta) (d_x \Vert r_{1:T} \Vert^2_2 + C_2(\delta) + C_3 \Vert r_{1:T}\Vert^2) + 6d_x^{3/2} C_B^2\Vert P_\star \Vert \sqrt{T} \\
   & \le 2C_3 C_1(\delta) \Vert r_{1:T}\Vert^2_2 +  C_1(\delta) C_2(\delta) + 6d_x^{3/2}  C_B^2\Vert P_\star \Vert \sqrt{T}
 \end{align*}
 where in the first inequality, we used the \textit{(i)} to have $\widetilde{P}_t = P_\star(K_t)$ for all $t\ge t(\delta)$, then used \textit{(iii)} to bound $\Vert P_\star(K_t) - P_\star(K_{t-1})\Vert \le \Vert P_\star(K_t) - P_\star\Vert + \Vert P_\star(K_{t-1}) - P_\star \Vert \le C_1(\delta) r_t^2$ for all $t\ge t(\delta)+1$. Next, we used \textit{(iv)}, to bound $\Vert B^\top \tilde{P}_t B + R \Vert \le (2\Vert P_\star\Vert \Vert B\Vert^2 + \Vert R \Vert) \le  (2 \Vert P_\star\Vert \Vert B \Vert^2 + 1)$.
 Finally we used  \emph{(v)} to bound $\sum_{t=t(\delta)}^T r_t^2 \Vert x_t\Vert$, and bounded $\Vert \sigma_{1:T}\Vert^2_2 \le 2d_x^{1/2} \sqrt{T}$.


Since $\mathcal{T}$ satisfies \eqref{eq:lazy}, we can easily verify that $\Vert r_{1:T} \Vert^2_2 \lesssim\log(T)\sqrt{T}$. To see that, note that $(r_t)_{t\ge1}$ depends on ${\cal T}$, and we can always find $\mathcal{T}' = \lbrace (C')^k : k \in \mathbb{N}\rbrace$ for $C'$ large enough such that the corresponding sequence $(r_t')_{t\ge1}$ satisfies $\Vert r_{1,T} \Vert^2_2 \le \Vert r_{1:T}'\Vert^2_2 \lesssim \log(T)\sqrt{T}$ since ${\cal T}$ satisfies $\eqref{eq:lazy}$.

Therefore, recalling the expressions of $t(\delta), C_1(\delta), C_2(\delta)$ and $C_3$, we have
\begin{align}\label{eq:I:regret:after}
   \widetilde{R}_T(\pi) - \widetilde{R}_{t(\delta) -1}(\pi) & \le C_4 \log(e/\delta) \log(T)\sqrt{T} + C_5 \log(e/\delta)^{31\gamma_\star^2}
\end{align}
with
\begin{align*}
  C_4 & =  c_4 \sigma^2 C_B^2 C_K^2 \Vert P_\star \Vert^{9.5} \log(e C_\circ \mathcal{G}_\circ d_x) d_x^{1/2}(d_x + d_u), \\
  C_5 & = \poly(\sigma, C_\circ, \mathcal{G}_\circ, \Vert P_\star\Vert, C_B, d_x,d_u, \gamma_\star),
\end{align*}
for some universal positive constant $c_4> 0$.

\medskip

\underline{\textbf{Step 3:}} We bound $\tilde{R}_{t(\delta)}(\pi)$ under the event $\mathcal{E}_\delta$. Note we can obtain the crude upper bound
\begin{align*}
  \widetilde{R}_{t(\delta) -1 }(\pi) & \le \max_{1 \le t \le t(\delta)} \Vert P_{\star,t} \Vert \sum_{t=1}^{t(\delta)} \Vert x_t \Vert^2 + d_x  (1 + \max(\Vert B \Vert^2, \Vert R\Vert) \sigma_t^2) \\
  & \le  \max_{1 \le t \le t(\delta)} \Vert P_{\star,t} \Vert \sum_{t=1}^{t(\delta)} \Vert x_t \Vert^2 + 5 d_x C_B^2 \sigma^2
\end{align*}
where we dropped the negative terms $-\Vert x_t \Vert^2_{\widetilde{P}_{t-1}}$ for $2 \le t \le t(\delta)$. Considering the definition of $P_{\star,t}$, we have
\begin{align*}
  \Vert P_{\star,t} \Vert \le  \begin{cases} 2 \Vert P_\star \Vert & \text{ if } \Vert B (\widetilde{K}_t - K_\star) \Vert  < \frac{1}{4\Vert P_\star\Vert^{3/2}} \text{ and } \Vert P_\star(\widetilde{K}_t)\Vert \le 2\Vert P_\star \Vert^2,\\
  4 C_\circ^2 \Vert P_\star \Vert   h(t) &   \text{ otherwise,}
\end{cases}
\end{align*}
where we upper bounded $\Vert P_{\star,t} \Vert \le 4 \Vert P_\star \Vert  C_\circ^2 h(t)$, using the fact that under $\Alg$, we have $\Vert \widetilde{K}_t \Vert^2 \le \max (\Vert K_\circ \Vert^2, h(t))$, and the fact that $\Vert P_\star\Vert \ge \max(\Vert Q\Vert, \Vert R\Vert)$. Thus,
$$
\max_{1  \le t\le t(\delta)}\Vert P_{\star,t}\Vert \le 4C_\circ^2 \Vert P_\star \Vert h(t).
$$
Thus, we may write
\begin{align*}
  \widetilde{R}_{t(\delta)- 1}(\pi) & \le  4C_\circ^2 \Vert P_\star \Vert h(t(\delta)) \sum_{t=1}^{t(\delta)} (\Vert x_t \Vert^2 + 5\sigma^2 d_x C_B^2).
\end{align*}
Therefore under event $\mathcal{E}_{\delta}$, we have
\begin{align*}
  \widetilde{R}_{t(\delta) - 1}(\pi) & \le  36 \sigma^2 d_x C_\circ^4 \Vert P_\star \Vert h(t(\delta)) f(t(\delta)),
\end{align*}
since when property \emph{(i)} holds at time $t(\delta)$, then it must mean that $\ell_{t(\delta)} = 1$, which means that $\sum_{s=1}^{t(\delta)} \Vert x_s \Vert^2 \le  \sigma^2d_xf(t(\delta))$. Hence, recalling the expression of $t(\delta)$, we obtain
\begin{equation} \label{eq:I:regret:before}
    \widetilde{R}_{t(\delta) - 1}(\pi)  \le C_5 \log(e/\delta)^{30\gamma_\star^2}
\end{equation}
where we note that the hidden universal constants hidden in $\poly(\cdot)$ may be chosen large enough so that $C_5$.

\medskip

\underline{\textbf{Step 4:}} (Putting everything together) Now, under the event $\mathcal{E}_{\delta}$, using \eqref{eq:I:regret:after} and \eqref{eq:I:regret:before} we have
\begin{align*}
  \widetilde{R}_T(\pi) & =  \widetilde{R}_T(\pi) -\widetilde{R}_{t(\delta) - 1 }(\pi)  + \widetilde{R}_{t(\delta) -1}(\pi) \\
  & \le  C_4 \log(e/\delta) \log(T)\sqrt{T} + C_5 \log(e/\delta)^{31\gamma_\star^2}
\end{align*}
where we note that the universal postive  constants hidden in $\poly(.)$ may be chosen large enough so that $C_5 \ge C_6$. Therefore, we have established that
\begin{align*}
  \mathcal{E}_\delta \subseteq \left\lbrace  \widetilde{R}_T(\pi) \le  C_4 \log(e/\delta) \log(T)\sqrt{T} + 2C_5 \log(e/\delta)^{31\gamma_\star^2} \right\rbrace  \\
\end{align*}
where $C_6 = 2C_5$. Now recalling the expression of $t(\delta)$ and that $\mathbb{P}(\mathcal{E}_\delta) \ge 1-t(\delta)\delta$ for all $\delta \in (0,1)$, we obtain that for all $\delta \in (0,1)$, we have
\begin{align}\label{eq:I:reg:whp}
  \mathbb{P}\left(   \widetilde{R}_T(\pi) \le  C_4 \log(e/\delta) \log(T)\sqrt{T} + C_6 \log(e/\delta)^{31\gamma_\star^2}  \right) \ge 1 - C_6\log(e/\delta)^{31\gamma_\star} \delta
\end{align}
where $C_6 = \poly(\sigma, C_\circ,  \mathcal{G}_\circ, \Vert P_\star \Vert, C_B, d_x, d_u, \gamma_\star)$.
Now, integrating \eqref{eq:I:reg:whp} using Lemma \ref{lem:integration}, yields the final result
\begin{equation*}
  \E[R_T(\pi)] \le \E[\widetilde{R}_T(\pi)] \le C_7 \log(T)\sqrt{T} + C_8
\end{equation*}
where
\begin{align*}
  C_7 & = c_7\log(e \sigma C_\circ  \mathcal{G}_\circ \Vert P_\star \Vert C_Bd_xd_u \gamma_\star)^2 \sigma^2 C_B^2 C_K^2 \Vert P_\star \Vert^{9.5} d_x^{1/2}(d_x + d_u) \\
  C_8 & = \poly(\sigma, C_\circ,  \mathcal{G}_\circ, \Vert P_\star \Vert, C_B, d_x, d_u, \gamma_\star)
\end{align*}
and where $c_7$ is a positive constant that only depends polynomially on $\gamma_\star$ -- the order $\poly(\cdot)$ may depend on $\gamma_\star$.

\medskip

\subsection{Proof of Theorem \ref{th:scenarioIIA} - Regret analysis in Scenario II - $A$ known}\label{sec:proof:th3}

The proof is very similar to that of Theorem \ref{th:scenarioI} (see \ref{sec:proof:th1}). The only difference is that now there are no input perturbation whenever $\Alg$ uses the certainty equivalence controller, and the error rates of the LSE are now better. We shall highlight these differences throughout the proof.

Again following Lemma \ref{lem:rd}, we define the \emph{proxy regret} as follows
\begin{equation*}
\widetilde{R}_T(\pi) =   \sum_{t=1}^T \Vert x_t \Vert_{P_t - \widetilde{P}_{t-1}}^2  +   \tr\left(\widetilde{P}_t - P_\star\right) + \alpha_t \widetilde{\sigma}^2\tr\left(B^\top \widetilde{P}_t B + R\right) + \Vert x_{1}\Vert^2_{\widetilde{P}_0}
\end{equation*}
where $ (\widetilde{P}_t)_{t\ge 1}$ is defined as in Lemma \ref{lem:rd} with $\alpha_t = 1_{\lbrace \widetilde{K}_t \neq K_t \rbrace}$, $\xi_t = \zeta_t $ and thus $\widetilde{\sigma}^2 \le 1$. Note by the same lemma, we have $\E[R_T(\pi)] \le \E[\widetilde{R}_T(\pi)]$.

\medskip

\underline{\textbf{Step 1:}} {(Defining the nice event)} We start by applying Theorem \ref{thm:nice:I}, which guarantees that the event
\begin{align}
\mathcal{E}_{\delta} = \left\lbrace \forall t \ge t(\delta), \begin{array}{rl} \textit{(i)} & \widetilde{K}_t  = K_t, \\
\textit{(ii)} & \Vert B (K_t - K_\star) \Vert \le (4\Vert P_\star\Vert^{3/2})^{-1} \\
\textit{(iii)} & \Vert P_\star(K_t) - P_\star \Vert  \le C_1(\delta) r_t^2  \\
\textit{(iv)} & \Vert P_\star(K_t) \Vert  \le 2 \Vert P_\star \Vert \\
\textit{(v)} & \sum_{t=t(\delta)}^T r_t^2 \Vert x_t \Vert^2 \le C_2(\delta)  + C_3 \Vert r_{1:T}\Vert^2_2
\end{array}
\right\rbrace
\end{align}
holds with probability at leat $1-t(\delta)\delta$. In the definition $\mathcal{E}_{\delta}$, we have
\begin{align*}
  r_t^2 & = i_t^{-1}, \\
    i_t &=\max\{t_k\in {\cal T}: t_k\le t\},\\
  t(\delta) & = \poly( \sigma, C_\circ, \mathcal{G}_\circ, \Vert P_\star\Vert, \mu_\star^{-1}, d_x, d_u,\gamma_\star) \log(e/\delta)^{18\gamma_\star^2}, \\
  C_1(\delta) & =  \frac{c_1\sigma^2 \Vert P_\star \Vert^8  (d_u + d_x )\gamma_\star \log\left(\frac{e \sigma C_K \Vert P_\star \Vert d_x d_u}{\mu_\star^2} \right)}{\mu_\star^2}  \log(e/\delta),\\
  C_2(\delta) & = \poly(\sigma, C_\circ, \mathcal{G}_\circ, \Vert P_\star \Vert, \mu_\star^{-1}, d_x, d_u, \gamma_\star ) \log(e/\delta)^{54\gamma_\star^2}, \\
  C_3 & = c_3 \sigma^2 \Vert P_\star \Vert^{3/2}   d_x,
\end{align*}
for some universal positive constants $c_1, c_3 > 0$.

The properties \textit{(i)-(v)} have the same interpretations as in \ref{sec:proof:th1}, with the distinction that this time, the error rates are much better. We use these properties to bound the \emph{proxy regret}.

\medskip

\underline{\textbf{Step 2:}} {(Regret from $t(\delta)$ onwards under the nice event)}  We bound $\widetilde{R}_T(\pi) - \widetilde{R}_{t(\delta)-1}(\pi)$ under the event $\mathcal{E}_{\delta}$. First let us note that under this event, we have $\alpha_t = 0$ for all $t\ge t(\delta)$. Therefore, we have
\begin{align*}
  \widetilde{R}_T(\pi) - \widetilde{R}_{t(\delta) -1}(\pi) & \le   C_1(\delta) \sum_{t=t(\delta)}^T r_t^2  (\Vert x_t \Vert^2 + d_x)  \\
  & \le  C_1(\delta) \left( d_x \Vert r_{1:T} \Vert^2_2 + 2\sum_{t=t(\delta)}^T r_t^2  \Vert x_t \Vert^2  \right)    \\
  & \le C_1(\delta) (d_x \Vert r_{1:T} \Vert^2_2 + 2C_2(\delta) + 2C_3 \Vert r_{1:T}\Vert^2)  \\
  & \le 3C_3 C_1(\delta) \Vert r_{1:T}\Vert^2_2 +  2C_1(\delta) C_2(\delta)
\end{align*}
where in the first inequality, we used the \textit{(i)} to have $\widetilde{P}_t = P_\star(K_t)$ for all $t\ge t(\delta)$, then used \textit{(iii)} to bound $\Vert P_\star(K_t) - P_\star(K_{t-1})\Vert \le \Vert P_\star(K_t) - P_\star\Vert + \Vert P_\star(K_{t-1}) - P_\star \Vert \le 2C_1(\delta) r_t^2$ for all $t\ge t(\delta)+1$.

Since $\mathcal{T}$ satisfies \eqref{eq:lazy}, we can easily verify that $\Vert r_{1:T} \Vert^2_2 \lesssim\log(T)$. To see that, note that $(r_t)_{t\ge1}$ depends on ${\cal T}$, and we can always find $\mathcal{T}' = \lbrace (C')^k : k \in \mathbb{N}\rbrace$ for $C'$ large enough such that the corresponding sequence $(r_t')_{t\ge1}$ satisfies $\Vert r_{1,T} \Vert^2_2 \le \Vert r_{1:T}'\Vert^2_2 \lesssim \log(T)$ since ${\cal T}$ satisfies $\eqref{eq:lazy}$.

Therefore, recalling the expressions of $t(\delta), C_1(\delta), C_2(\delta)$ and $C_3$, we have
\begin{align}\label{eq:IIb:regret:after}
  \widetilde{R}_T(\pi) - \widetilde{R}_{t(\delta) -1}(\pi) & \le C_4 \log(e/\delta) \log(T) + C_5 \log(e/\delta)^{55\gamma_\star^2}
\end{align}
with
\begin{align*}
 C_4 & =  \frac{c_4 \sigma^2  \Vert P_\star \Vert^{9.5}}{\mu_\star^2} \log\left(\frac{e C_K \Vert P \Vert_\star d_x d_u}{\mu_\star^2}\right) d_x(d_x + d_u) \gamma_\star, \\
 C_5 & = \poly(\sigma, C_\circ, \mathcal{G}_\circ, \Vert P_\star\Vert, \mu_\star^{-1}, d_x,d_u, \gamma_\star),
\end{align*}
for some universal positive constant $c_4> 0$.

\medskip

\underline{\textbf{Step 3:}} {(Regret up to $t(\delta)$ under the nice event)s} Now, we bound $\widetilde{R}_{t(\delta)}(\pi)$ under the event $\mathcal{E}_\delta$. Note we can obtain the crude upper bound
\begin{align*}
 \widetilde{R}_{t(\delta) -1 }(\pi) & \le \max_{1 \le t \le t(\delta)} \Vert P_{\star,t} \Vert \sum_{t=1}^{t(\delta)} \Vert x_t \Vert^2 + d_x \widetilde{\sigma}(1 + \max(\Vert B \Vert^2, \Vert R\Vert) ) \\
 & \le  \max_{1 \le t \le t(\delta)} \Vert P_{\star,t} \Vert \sum_{t=1}^{t(\delta)} \Vert x_t \Vert^2 + 2 d_x C_B^2
\end{align*}
where we dropped the negative terms $-\Vert x_t \Vert^2_{\widetilde{P}_{t-1}}$ for $2 \le t \le t(\delta)$. Recalling the definition of $P_{\star,t}$, we have
\begin{align*}
 \Vert P_{\star,t} \Vert \le  \begin{cases} 2 \Vert P_\star \Vert & \text{ if } \Vert B (\widetilde{K}_t - K_\star) \Vert  < \frac{1}{4\Vert P_\star\Vert^{3/2}} \text{ and } \Vert P_\star(\widetilde{K}_t)\Vert \le 2\Vert P_\star \Vert^2,\\
 4 C_\circ^2 \Vert P_\star \Vert   h(t) &   \text{ otherwise}
\end{cases}
\end{align*}
where we upper bounded $\Vert P_{\star,t} \Vert \le 4 \Vert P_\star \Vert  C_\circ^2 h(t)$, using the fact that under $\Alg$, we have $\Vert \widetilde{K}_t \Vert^2 \le \max (\Vert K_\circ \Vert^2, h(t))$, and the fact that $\Vert P_\star\Vert \ge \max(\Vert Q\Vert, \Vert R\Vert)$. Thus,
$$
\max_{1  \le t\le t(\delta)}\Vert P_{\star,t}\Vert \le 4C_\circ^2 \Vert P_\star \Vert h(t).
$$
Thus, we may write
\begin{align*}
 \widetilde{R}_{t(\delta)- 1}(\pi) & \le  4C_\circ^2 \Vert P_\star \Vert h(t(\delta)) \sum_{t=1}^{t(\delta)} (\Vert x_t \Vert^2 + 2 d_x C_B^2)
\end{align*}
Therefore under event $\mathcal{E}_{\delta}$, we have
\begin{align*}
 \widetilde{R}_{t(\delta) - 1}(\pi) & \le 12 \sigma^2 d_x C_\circ^2 C_B^2 \Vert P_\star \Vert h(t(\delta)) f(t(\delta)),
\end{align*}
since when property \emph{(i)} holds at time $t(\delta)$, then it must mean that $\ell_{t(\delta)} = 1$, which means that $\sum_{s=1}^{t(\delta)} \Vert x_s \Vert^2 \le  \sigma^2d_xf(t(\delta))$. Hence, recalling the expression of $t(\delta)$, we obtain

\begin{equation} \label{eq:IIb:regret:before}
   \widetilde{R}_{t(\delta) - 1}(\pi)  \le C_6 \log(e/\delta)^{45\gamma_\star^3}
\end{equation}
where $C_6 = \poly( \sigma,  C_\circ, \mathcal{G}_\circ, \Vert P_\star \Vert^2, \mu_\star^{-1}, C_B, d_x, d_u, \gamma_\star)$.

\medskip

\underline{\textbf{Step 4:}} (Putting everything together) Now, under the event $\mathcal{E}_{\delta}$, using \eqref{eq:IIb:regret:after} and \eqref{eq:IIb:regret:before}, we have
\begin{align*}
  \widetilde{R}_T(\pi) & =  \widetilde{R}_T(\pi) -\widetilde{R}_{t(\delta) - 1 }(\pi)  + \widetilde{R}_{t(\delta) -1}(\pi) \\
  & \le  C_4 \log(e/\delta) \log(T) + C_7 \log(e/\delta)^{55\gamma_\star^3}
\end{align*}
where we note that the universal postive  constants hidden in $\poly(.)$ may be chosen large enough so that $C_7 \ge 2C_5 + C_6$. Therefore, we have established that
\begin{align*}
  \mathcal{E}_\delta \subseteq \left\lbrace  \widetilde{R}_T(\pi) \le  C_4 \log(e/\delta) \log(T) + C_7 \log(e/\delta)^{55\gamma_\star^3} \right\rbrace.
\end{align*}

 Now recalling the expression of $t(\delta)$ and that $\mathbb{P}(\mathcal{E}_\delta) \ge 1-t(\delta)\delta$ for all $\delta \in (0,1)$, we obtain that for all $\delta \in (0,1)$, we have
\begin{align}\label{eq:IIb:reg:whp}
  \mathbb{P}\left(   \widetilde{R}_T(\pi) \le  C_4 \log(e/\delta) \log(T) + C_7 \log(e/\delta)^{55\gamma_\star^3}  \right) \ge 1 - C_7 \log(e/\delta)^{55\gamma_\star^3} \delta
\end{align}
where the hidden universal constants in $\poly(\cdot)$ defining $C_7$ may be chosen to be large enough for the above to hold. Now, integrating \eqref{eq:IIb:reg:whp} using Lemma \ref{lem:integration}, yields the final result
\begin{equation*}
  \E[R_T(\pi)] \le \E[\widetilde{R}_T(\pi)] \le C_8 \log(T) + C_9
\end{equation*}
where
\begin{align*}
  C_8 & =  \frac{ c_8\log(e \sigma C_\circ \mathcal{G}_\circ \Vert P_\star \Vert \mu_\star^{-1} C_B d_x d_u \gamma_\star)^2 \sigma^2 \Vert P_\star \Vert^{9.5} d_x(d_x + d_u)}{\mu_\star^{-1}} \\
  C_9 & = \poly(\sigma, C_\circ,  \mathcal{G}_\circ, \Vert P_\star \Vert, \mu_\star^{-1}, C_B, d_x, d_u, \gamma_\star)
\end{align*}
where $c_8$ is a positive constant that only depends polynomially on $\gamma_\star$, and the order of $\poly(\cdot)$ may depend on $\gamma_\star$.


\medskip

\subsection{Proof of Theorem \ref{th:scenarioIIB} - Regret analysis in Scenario II - $B$ known}\label{app:thlast}

Again, the proof is very similar to that of Theorems \ref{th:scenarioI} and \ref{th:scenarioIIB} (see \ref{sec:proof:th1} and \ref{sec:proof:th3}). Note that in this scenario, there are no input perturbations, and the LSE error rates are as fast as in scenario II -- ($A$ known). We shall highlight these differences throughout the proof.

Again following Lemma \ref{lem:rd}, we define the \emph{proxy regret} as follows
\begin{equation*}
\widetilde{R}_T(\pi) =   \sum_{t=1}^T \Vert x_t \Vert_{P_t - \widetilde{P}_{t-1}}^2  +   \tr\left(\widetilde{P}_t - P_\star\right)  + \Vert x_{1}\Vert^2_{\widetilde{P}_0}
\end{equation*}
where $ (\widetilde{P}_t)_{t\ge 1}$ is defined as in Lemma \ref{lem:rd} with $\alpha_t = 0$, $\xi_t =0$. Note by the same lemma, we have $\E[R_T(\pi)] \le \E[\widetilde{R}_T(\pi)]$.

\medskip

\underline{\textbf{Step 1:}} {(Defining the nice event)} We start by applying Theorem \ref{thm:nice:I}, which guarantees that the event
\begin{align}
\mathcal{E}_{\delta} = \left\lbrace \forall t \ge t(\delta), \begin{array}{rl} \textit{(i)} & \widetilde{K}_t  = K_t, \\
\textit{(ii)} & \Vert B (K_t - K_\star) \Vert \le (4\Vert P_\star\Vert^{3/2})^{-1} \\
\textit{(iii)} & \Vert P_\star(K_t) - P_\star \Vert  \le C_1(\delta) r_t^2  \\
\textit{(iv)} & \Vert P_\star(K_t) \Vert  \le 2 \Vert P_\star \Vert \\
\textit{(v)} & \sum_{t=t(\delta)}^T r_t^2 \Vert x_t \Vert^2 \le C_2(\delta)  + C_3 \Vert r_{1:T}\Vert^2_2
\end{array}
\right\rbrace
\end{align}
holds with probability at leat $1-t(\delta)\delta$. In the definition of $\mathcal{E}_{\delta}$, we have
\begin{align*}
  r_t^2 & = i_t^{-1}, \\
  i_t &=\max\{t_k\in {\cal T}: t_k\le t\},\\
  t(\delta) & = \poly( \sigma, C_\circ, \mathcal{G}_\circ, \Vert P_\star\Vert, d_x, \gamma_\star) \log(e/\delta)^{6\gamma_\star^2}, \\
  C_1(\delta) & =  c_6\sigma^2 \Vert P_\star \Vert^8  d_x \log\left( e \Vert P_\star \Vert d_x  \right) \log(e/\delta), \\
  C_2(\delta) & = \poly(\sigma, C_\circ, \mathcal{G}_\circ, \Vert P_\star \Vert, d_x,  \gamma_\star ) \log(e/\delta)^{15\gamma_\star^3}, \\
  C_3 & = 24 \sigma^2 \Vert P_\star \Vert^{3/2}   d_x,
\end{align*}
for some universal postive constants $c_1, c_3 > 0$.

The properties \textit{(i)-(v)} have the same interpretations as in \ref{sec:proof:th1}, with the distinction that this time again, the error rates are much better. We use these properties to bound the \emph{proxy regret}.

\medskip

\underline{\textbf{Step 2:}} {(Regret from $t(\delta)$ onwards under the nice event)}  We bound $\widetilde{R}_T(\pi) - \widetilde{R}_{t(\delta)-1}(\pi)$ under the event $\mathcal{E}_{\delta}$. First let us note that under this event, we have $\alpha_t = 0$ for all $t\ge t(\delta)$. Therefore, we have
\begin{align*}
  \widetilde{R}_T(\pi) - \widetilde{R}_{t(\delta) -1}(\pi) & \le   C_1(\delta) \sum_{t=t(\delta)}^T r_t^2  (\Vert x_t \Vert^2 + d_x)  \\
  & \le  C_1(\delta) \left( d_x \Vert r_{1:T} \Vert^2_2 + 2\sum_{t=t(\delta)}^T r_t^2  \Vert x_t \Vert^2  \right)    \\
  & \le C_1(\delta) (d_x \Vert r_{1:T} \Vert^2_2 + 2C_2(\delta) + 2C_3 \Vert r_{1:T}\Vert^2)  \\
  & \le 3C_3 C_1(\delta) \Vert r_{1:T}\Vert^2_2 +  2C_1(\delta) C_2(\delta)
\end{align*}
where in the first inequality, we used the \textit{(i)} to have $\widetilde{P}_t = P_\star(K_t)$ for all $t\ge t(\delta)$, then used \textit{(iii)} to bound $\Vert P_\star(K_t) - P_\star(K_{t-1})\Vert \le \Vert P_\star(K_t) - P_\star\Vert + \Vert P_\star(K_{t-1}) - P_\star \Vert \le 2C_1(\delta) r_t^2$ for all $t\ge t(\delta)+1$.

Since $\mathcal{T}$ satisfies \eqref{eq:lazy}, we can easily verify that $\Vert r_{1:T} \Vert^2_2 \lesssim\log(T)$ (see the proof in the previous scenario).

Therefore, recalling the expressions of $t(\delta), C_1(\delta), C_2(\delta)$ and $C_3$, we have
\begin{align}\label{eq:IIa:regret:after}
  \widetilde{R}_T(\pi) - \widetilde{R}_{t(\delta) -1}(\pi) & \le C_4 \log(e/\delta) \log(T) + C_5 \log(e/\delta)^{16\gamma_\star^2}
\end{align}
with
\begin{align*}
 C_4 & =  c_4 \sigma^4  \Vert P_\star \Vert^{9.5} \log\left(e \Vert P \Vert_\star d_x \right)d_x^2 \gamma_\star, \\
 C_5 & = \poly(\sigma, C_\circ, \mathcal{G}_\circ, \Vert P_\star\Vert, \mu_\star^{-1}, d_x,d_u, \gamma_\star),
\end{align*}
for some universal positive constant $c_4> 0$.

\medskip

\underline{\textbf{Step 3:}} {(Regret up to $t(\delta)$ under the nice event).} Now, we bound $\widetilde{R}_{t(\delta)}(\pi)$ under the event $\mathcal{E}_\delta$. Note we can obtain the crude upper bound
\begin{align*}
 \widetilde{R}_{t(\delta) -1 }(\pi) & \le \max_{1 \le t \le t(\delta)} \Vert P_{\star,t} \Vert \sum_{t=1}^{t(\delta)} \Vert x_t \Vert^2
\end{align*}
where we dropped the negative terms $-\Vert x_t \Vert^2_{\widetilde{P}_{t-1}}$ for $2 \le t \le t(\delta)$. Recalling the definition of $P_{\star,t}$, we have
\begin{align*}
 \Vert P_{\star,t} \Vert \le  \begin{cases} 2 \Vert P_\star \Vert & \text{ if } \Vert B (\widetilde{K}_t - K_\star) \Vert  < \frac{1}{4\Vert P_\star\Vert^{3/2}} \text{ and } \Vert P_\star(\widetilde{K}_t)\Vert \le 2\Vert P_\star \Vert^2,\\
 4 C_\circ^2 \Vert P_\star \Vert   h(t) &   \text{ otherwise}
\end{cases}
\end{align*}
where we upper bounded $\Vert P_{\star,t} \Vert \le 4 \Vert P_\star \Vert  C_\circ^2 h(t)$, using the fact that under $\Alg$, we have $\Vert \widetilde{K}_t \Vert^2 \le \max (\Vert K_\circ \Vert^2, h(t))$, and the fact that $\Vert P_\star\Vert \ge \max(\Vert Q\Vert, \Vert R\Vert)$. Thus,
$$
\max_{1  \le t\le t(\delta)}\Vert P_{\star,t}\Vert \le 4C_\circ^2 \Vert P_\star \Vert h(t).
$$
Thus, we may write
\begin{align*}
 \widetilde{R}_{t(\delta)- 1}(\pi) & \le  4C_\circ^2 \Vert P_\star \Vert h(t(\delta)) \sum_{t=1}^{t(\delta)} \Vert x_t \Vert^2.
\end{align*}
Therefore, under event $\mathcal{E}_{\delta}$, we have
\begin{align*}
 \widetilde{R}_{t(\delta) - 1}(\pi) & \le  4 \sigma^2 d_x C_\circ^2 C_B^2 \Vert P_\star \Vert h(t(\delta)) f(t(\delta))
\end{align*}
since when property \emph{(i)} holds at time $t(\delta)$, then it must mean that $\ell_{t(\delta)} = 1$, which means that $\sum_{s=1}^{t(\delta)} \Vert x_s \Vert^2 \le  \sigma^2d_xf(t(\delta))$. Hence, recalling the expression of $t(\delta)$, we obtain

\begin{equation} \label{eq:IIa:regret:before}
   \widetilde{R}_{t(\delta) - 1}(\pi)  \le C_6 \log(e/\delta)^{15\gamma_\star^3}
\end{equation}
where $C_6 = \poly( \sigma,  C_\circ, \mathcal{G}_\circ, \Vert P_\star \Vert, C_B, d_x, \gamma_\star)$.

\medskip

\underline{\textbf{Step 4:}} (Putting everything together) Now, under the event $\mathcal{E}_{\delta}$, using \eqref{eq:IIa:regret:after} and \eqref{eq:IIa:regret:before} we have
\begin{align*}
  \widetilde{R}_T(\pi) & =  \widetilde{R}_T(\pi) -\widetilde{R}_{t(\delta) - 1 }(\pi)  + \widetilde{R}_{t(\delta) -1}(\pi) \\
  & \le  C_4 \log(e/\delta) \log(T) + C_7 \log(e/\delta)^{16\gamma_\star^3}
\end{align*}
where we note that the universal positive  constants hidden in $\poly(.)$ may be chosen large enough so that $C_7 \ge C_5 + C_6$. Therefore, we have established that
\begin{align*}
  \mathcal{E}_\delta \subseteq \left\lbrace  \widetilde{R}_T(\pi) \le  C_4 \log(e/\delta) \log(T) + C7 \log(e/\delta)^{16\gamma_\star^3} \right\rbrace.
\end{align*}
 Now recalling the expression of $t(\delta)$ and that $\mathbb{P}(\mathcal{E}_\delta) \ge 1-t(\delta)\delta$ for all $\delta \in (0,1)$, we obtain that for all $\delta \in (0,1)$,
\begin{align}\label{eq:IIa:reg:whp}
  \mathbb{P}\left(   \widetilde{R}_T(\pi) \le  C_4 \log(e/\delta) \log(T) + C_7 \log(e/\delta)^{16\gamma_\star^3}  \right) \ge 1 - C_7 \log(e/\delta)^{16\gamma_\star^3} \delta
\end{align}
where the hidden universal constants in $\poly(\cdot)$ defining $C_7$ may be chosen large enough for the above to hold. Now, integrating \eqref{eq:IIa:reg:whp} using Lemma \ref{lem:integration}, yields the final result
\begin{equation*}
  \E[R_T(\pi)] \le \E[\widetilde{R}_T(\pi)] \le C_8 \log(T) + C_9
\end{equation*}
where
\begin{align*}
  C_8 & =  c_8\log(e \sigma C_\circ \mathcal{G}_\circ \Vert P_\star \Vert  C_B d_x \gamma_\star)^2  \sigma^4  \Vert P_\star \Vert^{9.5} d_x^2 \gamma_\star, \\
  C_9 & = \poly(\sigma, C_\circ,  \mathcal{G}_\circ, \Vert P_\star \Vert, C_B, d_x, \gamma_\star),
\end{align*}
where $c_8$ is a positive constant that only depends polynomially on $\gamma_\star$, and the order of $\poly(\cdot)$ may depend on $\gamma_\star$.

\newpage

\section{The Nice Event and its Likelihood} \label{app:nice}

In Appendix \ref{app:regretproof}, we have seen that the regret analysis relied on the definition of a "nice" event ${\cal E}_{\delta}$, and on the fact that its occurrence probability is close enough to 1. This appendix is devoted to presenting such event and establishing its likelihood for all the three envisioned scenarios. The main results are stated in Theorem \ref{thm:nice:I},  Theorem \ref{thm:nice:IIb}, and Theorem \ref{thm:nice:IIa} for Scenario I, Scenario II -- $A$ known, and Scenario II -- $B$ known, respectively. Their proofs follow the same line of reasoning and rely on the consistency of the least squares estimator (see Appendix \ref{app:lse}), perturbations bounds on Riccati equations (see Appendix \ref{app:control}), and the fact that $\Alg$ eventually just uses the certainty equivalence controller $K_t$ (see Appendix \ref{app:hysteresis}).

\subsection{Scenario I}

To analyse regret, we need to ensure the event ${\cal E}_{\delta}$ where for all $t \ge t(\delta)$
\begin{itemize}
  \item[(i)]  $ \widetilde{K}_t = K_t$
  \item[(ii)] $ \Vert B(K_t - K_\star) \Vert \le (4 \Vert P_\star \Vert^{3/2})^{-1} $
  \item[(iii)] $ \Vert P_\star(K_t) - P_\star \Vert \le C_1(\delta)  r_t^2 $
  \item[(iv)] $  \Vert P_\star(K_t) \Vert \le 2 \Vert P_\star \Vert $
  \item[(v)] $ \sum_{s=t(\delta)}^t r_t^2 \Vert x_t \Vert ^2 \le C_3 \Vert r_{1:T}\Vert^2_2 + C_2(\delta) $
\end{itemize}
holds with probability at least $1 - t(\delta)\delta$ for all $\delta \in (0,1)$. We shall precise $t(\delta), C_1(\delta), C_2(\delta)$ and $C_3$ in Theorem \ref{thm:nice:I}. As for the sequence $(r_t)_{t\ge 1}$, it is defined as
$$
\forall t \ge 1, \quad  r_t^2 = \frac{\log(ei_t)}{i_t^{1/2}}
$$
where $i_t = \max \lbrace t_k \in  {\cal T}: t_k \le t \rbrace$. We note that if $\mathcal{T} = \mathbb{N}$, then $i_t = t$, and if $\mathcal{T} = \lbrace e^k: t \in \mathbb{N} \rbrace$, then $i_t = e^{\lfloor \log(t)\rfloor}$.


\begin{theorem}\label{thm:nice:I}
  Assume ${\cal T}$ satisfies \eqref{eq:lazy}. Then under $\Alg$, for all $\delta \in (0,1)$, we have
  $$
  \mathbb{P}(\forall t \ge t(\delta), \ \mathrm{(i)-(v)\ hold}) \ge 1- t(\delta) \delta
  $$
  where
  \begin{align*}
    t(\delta) & =  \poly( \sigma, C_\circ, \mathcal{G}_\circ, \Vert P_\star\Vert,d_x, d_u,\gamma_\star) \log(e/\delta)^{12\gamma_\star}, \\
    C_1(\delta) & = c_1 \sigma^2 C_K^2 \Vert P_\star \Vert^8 \gamma_\star d_x^{-1/2} (d_x + d_u) \log(e C_\circ  \mathcal{G}_\circ d_x ) \log(e/\delta), \\
    C_2(\delta) & = \poly(\sigma, C_\circ, \mathcal{G}_\circ, \Vert P_\star \Vert, C_B, d_x, d_u, \gamma_\star ) \log(e/\delta)^{19\gamma_\star^3}, \\
    C_3 & = c_3\sigma^2 C_B^2 \Vert P_\star \Vert^{3/2} \gamma_\star d_x,
  \end{align*}
  for some universal positive constant $c_1,c_3 > 0$.
\end{theorem}

\begin{proof} The proof proceeds in the following steps.

\underline{{\bf Step 1:}} (Least squares estimation under $\Alg$) First, since $\mathcal{T}$ satisifes \eqref{eq:lazy}, we have by Proposition \ref{prop:I:LSE2} that under $\Alg$, the following holds
  \begin{align}\label{eq:nice:LSE}
  \max(\left\Vert A_t - A \right\Vert^2, \left\Vert B_t - B \right\Vert^2)
     & \le \frac{C_1 \sigma^2 C_K^2 \left(d \gamma_\star \log\left(\sigma C_\circ \mathcal{G}_\circ d_x t  \right)  + \log(e/\delta) \right)  }{(d_xt)^{1/2}} 
  \end{align}
  with probability at least $1-\delta$, provided that
  \begin{equation}\label{eq:nice:I:sc1}
     t^{1/4} \ge c \sigma^2 C_\circ^2 C_K^2 \Vert P_\star \Vert^{10} \left(d \gamma_\star \log( e\sigma C_\circ \mathcal{G}_\circ \Vert P_\star\Vert d_x d_u \gamma_\star) + \log(e/\delta)\right),
  \end{equation}
  for some  universal positive constants $C_1,c_1>0$. Constraining further $t$ to satisfy
  \begin{equation}\label{eq:nice:I:sc2}
     t^{1/2} \ge \frac{C_1 160^2 \sigma^2 C_K^2 \Vert P_\star \Vert^{10} \left(d \gamma_\star \log\left(\sigma C_\circ \mathcal{G}_\circ d_x t  \right)  + \log(e/\delta) \right)  }{(d_x)^{1/2} }
  \end{equation}
  ensures in addition that
  \begin{equation}\label{eq:nice:I:close}
    \max(\left\Vert A_t - A \right\Vert, \left\Vert B_t - B \right\Vert) \le \frac{1}{160 \Vert P_\star \Vert^5}.
  \end{equation}
  Now, using Lemma \ref{lem:technical}, we can find a universal positive constant $c_2>0$ such that if
  \begin{equation}\label{eq:nice:I:sc3}
    t^{1/4} \ge c_2 \sigma^2 C_\circ^2 C_K^2 \Vert P_\star \Vert^{10} \left(d \gamma_\star \log( e\sigma C_\circ \mathcal{G}_\circ \Vert P_\star\Vert d_x d_u \gamma_\star) + \log(e/\delta)\right),
  \end{equation}
  then conditions \eqref{eq:nice:I:sc1} and \eqref{eq:nice:I:sc2} also hold.

  We remind here that because $\mathcal{T}$ satisfies condition \eqref{eq:lazy}, if the constant $c_2>0$ is chosen large enough, we can claim that for all $t \in \mathbb{N}$  such that \eqref{eq:nice:I:sc1} holds, we have $K_t = K_{t_k}$ for some $t_k$ that also satisfies \eqref{eq:nice:I:sc1}. Therefore, applying Lemma \ref{lem:uniform_over_time}, ensures that
  \begin{equation}
    \mathbb{P}(\forall t \ge t_1(\delta),  \quad  \eqref{eq:nice:LSE} \text{ and } \eqref{eq:nice:I:close}) \ge 1-\delta
  \end{equation}
  with
  \begin{equation}
    t_1(\delta)^{1/4} = c_3  \sigma^2 C_\circ^2 C_K^2 \Vert P_\star \Vert^{10} \left(d \gamma_\star \log( e\sigma C_\circ \mathcal{G}_\circ \Vert P_\star\Vert d_x d_u \gamma_\star) + \log(e/\delta)\right).
  \end{equation}
  for some universal postive constant $c_3 >0$. Finally, a direct application of Proposition \ref{prop:DARE:bounds} ensures that
  \begin{equation}
    \mathcal{C}_{\delta} = \left\lbrace  \forall t \ge t_1(\delta), \quad  \begin{array}{rl}
      (ii) & \Vert B(K_t - K_\star) \Vert \le \frac{1}{4\Vert P_\star \Vert^{3/2}}\\
      (iii) & \Vert P_\star(K_t) - P_\star \Vert \le C_1(\delta) r_t^2  \\
      (iv) & \Vert P_\star(K_t) \Vert \le  2 \Vert P_\star \Vert  \\
    \end{array} \right\rbrace
  \end{equation}
  holds with probability $1-\delta$, with
  \begin{equation*}
    C_1(\delta) =  140 C_1 \sigma^2 C_K^2 \Vert P_\star \Vert^{8} d_x^{-1/2}(d_x + d_u) \gamma_\star \log\left(e \sigma C_\circ \mathcal{G}_\circ d_x\right)  \log(e/\delta).
  \end{equation*}

\medskip

\underline{{\bf Step 2:}} (Commitment to the certainty equivalence controller) Using Theorem \ref{thm:commit}, we have
  \begin{equation}
    \mathcal{D}_\delta = \left\lbrace  \forall t \ge t_2(\delta), \quad (i) \; \widetilde{K}_t = K \right\rbrace
  \end{equation}
  holds with probability at least $1 - 5t_2(\delta)\delta$, with
  \begin{equation}\label{eq:nice:sc4}
    t_2(\delta) = \poly( \sigma, C_\circ, \mathcal{G}_\circ, \Vert P_\star\Vert,d_x, d_u,\gamma_\star) \log(e/\delta)^{12\gamma_\star}
  \end{equation}
  such that $t_2(\delta) \ge t_1(\delta)$ (this can be ensured by taking the universal positive constants hidden $\poly(\cdot)$ large enough).

\medskip

\underline{{\bf Step 3:}} (Stability under the certainty equivalence controller)  Let us define the events
  \begin{align*}
    E_{1,\delta,T}  & = \left\lbrace \sum_{t = t_2(\delta)}^{T}  r_t^2 \Vert x_t \Vert^2 \le 8 \Vert P_\star \Vert^{3/2} \left( \Vert r_{1:T} \Vert_{\infty}^2 \Vert x_{t_2(\delta)}\Vert^2 + 6C_B^2\sigma^2 \left(\Vert r_{1:T} \Vert^2_2 +  \log(e/\delta) \right)  \right)     \right\rbrace, \\
    E_{2,\delta} & = \left \lbrace \forall t \ge t_2(\delta), \ \ (i) \text{  and  } (ii) \text{  hold }  \right\rbrace.
  \end{align*}
  Noting that $t_2(\delta)$ may be chosen so that $t_2(\delta) \ge d_x$, we obtain  by a direct application Proposition \ref{prop:stable:K_t},
  \begin{align*}
    \mathbb{P}\left( E_{1,\delta,T} \cup E_{2,\delta}^c  \right) \ge 1-\delta.
  \end{align*}

\underline{{\bf Step 4:}} (Putting everything together) To conclude, we note that under the event $\mathcal{C}_\delta \cap \mathcal{D}_\delta \cap (E_{1,\delta, T} \cup E_{2,\delta})$, the propoerties $(i)- (iv)$ hold for all $t\ge t_2(\delta)$. Additionally, under $\Alg$, it also holds that  $\Vert x_{t_2(\delta)}\Vert^2 \le \sigma^2 d_x f(t_2(\delta))$, therefore we have
  \begin{align*}
  & \sum_{t = t_2(\delta)}^{T}  r_t^2 \Vert x_t \Vert^2 \le 8 \Vert P_\star \Vert^{3/2} \left( \Vert r_{t_2(\delta):T} \Vert_{\infty}^2 \Vert x_{t_2(\delta)}\Vert^2 + 6C_B^2\sigma^2 \left(d_x \Vert r_{1:T} \Vert^2_2 +  \log(e/\delta) \right) \right) \\
   & \ \ \ \ \le 8 \Vert P_\star \Vert^{3/2} \left( \Vert r_{t_2(\delta):T} \Vert_{\infty}^2 \Vert x_{t_2(\delta)}\Vert^2 + 6C_B^2\sigma^2 \left(d_x \Vert r_{1:T} \Vert^2_2 +  \log(e/\delta) \right) \right) \\
   & \ \ \ \ \le 48 \sigma^2 \Vert P_\star \Vert^{3/2} C_B^2  d_x \Vert r_{1:T}\Vert^2_2 + \poly(\sigma, C_\circ, \mathcal{G}_\circ, \Vert P_\star \Vert, C_B, d_x, d_u, \gamma_\star ) \log(e/\delta) t_2(\delta)^{1+\gamma/2} \\
   & \ \ \ \ \le 48 \sigma^2 \Vert P_\star \Vert^{3/2} C_B^2  d_x \Vert r_{1:T}\Vert^2_2 + \poly(\sigma, C_\circ, \mathcal{G}_\circ, \Vert P_\star \Vert, C_B, d_x, d_u, \gamma_\star ) \log(e/\delta)^{19\gamma_\star^2}
\end{align*}
where we used the fact $\Vert r_{t:T}\Vert_\infty^2 \le \poly(\sigma, C_\circ, \mathcal{G}_\circ, \Vert P_\star \Vert, d_x, d_u, \gamma_\star)$. Thus, defining the property
$$
\text{(v)} \sum_{t = t_2(\delta)}^{T}  r_t^2 \Vert x_t \Vert^2 \le C_3 \Vert r_{1:T}\Vert^2_2 + C_2(\delta),
$$
where
\begin{align*}
  C_2(\delta) & = \poly(\sigma, C_\circ, \mathcal{G}_\circ, \Vert P_\star \Vert, C_B, d_x, d_u, \gamma_\star ) \log(e/\delta)^{19\gamma_\star^2}, \\
  C_3 & = 48 \sigma^2 \Vert P_\star \Vert^{3/2} C_B^2  d_x,
\end{align*}
we have shown that
$$
\mathcal{E}_{\delta} = \lbrace \forall t \ge t_2(\delta), \ \ (i)-(v) \text{ hold}\rbrace \subseteq  \mathcal{C}_\delta \cap \mathcal{D}_\delta \cap E_{1,\delta, T}
$$
Finally, we note that  $\mathcal{C}_\delta \cap \mathcal{D}_\delta \cap (E_{1,\delta, T} \cup E_{2,\delta}) = \mathcal{C}_\delta \cap \mathcal{D}_\delta \cap E_{1,\delta, T}$. Therefore, by a union bound, we have
\begin{align*}
  \mathbb{P}(\mathcal{E}_{\delta}) & \ge 1 -\mathbb{P}\left(\mathcal{C}_\delta^c \cup \mathcal{D}_\delta^c \cup (E_{1,\delta, t} \cup E_{d,\delta})^c \right) \\
  & \ge 1 - 2\delta - 5t_2(\delta) \delta \\
  & \ge 1 - 7 t_2(\delta)\delta.
\end{align*}
This gives the desired result with modified universal constants.
\end{proof}

\medskip
\medskip

\subsection{Scenario II -- $A$ known}

To analyse regret, we need to ensure the event $\mathcal{E}_\delta$ where for all $t \ge t(\delta)$
\begin{itemize}
  \item[(i)]  $ \widetilde{K}_t = K_t$
  \item[(ii)] $ \Vert B(K_t - K_\star) \Vert \le (4 \Vert P_\star \Vert^{3/2})^{-1} $
  \item[(iii)] $ \Vert P_\star(K_t) - P_\star \Vert \le C_1(\delta)  r_t^2 $
  \item[(iv)] $  \Vert P_\star(K_t) \Vert \le 2 \Vert P_\star \Vert $
  \item[(v)] $ \sum_{s=t(\delta)}^t r_t^2 \Vert x_t \Vert ^2 \le C_3 \Vert r_{1:T}\Vert^2_2 + C_2(\delta) $
\end{itemize}
holds with probability at least $1 - t(\delta)\delta$ for all $\delta \in (0,1)$. We shall precise $t(\delta), C_1(\delta), C_2(\delta)$ and $C_3$ in Theorem \ref{thm:nice:IIb}. As for the sequence $(r_t)_{t\ge 1}$, it is defined this time as
$$
\forall t \ge 1, \quad  r_t^2 = \frac{1}{i_t}
$$
where  $i_t = \max \lbrace t_k \in  {\cal T}: t_k \le t \rbrace$.

\medskip


\begin{theorem}\label{thm:nice:IIb}
  Assume ${\cal T}$ satisfies \eqref{eq:lazy}. Then under $\Alg$, for all $\delta \in (0,1)$, we have
  $$
  \mathbb{P}(\forall t \ge t(\delta), \ \ \mathrm{(i)-(v)\ hold}) \ge 1- t(\delta) \delta
  $$
  where
  \begin{align*}
    t(\delta) & = \poly( \sigma, C_\circ, \mathcal{G}_\circ, \Vert P_\star\Vert, \mu_\star^{-1}, d_x, d_u,\gamma_\star) \log(e/\delta)^{18\gamma_\star^2}, \\
    C_1(\delta) & =  \frac{c_1\sigma^2 \Vert P_\star \Vert^8  (d_u + d_x )\gamma_\star \log\left(\frac{e \sigma C_K \Vert P_\star \Vert d_x d_u}{\mu_\star^2} \right)}{\mu_\star^2}  \log(e/\delta),\\
    C_2(\delta) & = \poly(\sigma, C_\circ, \mathcal{G}_\circ, \Vert P_\star \Vert, \mu_\star^{-1}, d_x, d_u, \gamma_\star ) \log(e/\delta)^{54\gamma_\star^3}, \\
    C_3 & = c_3 \sigma^2 \Vert P_\star \Vert^{3/2}   d_x,
  \end{align*}
  for some universal positive constant $c_1,c_3 > 0$.
\end{theorem}

\begin{proof}[Proof of Theorem \ref{thm:nice:IIb}] The proof proceeds in the following steps.

  \underline{{\bf Step 1:}} (Least squares estimation under $\Alg$) First, since $\mathcal{T}$ satisfies \eqref{eq:lazy}, we have by Proposition \ref{prop:IIb:LSE2} that under Algorithm $\Alg$, the following holds
  \begin{align}\label{eq:nice:IIb:LSE}
  \left\Vert B_t - B \right\Vert^2  \le   \frac{C_1 \sigma^2 \left(d_u\gamma_\star \log(\sigma C_\circ \mathcal{G}_\circ d_x d_u t)  + d_x + \log(e/\delta)\right)}{\mu_\star^2 t} 
  \end{align}
  with probability at least $1-\delta$, provided that
  \begin{equation}\label{eq:nice:IIb:sc1}
       t^{1/2} \ge \frac{c \sigma^2 C_K^2\Vert P_\star \Vert^{10}}{\mu_\star}  \left(d_u \gamma_\star \log( e\sigma C_\circ \mathcal{G}_\circ \Vert P_\star\Vert d_x d_u \gamma_\star) + \log(e/\delta)\right)
  \end{equation}
  for some  universal positive constants $C_1,c_1>0$. Constraining further $t$ to satisfy
  \begin{equation}\label{eq:nice:IIb:sc2}
     t  \ge \frac{C_1 160^2 \sigma^2 \Vert P_\star \Vert^{10} \left(d_u\gamma_\star \log(\sigma C_\circ \mathcal{G}_\circ d_x d_u t)  + d_x + \log(e/\delta)\right)}{\mu_\star^2 }
  \end{equation}
  ensures in addition that
  \begin{equation}\label{eq:nice:IIb:close}
    \left\Vert B_t - B \right\Vert \le \frac{1}{160 \Vert P_\star \Vert^5}
  \end{equation}
  Now, using Lemma \ref{lem:technical}, we can find an universal positive constant $c_2>0$ such that if
  \begin{equation}\label{eq:nice:IIb:sc3}
    t^{1/2} \ge \frac{c_2\sigma^2  C_K^2 \Vert P_\star \Vert^{10}}{\mu_\star} \left(d \gamma_\star \log( e\sigma C_\circ \mathcal{G}_\circ \Vert P_\star\Vert d_x d_u \gamma_\star) + \log(e/\delta)\right),
  \end{equation}
  then conditions \eqref{eq:nice:IIb:sc1} and \eqref{eq:nice:IIb:sc2} also hold.

  We remind here that because $\mathcal{T}$ satisfies condition \eqref{eq:lazy}, if the constant $c_2>0$ is chosen large enough, we can claim that for all $t \in \mathbb{N}$  such that \eqref{eq:nice:IIb:sc1} holds, we have $K_t = K_{t_k}$ for some $t_k$ that also satisfies \eqref{eq:nice:IIb:sc1}. Therefore, applying Lemma \ref{lem:uniform_over_time}, ensures that
  \begin{equation}
    \mathbb{P}(\forall t \ge t_1(\delta), \ \  \eqref{eq:nice:IIb:LSE} \text{ and } \eqref{eq:nice:IIb:close}) \ge 1-\delta
  \end{equation}
  with
  \begin{equation}
    t_1(\delta)^{1/2} = \frac{c_3\sigma^2  C_K^2 \Vert P_\star \Vert^{10}}{\mu_\star} \left(d \gamma_\star \log( e\sigma C_\circ \mathcal{G}_\circ \Vert P_\star\Vert d_x d_u \gamma_\star) + \log(e/\delta)\right),
  \end{equation}
  for some universal postive constant $c_3 >0$. Finally, a direct application of Proposition \ref{prop:DARE:bounds} ensures that
  \begin{equation}
    \mathcal{C}_{1, \delta} = \left\lbrace  \forall t \ge t_1(\delta), \quad  \begin{array}{rl}
      (ii) & \max(\Vert B(K_t - K_\star) \Vert, \Vert K_t - K_\star \Vert) \le \frac{1}{4\Vert P_\star \Vert^{3/2}}\\
      (iv) & \Vert P_\star(K_t) \Vert \le  2 \Vert P_\star \Vert  \\
    \end{array} \right\rbrace
  \end{equation}
  holds with probability $1-\delta$.

  \underline{\bf Step 2:} (Commitment to the certainty equivalence controller)
  Using Theorem \ref{thm:commit}, we have
  \begin{equation}
    \mathcal{D}_\delta = \left\lbrace  \forall t \ge t_2(\delta), \quad (i) \; \widetilde{K}_t = K \right\rbrace
  \end{equation}
  holds with probability at least $1 - 5t_2(\delta)\delta$, with
  \begin{equation}\label{eq:nice:sc4}
    t_2(\delta) = \poly( \sigma, C_\circ, \mathcal{G}_\circ, \Vert P_\star\Vert, \mu_\star^{-1}, d_x, d_u,\gamma_\star) \log(e/\delta)^{6\gamma_\star}
  \end{equation}
  such that $t_2(\delta) \ge t_1(\delta)$ (this can be ensured by taking the universal positive constants hidden in $\poly(\cdot)$ large enough).

  \underline{\bf Step 3:} (Refined error rate under the certainty equivalence controller) Now, note that since the event $\mathcal{C}_{1,\delta} \cap \mathcal{D}_{1,\delta}$ holds with probability at least $1- 6t_2(\delta) \delta$, we may apply Proposition \ref{prop:IIb:LSE3} and obtain
  $$
  \Vert B_t - B  \Vert \le  \frac{c_4 \sigma^2}{\mu_\star^2 t } \left((d_u + d_x)\gamma_\star \log\left(\frac{e \sigma C_K \Vert P_\star \Vert d_x d_u}{\mu_\star^2} \right)  + \log(e/\delta) \right)
  $$
  provided that
  $$
  t \ge  \poly( \sigma, C_\circ, \mathcal{G}_\circ, \Vert P_\star\Vert, \mu_\star^{-1}, d_x, d_u,\gamma_\star) \log(e/\delta)^{18\gamma_\star^2}
  $$
  for some universal positive constant $c_4> 0$. Again, after using Lemma \ref{lem:uniform_over_time}, and using the perturbation bounds via Proposition \ref{prop:DARE:bounds} we obtain that  the event
  \begin{equation}
    \mathcal{C}_{2, \delta} = \left\lbrace  \forall t \ge t_3(\delta), \quad  \begin{array}{rl}
      (ii) & \Vert B(K_t - K_\star) \Vert \le \frac{1}{4\Vert P_\star \Vert^{3/2}}\\
      (iii) & \Vert P_\star (K_t) - P_\star \Vert \le C_1(\delta) r_t^2 \\
      (iv) & \Vert P_\star(K_t) \Vert \le  2 \Vert P_\star \Vert  \\
    \end{array} \right\rbrace
  \end{equation}
  holds with probability at least $1 - c_5 t_3(\delta)\delta$, where we define
  \begin{align*}
    t_3(\delta) & = \poly( \sigma, C_\circ, \mathcal{G}_\circ, \Vert P_\star\Vert, \mu_\star^{-1}, d_x, d_u,\gamma_\star) \log(e/\delta)^{18\gamma_\star^2} \\
    C_1(\delta) & =  \frac{c_6\sigma^2 \Vert P_\star \Vert^8  (d_u + d_x )\gamma_\star \log\left(\frac{e \sigma C_K \Vert P_\star \Vert d_x d_u}{\mu_\star^2} \right)}{\mu_\star^2}  \log(e/\delta)
  \end{align*}
  for some universal positive constants $c_5, c_6 > 0$.

  \underline{\bf Step 4:} (Stability under the certainty equivalence controller)
  Let us define the events
  \begin{align*}
    E_{1,\delta,T}  & = \left\lbrace \sum_{t = t_3(\delta)}^{T}  r_t^2 \Vert x_t \Vert^2 \le 8 \Vert P_\star \Vert^{3/2} \left( \Vert r_{1:T} \Vert_{\infty}^2 \Vert x_{t_2(\delta)}\Vert^2 + 3\sigma^2 \left(\Vert r_{1:T} \Vert^2_2 +  \log(e/\delta) \right)  \right)     \right\rbrace, \\
    E_{2,\delta} & = \left \lbrace \forall t \ge t_3(\delta), \ \ (i) \text{  and  } (ii) \text{  hold }  \right\rbrace.
  \end{align*}
  Noting that $t_2(\delta)$ may be chosen so that $t_2(\delta) \ge d_x$, we obtain  by a direct application Proposition \ref{prop:stable:K_t}, that
  \begin{align*}
    \mathbb{P}\left( E_{1,\delta,T} \cup E_{2,\delta}^c  \right) \ge 1-\delta.
  \end{align*}

  \underline{\bf Step 5:} (Putting everything together) To conclude, we note that under the event $\mathcal{C}_{2,\delta} \cap \mathcal{D}_\delta \cap (E_{1,\delta, T} \cup E_{2,\delta})$, propoerties $(i)- (iv)$ hold for all $t\ge t_3(\delta)$. Additionally, under Algorithm $\Alg$, it also holds that  $\Vert x_{t_3(\delta)}\Vert^2 \le \sigma^2 d_x f(t_3(\delta))$, therefore it follows that
  \begin{align*}
  & \sum_{t = t_3(\delta)}^{T}  r_t^2 \Vert x_t \Vert^2 \le 8 \Vert P_\star \Vert^{3/2} \left( \Vert r_{t_3(\delta):T} \Vert_{\infty}^2 \Vert x_{t_3(\delta)}\Vert^2 + 3\sigma^2 \left(d_x \Vert r_{1:T} \Vert^2_2 +  \log(e/\delta) \right) \right) \\
   & \ \ \ \ \le 8 \Vert P_\star \Vert^{3/2} \left( \Vert r_{t_3(\delta):T} \Vert_{\infty}^2 \Vert x_{t_3(\delta)}\Vert^2 + 3\sigma^2 \left(d_x \Vert r_{1:T} \Vert^2_2 +  \log(e/\delta) \right) \right) \\
   & \ \ \ \ \le 24 \sigma^2 \Vert P_\star \Vert^{3/2}   d_x \Vert r_{1:T}\Vert^2_2 + \poly(\sigma, C_\circ, \mathcal{G}_\circ, \Vert P_\star \Vert, \mu_\star^{-1},  d_x, d_u, \gamma_\star ) \log(e/\delta) t_3(\delta)^{1+\gamma/2} \\
   & \ \ \ \ \le 24 \sigma^2 \Vert P_\star \Vert^{3/2}   d_x \Vert r_{1:T}\Vert^2_2 + \poly(\sigma, C_\circ, \mathcal{G}_\circ, \Vert P_\star \Vert,  \mu_\star^{-1}, d_x, d_u, \gamma_\star ) \log(e/\delta)^{54\gamma_\star^3}
\end{align*}
where we used the fact $\Vert r_{t:T}\Vert_\infty^2 \le 1$. Thus, defining the property
$$
\text{(v)} \sum_{t = t_3(\delta)}^{T}  r_t^2 \Vert x_t \Vert^2 \le C_3 \Vert r_{1:T}\Vert^2_2 + C_2(\delta),
$$
where
\begin{align*}
  C_2(\delta) & = \poly(\sigma, C_\circ, \mathcal{G}_\circ, \Vert P_\star \Vert, \mu_\star^{-1}, d_x, d_u, \gamma_\star ) \log(e/\delta)^{54\gamma_\star^3}, \\
  C_3 & = 24 \sigma^2 \Vert P_\star \Vert^{3/2}   d_x,
\end{align*}
we have shown that
$$
\mathcal{E}_{\delta} = \lbrace \forall t \ge t_3(\delta), \  \mathrm{(i)-(v)\ hold}\rbrace \subseteq  \mathcal{C}_{2,\delta} \cap \mathcal{D}_\delta \cap E_{1,\delta, T}.
$$
Finally, we note that  $\mathcal{C}_{2,\delta} \cap \mathcal{D}_\delta \cap (E_{1,\delta, T} \cup E_{2,\delta}) = \mathcal{C}_{2,\delta} \cap \mathcal{D}_\delta \cap E_{1,\delta, T}$. Therefore, by a union bound, we have
\begin{align*}
  \mathbb{P}(\mathcal{E}_{\delta}) & \ge 1 -\mathbb{P}\left(\mathcal{C}_\delta^c \cup \mathcal{D}_\delta^c \cup (E_{1,\delta, t} \cup E_{d,\delta})^c \right) \\
  & \ge 1 - c_7 t_3(\delta) \delta
\end{align*}
for some universal positive constant $c_7> 0$. This gives the desired result with modified universal constants.
\end{proof}


\medskip
\medskip

\subsection{Scenario II -- $B$ known}

To analyse regret, we need to ensure the event $\mathcal{E}_\delta$ where for all $t \ge t(\delta)$
\begin{itemize}
  \item[(i)]  $ \widetilde{K}_t = K_t$
  \item[(ii)] $ \Vert B(K_t - K_\star) \Vert \le (4 \Vert P_\star \Vert^{3/2})^{-1} $
  \item[(iii)] $ \Vert P_\star(K_t) - P_\star \Vert \le C_1(\delta)  r_t^2 $
  \item[(iv)] $  \Vert P_\star(K_t) \Vert \le 2 \Vert P_\star \Vert $
  \item[(v)] $ \sum_{s=t(\delta)}^t r_t^2 \Vert x_t \Vert ^2 \le C_3 \Vert r_{1:T}\Vert^2_2 + C_2(\delta) $
\end{itemize}
holds with probability at least $1 - t(\delta)\delta$ for all $\delta \in (0,1)$. We shall precise $t(\delta), C_1(\delta), C_2(\delta)$ and $C_3$ in Theorem \ref{thm:nice:IIa}. As for the sequence $(r_t)_{t\ge 1}$, it is defined this time as
$$
\forall t \ge 1, \quad  r_t^2 = \frac{1}{i_t}
$$
where $i_t = \max \lbrace t_k \in  {\cal T}: t_k \le t \rbrace$.

\medskip

\begin{theorem}\label{thm:nice:IIa}
  Assume ${\cal T}$ satisfies \eqref{eq:lazy}. Then under $\Alg$, for all $\delta \in (0,1)$, we have
  $$
  \mathbb{P}(\forall t \ge t(\delta), \ \ \mathrm{(i)-(v)\ hold}) \ge 1- t(\delta) \delta
  $$
  where
  \begin{align*}
    t(\delta) & = \poly( \sigma, C_\circ, \mathcal{G}_\circ, \Vert P_\star\Vert, d_x, \gamma_\star) \log(e/\delta)^{6\gamma_\star^2}, \\
    C_1(\delta) & =  c_6\sigma^2 \Vert P_\star \Vert^8  d_x \log\left( e \Vert P_\star \Vert d_x  \right) \log(e/\delta), \\
    C_2(\delta) & = \poly(\sigma, C_\circ, \mathcal{G}_\circ, \Vert P_\star \Vert, d_x,  \gamma_\star ) \log(e/\delta)^{15\gamma_\star^3}, \\
    C_3 & = 24 \sigma^2 \Vert P_\star \Vert^{3/2}   d_x,
  \end{align*}
  for some universal positive constant $c_1,c_3 > 0$.
\end{theorem}

\begin{proof}[Proof of Theorem \ref{thm:nice:IIa}] The proof proceeds in the following steps.

  \underline{\bf Step 1:} (Least squares estimation under $\Alg$) First, since $\mathcal{T}$ satisfies \eqref{eq:lazy}, we have by Proposition \ref{prop:IIa:LSE:1} that under $\Alg$, the following holds
  \begin{align}\label{eq:nice:IIa:LSE}
  \left\Vert A_t - A \right\Vert^2  \le   \frac{C_1 \sigma^2 \left(d_x\gamma \log(e \sigma C_\circ \mathcal{G}_\circ d_x t)  + \log(e/\delta)\right)}{t} 
  \end{align}
  with probability at least $1-\delta$, provided that
  \begin{equation}\label{eq:nice:IIa:sc1}
       t  \ge c_1  (d_x \gamma_\star \log(e\sigma C_\circ \mathcal{G}_\circ d_x \gamma_\star)   + \log(e/\delta)),
  \end{equation}
  for some  universal positive constants $C_1,c_1>0$. Constraining further $t$ to satisfy
  \begin{equation}\label{eq:nice:IIa:sc2}
     t  \ge C_1 160^2 \sigma^2 \Vert P_\star \Vert^{10}  \left(d_x\gamma \log(e \sigma C_\circ \mathcal{G}_\circ d_x t)  + \log(e/\delta)\right)
  \end{equation}
  ensures in addition that
  \begin{equation}\label{eq:nice:IIa:close}
    \left\Vert A_t - A \right\Vert \le \frac{1}{160 \Vert P_\star \Vert^5}.
  \end{equation}
  Now, using Lemma \ref{lem:technical}, we can find an universal positive constant $c_2>0$ such that if
  \begin{equation}\label{eq:nice:IIa:sc3}
    t  \ge c_2\sigma^2  \Vert P_\star \Vert^{10} \left(d_x \gamma \log( e\sigma C_\circ \mathcal{G}_\circ \Vert P_\star\Vert d_x \gamma) + \log(e/\delta)\right),
  \end{equation}
  then conditions \eqref{eq:nice:IIa:sc1} and \eqref{eq:nice:IIa:sc2} also hold.

  We remind here that because $\mathcal{T}$ satisfies condition \eqref{eq:lazy}, if the constant $c_2>0$ is chosen large enough, we can claim that for all $t \in \mathbb{N}$  such that \eqref{eq:nice:IIa:sc1} holds, we have $K_t = K_{t_k}$ for some $t_k$ that also satisfies \eqref{eq:nice:IIa:sc1}. Therefore, applying Lemma \ref{lem:uniform_over_time} ensures that
  \begin{equation}
    \mathbb{P}(\forall t \ge t_1(\delta),  \quad  \eqref{eq:nice:IIa:LSE} \text{ and } \eqref{eq:nice:IIa:close}) \ge 1-\delta
  \end{equation}
  with
  \begin{equation}
    t_1(\delta)= c_3\sigma^2  \Vert P_\star \Vert^{10} \left(d_x \gamma_\star \log( e\sigma C_\circ \mathcal{G}_\circ \Vert P_\star\Vert d_x \gamma) + \log(e/\delta)\right),
  \end{equation}
  for some universal postive constant $c_3 >0$. Finally, a direct application of Proposition \ref{prop:DARE:bounds} ensures that
  \begin{equation}
    \mathcal{C}_{1, \delta} = \left\lbrace  \forall t \ge t_1(\delta), \quad  \begin{array}{rl}
      (ii) & \max(\Vert B(K_t - K_\star) \Vert, \Vert K_t - K_\star \Vert) \le \frac{1}{4\Vert P_\star \Vert^{3/2}}\\
      (iv) & \Vert P_\star(K_t) \Vert \le  2 \Vert P_\star \Vert  \\
    \end{array} \right\rbrace
  \end{equation}
  holds with probability $1-\delta$.

  \underline{\bf Step 2:} (Commitment to the certainty equivalence controller)
  Using Theorem \ref{thm:commit}, we have
  \begin{equation}
    \mathcal{D}_\delta = \left\lbrace  \forall t \ge t_2(\delta), \quad (i) \; \widetilde{K}_t = K \right\rbrace
  \end{equation}
  holds with probability at least $1 - 5t_2(\delta)\delta$, with
  \begin{equation}\label{eq:nice:sc4}
    t_2(\delta) = \poly( \sigma, C_\circ, \mathcal{G}_\circ, \Vert P_\star\Vert, \mu_\star^{-1}, d_x, d_u,\gamma_\star) \log(e/\delta)^{3\gamma_\star}
  \end{equation}
  such that $t_2(\delta) \ge t_1(\delta)$ (this can be ensured by taking the universal positive constants hidden in $\poly(\cdot)$ large enough)

  \underline{\bf Step 3:} (Refined error rate under the certainty equivalence controller) Now note that the event $\mathcal{C}_{1,\delta} \cap \mathcal{D}_{1,\delta}$ holds with probability at least $1- 6t_2(\delta) \delta$. Therefore, applying Proposition \ref{prop:IIb:LSE3} we obtain
  $$
  \Vert A_t - A  \Vert \le \frac{c_4 \sigma^2}{ t } \left( d_x  \log\left( e \Vert P_\star \Vert d_x  \right)  + \log(e/\delta) \right)
  $$
  provided that
  $$
  t \ge  \poly( \sigma, C_\circ, \mathcal{G}_\circ, \Vert P_\star\Vert,  d_x,\gamma_\star) \log(e/\delta)^{6\gamma_\star^2}
  $$
  for some universal positive constant $c_4> 0$. Again, after using Lemma \ref{lem:uniform_over_time}, and using the perturbation bounds via Proposition \ref{prop:DARE:bounds} we obtain that  the event
  \begin{equation}
    \mathcal{C}_{2, \delta} = \left\lbrace  \forall t \ge t_3(\delta), \quad  \begin{array}{rl}
      (ii) & \Vert B(K_t - K_\star) \Vert \le \frac{1}{4\Vert P_\star \Vert^{3/2}}\\
      (iii) & \Vert P_\star (K_t) - P_\star \Vert \le C_1(\delta) r_t^2 \\
      (iv) & \Vert P_\star(K_t) \Vert \le  2 \Vert P_\star \Vert  \\
    \end{array} \right\rbrace
  \end{equation}
  holds with probability at least $1 - c_5 t_3(\delta)\delta$, wehre we define
  \begin{align*}
    t_3(\delta) & = \poly( \sigma, C_\circ, \mathcal{G}_\circ, \Vert P_\star\Vert, d_x, \gamma_\star) \log(e/\delta)^{6\gamma_\star^2}, \\
    C_1(\delta) & =  c_6\sigma^2 \Vert P_\star \Vert^8  d_x \log\left( e \Vert P_\star \Vert d_x  \right) \log(e/\delta),
  \end{align*}
  for some universal positive constants $c_5, c_6 > 0$.

  \underline{\bf Step 4:} (Stability under the certainty equivalence controller)
  Let us define the events
  \begin{align*}
    E_{1,\delta,T}  & = \left\lbrace \sum_{t = t_3(\delta)}^{T}  r_t^2 \Vert x_t \Vert^2 \le 8 \Vert P_\star \Vert^{3/2} \left( \Vert r_{1:T} \Vert_{\infty}^2 \Vert x_{t_2(\delta)}\Vert^2 + 3\sigma^2 \left(\Vert r_{1:T} \Vert^2_2 +  \log(e/\delta) \right)  \right)     \right\rbrace, \\
    E_{2,\delta} & = \left \lbrace \forall t \ge t_3(\delta), \ \ (i) \text{  and  } (ii) \text{  hold }  \right\rbrace.
  \end{align*}
  Noting that $t_3(\delta)$ may be chosen so that $t_3(\delta) \ge d_x$, we obtain  by a direct application Proposition \ref{prop:stable:K_t}, that
  \begin{align*}
    \mathbb{P}\left( E_{1,\delta,T} \cup E_{2,\delta}^c  \right) \ge 1-\delta.
  \end{align*}

  \underline{\bf Step 5:} (Putting everything together) To conlcude, we note that under the event $\mathcal{C}_{2,\delta} \cap \mathcal{D}_\delta \cap (E_{1,\delta, T} \cup E_{2,\delta})$, the propoerties $(i)- (iv)$ hold for all $t\ge t_3(\delta)$. Additionally, under $\Alg$, it also holds that  $\Vert x_{t_3(\delta)}\Vert^2 \le \sigma^2 d_x f(t_3(\delta))$, therefore it follows that
  \begin{align*}
  & \sum_{t = t_3(\delta)}^{T}  r_t^2 \Vert x_t \Vert^2 \le 8 \Vert P_\star \Vert^{3/2} \left( \Vert r_{t_3(\delta):T} \Vert_{\infty}^2 \Vert x_{t_3(\delta)}\Vert^2 + 3\sigma^2 \left(d_x \Vert r_{1:T} \Vert^2_2 +  \log(e/\delta) \right) \right) \\
   & \ \ \ \ \le 8 \Vert P_\star \Vert^{3/2} \left( \Vert r_{t_3(\delta):T} \Vert_{\infty}^2 \Vert x_{t_3(\delta)}\Vert^2 + 3\sigma^2 \left(d_x \Vert r_{1:T} \Vert^2_2 +  \log(e/\delta) \right) \right) \\
   & \ \ \ \ \le 24 \sigma^2 \Vert P_\star \Vert^{3/2}   d_x \Vert r_{1:T}\Vert^2_2 + \poly(\sigma, C_\circ, \mathcal{G}_\circ, \Vert P_\star \Vert,   d_x, \gamma_\star ) \log(e/\delta) t_3(\delta)^{1+\gamma/2} \\
   & \ \ \ \ \le 24 \sigma^2 \Vert P_\star \Vert^{3/2}   d_x \Vert r_{1:T}\Vert^2_2 + \poly(\sigma, C_\circ, \mathcal{G}_\circ, \Vert P_\star \Vert,   d_x, \gamma_\star ) \log(e/\delta)^{15\gamma_\star^2}
\end{align*}
where we used the fact $\Vert r_{t:T}\Vert_\infty^2 \le 1$. Thus, defining the property
$$
\text{(v)} \sum_{t = t_3(\delta)}^{T}  r_t^2 \Vert x_t \Vert^2 \le C_3 \Vert r_{1:T}\Vert^2_2 + C_2(\delta),
$$
where
\begin{align*}
  C_2(\delta) & = \poly(\sigma, C_\circ, \mathcal{G}_\circ, \Vert P_\star \Vert, d_x,  \gamma_\star ) \log(e/\delta)^{15\gamma_\star^2}, \\
  C_3 & = 24 \sigma^2 \Vert P_\star \Vert^{3/2}   d_x,
\end{align*}
we have shown that
$$
\mathcal{E}_{\delta} = \lbrace \forall t \ge t_3(\delta), \  \mathrm{(i)-(v)\ hold}\rbrace \subseteq  \mathcal{C}_\delta \cap \mathcal{D}_\delta \cap E_{1,\delta, T}.
$$
Finally, we note that  $\mathcal{C}_{2,\delta} \cap \mathcal{D}_\delta \cap (E_{1,\delta, T} \cup E_{2,\delta}) = \mathcal{C}_{2,\delta} \cap \mathcal{D}_\delta \cap E_{1,\delta, T}$. Therefore, by a union bound, we have
\begin{align*}
  \mathbb{P}(\mathcal{E}_{\delta}) & \ge 1 -\mathbb{P}\left(\mathcal{C}_\delta^c \cup \mathcal{D}_\delta^c \cup (E_{1,\delta, t} \cup E_{d,\delta})^c \right) \\
  & \ge 1 - c_7 t_3(\delta) \delta
\end{align*}
for some universal positive constant $c_7> 0$. This gives the desired result with modified universal constants.
\end{proof}
\newpage

\section{Hysteresis Switching} \label{app:hysteresis}

In this appendix, we analyze the hysteresis switching scheme of $\Alg$. The main result is stated in Theorem \ref{thm:commit}, and essentially says that eventually, $\Alg$ just uses the certainty equivalence controller $K_t$. The proof of this result relies mainly on the consistency of the least squares estimator (see Appendix \ref{app:lse}), and on the perturbation bounds on Riccati equations (see Appendix \ref{app:control}). The stability behaviour of the resulting dynamical system is also instrumental in the analysis (see Appendix \ref{app:lds}).


\subsection{Main result}

The following theorem says that after time $t(\delta)$, $\Alg$ only uses the certainty equivalence controller $K_t$. In its proof, we will use lemmas presented later in this appendix.

\begin{theorem}\label{thm:commit} Assume that ${\cal T}$ satisfies \eqref{eq:lazy}. For all $\delta \in (0,1)$, there exists a stopping time $\upsilon(\delta)$, such that $\Alg$ uses the certainty equivalence controller at $\upsilon(\delta)$, and such that the stopping time $\tau(\delta) = \inf\lbrace t > \upsilon(\delta): \ \Alg \text{ uses stabilizing controller at time } t\rbrace$ verifies
\begin{align*}
  \mathbb{P}(  \upsilon(\delta) \le t(\delta), \tau(\delta) = \infty) \ge 1- 5t(\delta)\delta
\end{align*}
where
$$
t(\delta) = \poly( \sigma, C_\circ, \mathcal{G}_\circ, \Vert P_\star\Vert, \mu_\star^{-1},d_x, d_u,\gamma_\star) \log(e/\delta)^{3\gamma_\star \beta}
$$
and where the order of the polynomial only depends on $\gamma$, and $\beta = 4$, $\alpha=0$ in Scenario I, $\beta=2$ in Scenario II -- ($A$ known), and $\beta=1$ in Scenario II -- ($B$ known). We note that in Scenario I and Scenario II -- ($B$ known), we have $\poly( \sigma, C_\circ, \mathcal{G}_\circ, \Vert P_\star\Vert, \mu_\star^{-1},d_x, d_u,\gamma_\star) = \poly( \sigma, C_\circ, \mathcal{G}_\circ, \Vert P_\star\Vert,d_x, d_u,\gamma_\star)$ since we have no more dependency on $\mu_\star$.
\end{theorem}
\begin{proof}[Proof of Theorem \ref{thm:commit}]
  The proof for the three scenarios are very similar. We provide the proof for Scenario I, as for the other two scenarios, we simply highlight the differences in the proof.

  \paragraph{Scenario I.} By Lemma \ref{lem:I:commit_conditions}, we have  for all $\delta \in (0,1)$, $\mathbb{P}(E_{1,\delta}) \ge 1-\delta$, provided that
  $$
  t_1(\delta)^{1/4} = c \sigma C_\circ^2 C_K^2 \Vert P_\star \Vert^{10} d \gamma_\star \log(e\sigma C_\circ \mathcal{G}_\circ \Vert P_\star \Vert d_x d_u \gamma_\star ) \log(e/\delta)
  $$
  for some universal positive constant $c_1> 0$ large enough so that $t_1(\delta) \ge d_x + \log(e/\delta)$. We note here that we may have the crude upper bound $C_K \le C_\circ \Vert P_\star\Vert$. Now defining
  \begin{align*}
    \upsilon(\delta) & = \inf\left\lbrace t > t_1(\delta): \  \ell_t = 1 \right\rbrace, \\
    \tau(\delta) & = \inf\left\lbrace t > \upsilon(\delta)\;: \   \ell_t = 0 \right\rbrace.
  \end{align*}
  and denoting
  $$
  t(\delta) =  c_2 (C_\circ \mathcal{G}_\circ)^{8/\gamma} \Vert P_\star \Vert^{3/\gamma} t_1(\delta)^{3\gamma_\star} )
  $$
  where $c_2> 0$ is a universal positive constant. Note that $t(\delta) \ge c_2 (C_\circ \mathcal{G}_\circ)^{8/\gamma} t_1(\delta)^{3\gamma_\star}$ (recall that $\Vert P_\star \Vert\ge 1$). Thus, provided $c_2$ is large enough, we may apply Lemma \ref{lem:hook} and obtain
  $$
  \Pr( \upsilon(\delta) \le t(\delta) ) \ge 1-\delta.
  $$
  Furthermore, note that $t(\delta)\ge  c_2\Vert P_\star \Vert^{3/2} t_1(\delta)$. Hence, provided $c_2$ is large enough, we may apply Lemma \ref{lem:commit}, and obtain
  $$
  \Pr(  \upsilon(\delta) \le t(\delta), \tau(\delta) < \infty ) \le 4 t(\delta) \delta.
  $$
  Therefore, we have
  \begin{align*}
      \mathbb{P}(  \upsilon(\delta) \le t(\delta), \tau(\delta) = \infty) & = \Pr(\upsilon(\delta) \le t(\delta))- \mathbb{P}(  \upsilon(\delta) \le t(\delta), \tau(\delta) < \infty) \\
      & \le 1 -\delta - 4 t(\delta) \delta \\
      & \ge 1- 5t(\delta) \delta,
  \end{align*}
  where $c_2$ may be chosen sufficiently large so that $t(\delta) \ge 1$.

  \medskip

  \paragraph{Scenario II -- ($A$ known).} We start by applying Lemma \ref{lem:IIb:commit_conditions}, to obtain that $\Pr(E_{2,\delta}) \ge 1-\delta$, provided that
  $$
  t_1(\delta)^{1/2} = \frac{c\sigma^2 C_K \Vert P_\star \Vert^{10}}{\mu_\star^2} \left(d \gamma_\star \log(e\sigma C_\circ \mathcal{G}_\circ) \Vert P_\star \Vert \mu_\star^{-1} d_x d_u \gamma_\star ) + \log(e/\delta) \right).
  $$
 The remaining of the proof follows similarly as in Scenario I.

  \medskip

  \paragraph{Scenario II -- ($B$ known).} We start by applying Lemma \ref{lem:IIa:commit_conditions}, to obtain that $\Pr(E_{3,\delta}) \ge 1-\delta$, provided that
  $$
    t_1(\delta) = c  \Vert P_\star \Vert^{10} (d_x \gamma_\star \log( e\sigma C_\circ \mathcal{G}_\circ \Vert P_\star\Vert d_x \gamma_\star) +\log(e/\delta)).
  $$
The remaining of the proof follows similarly as in Scenario I.
\end{proof}

\medskip

\subsection{The time it takes for $\Alg$ to use the certainty equivalence controller}

Lemma \ref{lem:hook} quantifies the probability of switching to the certainty equivalence controller (i.e., $\widetilde{K}_t = K_t$) at some time, say between $i$ and $j$. The proof of Lemma \ref{lem:hook} relies on Lemma \ref{lem:lti}) and Proposition \ref{prop:caution}. The stabilizing controller (i.e., $\widehat{K}_t = K_\circ)$ will eventually  bring the system to a stable behaviour, thus to a state where it can attempt the certainty equivalence controller.

\begin{lemma}\label{lem:hook}
  Under $\Alg$, for all $\delta \in (0,1)$ we have
  $$
  \mathbb{P}\left( \exists k \in \lbrace i, \dots,  j \rbrace: \ell_k = 1 \right) \ge 1 - \delta
  $$
  provided that
  $$
  i \ge d_x + \log(e/\delta) \quad \text{and}  \quad j \ge c(C_\circ \mathcal{G}_\circ)^{8/\gamma}i^{3\gamma_\star}
  $$
  for some universal positive constant $c> 0$. Refer to the pseudo-code of $\Alg$ for the definition of $\ell_k$.
\end{lemma}

\begin{proof}[Proof of Lemma \ref{lem:hook}]
  We start by defining the following events.
  \begin{align*}
      \mathcal{E}_{i,j} & = \left\lbrace  \exists k \in \lbrace i, \dots, j \rbrace: \quad \ell_k = 1 \right\rbrace, \\
      \mathcal{A}_{\delta, t} & = \left\lbrace \sum_{s=0}^t \Vert x_s \Vert^2  \le C_1 \sigma^2 \mathcal{G}_\circ^2 C_\circ^2 (d_x t^{1+2\gamma} + \log(e/\delta)) \right\rbrace.
  \end{align*}
  We have, by Proposition \ref{prop:caution}, that the event $\mathcal{A}_{\delta,t}$ holds with probability at least $1-\delta$ for some universal positive constant $C_1 > 1$. Our analysis is essentially the same for all the envisioned scenarios. Therefore, to avoid unnecessary rewriting, we denote
  $$
  \forall s \ge 0: \quad \xi_s = \begin{cases} \nu_s  & \text{ if in Scenario I}\\
  \zeta_s & \text{ if in Scenario II -- (}A\text{ known)} \\
  0 & \text{ if in Scenario II -- (}B\text{ known)}
  \end{cases}
  $$
  where we remark that for all $s\ge 0$, $\xi_s$ is a zero-mean, sub-gaussian random vector that has variance proxy at worst $\sigma^2$ for all $s\ge d_x$.

  Now, provided that $i \ge \log(e/\delta)$, observe that under the event $\mathcal{A}_{\delta, t} \cap \mathcal{E}_{i,j}^c$ the following holds
  \begin{itemize}
    \item [\it (a)] $\sum_{s=0}^j \Vert x_s \Vert^2 > \sigma^2 d_x f(j) = \sigma^2 d_x j^{1+\gamma/2}$,
    \item [\it (b)] $\sum_{s=0}^i \Vert x_s \Vert^2 \le C_1 \sigma^2 \mathcal{G}_\circ^2 C_\circ^2  (d_x i^{1+2\gamma} + \log(e/\delta)) \le 2C_1 \sigma^2 \mathcal{G}_\circ^2 C_\circ^2  d_x i^{3\gamma_\star}$,
    \item [\it (c)] $\forall s \in \lbrace i, \dots, j \rbrace$, $u_s = K_\circ x_s + \xi_s $.
  \end{itemize}

  Consider the following dynamical system
  $$
  \forall s \ge 0: \qquad y_{s+1} = (A+ BK_\circ) y_s + B \xi_{i+s} + \eta_{i+s}  \quad \text{with} \quad y_0 = x_i
  $$
  and note that when \emph{(c)} holds, we have $(x_i, \dots, x_j) = (y_0, \dots, y_{j-i})$, thus, $\sum_{s=i}^j \Vert x_s \Vert^2 = \sum_{s=0}^{j-i} \Vert y_s \Vert^2$. On the other hand, the event
  \begin{align*}
    \mathcal{E}_{1,\delta,i,j} = \left\lbrace  \sum_{s=0}^{j-i} \Vert y_s \Vert^2 \le 2 \mathcal{G}_{\circ}^2 \left(\Vert x_i\Vert^2 + 2 \sigma^2 C_\circ^2  (d_x (j-i) + \log(e/\delta) )\right) \right\rbrace
  \end{align*}
  holds with probability at least $1-\delta$ provided $i \ge d_x$. This follows from Lemma \ref{lem:lti}. Therefore, provided $i\ge d_x + \log(e/\delta)$, under the event $\mathcal{A}_{\delta,t}\cap \mathcal{E}_{i,j}^c \cap \mathcal{E}_{1,\delta,i,j}$, we have
  \begin{align*}
    \sum_{s=0}^{j} \Vert x_s \Vert^2 & \le \sum_{s=0}^{i} \Vert x_s\Vert^2 + \sum_{s=i}^j \Vert x_s \Vert^2 \\
    & \le 2C_1 \sigma^2 C_\circ^2 \mathcal{G}_\circ^2 d_x i^{3\gamma_\star} + \sum_{s=0}^{j-i} \Vert y_s \Vert^2 & (\textit{Because (b) and (c) hold})\\
    & \le  2C_1 \sigma^2 C_\circ^2 \mathcal{G}_\circ^2 d_x i^{3\gamma_\star} + 2\mathcal{G}_\circ^2 \Vert x_i \Vert^2 + 2 \sigma^2 C_\circ^2 \mathcal{G}_\circ^2  d_x j  & (\textit{Under }\mathcal{E}_{1,\delta,i,j})\\
    & \le 4 C_1 \sigma^2 C_\circ^2 \mathcal{G}_\circ^4 d_x (i^{3\gamma_\star} + j). & (\textit{Because (b)  holds})
  \end{align*}

  After some elementary calculations, we obtain that if $
  j \ge (8C_1C_\circ^2 \mathcal{G}_\circ^4)^{2/\gamma}, j \ge i^{3\gamma_\star}$,  and $i\ge d_x + \log(e/\delta)$, then, under the event
  $\mathcal{A}_{\delta,t}\cap \mathcal{E}_{i,j}^c \cap \mathcal{E}_{1,\delta,i,j}$, we have
  $$
  \sum_{s=0}^{j} \Vert x_s \Vert^2 \le  \sigma^2 d_x f(j).
  $$
 But this cannot hold under $\mathcal{E}_{i,j}^c$ otherwise it would contradict property \emph{(a)}, therefore it must be that $ \mathcal{A}_{\delta,t}\cap \mathcal{E}_{1,\delta,i,j} \subseteq \mathcal{E}_{i,j}$. Hence by union bound, we have
  $$
  \mathbb{P}\left( \mathcal{E}_{i,j}\right) \ge 1 -\mathbb{P}(\mathcal{A}_{\delta,t}^c \cup \mathcal{E}_{1,\delta,i,j}^c) \ge 1-2\delta
  $$
  provided that
  $$
  i \ge d_x + \log(e/\delta) \quad \text{and}  \quad j \ge c(C_\circ \mathcal{G}_\circ)^{8/\gamma}i^{3\gamma_\star}
  $$
  for some universal positive constant $c> 0$. Reparametrizing by $\delta' = 2\delta$ yields the desired result with modified positive constants.
\end{proof}

\medskip

\subsection{Consistency of LSE leads to commitment}

\paragraph{The event that leads to commitment.} We note that the conditions under which $\Alg$ switches to the certainty equivalence controller vary depending on which scenario we are in. Therefore, we are constrained to define the event that leads to commitment in each of the three envisioned scenarios.

For scenario I, we define for all $\delta \in (0,1)$, the event of interest as
\begin{equation}\label{eq:eventIcommit}
  E_{1,\delta} = \left\lbrace \forall t\ge t_1(\delta), \ \  \begin{array}{rl}
    (i) &  \Vert B (K_t - K_\star) \Vert \le \frac{1}{\Vert P_\star\Vert^{3/2}} \\
    (ii) & \Vert K_t \Vert^{2} \le h(t) \\
    (iii) & \lambda_{\min}\left(\sum_{s=0}^{t-1} \begin{bmatrix}
      x_s \\
      u_s
    \end{bmatrix} \begin{bmatrix}
      x_s \\
      u_s
    \end{bmatrix}^\top \right) \ge t^{1/4}
  \end{array}\right\rbrace.
\end{equation}
The definition of the time $t_1(\delta)$ is made precise in the following result, proved in \ref{app:d5}.

\begin{lemma}\label{lem:I:commit_conditions}
  For all $\delta \in (0,1)$, let $E_{1,\delta}$ be defined as in \eqref{eq:eventIcommit}.  Under Algorithm $\Alg$, assuming that $\mathcal{T}$ satisifes \eqref{eq:lazy}, then, for all $\delta\in (0,1)$,
  $$
  \mathbb{P}\left(E_{1,\delta} \right) \ge 1-\delta,
  $$
  provided that
  $$
  t_1(\delta)^{1/4} = c \sigma C_\circ^2 C_K^2 \Vert P_\star \Vert^{10} \left(d \gamma_\star \log(e\sigma C_\circ \mathcal{G}_\circ) \Vert P_\star \Vert d_x d_u \gamma_\star ) + \log(e/\delta) \right).
  $$
\end{lemma}

For Scenario II--($A$ known), we define for all $\delta \in (0,1)$, the event of interest as
\begin{equation}\label{eq:eventIIacommit}
  E_{2,\delta} = \left\lbrace \forall t\ge t_2(\delta),\ \ \begin{array}{rl}
    (i) &  \Vert B (K_t - K_\star) \Vert \le \frac{1}{\Vert P_\star\Vert^{3/2}} \\
    (ii) & \Vert K_t \Vert^{2} \le h(t) \\
    (iii) & \lambda_{\min}\left(\sum_{s=0}^{t-1}u_s u_s^\top \right) \ge t^{1/2}
  \end{array}\right\rbrace.
\end{equation}

\begin{lemma}\label{lem:IIa:commit_conditions}
  For all $\delta \in (0,1)$, let $E_{2,\delta}$ be defined as in \eqref{eq:eventIIacommit}.  Under Algorithm $\Alg$, assuming that $\mathcal{T}$ satisifes \eqref{eq:lazy}, then, for all $\delta\in (0,1)$,
  $$
  \mathbb{P}\left(E_{2,\delta} \right) \ge 1-\delta,
  $$
  provided that
  $$
t_2(\delta)^{1/2} = \frac{c\sigma^2 C_K \Vert P_\star \Vert^{10}}{\mu_\star^2} \left(d \gamma_\star \log(e\sigma C_\circ \mathcal{G}_\circ) \Vert P_\star \Vert \mu_\star^{-1} d_x d_u \gamma_\star ) + \log(e/\delta) \right),
  $$
\end{lemma}

For Scenario II-($B$ known),
\begin{equation}\label{eq:eventIIbcommit}
  E_{3,\delta} = \left\lbrace \forall t\ge t_3(\delta),\ \ \begin{array}{rl}
    (i) &  \Vert B (K_t - K_\star) \Vert \le \frac{1}{\Vert P_\star\Vert^{3/2}} \\
    (ii) & \Vert K_t \Vert^{2} \le h(t)
  \end{array}\right\rbrace.
\end{equation}

\begin{lemma}\label{lem:IIb:commit_conditions}
  For all $\delta \in (0,1)$, let $E_{3,\delta}$ be defined as in \eqref{eq:eventIIbcommit}.  Under Algorithm $\Alg$, assuming that $\mathcal{T}$ satisifes \eqref{eq:lazy}, then, for all $\delta\in (0,1)$,
  $$
  \mathbb{P}\left(E_{3,\delta} \right) \ge 1-\delta,
  $$
  provided that
  $$
    t_3(\delta) = c  \Vert P_\star \Vert^{10} (d_x \gamma_\star \log( e\sigma C_\circ \mathcal{G}_\circ \Vert P_\star\Vert d_x \gamma_\star) +\log(e/\delta)).
  $$
\end{lemma}

The proofs of Lemma \ref{lem:I:commit_conditions}, Lemma  \ref{lem:IIa:commit_conditions}, and Lemma \ref{lem:IIb:commit_conditions} are presented in \ref{app:d5}. They rely on the consistency of the Least squares algorithm under $\Alg$ when ${\cal T}$ satisfies \eqref{eq:lazy}.

In the definitions of the events $E_{1,\delta}$, $E_{2,\delta}$, and $E_{3,\delta}$, property (i) is the most important for establishing the commitment lemma that we shall provide shortly.

\medskip

\subsection{The commitment lemma}

Lemma \ref{lem:commit} states that $\Alg$ will eventually only use the certainty equivalence conotroller, provided that $K_t$ is sufficiently close to $K_\star$ (this is captured by property (i) in the aforementioned events).

\begin{lemma}\label{lem:commit}
Assume that $\mathcal{T}$ satisfies \eqref{eq:lazy}. Assume that for all $\delta \in (0,1)$,
  $$
  \mathbb{P}(E_{p,\delta}) \ge 1 -\delta
  $$
  for some $t_p(\delta) \ge d_x + \log(e/\delta)$ for $p \in \lbrace 1,2,3\rbrace$ and where $E_{1,\delta}$, $E_{2,\delta}$, and $E_{3,\delta}$ are defined in \eqref{eq:eventIcommit}, \eqref{eq:eventIIbcommit}, and \eqref{eq:eventIIacommit}, respectively. Define the following stopping times
  \begin{align*}
    \upsilon(\delta) & = \inf\left\lbrace t > t_p(\delta):   \ell_t = 1 \right\rbrace, \\
    \tau(\delta) & = \inf\left\lbrace t > \upsilon(\delta):   \ell_t = 0 \right\rbrace.
  \end{align*}
  Then for all $\delta \in (0,1)$, we have
  $$
  \mathbb{P}\left(  \upsilon(\delta) \le k ,  \tau(\delta) < \infty  \right) \le 4k\delta
  $$
  provided that $k \ge 18^{2/\gamma} \Vert P_\star\Vert^{3/\gamma} t_p(\delta)$ for some universal positive constant $c> 0$.
\end{lemma}

\begin{proof}[Proof of Lemma \ref{lem:commit}] For ease of notation, we drop the dependency of $\upsilon(\delta)$, and $\tau(\delta)$
on $\delta$, and simply write $\upsilon$, and $\tau$. Furthermore, we shall refer to
$$
\forall s \ge 0 \quad \xi_{s} = \begin{cases}
  \nu_s & \text{If in Scenario I} \\
  0 & \text{If in Scenario II --(}A \text{ known) o in Scenario II --(}B \text{ known).}
\end{cases}
$$
We note that $\xi_s$ is a zero-mean, sub-gaussian random vector with variance proxy at most $\sigma^2$ provided $s \ge d_x$. Fix $p \in \lbrace 1,2 , 3\rbrace $

  We have
  \begin{align*}
    \mathbb{P}\left(  \upsilon \le k, \tau  < \infty  \right) &\le \mathbb{P}\left(  \lbrace \upsilon \le k, \tau < \infty \rbrace  \cap E_{p,\delta} \right) + \Pr \left( {\cal E}_{1,\delta}^c \right) \\
    &\le \mathbb{P}\left(  \lbrace \upsilon \le k, \tau < \infty \rbrace \cap E_{p,\delta} \right) + \delta \\
    & \le \sum_{i=t(\delta)+1}^k \Pr( \lbrace \upsilon = i, \tau < \infty\rbrace  \cap E_{p,\delta}) + \delta \\
    & \le \sum_{i=t(\delta)+1}^k \sum_{j=i+1}^\infty \Pr( \lbrace \upsilon = i, \tau = j\rbrace  \cap E_{p,\delta}) +  \delta.
  \end{align*}

  \paragraph{Computing the probabilities $ \Pr( \lbrace \upsilon = i, \tau = j \rbrace \cap E_{p,\delta})$.} Let $j > i$. First, let us note that under the event $\lbrace \upsilon = i, \tau = j \rbrace \cap  E_{p,\delta} $, the following must hold
  \begin{itemize}
    \item[\it (a)] For all $ i \le  t < j $, $u_t = K_t x_t + \xi_t$ (since $\ell_t = 0$ and conditions \emph{(ii)-(iii)} hold),
    \item[\it (b)] $\sum_{s=0}^i \Vert x_s \Vert^2 < \sigma^2 d_x f(i) = \sigma^2 d_x i^{1+\gamma/2}$,
    \item[\it (c)] $\sum_{s=0}^j \Vert x_s \Vert^2 > \sigma^2 d_x g(j) = \sigma^2 d_x j^{1+\gamma}$.
  \end{itemize}
  First, we use a truncation trick, and define
  $$
  \forall t \ge 0, \ \ \widetilde{K}_t = (K_t - K_\star) 1_{\left\lbrace \Vert K_{i+s} - K_\star \Vert\le \frac{1}{4 \Vert P_\star \Vert^{3/2}} \right\rbrace} + K_\star
  $$
  and note that under the event ${E}_{p,\delta}$, we have $K_t = \widetilde{K}_t$ for all $t\ge t(\delta)$. Now consider the following dynamical system
  $$
  \forall s \ge 1:\quad y_{s+1} = (A + B \widetilde{K}_{i+s} )y_s + B\xi_{i+s} + \eta_{i+s} \quad \text{with} \quad y_0 = x_i.
  $$
  We note that under the event $E_{p,\delta}$, we have $(x_{i}, \dots, x_{j}) = (y_0, \dots, y_{j-i})$. On the other hand, the event
  $$
  \mathcal{E}_{2,\delta, i, j} = \left\lbrace \sum_{s=0}^{j-i} \Vert y_s \Vert^2  \le 8 \Vert P_\star\Vert^{3/2}\left(\Vert x_i\Vert^2 + 6 \sigma^2 C_\circ^2 (d_x (j-i) + \log(j^2 /\delta))\right) \right\rbrace
  $$
  holds with probability at least $1-\delta/j^2$, provided that $i \ge d_x$. This follows by Lemma \ref{lem:perturbed} (see also Proposition \ref{prop:stable:K_t}). Therefore, under the event $E_{p,\delta} \cap \mathcal{E}_{2,\delta,i,j} \cap \lbrace \upsilon = i, \tau = j \rbrace$, we have
  \begin{align*}
    \sum_{s=0}^j \Vert x_s \Vert^2 & \le \sum_{s=0}^i \Vert x_s \Vert^2 + \sum_{s=i}^j \Vert x_s \Vert^2 \\
    & \le \sigma^2 d_x i^{1+\gamma/2} + \sum_{s=i}^j \Vert x_s \Vert^2  & (\emph{Because (b) holds})\\
    & \le \sigma^2 d_x i^{1+\gamma/2} + 8 \Vert P_\star \Vert^{3/2} \Vert x_i \Vert^2 + 6 \sigma^2 C_\circ^2 (d_x+2) j & (\emph{Under }\mathcal{E}_{2,\delta,i,j} \cap E_{p,\delta})\\
    & \le 9 \sigma^2 \Vert P_\star \Vert^{3/2}d_x  i^{1+\gamma/2} + 3 \sigma^2 C_\circ^2 d_x j & (\emph{Because (a) holds})
  \end{align*}
  provided that $i \ge \log(e/\delta)$.

  After some elementary computations, we obtain that if $ j > 18^{2/\gamma} \Vert P_\star \Vert^{3/\gamma}$, and $i\ge \log(e/\delta)$, then under the event $E_{p,\delta} \cap \mathcal{E}_{2,\delta,i,j} \cap \lbrace \upsilon = i, \tau = j \rbrace$, we have
  $$
  \sum_{s=0}^j \Vert x_s \Vert^2 \le \sigma^2 d_x j^{1+\gamma}.
  $$
But this cannot be under the event $\lbrace \upsilon = i, \tau = j \rbrace \cap E_{p,\delta}$, otherwise it would contradict \emph{(c)}. Therefore, it must be that $\lbrace \upsilon = i, \tau = j \rbrace \cap E_{p,\delta} \subseteq \mathcal{E}_{2,\delta,i,j}^c$. Hence
  \begin{equation*}
    \mathbb{P}( \lbrace \upsilon = i, \tau = j \rbrace \cap E_{p,\delta} ) \le \mathbb{P}( \mathcal{E}_{2,\delta,i,j}^c ) \le \frac{\delta}{j^2}
  \end{equation*}
  provided that
  $$
  i \ge d_x + \log(e/\delta) \qquad \text{and}\qquad j > 18^{2/\gamma} \Vert P_\star \Vert^{3/\gamma}.
  $$
  Let us remind here that $j>i$ and $i > t(\delta)$.

  \paragraph{Concluding step.} To conclude, provided that $k \ge 18^{2/\gamma} \Vert P_\star\Vert^{3/\gamma} (d_x + \log(e/\delta))$, we have
  \begin{align*}
    \mathbb{P}\left( \upsilon \le k, \tau(\delta) < \infty \right)  & \le \sum_{i=t(\delta)+1}^k \sum_{j=i+1}^\infty \Pr( \upsilon = i, \tau = j, E_{p,\delta}) +  \delta \\
    & \le \sum_{i=t(\delta)+1}^k \sum_{j=i+1}^\infty \frac{\delta}{j^2}  +  \delta  \\
    & \le \frac{\pi^2k \delta}{6} +  \delta \le 4k\delta.
  \end{align*}
\end{proof}

\medskip

\subsection{Remaining proofs}\label{app:d5}

%

\begin{proof}[Proof of Lemma \ref{lem:I:commit_conditions}]
Using Lemma \ref{lem:I:endwe}, we have already established that:
\begin{equation}\label{eq:I:lse:whp:uniform}
  \mathbb{P}\left(  \forall t \ge t_4(\delta): \Vert B(K_t - K_\star)\Vert \le \frac{1}{5\Vert P_\star \Vert^{3/2}} \quad \text{and} \quad \Vert K_t \Vert^2 \le h(t)  \right) \ge 1-\delta
\end{equation}
with
\begin{equation}\label{eq:I:lse:sc:uniform}
  t_4(\delta)^{1/4} = c_1 \sigma C_\circ^2 \Vert P_\star \Vert^{10} \left( d\gamma_\star \log(e\sigma C_\circ \mathcal{G}_\circ \Vert P_\star \Vert d_x d_u \gamma_\star) + \log(e/\delta)   \right)
\end{equation}
for some universal positive constant $c_1 > 0$. Now we may use Proposition \ref{prop:I:se:enough} to obtain that:
\begin{equation*}
  \mathbb{P}\left( \lambda_{\min}\left(\sum_{s=0}^{t-1} \begin{bmatrix}
    x_s \\
    u_s
  \end{bmatrix} \begin{bmatrix}
    x_s \\
    u_s
  \end{bmatrix}^\top \right) \ge \frac{C_2\sigma^2\sqrt{d_x t}}{C_K^2}  \right) \ge 1 -\delta
\end{equation*}
provided that
$
t^{1/4} \ge c_2 \sigma C_\circ^2 C_K^2 \Vert P_\star \Vert^{10} (d\gamma_\star \log(e \sigma C_\circ \mathcal{G}_\circ \Vert P_\star \Vert d_x d_u \gamma_\star) + \log(e/\delta) )
$
for some universal positive constants $C_2, c_2 > 0$. From which we may conclude that
\begin{equation*}
  \mathbb{P}\left( \lambda_{\min}\left(\sum_{s=0}^{t-1} \begin{bmatrix}
    x_s \\
    u_s
  \end{bmatrix} \begin{bmatrix}
    x_s \\
    u_s
  \end{bmatrix}^\top \right) \ge t^{1/4} \right) \ge 1 -\delta
\end{equation*}
provided that
$
t^{1/4} \ge c_3 \sigma C_\circ^2 C_K^2 \Vert P_\star \Vert^{10} (d\gamma_\star \log(e \sigma C_\circ \mathcal{G}_\circ \Vert P_\star \Vert d_x d_u \gamma_\star) + \log(e/\delta) ),
$
for some universal positive constant $c_3>0$ that is chosen to be large enough so that  $\frac{C_2\sigma^2\sqrt{d_x t}}{C_K^2} \ge t^{1/4}$. Next, we apply Lemma \ref{lem:uniform_over_time} to obtain the following bound that holds uniformly over time.
\begin{equation}\label{eq:I:se:whp:uniform}
  \mathbb{P}\left(\forall t \ge t_5(\delta):\quad \lambda_{\min}\left( \sum_{s=0}^{t-1} \begin{bmatrix}
    x_s \\ u_s
  \end{bmatrix} \begin{bmatrix}
    x_s \\ u_s
  \end{bmatrix} ^\top\right) > t^{1/4}\right) \ge 1- \delta
\end{equation}
where
\begin{equation}\label{eq:I:se:sc:uniform}
  t_5(\delta)^{1/4} = c_4 \sigma C_\circ^2 C_K^2 \Vert P_\star \Vert^{10} (d \gamma_\star \log(e \sigma C_\circ \mathcal{G}_\circ \Vert P_\star \Vert d_x d_u \gamma_\star) + \log(e/\delta))
\end{equation}
for some universal constant $c_4 > 0$. Now, using a union bound, we obtain from \eqref{eq:I:lse:whp:uniform}, and \eqref{eq:I:se:whp:uniform} that
\begin{equation}
  \mathbb{P}\left( E_{1,\delta} \right) \ge 1- 2\delta
\end{equation}
with
$$
t_1(\delta)^{1/4} = c \sigma C_\circ^2 C_K^2 \Vert P_\star \Vert^{10} \left(d \gamma_\star \log(e\sigma C_\circ \mathcal{G}_\circ) \Vert P_\star \Vert d_x d_u \gamma_\star ) + \log(e/\delta) \right)
$$
for some universal positive constant $c>0$ chosen large enough so that $t_1(\delta) \ge \max(t_4 (\delta), t_5(\delta))$.
\end{proof}

\medskip

\begin{proof}[Proof of Lemma \ref{lem:IIa:commit_conditions}]
Using Lemma \ref{lem:I:endwe}, we have already established that:
\begin{equation}\label{eq:I:lse:whp:uniform}
  \mathbb{P}\left(  \forall t \ge t_4(\delta): \Vert B(K_t - K_\star)\Vert \le \frac{1}{5\Vert P_\star \Vert^{3/2}} \quad \text{and} \quad \Vert K_t \Vert^2 \le h(t)  \right) \ge 1-\delta
\end{equation}
with
\begin{equation}\label{eq:I:lse:sc:uniform}
  t_4(\delta)^{1/2} =  c  \Vert P_\star \Vert^{10} (d \gamma_\star \log( e\sigma C_\circ \mathcal{G}_\circ \Vert P_\star\Vert d_x d_u \gamma_\star) + \log(e/\delta))
\end{equation}
for some universal positive constant $c_1 > 0$. Now we may use Proposition \ref{prop:IIb:se} to obtain that:
\begin{equation*}
  \mathbb{P}\left( \lambda_{\min}\left(\sum_{s=0}^{t-1} u_s u_s^\top \right) \ge \frac{\mu_\star^2 t}{10}   \right) \ge 1 -\delta
\end{equation*}
provided that
$
t^{1/2} \ge c_2 \frac{\sigma^2 C_K \Vert P_\star \Vert^{10}}{\mu_\star} (d\gamma_\star \log(e \sigma C_\circ \mathcal{G}_\circ \Vert P_\star \Vert d_x d_u \gamma_\star) + \log(e/\delta) )
$
for some universal positive constants $c_2 > 0$. From which we may conclude that
\begin{equation*}
  \mathbb{P}\left( \lambda_{\min}\left(\sum_{s=0}^{t-1}u_s u_s^\top \right) \ge t^{1/2} \right) \ge 1 -\delta
\end{equation*}
provided that
$
t^{1/4} \ge \frac{c_3\sigma^2 C_K \Vert P_\star \Vert^{10}}{\mu_\star^2} (d\gamma_\star \log(e \sigma C_\circ \mathcal{G}_\circ \Vert P_\star \Vert d_x d_u \gamma_\star) + \log(e/\delta) ),
$
for some universal positive constant $c_3>0$ that is chosen to be large enough so that  $\frac{\mu_\star^2 t}{10} \ge t^{1/2}$. Next, we apply Lemma \ref{lem:uniform_over_time} to obtain the following bound that holds uniformly over time.
\begin{equation}\label{eq:I:se:whp:uniform}
  \mathbb{P}\left(\forall t \ge t_5(\delta):\quad \lambda_{\min}\left( \sum_{s=0}^{t-1} u_s u_s^\top \right) > t^{1/2}\right) \ge 1- \delta
\end{equation}
where
\begin{equation}\label{eq:I:se:sc:uniform}
  t_5(\delta)^{1/2} =  \frac{c_4\sigma^2 C_K \Vert P_\star \Vert^{10}}{\mu_\star^2} (d \gamma_\star \log(e \sigma C_\circ \mathcal{G}_\circ \Vert P_\star \Vert \mu_\star^{-1}d_x d_u \gamma_\star) + \log(e/\delta))
\end{equation}
for some universal constant $c_4 > 0$. Now, using a union bound we obtain from \eqref{eq:I:lse:whp:uniform}, and \eqref{eq:I:se:whp:uniform} that
\begin{equation}
  \mathbb{P}\left( E_{2,\delta}\right) \ge 1- 2\delta
\end{equation}
with
$$
t_2(\delta)^{1/2} = \frac{c\sigma^2 C_K \Vert P_\star \Vert^{10}}{\mu_\star^2} \left(d \gamma_\star \log(e\sigma C_\circ \mathcal{G}_\circ) \Vert P_\star \Vert \mu_\star^{-1} d_x d_u \gamma_\star ) + \log(e/\delta) \right)
$$
for some universal positive constant $c>0$ chosen large enough so that $t_2(\delta) \ge \max(t_4 (\delta), t_5(\delta))$.
\end{proof}

\medskip

\begin{proof}[Proof of Lemma \ref{lem:IIb:commit_conditions}]
  Let $t \ge 0$ and $\delta \in(0,1)$. Assume the following condition holds.
  \begin{equation}\label{eq:IIb:lse:sc1}
  t^{1/2} \ge  C 160^2 \Vert P_\star \Vert^{10}  \left(d_x \gamma \log\left(\sigma C_\circ \mathcal{G}_\circ d_x t  \right)  + \log(e/\delta) \right).
  \end{equation}
  Then, by Proposition \ref{prop:IIa:LSE:1}, if the condition
  \begin{equation}\label{eq:IIb:lse:sc2}
    t\ge c_1  (d_x \gamma_\star \log(e\sigma C_\circ \mathcal{G}_\circ d_x\gamma_\star) + \log(e/\delta)),
  \end{equation}
  also holds, then we have
  \begin{equation}
    \mathbb{P}\left(  \Vert A_t - A\Vert^2 \le \frac{1}{160^2 \Vert P_\star \Vert^{10}}\right) \ge 1- \delta.
  \end{equation}

  Note that $\mathcal{T}$ satisfies condition \eqref{eq:lazy}. Therefore, provided that the constant $c_1>0$ is chosen large enough, we can claim that for all $t \in \mathbb{N}$  such that \eqref{eq:IIb:lse:sc1} and \eqref{eq:IIb:lse:sc2} hold, we have $K_t = K_{t_k}$ for some $t_k$ that also satisfies \eqref{eq:IIb:lse:sc1} and \eqref{eq:II:lse:sc2}.

 Using Propostion \ref{prop:DARE:bounds}, we may conclude that provided \eqref{eq:IIb:lse:sc1} and \eqref{eq:IIb:lse:sc2} hold, we have
  $$
  \mathbb{P}\left( \max(\Vert K_t - K_\star \Vert, \Vert B(K_t - K_\star) \Vert) \le \frac{1}{5\Vert P_\star \Vert^{3/2}} \right) \ge 1-\delta.
  $$
  By Lemma \ref{lem:technical}, we can find a universal positive constant $c_2>0$ such that
  $$
  t \ge  c_2 \Vert P_\star \Vert^{10} (d_{x} \gamma_\star \log( e\sigma C_\circ \mathcal{G}_\circ \Vert P_\star\Vert d_x \gamma_\star) + \log(e/\delta))
  $$
  implies that the conditions \eqref{eq:IIb:lse:sc1} and \eqref{eq:IIb:lse:sc2} hold, and $\Vert K_t \Vert^2 \le h(t)$. Now using Lemma \ref{lem:uniform_over_time} gives that
  $$
  \mathbb{P}(   E_{3,\delta}) \ge 1 -\delta
  $$
  where
  $$
  t_3(\delta) = c  \Vert P_\star \Vert^{10} (d_x \gamma_\star \log( e\sigma C_\circ \mathcal{G}_\circ \Vert P_\star\Vert d_x \gamma_\star) +\log(e/\delta))
  $$
  for some universal positive constant $c > c_2$.
\end{proof}

\newpage
\section{The Least Squares Estimator and its Error Rate} \label{app:lse}

In this appendix, we study the performance of the Least Squares Estimator (LSE) of $(A,B)$. We first provide the pseudo-code of the algorithm giving this estimator, see Algorithm \ref{algo:subcce}. The error rate of the LSE in Scenario I is characterized in Propositions \ref{prop:I:LSE1} and \ref{prop:I:LSE2}. The error rate is analyzed for Scenario II -- $A$ known in Proposition \ref{prop:IIb:LSE1}, \ref{prop:IIb:LSE2}, and \ref{prop:IIb:LSE3}. Propositions \ref{prop:IIa:LSE:1} and \ref{prop:IIa:LSE3} upper bound the error rate in Scenario II -- $B$ known. It is worth mentioning that these results are established without the use of a doubling trick that would essentially mean that the controller is fixed. Here, in $\Alg$, we allow the controller to change over time. Further note that the results hold under $\Alg$ regardless of how $\mathcal{T}$ is chosen provided it satisifes \eqref{eq:lazy}. To derive upper bounds on the LSE error rate, we extensively exploit the results related to the smallest eigenvalue of the covariates matrix, presented in the next appendix.

The performance of $\Alg$ depends on the error rate of the LSE, but also on how well the certainty equivalence controller $K_t$ approximates the optimal controller $K_{\star}$. In Lemmas \ref{lem:I:endwe} and \ref{lem:II:end:we}, we present upper bounds of $\|K_t-K_{\star}\|$ in the different  scenarios. These lemmas are also used to refine the error rates of the LSE.

\subsection{Pseudo-code of the LSE}

\medskip

\begin{algorithm}[H]\label{algo:subcce}
\SetAlgoLined
\LinesNumbered
\SetKwInOut{Input}{input}\SetKwInOut{Output}{output}
 \Input{ Sample path $(x_0, u_0, \dots, x_{t-1}, u_{t-1}, x_t)$ and the cost matrices $Q$ and $R$}
 \Output{ Estimator $(A_t,B_t)$ of $(A,B)$}
 \BlankLine
 \If{$B$ known}{
 $A_t \gets \left(\sum_{s=0}^{t-1} (x_{s+1} - B u_s)x_{s}^\top\right) \left( \sum_{s=0}^{t-1}  x_s x_s^\top \right)^{\dagger}$\;
 }
 \If{$A$ known}{
 $B_t \gets \left(\sum_{s=0}^{t-2} (x_{s+1} - A x_s)u_{s}^\top\right) \left( \sum_{s=0}^{t-2}  u_s u_s^\top \right)^{\dagger}$\;
 }
 \If{$(A,B)$ are unkown}{
  $\begin{bmatrix} A_t & B_t \end{bmatrix} \gets \left(\sum_{s=0}^{t-2}   x_{s+1} \begin{bmatrix} x_{s} \\ u_s\end{bmatrix}^{\top}\right) \left(   \sum_{s=0}^{t-2} \begin{bmatrix} x_{s} \\ u_s\end{bmatrix} \begin{bmatrix} x_{s} \\ u_s\end{bmatrix}^{\top} \right)^{\dagger}$\;
 }
 \caption{Least Squares Estimation (LSE)}
\end{algorithm}


\subsection{Error rate in Scenario I}

\begin{proposition}\label{prop:I:LSE1}
  Under $\Alg$, for all $\delta \in (0,1)$,
  \begin{equation*}
  \max(\left\Vert A_t - A \right\Vert^2, \left\Vert B_t - B \right\Vert^2)
     \le \frac{C  \sigma^2 \left(d \gamma_\star \log\left(\sigma C_\circ \mathcal{G}_\circ d_x d_u t  \right)  + \log(e/\delta) \right)  }{t^{1/4}}
  \end{equation*}
  holds with probability at least $1 - \delta$, provided that
  $
  t\ge c \sigma^4 C_\circ^8 (d \gamma_\star \log(e\sigma C_\circ \mathcal{G}_\circ d_x d_u \gamma_\star) + \log(e/\delta))
  $
  for some positive constants $C,c > 0$.
\end{proposition}

\begin{proof}[Proof of Proposition \ref{prop:I:LSE1}]

  Fix $t \ge 2$. We start by writing the LSE error
  \begin{align*}
    \begin{bmatrix}
      A_t - A & B_t - B
    \end{bmatrix} & =  \left(\sum_{s=0}^{t-2} \eta_s \begin{bmatrix} x_s \\ u_s\end{bmatrix}^\top  \right) \left(  \sum_{s=0}^{t-2}  \begin{bmatrix} x_s \\ u_s \end{bmatrix}\begin{bmatrix} x_s \\ u_s \end{bmatrix}^\top \right)^{\dagger}.
  \end{align*}
  For ease of notation, let $y_s = \begin{bmatrix}
    x_s \\u_s
  \end{bmatrix}$ for all $s\ge 1$. Provided the $\sum_{s=0}^{t-2}  y_s y_s^\top$ is invertible, we can decompose the error as
  \begin{align*}
    \left\Vert \begin{bmatrix}
        A_t - A & B_t - B
      \end{bmatrix} \right\Vert^2
       \le \left\Vert \left(\sum_{s=0}^{t-2}  y_s y_s^\top \right)^{-1/2}   \left(\sum_{s=0}^{t-2}  y_s \eta_s^{\top} \right) \right\Vert^2\frac{1}{\lambda_{\min}\left( \sum_{s=0}^{t-2} y_s y_s^\top \right)}.
  \end{align*}

  \underline{\textbf{Step 1:}} (Bounding the smallest eigenvalue) Define the event
  $$
  \mathcal{B}_{t} = \left\lbrace \lambda_{\min}\left(\sum_{s=0}^{t-2} y_s y_s^\top \right) \ge \left(\frac{t}{4}\right)^{1/4}   \right\rbrace.
  $$
  Using Proposition \ref{prop:I:we}, we know that the event $\mathcal{B}_t$ holds with probability at least $1-\delta$ provided
 $$
 (\star) \quad t\ge c_1 \sigma^4 C_\circ^8 (\gamma_\star d \log(e\sigma C_\circ \mathcal{G}_\circ d_x d_u \gamma_\star) + \log(e/\delta)),
 $$
 for some $c_1 > 0$.
 Under the event $\mathcal{B}_t$, the estimation error may be upper bounded as
  \begin{align*}
    \left\Vert \begin{bmatrix}
        A_t - A & B_t - B
      \end{bmatrix} \right\Vert^2 & \le  \frac{\sqrt{2}}{t^{1/4}} \left\Vert \left(\sum_{s=0}^{t-2}  y_s y_s^\top \right)^{-1/2}   \left(\sum_{s=0}^{t-2}  y_s \eta_s^{\top} \right) \right\Vert^2, \\
      & \le  \frac{2\sqrt{2}}{t^{1/4}}   \left\Vert \left(\sum_{s=0}^{t-2}  y_s y_s^\top + \frac{t^{1/4}}{\sqrt{2}} I_{d_x + d_u} \right)^{-1/2}   \left(\sum_{s=0}^{t-2}  y_s \eta_s^{\top} \right) \right \Vert^2,
  \end{align*}
  where we used, for the last inequality, the fact that $2 \sum_{s=0}^{t-2} y_s y_s^\top \succeq \sum_{s=0}^{t-2} y_s y_s^\top + \frac{t^{1/4}}{\sqrt{2}}I_{d_x + d_u}$.

  \underline{\textbf{Step 2:}} (Bounding the self-normalized term) Define the events
  \begin{align*}
  \mathcal{E}_{1,t,\delta} & = \Bigg\lbrace  \left\Vert \left(\sum_{s=0}^{t-2}  y_s y_s^\top + \frac{t^{1/4}}{\sqrt{2}} I_{d_x + d_u} \right)^{\!\!-1/2} \!\!  \left(\sum_{s=0}^{t-2}  y_s \eta_s^{\top} \right) \right \Vert^2 \\
  &\qquad\qquad \qquad \le 7 \sigma^2\log\left(\frac{e^d \det\left(\frac{\sqrt{2}}{t^{1/4}} \sum_{s=0}^{t-2} y_s y_s^\top + I_{d} \right)}{ \delta}\right) \Bigg\rbrace, \\
  \mathcal{E}_{2,t,\delta} & = \left\lbrace  \sum_{s=0}^{t-2} \Vert \nu_s \Vert^2 \le \sigma^2 \sqrt{d_x} (  2 d_u t + 3\log(e/\delta))   \right\rbrace, \\
  \mathcal{A}_{t,\delta} & = \left\lbrace   \sum_{s=0}^{t} \Vert x_s \Vert^2 \le c_2 \sigma^2 C_\circ^2 \mathcal{G}_\circ^2 (d_x t^{1+2\gamma}  + \log(1/\delta)) \right\rbrace.
  \end{align*}
Then by Proposition \ref{prop:caution}, the event $\mathcal{A}_{t,\delta}$ holds with probability at least $1-\delta$ for some absolution positive constant $c_2>0$. By proposition \ref{prop:SNP++}, the event  $\mathcal{E}_{1,t}$ holds with probability at least $1-\delta$, and by Hanson-Wright inequality (see Proposition \ref{prop:HW}), the event $\mathcal{E}_{2,t,\delta}$ holds with probability $1-\delta$. Under the event $\mathcal{A}_{t,\delta} \cap \mathcal{E}_{1,t,\delta} \cap \mathcal{E}_{2,t,\delta}$, we have
  \begin{align*}
    \left(\det\left( \frac{\sqrt{2}}{t^{1/4}} \sum_{s=0}^{t-2}  y_s y_s^\top + I_{d} \right) \right)^{1/d} & \le \frac{\sqrt{2}}{t^{1/4}} \sum_{s=0}^{t-2}  \Vert y_s \Vert^2  + 1 \\
    & \le \frac{\sqrt{2}}{t^{1/4}} \sum_{s=0}^{t-2}  \Vert x_s \Vert^2 + \Vert u_s \Vert^2  + 1 \\
    & \le \frac{\sqrt{2}}{t^{1/4}} \sum_{s=0}^{t-2}  \Vert x_s \Vert^2 + \Vert \widetilde{K}_s \Vert^2  \Vert x_s \Vert^2  + \Vert \nu_s\Vert^2 + 1 \\
    & \le  \sqrt{2}(4c_2+6) \sigma^2 C_\circ^4 \mathcal{G}_\circ^2 d_x d_u  t^{3/4 + 5\gamma/2}  \\
    & \le  \sqrt{2}(4c_2+6) \sigma^2 C_\circ^4 \mathcal{G}_\circ^2 d_x d_u  t^{4\gamma_\star},
  \end{align*}
  where we used $\Vert \widetilde{K}_s \Vert^2 \le C_K^2 h(s) $, assumed that
  $$
  (\star\star)\quad t \ge \log(e/\delta),
  $$
  and used $\gamma_\star=\max(1,\gamma)$ to obtain the last two inequalities. Therefore, assuming $(\star\star)$, we have under the event $\mathcal{A}_{t,\delta} \cap \mathcal{E}_{t,\delta}$ that
  \begin{align*}
    \log\left(\frac{e^d \det\left( \frac{\sqrt{2}}{t^{1/4}} \sum_{s=0}^{t-2}  y_s y_s^\top + I_{d} \right)}{ \delta}\right) & \le  c_3 \sigma^2 \left(d \gamma_\star \log\left(\sigma C_\circ \mathcal{G}_\circ d_x d_u t  \right)  + \log(e/\delta) \right),
  \end{align*}
  for some universal positive constants.

  \underline{\textbf{Step 3:}}\ (Putting everything together) To conclude, under the event $\mathcal{A}_{t,\delta} \cap \mathcal{B}_{t} \cap \mathcal{E}_{t,\delta}$,
   \begin{equation}\label{eq:lse:I:we}
  \max(\left\Vert A_t - A \right\Vert^2, \left\Vert B_t - B \right\Vert^2)
     \le \frac{C  \sigma^2 \left(d \gamma_\star \log\left(\sigma C_\circ \mathcal{G}_\circ d_x d_u t  \right)  + \log(e/\delta) \right)  }{t^{1/4}},
  \end{equation}
  where $C = 2\sqrt{2}c_3$. Therefore, the upper bound \eqref{eq:lse:I:we} holds with probability $1-3\delta$ when $(\star)$ and $(\star\star)$ hold. These two conditions hold whenever
  $$
  t\ge c \sigma^4 C_\circ^8 (d \gamma_\star \log(e\sigma C_\circ \mathcal{G}_\circ d_x d_u \gamma_\star) + \log(e/\delta)),
  $$
for $c$ is large enough. Reparametrizing by $\delta' = \delta/3$ yields the desired guarantee with modified universal positive constants.
\end{proof}

\medskip

\begin{lemma}\label{lem:I:endwe} Assume that $\mathcal{T}$ satisfies \eqref{eq:lazy}. Then, under Algorithm $\Alg$, we have, for all $\delta \in (0,1)$,
  \begin{align*}
  \mathbb{P}\left( \forall t \ge t(\delta), \ \
  \begin{array}{rl}
    (i) &\max(\Vert K_t - K_\star \Vert, \Vert B(K_t - K_\star) \Vert) \le  \frac{1}{5\Vert P_\star \Vert^{3/2}} \\
    (ii) &\Vert K_t \Vert^2 \le h(t) \\
  \end{array}\right)  \ge 1 -\delta
  \end{align*}
  where
$  t(\delta)^{1/4} =  c \sigma C_\circ^2 \Vert P_\star \Vert^{10} (d \gamma_\star \log( e\sigma C_\circ \mathcal{G}_\circ \Vert P_\star\Vert d_x d_u \gamma_\star) + \log(e/\delta))
$, for some universal positive constant $c>0$.
\end{lemma}

\begin{proof}[Proof of Lemma \ref{lem:I:endwe}]
  Let $t \ge 0$ and $\delta \in(0,1)$. Assume that:
  \begin{equation}\label{eq:I:lse:sc1}
  t^{1/4} \ge  C 160^2 \sigma^2 \Vert P_\star \Vert^{10}  \left(d \gamma_\star \log\left(\sigma C_\circ \mathcal{G}_\circ d_x d_u t  \right)  + \log(e/\delta) \right).
  \end{equation}
  Then, by Proposition \ref{prop:I:LSE1}, if the condition
  \begin{equation}\label{eq:I:lse:sc2}
    t\ge c_1 \sigma^4 C_\circ^8 (d \gamma_\star \log(e\sigma C_\circ \mathcal{G}_\circ d_x d_u \gamma_\star) + \log(e/\delta)),
  \end{equation}
  also holds, then we have
  \begin{equation}
    \mathbb{P}\left( \max(\Vert A_t - A \Vert^2, \Vert B_t - B\Vert^2) \le \frac{1}{160^2 \Vert P_\star \Vert^{10}}\right) \ge 1- \delta.
  \end{equation}
Since $\mathcal{T}$ satisfies condition \eqref{eq:lazy}, if the constant $c_1>0$ is chosen large enough, we can claim that for all $t \in \mathbb{N}$  such that \eqref{eq:I:lse:sc1} and \eqref{eq:I:lse:sc2} hold, we have $K_t = K_{t_k}$ for some $t_k$ that also satisfies \eqref{eq:I:lse:sc1} and \eqref{eq:I:lse:sc2}.

  Using the perturbation bounds of Propostion \ref{prop:DARE:bounds}, we conclude that, provided \eqref{eq:I:lse:sc1} and \eqref{eq:I:lse:sc2} hold, we have
  $$
  \mathbb{P}\left( \max(\Vert K_t - K_\star \Vert, \Vert B(K_t - K_\star) \Vert) \le \frac{1}{5\Vert P_\star \Vert^{3/2}} \right) \ge 1-\delta.
  $$
  By Lemma \ref{lem:technical}, we can find a universal positive constant $c_2>0$ such that
  $$
  t^{1/4} \ge  c_2 \sigma^2 C_\circ^2 \Vert P_\star \Vert^{10} (d \gamma_\star \log( e\sigma C_\circ \mathcal{G}_\circ \Vert P_\star\Vert d_x d_u \gamma_\star) + \log(e/\delta))
  $$
  implies that the conditions \eqref{eq:I:lse:sc1} and \eqref{eq:I:lse:sc2} hold, and $\Vert K_t\Vert^2 \le h(t)$. Finally, to obtain a bound that holds uniformly over time, we use Lemma \ref{lem:uniform_over_time} and obtain
  $$
  \mathbb{P}\left( \forall t \ge t(\delta),\ \ \begin{array}{rl}
  (i) & \max(\Vert K_t - K_\star \Vert, \Vert B(K_t - K_\star) \Vert) \le  \frac{1}{5\Vert P_\star \Vert^{3/2}} \\
  (ii) & \Vert K_t \Vert^2 \le h(t)
  \end{array} \right)  \ge 1 -\delta,
  $$
  where
  $
  t(\delta)^{1/4} = c \sigma^2 C_\circ^2 \Vert P_\star \Vert^{10} (d \gamma_\star \log( e\sigma C_\circ \mathcal{G}_\circ \Vert P_\star\Vert d_x d_u \gamma_\star) + \log(e/\delta))
  $
  for some universal positive constant $c > c_2$.
\end{proof}

\medskip

\begin{proposition}\label{prop:I:LSE2} Assume that $\mathcal{T}$ satisfies \eqref{eq:lazy}.
  Under algorithm $\Alg$, for all $\delta \in (0,1)$,
  \begin{equation*}
  \max(\left\Vert A_t - A \right\Vert^2, \left\Vert B_t - B \right\Vert^2)
     \le \frac{C C_K^2 \left(d \gamma_\star \log\left(\sigma C_\circ \mathcal{G}_\circ d_x t  \right)  + \log(e/\delta) \right)  }{(d_xt)^{1/2}}
  \end{equation*}
  holds with probability at least $1 - \delta$, when
  $$
  t^{1/4} \ge c \sigma C_\circ^2 C_K^2 \Vert P_\star \Vert^{10} \left(d \gamma_\star \log( e\sigma C_\circ \mathcal{G}_\circ \Vert P_\star\Vert d_x d_u \gamma_\star) + \log(e/\delta)\right),
  $$
  for some positive constants $C,c > 0$.
\end{proposition}
\begin{proof} The proof differs from that of Proposition \ref{prop:I:LSE1} only in the second step, where instead we define the event
$$
\mathcal{B}_t = \left\lbrace  \lambda_{\min}\left( \sum_{s=0}^{t-2} \begin{bmatrix}
  x_s \\
  u_s
\end{bmatrix} \begin{bmatrix}
  x_s \\
  u_s
\end{bmatrix}^\top   \right) \ge \frac{C\sigma^2\sqrt{d_x t}}{C_K^2}   \right\rbrace,
$$
where $C>0$ is some universal constant to be defined through Proposition \ref{prop:I:se:enough}. Indeed, using Lemma \ref{lem:I:endwe}, we may apply Proposition \ref{prop:I:se:enough} which guarantees that $\mathcal{B}_t$ holds with probability at least $1  -\delta$, provided that
$$
t^{1/4} \ge c \sigma^2 C_\circ^2 C_K^2 \Vert P_\star \Vert^{10} \left(d \gamma_\star \log( e\sigma C_\circ \mathcal{G}_\circ \Vert P_\star\Vert d_x d_u \gamma_\star) + \log(e/\delta)\right)
$$
for some positive constant $c> 0$. The remaining steps are identical to those of the proof of Proposition \ref{prop:I:LSE1}.
\end{proof}

\medskip

\subsection{Error rate in Scenario II -- $A$ known}

\begin{proposition}\label{prop:IIb:LSE1}
  Under $\Alg$, for all $\delta \in (0,1)$,
    $$
  \Vert B_t - B  \Vert \le \frac{C \sigma^2 \left(d_u\gamma_\star \log(\sigma C_\circ \mathcal{G}_\circ d_x d_u t)  + d_x + \log(e/\delta)\right)}{t^{1/2}}
  $$
  with probability at least $1-\delta$, provided that
  $$
  t \ge c_1 \sigma^2 ( d_x \gamma_\star \log(e \sigma C_\circ\mathcal{G}_\circ d_x \gamma_\star)  + \log(e/\delta))
  $$
  for some universal positive constants $C, c > 0$.
\end{proposition}

\begin{proof}[Proof of Proposition \ref{prop:IIb:LSE1}]
Fix $t\ge 2$. We start by writing the estimation error as
\begin{align*}
  B_t - B = \left( \sum_{s=0}^{t-2}   \eta_s u_s^\top \right) \left( \sum_{s=0}^{t-2} u_s u_s^\top \right)^{\dagger}.
\end{align*}
When $\sum_{s=0}^{t-2} u_s u_s^\top$ is invertible, we can decompose the estimation error as
\begin{align*}
  \Vert B_t - B\Vert^2  \le \left\Vert \left(\sum_{s=0}^{t-2} u_s u_s^\top\right)^{-1/2} \left( \sum_{s=0}^{t-2} u_s \eta_s^\top \right)  \right\Vert^2 \frac{1}{\lambda_{\min}\left( \sum_{s=0}^{t-2} u_s u_s^\top \right)}.
\end{align*}

\underline{\bf Step 1:} {(Bounding the smallest eigenvalue)} Define the event
\begin{align*}
  \mathcal{B}_t = \left\lbrace \lambda_{\min}\left( \sum_{s=0}^t u_s u_s^\top\right) \ge \left(\frac{t}{2}\right)^{1/2} \right\rbrace.
\end{align*}
Using Proposition \ref{prop:IIb:we}, we have the event $\mathcal{B}_t$ holds with probability at least $1-\delta$, if the following condition $(\star)$ hold:
$$
(\star) \quad t\ge c_1 ( d_u \gamma_\star \log(e \sigma C_\circ\mathcal{G}_\circ d_x \gamma_\star)  + \log(e/\delta)).
$$
Under the event $\mathcal{B}_t$, the estimation error may be upper bounded as
\begin{align*}
  \Vert B_t - B \Vert^2 & \le \frac{\sqrt{2}}{t^{1/2}} \left\Vert  \left(\sum_{s=0}^{t-2} u_s u_s^\top \right)^{-1/2} \left( \sum_{s=0}^{t-2}  u_s \eta_s^\top\right)   \right\Vert^2 \\
  & \le \frac{2\sqrt{2}}{t^{1/2}} \left\Vert  \left(\sum_{s=0}^{t-2} u_s u_s^\top + \frac{t^{1/2}}{\sqrt{2}}  \right)^{-1/2} \left( \sum_{s=0}^{t-2}  u_s \eta_s^\top\right)   \right\Vert^2,
\end{align*}
where we used the fact that $2 \sum_{s=0}^{t-2} u_s u_s^\top \succeq \sum_{s=0}^{t-2} u_s u_u^\top + \frac{t^{1/2}}{\sqrt{2}}$.

\underline{\bf Step 2:} {(Bounding the self-normalized term)} Consider the events,
\begin{align*}
\mathcal{E}_{1, t,\delta} & = \Bigg\lbrace  \left\Vert \left(\sum_{s=0}^{t-2}  u_s u_s^\top + \frac{t^{1/2}}{\sqrt{2}} I_{d_u} \right)^{\!\!-1/2} \!\!  \left(\sum_{s=0}^{t-2}  u_s \eta_s^{\top} \right) \right \Vert^2\\
&\quad\qquad \qquad \le 7 \sigma^2\log\left(\frac{e^d \det\left(\frac{\sqrt{2}}{t^{1/2}} \sum_{s=0}^{t-2} u_s u_s^\top + I_{d_u} \right)}{ \delta}\right) \Bigg\rbrace, \\
\mathcal{E}_{2, t,\delta} & = \left\lbrace  \sum_{s=0}^{t-2} \Vert \zeta_s \Vert^2 \le \sigma^2 (2 d_u t + 3 \log(1/\delta)) \right\rbrace, \\
\mathcal{A}_{t,\delta} & = \left\lbrace   \sum_{s=0}^{t} \Vert x_s \Vert^2 \le c_2 \sigma^2 C_\circ^2 \mathcal{G}_\circ^2 (d_x t^{1+2\gamma}  + \log(1/\delta)) \right\rbrace.
\end{align*}
By Proposition \ref{prop:caution}, the event $\mathcal{A}_{t,\delta}$ holds with probability at least $1-\delta$ for some universal positive constant $c_2 > 0$. By Proposition \ref{prop:SNP++}, the event $\mathcal{E}_{1, t,\delta}$ holds with probability at least $1-\delta$, and by Hanson-Wright inequality (See Proposition \ref{prop:HW}) the event $\mathcal{E}_{2, t, \delta}$ holds with probability at least $1-\delta$. Under the event $\mathcal{A}_{t,\delta} \cap \mathcal{E}_{1, t,\delta}$ we have
\begin{align*}
  \left( \det\left( \frac{\sqrt{2}}{t^{1/2}} \sum_{s=0}^{t-2} u_s u_s^\top  + I_{d_u}  \right)\right)^{\frac{1}{d}} & \le  \frac{\sqrt{2}}{t^{1/2}} \sum_{s=0}^{t-2} \Vert u_s \Vert^2 + 1 \\
  & \le  \frac{\sqrt{2}}{t^{1/2}}  \sum_{s=0}^{t-2} \Vert\widetilde{K}_s \Vert^2 \Vert x_s \Vert^2 + \Vert \zeta_s \Vert^2  +1  \\
  & \le \sqrt{2} C_\circ^2 t^{(\gamma-1)/2}  \sum_{s=0}^{t-2}  \Vert x_s \Vert^2 + \Vert \zeta_s \Vert^2  +1 \\
  & \le 2\sqrt{2}(c_2 + 3)     \sigma^2 C_\circ^4 \mathcal{G}_\circ^2 d_x d_u  t^{(1+ 3\gamma)/2},
\end{align*}
where in the last two inequalities, we used the fact $\Vert \widetilde{K}_s \Vert^2 \le \Vert K_\circ \Vert^2 h(t)$, we assumed that
$$(\star\star)\quad  t \ge \log(e/\delta)
$$
holds and used $\gamma_\star = \max(\gamma,1)$. Therefore, assuming that $(\star\star)$ holds, under the event $\mathcal{E}_{1,t,\delta} \cap \mathcal{E}_{2,t,\delta} \cap \mathcal{A}_{t,\delta}$, we have
\begin{align*}
  \log\left( \frac{e^d \det\left( \frac{\sqrt{2}}{t^{1/2}} \sum_{s=0}^{t-2} u_s u_s^\top  + I_{d_u}\right)}{\delta} \right) \le c_3 \sigma^2 \left(d_u\gamma_\star \log(\sigma C_\circ \mathcal{G}_\circ d_x d_u t)  + \log(e/\delta)\right)
\end{align*}
for some universal constant $c_3 > 0$.
\medskip

\underline{\bf Step 3:} {(Putting everything together)} To conclude, under the event $\mathcal{A}_{t,\delta} \cap \mathcal{B}_t \cap \mathcal{E}_{1,t,\delta} \cap \mathcal{E}_{2,t,\delta}$, we have
\begin{align}\label{eq:IIb:LSE1}
  \Vert B_t - B  \Vert \le \frac{C \sigma^2\left(d_u\gamma_\star \log(\sigma C_\circ \mathcal{G}_\circ d_x d_u t)  + \log(e/\delta)\right)}{t^{1/2}}
\end{align}
for some universal constant $C > 0$. Therefore, the upper bound \eqref{eq:IIb:LSE1} holds with probability at least $1-\delta$, provided that $(\star)$ and $(\star\star)$ hold. These two conditions hold whenever
\begin{align*}
  t \ge c  \sigma^2 (d_u \gamma_\star \log(e\sigma C_\circ \mathcal{G}_\circ d_x \gamma_\star)   + \log(e/\delta))
\end{align*}
for some universal constant $c> 0$. This concludes the proof.
\end{proof}

\medskip

\begin{lemma}\label{lem:II:end:we} Assume that $\mathcal{T}$ satisfies \eqref{eq:lazy}. Under $\Alg$,  for all $\delta \in (0,1)$, we have
  \begin{align*}
  \mathbb{P}\left( \forall t \ge t(\delta), \ \  \begin{array}{rl}
  (i) & \max(\Vert K_t - K_\star \Vert, \Vert B(K_t - K_\star) \Vert) \le  \frac{1}{5\Vert P_\star \Vert^{3/2}} \\
  (ii) & \Vert K_t \Vert^2 \le h(t)
  \end{array} \right)  \ge 1 -\delta,
  \end{align*}
  where
$t(\delta)^{1/2} =  c  \sigma^2 \Vert P_\star \Vert^{10} (d \gamma_\star \log( e\sigma C_\circ \mathcal{G}_\circ \Vert P_\star\Vert d_x d_u \gamma_\star) + \log(e/\delta))$  for some universal positive constant $c>0$.
\end{lemma}

\begin{proof}
  Let $t \ge 0$ and $\delta \in(0,1)$. Assume that the following condition holds.
  \begin{equation}\label{eq:II:lse:sc1}
  t^{1/2} \ge  C 160^2 \sigma^2 \Vert P_\star \Vert^{10}  \left(d_u \gamma_\star \log\left(\sigma C_\circ \mathcal{G}_\circ d_x d_u t  \right)  + d_x + \log(e/\delta) \right).
  \end{equation}
  Then, by proposition \ref{prop:IIb:LSE1}, if the condition
  \begin{equation}\label{eq:II:lse:sc2}
    t\ge c_1  (d_u \gamma_\star \log(e\sigma C_\circ \mathcal{G}_\circ d_x d_u \gamma_\star) + \log(e/\delta)),
  \end{equation}
  also holds, then we have
  \begin{equation}
    \mathbb{P}\left(  \Vert B_t - B\Vert^2 \le \frac{1}{160^2 \Vert P_\star \Vert^{10}}\right) \ge 1- \delta.
  \end{equation}

Since $\mathcal{T}$ satisfies condition \eqref{eq:lazy}, if the constant $c_1>0$ is chosen large enough, we can claim that for all $t \in \mathbb{N}$  such that \eqref{eq:II:lse:sc1} and \eqref{eq:II:lse:sc2} hold, we have $K_t = K_{t_k}$ for some $t_k$ that also satisfies \eqref{eq:II:lse:sc1} and \eqref{eq:II:lse:sc2}.

  Using Propostion \ref{prop:DARE:bounds}, we may conclude that when \eqref{eq:II:lse:sc1} and \eqref{eq:II:lse:sc2} hold, we have
  $$
  \mathbb{P}\left( \max(\Vert K_t - K_\star \Vert, \Vert B(K_t - K_\star) \Vert) \le \frac{1}{5\Vert P_\star \Vert^{3/2}} \right) \ge 1-\delta.
  $$
  By Lemma \ref{lem:technical}, we can find a universal positive constant $c_2>0$ such that
  $$
  t^{1/2} \ge  c_2 \sigma^2 \Vert P_\star \Vert^{10} (d_{u} \gamma_\star \log( e\sigma C_\circ \mathcal{G}_\circ \Vert P_\star\Vert d_x d_u \gamma_\star) + \log(e/\delta))
  $$
  implies that the conditions \eqref{eq:II:lse:sc1} and \eqref{eq:II:lse:sc2} hold, and $\Vert K_t \Vert^2 \le h(t)$. Now using Lemma \ref{lem:uniform_over_time} yields
  $$
  \mathbb{P}\left( \forall t \ge t(\delta), \ \  \begin{array}{rl}
  (i) & \max(\Vert K_t - K_\star \Vert, \Vert B(K_t - K_\star) \Vert) \le  \frac{1}{5\Vert P_\star \Vert^{3/2}} \\
  (ii) & \Vert K_t \Vert^2 \le h(t)
  \end{array} \right)  \ge 1 -\delta,
  $$
  where
  $
  t(\delta)^{1/2} = c  \sigma^2 \Vert P_\star \Vert^{10} (d_u \gamma_\star \log( e\sigma C_\circ \mathcal{G}_\circ \Vert P_\star\Vert d_x d_u \gamma_\star) + d_x +\log(e/\delta))
  $
for some universal positive constant $c > c_2$.
\end{proof}

\medskip

\begin{proposition}\label{prop:IIb:LSE2} Assume that $\mathcal{T}$ satisfies \eqref{eq:lazy}. Under $\Alg$, for all $\delta \in (0,1)$,
  $$
  \Vert B_t - B  \Vert \le \frac{C \sigma^2 \left(d_u\gamma_\star \log(\sigma C_\circ \mathcal{G}_\circ d_x d_u t)  + d_x + \log(e/\delta)\right)}{\mu_\star^2 t}
  $$
  holds with probability at least $1-\delta$, provided that
  $$
  t^{1/2} \ge \frac{c \sigma^2 C_K^2\Vert P_\star \Vert^{10}}{\mu_\star}  \left(d_u \gamma_\star \log( e\sigma C_\circ \mathcal{G}_\circ \Vert P_\star\Vert d_x d_u \gamma_\star) + \log(e/\delta)\right)
  $$
  for some positive constants $C,c > 0$.
\end{proposition}
\begin{proof} The proof differs from that of Proposition \ref{prop:IIb:LSE1} only in the second step, where instead we define the event
$$
\mathcal{B}_t = \left\lbrace  \lambda_{\min}\left( \sum_{s=0}^{t-2} u_s u_s^\top  \right) \ge C \mu_\star^2   t   \right\rbrace,
$$
where $\mu_\star^2 = \min(\lambda_{\min}( K_\star K_\star^\top),1)$, and $C>0$ is some universal constant to be defined through Proposition \ref{prop:IIb:se}. Indeed, using Lemma \ref{lem:II:end:we}, we may apply Proposition \ref{prop:IIb:se} which guarantees that $\mathcal{B}_t$ holds with probability at least $1  -\delta$, provided that
$$
t^{1/2} \ge c \frac{\sigma^2 C_K^2\Vert P_\star \Vert^{10}}{\mu_\star}  \left(d_u \gamma_\star \log( e\sigma C_\circ \mathcal{G}_\circ \Vert P_\star\Vert d_x d_u \gamma_\star) + \log(e/\delta)\right)
$$
for some positive constant $c> 0$. The remaining steps are identical to those of the proof of Proposition \ref{prop:IIb:LSE1}.
\end{proof}

\medskip

\begin{proposition}\label{prop:IIb:LSE3}
  Assume that $\mathcal{T}$ satisfies \eqref{eq:lazy}, and that under $\Alg$, for all $\delta \in (0,1)$, we have
  \begin{align*}
    \Pr\left( \forall t \ge t(\delta), \begin{array}{rl}
    (i) & \widetilde{K}_t = K_t \\
    (ii) & \max(\Vert B (K_t - K_\star), \Vert K_t - K_\star \Vert) \Vert \le (4 \Vert P_\star \Vert^{3/2} )^{-1}
    \end{array}  \right) \ge 1 -t(\delta)(\delta)
  \end{align*}
  for some $t(\delta) \ge \log(e/\delta)$. Then for all $\delta \in (0,1)$, the following
  $$
  \Vert B_t - B  \Vert \le  \frac{c \sigma^2}{\mu_\star^2 t } \left((d_u + d_x)\gamma_\star \log\left(\frac{e \sigma C_K \Vert P_\star \Vert d_x d_u}{\mu_\star^2} \right)  + \log(e/\delta) \right)
  $$
  holds with probability at least $1-c_1 t(\delta) \delta $ provided that
  \begin{align*}
    t \ge c_2  \max( t(\delta)^{3\gamma_\star},\sigma^4 (d_u \gamma_\star \log(e\sigma C_\circ \mathcal{G}_\circ d_x \gamma_\star)   + \log(e/\delta))^2
  \end{align*}
  for some universal constants $C,c_1,c_2> 0$.
\end{proposition}

\begin{proof}[Proof of Proposition \ref{prop:IIb:LSE3}]
Fix $t\ge 2$. We start by writing the estimation error as
\begin{align*}
  B_t - B = \left( \sum_{s=0}^{t-2}   \eta_s u_s^\top \right) \left( \sum_{s=0}^{t-2} u_s u_s^\top \right)^{\dagger}.
\end{align*}
When $\sum_{s=0}^{t-2} u_s u_s^\top$ is invertible, we can decompose the estimation error as
\begin{align*}
  \Vert B_t - B\Vert^2  \le \left\Vert \left(\sum_{s=0}^{t-2} u_s u_s^\top\right)^{-1/2} \left( \sum_{s=0}^{t-2} u_s \eta_s^\top \right)  \right\Vert^2 \frac{1}{\lambda_{\min}\left( \sum_{s=0}^{t-2} u_s u_s^\top \right)}.
\end{align*}

\underline{\bf Step 1:} {(Bounding the smallest eigenvalue)} We define the event
$$
\mathcal{B}_t = \left\lbrace  \lambda_{\min}\left( \sum_{s=0}^{t-2} u_s u_s^\top  \right) \ge C_1 \mu_\star^2   t   \right\rbrace,
$$
where $\mu_\star^2 = \min(\lambda_{\min}( K_\star K_\star^\top),1)$, and $C_1>0$ is some universal constant to be defined through Proposition \ref{prop:IIb:se}. Indeed, using Lemma \ref{lem:II:end:we}, we may apply Proposition \ref{prop:IIb:se} which guarantees that $\mathcal{B}_t$ holds with probability at least $1  -\delta$, provided that
$$
(\star)\quad t^{1/2} \ge c_1 \frac{\sigma^2 C_K^2\Vert P_\star \Vert^{10}}{\mu_\star}  \left(d_u \gamma_\star \log( e\sigma C_\circ \mathcal{G}_\circ \Vert P_\star\Vert d_x d_u \gamma_\star) + \log(e/\delta)\right)
$$
for some positive constant $c_1> 0$. Under the event $\mathcal{B}_t$, the estimation error may be upper bounded as
\begin{align*}
  \Vert B_t - B \Vert^2 & \le \frac{1}{C_1\mu_\star^2 t} \left\Vert  \left(\sum_{s=0}^{t-2} u_s u_s^\top \right)^{-1/2} \left( \sum_{s=0}^{t-2}  u_s \eta_s^\top\right)   \right\Vert^2 \\
  & \le \frac{2}{C_1\mu_\star^2 t}\left\Vert  \left(\sum_{s=0}^{t-2} u_s u_s^\top + C_1 \mu_\star^2 t I_{d_u}  \right)^{-1/2} \left( \sum_{s=0}^{t-2}  u_s \eta_s^\top\right)   \right\Vert^2,
\end{align*}
where we used the fact that $2 \sum_{s=0}^{t-2} u_s u_s^\top \succeq \sum_{s=0}^{t-2} u_s u_u^\top +  C_1 \mu_\star^2 t I_{d_u}$.

\underline{\bf Step 2:} {(Bounding the self-normalized term)} Consider the events,
\begin{align*}
\mathcal{E}_{1, t,\delta} & = \Bigg\lbrace  \left\Vert \left(\sum_{s=0}^{t-2}  u_s u_s^\top + C_1 \mu_\star^2t  I_{d_u} \right)^{\!\!-1/2} \!\!  \left(\sum_{s=0}^{t-2}  u_s \eta_s^{\top} \right) \right \Vert^2\\
&\quad\qquad \qquad \le 7 \sigma^2\log\left(\frac{e^{d_x} \det\left(\frac{1}{C_1 \mu_\star^2 t} \sum_{s=0}^{t-2} u_s u_s^\top + I_{d_u} \right)}{ \delta}\right) \Bigg\rbrace, \\
\mathcal{E}_{2, t,\delta} & = \left\lbrace  \sum_{s=0}^{t-2} \Vert \zeta_s \Vert^2 \le \sigma^2 (2 d_u t + 3 \log(1/\delta)) \right\rbrace, \\
E_{1,\delta,t} & = \left\lbrace \sum_{t(\delta)}^t \Vert x_t \Vert^2  \le 8\Vert P_\star \Vert^{3/2}    (\Vert x_{t(\delta)} \Vert + 3 \sigma^2 (d_x t + \log(e/\delta) ))  \right\rbrace \\
\mathcal{C}_{1,\delta} & = \left\lbrace  \forall t \ge t(\delta), \begin{array}{rl}
(i) & \widetilde{K}_t = K_t \\
(ii) & \max(\Vert B (K_t - K_\star), \Vert K_t - K_\star \Vert) \Vert \le (4 \Vert P_\star \Vert^{3/2} )^{-1}
\end{array}   \right\rbrace
\end{align*}
By Proposition \ref{prop:caution}, the event $\mathcal{A}_{t,\delta}$ holds with probability at least $1-\delta$ for some universal positive constant $c_2 > 0$. By Proposition \ref{prop:SNP++}, the event $\mathcal{E}_{1, t,\delta}$ holds with probability at least $1-\delta$, and by Hanson-Wright inequality (See Proposition \ref{prop:HW}) the event $\mathcal{E}_{2, t, \delta}$ holds with probability at least $1-\delta$.
By Proposition \ref{prop:stable:K_t} we have $\mathbb{P}( E_{1,\delta,t} \cup \mathcal{C}_{1,\delta}^c) \ge 1-\delta$, and by assumption the event $\mathcal{C}_{1,\delta}^c)$ holds with probability at least $1-t(\delta)\delta$.
Under the event $\mathcal{E}_{2,t} \cap E_{1,\delta} \cap \mathcal{C}_{1,\delta}$ we have
\begin{align*}
  \left( \det\left( \frac{1}{C \mu_\star^2t} \sum_{s=0}^{t-2} u_s u_s^\top  + I_{d_u}  \right)\right)^{\frac{1}{d_u}}\!\!\! & \le  \frac{1}{C \mu_\star^2t}  \sum_{s=0}^{t-2} \Vert u_s \Vert^2 + 1 \\
  & \le  \frac{2}{C \mu_\star^2t}   \sum_{s=0}^{t-2} \Vert\widetilde{K}_s \Vert^2 \Vert x_s \Vert^2 +  \Vert \zeta_s \Vert^2   +1  \\
  & \le  \frac{2}{C \mu_\star^2t}   \left(\sum_{s=0}^{t(\delta)}\widetilde{K}_s \Vert x_s \Vert^2  +\sum_{s=t(\delta)}^{t-2}  4C_{K}^2\Vert x_s \Vert^2 + \sum_{s=0}^{t-2} \Vert \zeta_s \Vert^2 \right)  +1.
\end{align*}
Now always under $\mathcal{E}_{2,t} \cap E_{1,\delta} \cap \mathcal{C}_{1,\delta}$, provided $t \ge t(\delta)^{1+ 3\gamma/2}$ we have
\begin{align*}
\sum_{s=0}^{t(\delta)}\widetilde{K}_s \Vert x_s \Vert^2 & \le \Vert K_\circ \Vert^2 h(t(\delta)) \sum_{s=0}^{t(\delta)} \Vert x_s \Vert^2  \le \sigma^2 d_x \Vert K_\circ \Vert t(\delta)^{1+3\gamma/2} \le  \sigma^2 d_x \Vert K_\circ \Vert t
\end{align*}
where we used the fact $\Vert \widetilde{K}_s \Vert^2 \le \Vert K_\circ \Vert^2 h(t)$, and the fact at time $t=t(\delta)$ $\widetilde{K}_t = K_t$, so that $\sum_{s=0}^{t(\delta)} \Vert x_s \Vert^2 \le \sigma^2 d_x f(t(\delta))$. Furthermore, we have
\begin{align*}
  \sum_{s=t(\delta)}^{t-2}  4C_{K}^2\Vert x_s \Vert^2 & \le 32C_{K}^2 \Vert P_\star \Vert^{3/2}    (\Vert x_{t(\delta)} \Vert + 3 \sigma^2 (d_x t + \log(e/\delta) )) \\
  & \le 96C_{K}^2 \Vert P_\star \Vert^{3/2} \sigma^2  d_x   (2t(\delta)^{1+\gamma/2}   + t ) \\
  & \le 200 C_K^2 \Vert P_\star \Vert^{3/2} \sigma^2  d_x   t.
\end{align*}
Thus, provided that $t \ge t(\delta)^{1 + 3/2\gamma}$, we obtain that under the event $\mathcal{E}_{2,t} \cap E_{1,\delta} \cap \mathcal{C}_{1,\delta}$ we have
\begin{align*}
  \left( \det\left( \frac{1}{C \mu_\star^2t} \sum_{s=0}^{t-2} u_s u_s^\top  + I_{d_u}  \right)\right)^{\frac{1}{d}} & \le  \frac{c_1  \sigma^2 C_K^2 \Vert P_\star\Vert^{3/2}  d_x d_u}{\mu_\star^2}
\end{align*}
for some universal positive constant $c_1 > 0$. Denote
$$
(\star\star)\quad  t \ge  t(\delta)^{1 + 3/2\gamma}
$$
Therefore, assuming that $(\star\star)$ holds, under the event $\mathcal{E}_{1,\delta,t} \cap \mathcal{E}_{2,t} \cap E_{1,\delta} \cap \mathcal{C}_{1,\delta}$, we have
\begin{align*}
  \log\left(\frac{e^{d_x} \det\left(\frac{1}{C_1 \mu_\star^2 t} \sum_{s=0}^{t-2} u_s u_s^\top + I_{d_u} \right)}{ \delta}\right) \le    c_4 \left(d \log\left(\frac{e \sigma C_K \Vert P_\star \Vert d_x d_u}{\mu_\star^2} \right)  + \log(e/\delta) \right)
\end{align*}
for some universal constant $c_4 > 0$.
\medskip

\underline{\bf Step 3:} {(Putting everything together)} To conclude, under the event $\mathcal{B}_t \cap \mathcal{E}_{1,\delta,t} \cap \mathcal{E}_{2,t} \cap E_{1,\delta} \cap \mathcal{C}_{1,\delta}$, we have
\begin{align}\label{eq:IIb:LSE1}
  \Vert B_t - B  \Vert \le  \frac{C \sigma^2}{\mu_\star^2 t } \left(d\log\left(\frac{e \sigma C_K \Vert P_\star \Vert d_x d_u}{\mu_\star^2} \right)  + \log(e/\delta) \right)
\end{align}
for some universal constant $C > 0$. Therefore, the upper bound \eqref{eq:IIb:LSE1} holds with probability at least $1- Ct(\delta)\delta$, provided that $(\star)$ and $(\star\star)$ hold. These two conditions hold whenever
\begin{align*}
  t \ge c  \max( t(\delta)^{3\gamma_\star}, \sigma^4 (d_u \gamma_\star \log(e\sigma C_\circ \mathcal{G}_\circ d_x \gamma_\star)   + \log(e/\delta))^2 )
\end{align*}
for some universal constants $C,c> 0$. This concludes the proof.
\end{proof}

\medskip
\medskip

\subsection{Error rate in Scenario II - $B$ known}

\begin{proposition}\label{prop:IIa:LSE:1}
  Under $\Alg$, for all $\delta \in (0,1)$,
  $$
  \Vert A_t - A \Vert^2 \le  \frac{C \sigma^2 \left(d_x\gamma \log(e \sigma C_\circ \mathcal{G}_\circ d_x t)  + \log(e/\delta)\right)}{t}
  $$
  with probability at least $1-\delta$, when
  $
   t \ge c  (d_x \gamma_\star \log(e\sigma C_\circ \mathcal{G}_\circ d_x \gamma_\star)   + \log(e/\delta)),
  $
  for some universal positive constants $C, c > 0$.
\end{proposition}

\begin{proof}[Proof of Proposition \ref{prop:IIa:LSE:1}]
Fix $t\ge 2$. We start by writing the estimation error as
\begin{align*}
  A_t - A = \left( \sum_{s=0}^{t-2}   \eta_s x_s^\top \right) \left( \sum_{s=0}^{t-2} x_s x_s^\top \right)^{\dagger}.
\end{align*}
When $\sum_{s=0}^{t-2} x_s x_s^\top$ is invertible, we can decompose the estimation error as
\begin{align*}
  \Vert A_t - A\Vert^2  \le \left\Vert \left(\sum_{s=0}^{t-2} x_s x_s^\top\right)^{-1/2} \left( \sum_{s=0}^{t-2} x_s \eta_s^\top \right)  \right\Vert^2 \frac{1}{\lambda_{\min}\left( \sum_{s=0}^{t-2} x_s x_s^\top \right)}.
\end{align*}

\underline{\bf Step 1:} {(Bounding the smallest eigenvalue)} Define the event
\begin{align*}
  \mathcal{B}_t = \left\lbrace \lambda_{\min}\left( \sum_{s=0}^t x_s x_s^\top\right) \ge C_1 t  \right\rbrace.
\end{align*}
By Proposition \ref{prop:IIa:se}, the event $\mathcal{B}_t$ holds with probability at least $1-\delta$, provided that the condition $(\star)$ holds:
$$
(\star) \ \ \  t\ge c_1 \sigma^2 ( d_x \gamma_\star \log(e \sigma C_\circ\mathcal{G}_\circ d_x \gamma_\star)  + \log(e/\delta)),
$$
for some universal constants $C_1, c_1> 0$. Under the event $\mathcal{B}_t$, we have
\begin{align*}
  \Vert A_t - A \Vert^2 & \le \frac{1}{C_1 t} \left\Vert  \left(\sum_{s=0}^{t-2} x_s x_s^\top \right)^{-1/2} \left( \sum_{s=0}^{t-2}  x_s \eta_s^\top\right)   \right\Vert^2 \\
  & \le \frac{2}{C_1t} \left\Vert  \left(\sum_{s=0}^{t-2} x_s x_s^\top + C_1 t  I_{d_x} \right)^{-1/2} \left( \sum_{s=0}^{t-2}  x_s \eta_s^\top\right)   \right\Vert^2,
\end{align*}
where we used the fact that $2 \sum_{s=0}^{t-2} x_s x_s^\top \succeq \sum_{s=0}^{t-2} x_s x_s^\top + C_1 t $.

\underline{\bf Step 2:} {(Bounding the self-normalized term)}
Consider the events,
\begin{align*}
\mathcal{E}_{ t,\delta} & = \Bigg\lbrace  \left\Vert \left(\sum_{s=0}^{t-2}  x_s x_s^\top + C_1 t  I_{d_u} \right)^{\!\!-1/2} \!\!  \left(\sum_{s=0}^{t-2}  x_s \eta_s^{\top} \right) \right \Vert^2  \\
&\qquad\qquad\qquad \le 7 \sigma^2\log\left(\frac{e^{d_x} \det\left(\frac{1}{C_1t} \sum_{s=0}^{t-2} x_s x_s^\top + I_{d_x} \right)}{ \delta}\right) \Bigg\rbrace, \\
\mathcal{A}_{t,\delta} & = \left\lbrace   \sum_{s=0}^{t} \Vert x_s \Vert^2 \le c_2 \sigma^2 C_\circ^2 \mathcal{G}_\circ^2 (d_x t^{1+2\gamma}  + \log(1/\delta)) \right\rbrace.
\end{align*}
By Proposition \ref{prop:caution}, the event $\mathcal{A}_{t,\delta}$ holds with probability at least $1-\delta$ for some universal positive constant $c_2 > 0$ (involved in the definition of $\mathcal{A}_{t,\delta}$). By Proposition \ref{prop:SNP++}, the event $\mathcal{E}_{t,\delta}$ holds with probability at least $1-\delta$. Under the event $\mathcal{A}_{t,\delta} \cap \mathcal{E}_{t,\delta}$, we have
\begin{align*}
 \left(\det\left( \frac{1}{C_1t} \sum_{s=0}^{t-2} x_s x_s^\top  + I_{d_u}  \right)\right)^{\frac{1}{d_x}} & \le  \frac{1}{C_1 t} \sum_{s=0}^{t-2} \Vert x_s \Vert^2 + 1 \\
 & \le  \frac{2c_1}{C_1} \sigma^2 C_\circ^2 \mathcal{G}_\circ^2 d_x t^{2\gamma},
\end{align*}
where we assumed that $(\star\star)$ : $t \ge \log(e/\delta)$. Therefore, assuming that $(\star\star)$ holds, we have, under the event $\mathcal{E}_{t,\delta} \cap \mathcal{E}_{2,t,\delta} \cap \mathcal{A}_{t,\delta}$,
\begin{align*}
 \log\left( \frac{e^d_x \det\left( \frac{1}{C_1t} \sum_{s=0}^{t-2} x_s x_s^\top  + I_{d_x}\right)}{\delta} \right) \le c_3  \sigma^2 \left(d_x\gamma \log(e \sigma C_\circ \mathcal{G}_\circ d_x t)  + \log(e/\delta)\right),
\end{align*}
for some universal constant $c_3 > 0$.
\medskip

\underline{\bf Step 3:} {(Putting everything together)} To conclude, under the event $\mathcal{A}_{t,\delta} \cap \mathcal{B}_t \cap \mathcal{E}_{t,\delta}$, we have
\begin{align}\label{eq:IIa:LSE1}
 \Vert A_t - A  \Vert \le \frac{C \sigma^2  \left(d_x\gamma \log(e \sigma C_\circ \mathcal{G}_\circ d_x t)  + \log(e/\delta)\right)}{t}
\end{align}
for some universal constant $C > 0$. Therefore the upper bound \eqref{eq:IIb:LSE1} holds with probability at least $1-2\delta$, provided that $(\star)$ and $(\star\star)$ hold. These two conditions hold whenever
\begin{align*}
 t \ge c  \sigma^2(  d_x \gamma_\star \log(e\sigma C_\circ \mathcal{G}_\circ d_x \gamma_\star)   + \log(e/\delta)),
\end{align*}
for some universal constant $c> 0$. This concludes the proof.

\end{proof}

\medskip

\begin{proposition}\label{prop:IIa:LSE3}
  Assume that $\mathcal{T}$ satisfies \eqref{eq:lazy}, and that under $\Alg$, for all $\delta \in (0,1)$, we have
  \begin{align*}
    \Pr\left( \forall t \ge t(\delta), \begin{array}{rl}
    (i) & \widetilde{K}_t = K_t \\
    (ii) & \max(\Vert B (K_t - K_\star), \Vert K_t - K_\star \Vert) \Vert \le (4 \Vert P_\star \Vert^{3/2} )^{-1}
    \end{array}  \right) \ge 1 -t(\delta)\delta
  \end{align*}
  for some $t(\delta) \ge \log(e/\delta)$. Then for all $\delta \in (0,1)$, the following
  $$
    \Vert A_t - A  \Vert \le  \frac{C \sigma^2}{ t } \left( d_x  \log\left( e \Vert P_\star \Vert d_x  \right)  + \log(e/\delta) \right)
  $$
  holds with probability at least $1-c_1t(\delta)\delta$
  \begin{align*}
    t \ge c_2  \max( t(\delta)^{2\gamma_\star},\sigma^2 (d_x \gamma_\star \log(e\sigma C_\circ \mathcal{G}_\circ d_x \gamma_\star)   + \log(e/\delta)) )
  \end{align*}
  for some universal constants $C,c_1,c_2> 0$.
\end{proposition}

\begin{proof}[Proof of Proposition \ref{prop:IIa:LSE3}]
  Fix $t\ge 2$. We start by writing the estimation error as
  \begin{align*}
    A_t - A = \left( \sum_{s=0}^{t-2}   \eta_s x_s^\top \right) \left( \sum_{s=0}^{t-2} x_s x_s^\top \right)^{\dagger}.
  \end{align*}
  When $\sum_{s=0}^{t-2} x_s x_s^\top$ is invertible, we can decompose the estimation error as
  \begin{align*}
    \Vert A_t - A\Vert^2  \le \left\Vert \left(\sum_{s=0}^{t-2} x_s x_s^\top\right)^{-1/2} \left( \sum_{s=0}^{t-2} x_s \eta_s^\top \right)  \right\Vert^2 \frac{1}{\lambda_{\min}\left( \sum_{s=0}^{t-2} x_s x_s^\top \right)}.
  \end{align*}

  \underline{\bf Step 1:} {(Bounding the smallest eigenvalue)} Define the event
  \begin{align*}
    \mathcal{B}_t = \left\lbrace \lambda_{\min}\left( \sum_{s=0}^t x_s x_s^\top\right) \ge C_1 t  \right\rbrace.
  \end{align*}
  By Proposition \ref{prop:IIa:se}, the event $\mathcal{B}_t$ holds with probability at least $1-\delta$, provided that the condition $(\star)$ holds:
  $$
  (\star) \ \ \  t\ge c_1 \sigma^2 ( d_x \gamma_\star \log(e \sigma C_\circ\mathcal{G}_\circ d_x \gamma_\star)  + \log(e/\delta)),
  $$
  for some universal constants $C_1, c_1> 0$. Under the event $\mathcal{B}_t$, we have
  \begin{align*}
    \Vert A_t - A \Vert^2 & \le \frac{1}{C_1 t} \left\Vert  \left(\sum_{s=0}^{t-2} x_s x_s^\top \right)^{-1/2} \left( \sum_{s=0}^{t-2}  x_s \eta_s^\top\right)   \right\Vert^2 \\
    & \le \frac{2}{C_1t} \left\Vert  \left(\sum_{s=0}^{t-2} x_s x_s^\top + C_1 t  I_{d_x} \right)^{-1/2} \left( \sum_{s=0}^{t-2}  x_s \eta_s^\top\right)   \right\Vert^2,
  \end{align*}
  where we used the fact that $2 \sum_{s=0}^{t-2} x_s x_s^\top \succeq \sum_{s=0}^{t-2} x_s x_s^\top + C_1 t $.

\underline{\bf Step 2:} {(Bounding the self-normalized term)} Consider the events,
\begin{align*}
\mathcal{E}_{1, t,\delta} & = \Bigg\lbrace  \left\Vert \left(\sum_{s=0}^{t-1}  x_s x_s^\top + C_1 t  I_{d_u} \right)^{\!\!-1/2} \!\!  \left(\sum_{s=0}^{t-1}  x_s \eta_s^{\top} \right) \right \Vert^2,\\
&\quad\qquad \qquad \le 7 \sigma^2\log\left(\frac{e^{d_x} \det\left(\frac{1}{C_1 t} \sum_{s=0}^{t-1} x_s x_s^\top + I_{d_u} \right)}{ \delta}\right) \Bigg\rbrace, \\
E_{1,\delta,t} & = \left\lbrace \sum_{t(\delta)}^t \Vert x_t \Vert^2  \le 8\Vert P_\star \Vert^{3/2}    (\Vert x_{t(\delta)} \Vert + 3 \sigma^2 (d_x t + \log(e/\delta) ))  \right\rbrace, \\
\mathcal{C}_{1,\delta} & = \left\lbrace  \forall t \ge t(\delta), \begin{array}{rl}
(i) & \widetilde{K}_t = K_t \\
(ii) & \max(\Vert B (K_t - K_\star), \Vert K_t - K_\star \Vert) \Vert \le (4 \Vert P_\star \Vert^{3/2} )^{-1}
\end{array}   \right\rbrace.
\end{align*}
 By Proposition \ref{prop:SNP++}, the event $\mathcal{E}_{1, t,\delta}$ holds with probability at least $1-\delta$. By Proposition \ref{prop:stable:K_t} we have $\mathbb{P}( E_{1,\delta,t} \cup \mathcal{C}_{1,\delta}^c) \ge 1-\delta$, and by assumption the event $\mathcal{C}_{1,\delta}^c$ holds with probability at least $1-t(\delta)\delta$.
Under the event $E_{1,\delta} \cap \mathcal{C}_{1,\delta}$ we have
\begin{align*}
  \left( \det\left( \frac{1}{C_1 \mu_\star^2t} \sum_{s=0}^{t-2} x_s x_s^\top  + I_{d_u}  \right)\right)^{\frac{1}{d_x}}\!\!\! & \le  \frac{1}{C_1 t}  \sum_{s=0}^{t-1} \Vert x_s \Vert^2 + 1 \\
  & \le  \frac{2}{C_1 t}   \left(\sum_{s=0}^{t(\delta)} \Vert x_s \Vert^2  +\sum_{s=t(\delta)}^{t-2} \Vert x_s \Vert^2 \right)  +1.
\end{align*}
Now always under $\mathcal{E}_{2,t} \cap E_{1,\delta} \cap \mathcal{C}_{1,\delta}$, provided $t \ge t(\delta)^{1+ \gamma/2}$ we have
\begin{align*}
\sum_{s=0}^{t(\delta)}\widetilde{K}_s \Vert x_s \Vert^2 & \le  \sigma^2 d_x t(\delta)^{1+ \gamma/2} \le  \sigma^2 d_x  t
\end{align*}
where we used the fact at time $t=t(\delta)$, $\widetilde{K}_t = K_t$, so that $\sum_{s=0}^{t(\delta)} \Vert x_s \Vert^2 \le \sigma^2 d_x f(t(\delta))$. Furthermore, we have
\begin{align*}
  \sum_{s=t(\delta)}^{t-1}  \Vert x_s \Vert^2 & \le 8 \Vert P_\star \Vert^{3/2}    (\Vert x_{t(\delta)} \Vert + 3 \sigma^2 (d_x t + \log(e/\delta) )) \\
  & \le 21  \Vert P_\star \Vert^{3/2} \sigma^2  d_x   (t(\delta)^{1+\gamma/2}   + t ) \\
  & \le 42  \Vert P_\star \Vert^{3/2} \sigma^2  d_x   t.
\end{align*}
Thus, provided that $t \ge t(\delta)^{1 + 1/2\gamma}$, we obtain that under the event $E_{1,\delta} \cap \mathcal{C}_{1,\delta}$ we have
\begin{align*}
  \left( \det\left( \frac{1}{C_1t} \sum_{s=0}^{t-1} x_s x_s^\top  + I_{d_u}  \right)\right)^{\frac{1}{d_x}} & \le  c_1  \sigma^2 \Vert P_\star\Vert^{3/2}  d_x
\end{align*}
for some universal positive constant $c_1 > 0$. Denote
$$
(\star\star)\quad  t \ge  t(\delta)^{1 + 1/2\gamma}
$$
Therefore, assuming that $(\star\star)$ holds, under the event $\mathcal{E}_{1,\delta,t} \cap E_{1,\delta} \cap \mathcal{C}_{1,\delta}$, we have
\begin{align*}
  \log\left(\frac{e^{d_x} \det\left(\frac{1}{C_1 \mu_\star^2 t} \sum_{s=0}^{t-2} u_s u_s^\top + I_{d_u} \right)}{ \delta}\right) \le    c_4 \left(d_x \log\left( e\sigma \Vert P_\star \Vert d_x \right)  + \log(e/\delta) \right)
\end{align*}
for some universal constant $c_4 > 0$.
\medskip

\underline{\bf Step 3:} {(Putting everything together)} To conclude, under the event $\mathcal{B}_t \cap \mathcal{E}_{1,\delta,t} \cap \mathcal{E}_{2,t} \cap E_{1,\delta} \cap \mathcal{C}_{1,\delta}$, we have
\begin{align}\label{eq:IIa:LSE1}
  \Vert A_t - A  \Vert \le  \frac{C \sigma^2}{ t } \left( d_x  \log\left( e \Vert P_\star \Vert d_x  \right)  + \log(e/\delta) \right)
\end{align}
for some universal constant $C > 0$. Therefore, the upper bound \eqref{eq:IIa:LSE1} holds with probability at least $1- Ct(\delta)\delta$, provided that $(\star)$ and $(\star\star)$ hold. These two conditions hold whenever
\begin{align*}
  t \ge c  \max( t(\delta)^{2\gamma_\star},\sigma^2 (d_x \gamma_\star \log(e\sigma C_\circ \mathcal{G}_\circ d_x \gamma_\star)   + \log(e/\delta)) )
\end{align*}
for some universal constants $C,c> 0$. This concludes the proof.
\end{proof}

\newpage
\section{Smallest Eigenvalue of the Cumulative Covariates Matrix} \label{app:se}

This appendix is devoted to the analysis of the smallest eigenvalue of the cumulative covariates matrix. This eigenvalue should exhibit an appropriate scaling so that the LSE performs well. We first provide a generic recipe for the analysis of this eigenvalue, and then apply it to the three scenarios. For Scenario I, the results are stated in Proposition \ref{prop:I:we} and \ref{prop:I:se:enough}. Our analysis of Scenario II -- $A$ known is summarized in Propositions \ref{prop:IIb:we} and \ref{prop:IIb:se}. Finally, for Scenario II -- $B$ known, we establish Proposition \ref{prop:IIa:se}.

%

\subsection{A generic recipe}
In the three scenarios, we will have to obtain high probability bounds on the smallest eigenvalue of a matrix of the form $\sum_{s=1}^t y_s y_s^\top$ where $y_s = z_s + M_s \xi_s$ where $\xi_s$ is a random variable independent of $z_{1}, \dots, z_s$ and $M_{1}, \dots, M_s$ for all $s\ge 1$. The need for such guarantee stems mainly from the analysis of the least squares estimator. Because of this  structure, common to the three settings, our proofs for the different scenarios will be similar in spirit up to some technical details that are mainly related to the nature of the sequence of matrices $(M_s)_{s\ge 1}$. We shall now sketch a generic recipe for our proofs.

\paragraph{Sketch of the recipe.} The first step is to use Lemma \ref{lem:selb}, which will allow us to lower bound\footnote{Here we mean lower bound in the Löwner partial order over symmetric matrices.}  $\sum_{s=1}^t y_s y_s^\top$. We obtain, for all $\lambda > 0$,
$$
\sum_{s=1}^t y_s y_s^\top \! \succeq \! \underbrace{\sum_{s=1}^t (M_s\xi_s) (M_s\xi_s)^\top}_{\substack{\textit{Random Matrix} \\ \textit{}}} \!\! - \underbrace{\left(\sum_{s=1}^tz_s (M_s\xi_s)^{\top} \! \right)^{\!\!\!\top} \!\!\! \left(\sum_{s=1}^t z_s z_s^\top + \lambda I_d \right)^{\!\!\!-1} \!\!\!\left(\sum_{s=1}^t z_s (M_s\xi_s)^\top \! \right)}_{\textit{Self-Normalized Matrix Valued  Process}} \! - \lambda I_d.
$$
Then, we bound the random matrix (first term) using conditional independence via Proposition \ref{prop:RM++}. Finally, we also bound the Self-Normalized Matrix Process (the second term) using Proposition \ref{prop:SNP++}.

\paragraph{Ingredients of the recipe.} Let us now list the main lemmas and propositions used above. Their proofs are presented in \ref{proofs:tools}.

\medskip

\begin{lemma}\label{lem:selb} Let $(y_t)_{t\ge 1}$, $(z_t)_{t \ge 1}$, and $(\xi_t)_{t \ge 1}$ be three sequences of vectors in $\mathbb{R}^d$ satisfying, for all $s \ge 0$, the linear relation $\  y_{s} = z_s + \xi_s$. Then, for all $\lambda > 0$, all $t\ge 1$ and all $\varepsilon \in (0,1]$, we have
  $$
  \sum_{s=1}^t y_s y_s^\top \succeq \sum_{s=1}^t \xi_s \xi_s^\top + (1-\varepsilon) \sum_{s=1}^t z_s z_s^\top  - \frac{1}{\varepsilon}\left(\sum_{s=1}^tz_s \xi_s \right)^{\!\!\top}\!\! \left(\sum_{s=1}^t z_s z_s^\top + \lambda I_d \right)^{\!\!-1}\!\! \left(\sum_{s=1}^t z_s \xi_s^\top \right) - \varepsilon \lambda I_d.
  $$
\end{lemma}

\medskip

\begin{proposition}\label{prop:RM++}
  Let $(\mathcal{F}_t)_{t\ge 0}$ be a filtration over the underlying probability space. Let $(\xi_t)_{t \ge 1}$ be a sequence of independent, zero-mean, $\sigma^2$-sub-gaussian, isotropic random vectors taking values in $\mathbb{R}^p$ and such that $\xi_t$ is $\mathcal{F}_t$-measurable for all $t\ge1$. Let $(M_t)_{t\ge1}$ be a sequence of random matrices taking values in $\mathbb{R}^{d \times p}$, such that $M_t$ is $\mathcal{F}_{t-1}$-measurable and its norm $\Vert M_s\Vert$ is bounded a.s.. Let $m = (m_t)_{t\ge1}$ refer to the sequence of such bounds (e.g. $\Vert M_s\Vert\le m_s$). Then
  $$
  \mathbb{P}\left( \left\Vert  \sum_{s=1}^t (M_s \xi_s) (M_s \xi_s)^\top - \sum_{s=1}^t M_s M_s^\top \right\Vert > 8\sigma^2 \Vert m_{1:t}\Vert^2_2 \max\left(\sqrt{\frac{2\rho + 5 d}{r_t^2} }, \frac{2\rho + 5 d}{r_t^2} \right)\right) \le 2 e^{-\rho}
  $$
  where $\Vert m_{1:t}\Vert_\infty = \max_{1 \le s \le t}\vert m_s\vert$,  $\Vert m_{1:t}\Vert_2 = \sqrt{\sum_{s=1}^t \vert m_s\vert^2}$, and $r_t = \Vert m_{t:t} \Vert_2 / \Vert m_{1:t}\Vert_\infty$.
\end{proposition}

\begin{remark} For our purposes, $\Vert m_{1:t} \Vert_2^2$ and $r_t^2$ will be either of order $\mathcal{O} (1)$ and $\mathcal{O}(t)$ respectively, or  of order $\mathcal{O}(\log(t))$ and $\mathcal{O}(\sqrt{t})$ respectively. These scalings will depend on the scenario considered.
\end{remark}

An immediate consequence of Proposition \ref{prop:RM++} is:

\begin{corollary}\label{cor:RM++}
Under the same assumptions on $(\xi_t)_{t\ge1}$ and $(M_s)_{t\ge1}$ as in Proposition \ref{prop:RM++}, if we further assume that $\sup_{s\ge1}\vert m_s \vert \le m$, then we have for all $\rho >0$, $\varepsilon \in (0, 1)$, and for all $t\ge \min\left( \frac{8^2 (\sigma m)^4}{\varepsilon^2}, \frac{8(\sigma m)^2}{\varepsilon} \right)(5d + 2\rho)$,
$$
  \mathbb{P}\left(   \sum_{s=1}^{t}  M_s M_s^\top - \varepsilon t I_{d} \preceq \sum_{s=1}^{t-1} (M_s\xi_s)(M_s\xi_s)^\top \preceq   \sum_{s=1}^{t}  M_s M_s^\top + \varepsilon t I_{d}   \right) \ge 1 -2e^{-\rho}.
$$
In particular if $M_s = I_d$, $(\xi_t)_{t\ge1}$ are now taking values in $\mathbb{R}^d$, we have for all $\rho >0$, $\varepsilon \in (0, 1)$, and for all $t\ge \min\left( \frac{8^2 \sigma ^4}{\varepsilon^2}, \frac{8\sigma^2}{\varepsilon} \right)(5d + 2\rho)$,
$$
  \mathbb{P}\left(   (1- \varepsilon) t I_{d} \preceq \sum_{s=1}^{t-1} \xi_s \xi_s^\top \preceq   (1+\varepsilon) t I_{d}   \right) \ge 1 -2e^{-\rho}.
$$
\end{corollary}

\begin{proposition}[Self-normalized matrix processes]\label{prop:SNP++} Let $(\mathcal{F}_t)_{t\ge 0}$ be a filtration over the underlying probability space. Let $(\xi_t)_{t \ge 1}$ be a sequence of independent, zero-mean, $\sigma^2$-sub-gaussian, isotropic random vectors taking values in $\mathbb{R}^p$ and such that $\xi_t$ is $\mathcal{F}_t$-measurable for all $t\ge1$.
Let $(M_t)_{t\ge1}$ be a sequence of random matrices taking values in $\mathbb{R}^{d \times p}$, such that $M_t$ is $\mathcal{F}_{t-1}$-measurable and its norm $\Vert M_s\Vert$ is bounded a.s.. Let $m = (m_t)_{t\ge1}$ refer to the sequence of such bounds. Let $(z_t)_{t \ge 1}$ be a sequence of random vectors taking values in $\mathbb{R}^{d}$, such that $z_t$ is $\mathcal{F}_{t-1}$-measurable for all $t\ge 1$. Then for all positive definite matrix $V\succ 0$, the following self-normalized matrix process defined by
$$
\forall t\ge 1,\ \   S_t(z,M\xi) \triangleq  \left(\sum_{s=1}^t z_s (M_s \xi_s)^\top\right)^\top \left(\sum_{s=1}^t z_s z_s^\top + V\right)^{-1} \left(\sum_{s=1}^t z_s (M_s \xi_s)^\top\right)
$$
satisfies, for all $\rho \ge 1$ and $t\ge 1$,
$$
\mathbb{P}\left[ \left\Vert S_t (z, M\xi) \right\Vert > \sigma^2 \Vert m_{1:t} \Vert^2_\infty \left( 2\log\det\left( V^{-1} \sum_{s=1}^t z_s z_s^\top + I_d  \right) +  7d +  4\rho \right) \right] \le e^{-\rho}.
$$
\end{proposition}

\subsection{Application to Scenario I}

We now apply the recipe described in the previous subsection to lower bound the smallest eigenvalue of the cumulative covariates matrix in Scenario I. We first prove the following result, that will then be refined.

\begin{proposition}[Sufficent exploration]\label{prop:I:we}
  Under Algorithm $\Alg$, for all $\delta \in (0,1)$,
  \begin{equation*}
    \lambda_{\min}\left(\sum_{s=0}^{t-1} \begin{bmatrix} x_s  \\ u_s \end{bmatrix} \begin{bmatrix} x_s  \\ u_s \end{bmatrix}^\top\right) \ge \left(\frac{t}{2}\right)^{1/4}
  \end{equation*}
  holds with probability at least $1-\delta$, provided that
  \begin{equation*}
      t \ge c \sigma^4 C_\circ^8  ((d_x + d_u) \gamma_\star \log(e\sigma C_\circ \mathcal{G}_\circ d_x d_u \gamma_\star )  + \log(e/\delta))
    \end{equation*}
    for some universal positive constant $c> 0$.
  %
\end{proposition}

\begin{proof}[Proof of Proposition \ref{prop:I:we}]
    Define for all $t \ge 1$, the event
    $$
    \mathcal{E}_{1,t} = \left\lbrace \exists i \in \lbrace t/2,\dots, t-1\rbrace: \ \  \lambda_{\min}\left(\sum_{s=0}^{i-1} \begin{bmatrix} x_s \\ u_s  \end{bmatrix}\begin{bmatrix} x_s \\ u_s  \end{bmatrix}^\top\right) \ge i^{1/4}  \right\rbrace.
    $$
    Let us recall that under $\Alg$, we have
    \begin{equation*}
      u_t \gets \begin{cases}
         K_t  x_t  + \nu_t & \text{if } \ell_t = 1 \text{ and } \Vert K_t \Vert^2 \le h(t), \text{ and } \lambda_{\min} \left(\sum_{s=0}^{t-1} \begin{bmatrix}x_s \\ u_s \end{bmatrix} \begin{bmatrix}x_s \\ u_s \end{bmatrix}^{\top}  \right) \ge t^{1/4} \\
         K_\circ   x_t + \nu_{t}  & \text{otherwise.}
      \end{cases}
    \end{equation*}
Define for all $s\ge 1$,
    $$
    y_s = \begin{bmatrix}
      x_s \\
      u_s
  \end{bmatrix}, \quad z_{s}=\begin{bmatrix}
      A x_{s-1} + B u_{s-1} \\
      K_\circ (A x_{s-1} +B u_{s-1})
    \end{bmatrix},
    \quad M_\circ = \begin{bmatrix}
      I_{d_x} & O \\
      K_\circ & I_{d_u}
    \end{bmatrix},
    \quad \text{and} \quad \xi_{s} = \begin{bmatrix}
    \eta_{s-1} \\
    \nu_s
    \end{bmatrix}.
    $$
    Note that under the event $\mathcal{E}_{1,t}^c$ we have $y_s = z_s + M_\circ \xi_s$ for all $s \in \lbrace t/2, \dots, t-1\rbrace$. Applying  Lemma \ref{lem:selb}, we obtain
    \begin{align*}
      \sum_{s=t/2}^t y_s y_s^\top   \succeq \! \! \sum_{s=t/2}^t (M_\circ \xi_{s}) (M_\circ \xi_{s})^\top -  I_{d_u} -  \left\Vert \left(\sum_{s=t/2}^{t} z_s z_s^\top + I_{d_us} \right)^{\!\!\!-1/2}\!\!\! \left(\sum_{s=t/2}^{t} z_s (M_\circ \xi_s )^\top \right)\right\Vert^2 I_{d_u}.
    \end{align*}
    For all $\delta \in (0,1)$ and $t\ge 1$, we define the following events
    \begin{align*}
      \mathcal{A}_{\delta, t}  & = \left\lbrace \sum_{s=0}^t \Vert x_s  \Vert \le C_1 \sigma^2 C_\circ^2\mathcal{G}_\circ^2 (d_x t^{1+2\gamma} + \log(e/\delta)) \right\rbrace, \\
      \mathcal{E}_{2, \delta, t} & = \Bigg\lbrace \left\Vert \left(\sum_{s=t/2}^{t-1} z_s z_s^\top + I_d \right)^{-1/2} \left(\sum_{s=t/2}^{t-1} z_s (M_\circ\xi_s)^\top \right)\right\Vert^2\\
    &\qquad \qquad \qquad \le 7 \sqrt{d_x} \sigma^2  \log\left( \frac{e^{d} \det\left(\sum_{s=t/2}^{t-1} z_s z_s^\top  + I_{d_x}  \right)   }{\delta} \right) \Bigg\rbrace, \\
      \mathcal{E}_{3, t} & = \left \lbrace  \sum_{s=t/2}^{t-1} \xi_s \xi_s^\top  \succeq
      \begin{bmatrix}
        (t/2)I_{d_x} & O \\
        O & \sigma^2\sqrt{d_x t}/2I_{d_u}
      \end{bmatrix}  \right\rbrace, \\
      \mathcal{E}_{4, t} & = \left\lbrace  \lambda_{\max}\left(\sum_{s=0}^{t-1} \eta_s \eta_s^\top  \right) \le \frac{3t}{2} \right\rbrace.
    \end{align*}
 In view of Proposition \ref{prop:caution}, $\mathbb{P}\left( \mathcal{A}_{t,\delta}\right) \ge 1- \delta$. From Proposition \ref{prop:SNP++}, we have $\mathbb{P}(\mathcal{E}_{2,\delta,t}) \ge 1- \delta$. From Proposition \ref{prop:RM++}, $\mathbb{P}(\mathcal{E}_{3,t}) \ge 1- \delta$ provided that $t \ge c_1 \sigma^2 (d + \log(e/\delta))$, where we first normalize to obtain  $\sum_{s=t/2}^t (\E[\xi_t \xi_t^\top]^{-1/2} \xi_s) (\E[\xi_t \xi_t^\top]^{-1/2}\xi_s)^\top$ then apply the proposition to get the high probability bound. We have by Proposition \ref{prop:RM++}, that $\mathbb{P}(\mathcal{E}_{4,t}) \ge 1 - \delta$ provided that $t\ge \sigma^2 (d_x + \log(e/\delta))$.

Provided that $t\ge \log(e/\delta)$, we have under the event $ \mathcal{A}_{\delta,t} \cap \mathcal{E}_{1,t}^c\cap \mathcal{E}_{2,\delta,t}\cap \mathcal{E}_{4,t} $ that
    \begin{align*}
      \det\left(\sum_{s=t/2}^{t-1} z_s z_s^\top  + I_{d_x}  \right)^{1/d} & \le \sum_{s=t/2}^t \Vert z_s \Vert^2 + 1 \\
      & \le \sum_{s=t/2}^t 2 C_\circ^2 \Vert x_{s} - \eta_{s-1} \Vert^2  +1 \\
      & \le 4 C_\circ^2 \sum_{s=0}^t \Vert x_s \Vert^2 + \Vert \eta_{s-1}\Vert^2 + 1 \\
      & \le 21 C_1 \sigma^2 C_\circ^4 \mathcal{G}_\circ^2 d_x t ^{3\gamma^\star}.
    \end{align*}
    Therefore, provided that $t\ge \log(e/\delta)$, we have under the event  $ \mathcal{A}_{\delta,t} \cap \mathcal{E}_{1,t}^c\cap \mathcal{E}_{2,\delta,t}\cap \mathcal{E}_{4,t} $ that
    \begin{align*}
       \left\Vert \left(\sum_{s=t/2}^{t-1} z_s z_s^\top + I_d \right)^{-1/2} \left(\sum_{s=t/2}^{t-1} z_s (M_\circ\xi_s)^\top \right)\right\Vert^2 \le C_2\sqrt{d_x}\sigma^2 ( d\gamma_\star \log(e \sigma C_\circ \mathcal{G}_\circ d_x t) +\log(e/\delta))
    \end{align*}
    for some universal positive constant $C_2 > 0$. Furthermore, if $t\ge \sigma^4 d_x$, under the event $\mathcal{E}_{3,t}$, we have
    \begin{align*}
      \sum_{t/2}^t (M_\circ \xi_s) (M_\circ \xi_s)^{\top} & \succeq  M_\circ
      \begin{bmatrix}
        (t/2)I_{d_x} & O \\
        O & \sigma^2\sqrt{d_x t}/2I_{d_u}
      \end{bmatrix} M_\circ^\top  \\
      &  \succeq  \frac{1}{2}\begin{bmatrix}
      tI_{d_x} & tK_\circ^\top \\
      tK_\circ  & tK_\circ^\top K_\circ + \sigma^2 \sqrt{d_x t}I_{d_u}
    \end{bmatrix} \\
    & \succeq \frac{t}{2} \min\left(\frac{\sigma^2\sqrt{d_x}}{2 \Vert K_\circ \Vert^2 \sqrt{t} + \sigma^2\sqrt{d_x}} , \frac{\sigma^2\sqrt{d_x}}{2\sqrt{t}} \right) I_{d} \\
    & \succeq \frac{\sigma^2 \sqrt{d_x}}{2} \min\left(\frac{t}{2 \Vert K_\circ \Vert^2 \sqrt{t} + \sigma^2\sqrt{d_x}} , \frac{\sqrt{t}}{2} \right) I_{d} \\
    & \succeq \frac{\sigma^2\sqrt{d_xt}}{6C_\circ^2}  I_d,
    \end{align*}
    where we used  Lemma \ref{lem:psd:ineq1} (with $\alpha = 1/2$, and $\beta=1$).

Therefore, provided that $t\ge \sigma^2 d_x$ and $t\ge \log(e/\delta)$, we have under the event $ \mathcal{A}_{\delta,t} \cap \mathcal{E}_{1,t}^c\cap \mathcal{E}_{2,\delta,t}\cap\mathcal{E}_{3,t} \cap\mathcal{E}_{4,t}$ that
    \begin{align*}
      \lambda_{\min}\left(\sum_{s=0}^{t-1} u_s u_s^\top\right) \ge \frac{\sigma^2\sqrt{d_xt}}{6C_\circ^2} - 1 - C_2\sqrt{d_x}\sigma^2 ( d\gamma_\star \log(e \sigma C_\circ \mathcal{G}_\circ d_x t) +\log(e/\delta)).
    \end{align*}

    Using Lemma \ref{lem:technical}, there exists $c_3> 0$ such that if
    \begin{align}\label{eqI:we:sc}
      t \ge c_3 \sigma^4 C_\circ^8  (d \gamma_\star \log(e\sigma C_\circ \mathcal{G}_\circ d_x d_u \gamma_\star )  + \log(e/\delta))
    \end{align}
    then
    \begin{itemize}
      \itemsep0em
      \item $\frac{\sigma^2\sqrt{d_xt}}{6C_\circ^2} - 1 - C_2\sqrt{d_x}\sigma^2 ( d\gamma_\star \log(e \sigma C_\circ \mathcal{G}_\circ d_x t) +\log(e/\delta)) \ge \frac{\sigma^2\sqrt{d_xt}}{10C_\circ^2} > t^{1/4}$,
      \item $t\ge \log(e/\delta)$,
      \item $t\ge \sigma^4 d_x$.
    \end{itemize}
    Therefore, if condition \eqref{eqI:we:sc} holds, we have under $\mathcal{A}_{\delta,t} \cap \mathcal{E}_{1,t}^c\cap \mathcal{E}_{2,\delta,t}\cap\mathcal{E}_{3,t} \cap\mathcal{E}_{4,t}$
    $$
    \lambda_{\min}\left(\sum_{s=0}^{t-1} u_s u_s^\top \right) > t^{1/4}.
    $$
 But this cannot hold under the event $\mathcal{E}_{1,t}^c$, therefore it must be that $  \mathcal{A}_{\delta,t} \cap \mathcal{E}_{2,\delta,t} \cap \mathcal{E}_{3,t} \cap\mathcal{E}_{4,t} \subseteq  \mathcal{E}_{1,t}$ which in turns impies that
    $$
    \mathbb{P}\left( \mathcal{E}_{1,t} \right) \ge 1 - \mathbb{P}( \mathcal{A}_{\delta,t}^c \cup \mathcal{E}_{2,\delta,t}^c\cup\mathcal{E}_{3,t}^c \cup\mathcal{E}_{4,t}^c ) \ge 1- 4\delta.
    $$
    reparametrizing by $\delta' = 4\delta$ gives the desired result with modified universal postive constants.
\end{proof}


\newpage
\begin{proposition}[Sufficient exploration with refined rates]\label{prop:I:se:enough} Under $\Alg$, assume that for all $\delta \in (0,1)$, we have
  $$
  \mathbb{P}\left(  \forall t \ge t(\delta),\ \ \Vert K_t - K_\star \Vert \le 1  \right) \ge 1-\delta
  $$
  for some $t(\delta) \ge 1$. Then for all $\delta \in(0,1)$,
  $$
  \lambda_{\min}\left( \sum_{s=0}^{t-1} \begin{bmatrix}
    x_s \\
    u_s
  \end{bmatrix} \begin{bmatrix}
    x_s \\
    u_s
  \end{bmatrix}^\top \right) \ge \frac{C\sigma^2\sqrt{d_xt}}{C_K^2}
  $$
  with probability at least $1-\delta$, provided that
  \begin{equation}\label{eq:I:se:sc2}
    t^{1/2} \ge c \max(t(\delta)^{1/2}, C_K^4 (d\gamma_\star \log(e \sigma C_\circ \mathcal{G}_\circ \Vert P_\star \Vert d_x d_u \gamma_\star) + \log(e/\delta)) )
  \end{equation}
  for some universal positive constants $C,c> 0$.
\end{proposition}

\begin{proof}[Proof of Proposition \ref{prop:I:se:enough}]
Let us start by defining
  $$
  \mathcal{E}_{1,\delta} =   \left\lbrace \forall t \ge t(\delta), \  \Vert K_t - K_\star \Vert \le 1 \right\rbrace.
  $$
  Now, we recall that under $\Alg$, we have $u_t = (1-\alpha_t)(K_t  x_t + \nu_t) + \alpha_t (K_\circ  x_t +\zeta_t)$ for all $t \ge 1$, where we defined
 \begin{equation}
   \forall t \ge 1,\ \  \alpha_t =  \begin{cases}
     0 & \text{if } \ell_t = 1 \text{ and } \Vert K_t \Vert^2 \le h(t), \text{ and } \lambda_{\min} \left(\sum_{s=0}^{t-1} \begin{bmatrix}x_s \\ u_s \end{bmatrix} \begin{bmatrix} x_s \\ u_s \end{bmatrix}^{\top}  \right) \ge t^{1/4} \\
     1 & \text{otherwise.}
 \end{cases}
\end{equation}
Let $\widetilde{K}_s = (1-\alpha_s)K_s 1_{\lbrace \Vert K_s - K_\star\Vert\le 1\rbrace} + \alpha_s K_\circ$, and note that $\Vert \widetilde{K}_s \Vert \le 2 C_K$. Thus, under the event $\mathcal{E}_{1,\delta}$, we have
\begin{equation}
  \begin{bmatrix}
    x_s \\
    u_s
  \end{bmatrix} = \begin{bmatrix}
    A x_{s-1} +  B u_s \\
    \widetilde{K}_s A x_{s-1} +  \widetilde{K}_s B u_{s-1}
\end{bmatrix} + \begin{bmatrix}
  I_{d_x} & O \\
  \widetilde{K}_s & I_{d_u}
\end{bmatrix} \begin{bmatrix}
  \eta_{s-1} \\
  \nu_s
\end{bmatrix}.
\end{equation}
Denote for all $s\ge 1$,
\begin{align*}
  y_s  = \begin{bmatrix}
    x_s \\
    u_s
\end{bmatrix}, \quad
  z_s = \begin{bmatrix} A x_{s-1} + B u_s \\
  \widetilde{K}_s A x_{s-1} + K_s B u_s
  \end{bmatrix}, & \quad
  M_s = \begin{bmatrix}
  I_{d_x} & O \\
  \widetilde{K}_s & I_{d_u}
  \end{bmatrix}, \quad
  \xi_{s}  = \begin{bmatrix}
    \eta_{s-1} \\
    \nu_s
\end{bmatrix}, \\
 \xi_{1,s} = \begin{bmatrix}
    \eta_{s-1} \\ 0
  \end{bmatrix}, \quad  \text{and} & \quad \xi_{2,s} = \begin{bmatrix}
    0 \\
    \nu_s
\end{bmatrix}.
\end{align*}
Now, under the event $\mathcal{E}_{1,\delta}$, we may simply write $y_s = z_s + M_s \xi_s = z_s + M_s \xi_{1,s} + \xi_{2,s}$ for all $s \ge 1$. Let us note that  $\xi_s$ is independent of $(M_0, \dots, M_s)$ and $(z_0, \dots, z_s)$, and that $\Vert M_s \Vert \le \sqrt{5}C_K$.  Lemma \ref{lem:selb} ensures that
\begin{align*}
\sum_{s=0}^{t-1}
y_s y_s^\top \! \succeq \!\!\! \sum_{s=t(\delta)}^{t-1} \!\!(M_s \xi_s) (M_s \xi_s)^\top \!\! -\! \!\left( \sum_{s=t(\delta)}^{t-1}\!\! z_s (M_s \xi_s)^{\!\top} \!\!  \right)^{\!\!\!\top}\!\!\!\!\left(\sum_{s=t(\delta)}^t \!\! z_s z_s^\top + \lambda I_d \!\right)^{\!\!\!\!-1}\!\!\!\!\left( \sum_{s=t(\delta)}^{t-1} \!\! z_s (M_s \xi_s)^{\!\top} \!\!\right)\!\! -\! I_d.
\end{align*}

\paragraph{Upper bounding the self-normalized term.} Define the following events
\begin{align*}
  \mathcal{A}_{\delta,t} & = \left\lbrace \sum_{s=0}^{t} \Vert x_s \Vert^2 \le C_1  \sigma^2 C_\circ^2 \mathcal{G}_\circ^2 (d_x t^{1+2\gamma} + \log(e/\delta))  \right\rbrace, \\
  \mathcal{E}_{2,\delta,t} & = \Bigg\lbrace  \left\Vert \left( \sum_{s=t(\delta)}^{t-1} z_s z_s^\top + I_d \right)^{-1/2} \left(\sum_{s=t(\delta)}^{t-1} z_s (M_s \xi_s)^\top \right) \right\Vert^2 \\
  &\qquad\qquad \le 35 C_K^2\sqrt{d_x}\sigma^2\log\left(\frac{e^d \det\left(\sum_{s=t(\delta)}^{t-1} z_s z_s^\top + I_d\right) }{\delta} \right)   \Bigg\rbrace, \\
  \mathcal{E}_{3, t} & = \left\lbrace  \lambda_{\max}\left(\sum_{s=0}^{t-1} \eta_s \eta_s^\top  \right) \le \frac{3t}{2} \right\rbrace.
\end{align*}
We have by Proposition \ref{prop:caution} that the event $\mathcal{A}_{\delta,t}$ holds with probability at least $1-\delta$ for some universal positive constant $C_1 > 1$. By Proposition \ref{prop:SNP++}, the event $\mathcal{E}_{2,\delta,t}$ holds with probability at least $1-\delta$. Finally the event $\mathcal{E}_{3,t}$ holds with probability at least $1-\delta$ provided that $t\ge c_1 \sigma^2 (d_x + \log(e/\delta))$.

Provided that $t\ge \log(e/\delta)$, under the event $\mathcal{A}_{t,\delta} \cap \mathcal{E}_{\delta,t} \cap \mathcal{E}_{\delta,t}$ we have
\begin{align*}
  \det\left(\sum_{s=t(\delta)}^{t-1} z_s z_s^\top + I_d\right)^{ 1/d}  & \le \sum_{s=t(\delta)}^{t-1} \Vert z_s \Vert^2 + 1 \\
  & \le \sum_{s=0}^{t-1} 4C_K^2  \Vert x_{s+1} - \eta_{s} \Vert^2 +1 \\
  & \le 8C_K^2 \sum_{s=0}^{t-1} \Vert x_{s+1} \Vert^2 + \Vert \eta_{s} \Vert^2 + 1 \\
  & \le 8 (3C_1 + 4)\sigma^2 C_K^2 C_\circ^2 \mathcal{G}_\circ^2 d_x t^{3\gamma_\star}.
\end{align*}
Thus provided that $t\ge \log(e/\delta)$, under the event $\mathcal{A}_{t,\delta} \cap \mathcal{E}_{2,\delta,t} \cap \mathcal{E}_{3,t}$ we have that
\begin{align*}
  \Bigg\Vert \bigg( \sum_{s=t(\delta)}^{t-1} z_s z_s^\top & + I_d \bigg)^{\!\!\!\!-1/2} \!\!\!\bigg( \sum_{s=t(\delta)}^{t-1}z_s (M_s \xi_s)^\top \bigg) \Bigg\Vert^2 \\
  &\le 35 C_K^2\sqrt{d_x}\sigma^2 ( d \gamma_\star \log(e\sigma C_\circ \mathcal{G}_\circ \Vert P_\star \Vert d_x t) + \log(e/\delta) ).
\end{align*}

\paragraph{Lower bounding $\sum_{s=t(\delta)}^{t-1} (M_s \xi_s) (M_s \xi_s)^\top$.} This is the most challenging task, and we shall break it into several steps. First, we note that $\xi_{1,s}$ and $\xi_{2,s}$ are independent (by design of $\Alg$), and $M_s \xi_s = M_s \xi_{1,s} + \xi_{2,s}$. We use Lemma \ref{lem:selb} to write:
\begin{align*}
  \sum_{s=t(\delta)}^t (M_s\xi_s) (M_s \xi_s)^\top \succeq  \xi_2^\top \xi_2  + \frac{1}{2} \xi_1^\top \xi_1  -  2 \xi_1^\top \xi_2( \xi_2^\top \xi_2 + I_{d})^{-1} \xi_2^\top \xi_1 - \frac{1}{2} I_d
\end{align*}
where we define, for ease of notations, the tall matrices $\xi_1^\top = \begin{bmatrix} M_1\xi_{1,1} & \hdots & M_t\xi_{1,t} \end{bmatrix}$ and $\xi_2^\top = \begin{bmatrix} \xi_{2,t} & \hdots & \xi_{2,t} \end{bmatrix}$, so that we have
\begin{align*}
  \xi_1^\top \xi_1 & = \sum_{s=t(\delta)}^{t-1} (M_s \xi_{1,s})(M_s \xi_{1,s})^\top, \\
  \xi_2^\top \xi_2 & = \sum_{s=t(\delta)}^{t-1}  \xi_{2,s} \xi_{2,s}^\top,\\
   \xi_1^\top \xi_2( \xi_2^\top \xi_2 +  I_{d})^{-1} \xi_2^\top \xi_1 & \! \! = \\
   \!\! \left( \sum_{s=t(\delta)}^{t-1} (M_s \xi_{1,s}) \xi_{2,s}^\top  \!\right)^{\!\!\!\!\top} & \!\! \!\left(  \sum_{s=t(\delta)}^{t-1}  (M_s \xi_{1,s}) (M_s \xi_{1,s})^\top \!\! + \!I_{d}\!  \right)^{\!\!\!-1}  \!\!\!  \left( \sum_{s=t(\delta)}^{t-1} (M_s \xi_{1,s}) \xi_{2,s}^\top  \!\right)\!.
\end{align*}
\medskip
\underline{\textbf{Step 1:}} We first derive a lower bound on the smallest eigenvalue of $\xi_1 \xi_1^\top$. We have
$$
  \xi_1 \xi_1^\top  = \sum_{s=t(\delta)}^{t-1} \xi_{1,s} \xi_{1,s}^\top \succeq
  \begin{bmatrix}
    \frac{\lambda}{ \Vert \sum_{s=0}^{t-1} (\widetilde{K}_{s+1} \eta_{s}) (\widetilde{K}_{s+1} \eta_{s})^\top  \Vert + \lambda} \sum_{s=0}^{t-1} \eta_s \eta_s^\top & O \\
    O & -\lambda I_{d_u}
  \end{bmatrix}.
$$
Indeed, we have for all $\lambda > 0$,
\begin{align*}
  \xi_1 \xi_1^\top & = \sum_{s=1}^t \xi_{1,s} \xi_{1,s}^\top \\
  & =
  \sum_{s=1}^{t-1}
  \begin{bmatrix}
    \eta_{s-1} \eta_{s-1}^\top & \eta_{s-1} ( \widetilde{K}_s\eta_{s-1})^\top \\
    \widetilde{K}_s \eta_{s-1} \eta_{s-1}^\top & ( \widetilde{K}_s  \eta_{s-1}) ( \widetilde{K}_s \eta_{s-1})^\top
  \end{bmatrix}\\
  & =\begin{bmatrix}
    \eta^\top \eta & \eta^\top X \\
    X \eta & X^\top X
\end{bmatrix} \\
& \succeq \begin{bmatrix}
  \frac{\lambda}{\Vert X \Vert^2 +  \lambda}\eta^\top \eta & 0 \\
  O & -\lambda I_{d_u}
\end{bmatrix},
\end{align*}
where we defined, for ease of notations, the tall matrices
$
\eta^\top  = \begin{bmatrix}
 \eta_0 &
\hdots &
 \eta_{t-1}
\end{bmatrix}$, $
X^\top = \begin{bmatrix}
 K_1 \eta_{0} &
\hdots &
 K_t \eta_{t-1}
\end{bmatrix},
$
and used Lemma \ref{lem:selb2} to obtain the last inequality. We may apply Corollary \ref{cor:RM++} to bound from below the above inequality. First define the events
\begin{align*}
  \mathcal{E}_{4,t} & = \left\lbrace   2 (t-t(\delta))I_{d_x} \succeq   \sum_{s=t(\delta)}^{t-1} \eta_s \eta_s^\top  \succeq \frac{t -t(\delta)}{2} I_{d_x}  \right\rbrace, \\
  \mathcal{E}_{5,t} & = \left\lbrace \left\Vert  \sum_{s=t(\delta)}^{t-1} (\widetilde{K}_{s+1} \eta_s) (\widetilde{K}_{s+1} \eta_s)^\top   \right\Vert \le  8 C_K^2 (t-t(\delta))  \right\rbrace.
\end{align*}
By proposition \ref{prop:RM++}, the event $\mathcal{E}_{4,t}$ holds with probability at least $1-\delta$, provided that $t\ge c_2 \sigma^2(d_x +\log(e/\delta))$ for some universal positive constant $c_2> 0$. By Propostion \ref{prop:RM++}, the event $\mathcal{E}_{5,t}$ holds with probability $1-\delta$, provided that $t \ge c_3 \sigma^2(d_x + \log(e/\delta))$ for some universal positive constants $C_3,c_3 > 0$.  Therefore, provided that $t \ge 2 t(\delta)$ and $\lambda = \epsilon \sqrt{t}$ under the event $\mathcal{E}_{4,t} \cap \mathcal{E}_{5,t}$ we have
\begin{align*}
  \xi_{1}^\top \xi_{1}  & \succeq \begin{bmatrix}
    \frac{\epsilon(t -t(\delta))\sqrt{t}}{8C_K^2 (t-t(\delta)) + \epsilon \sqrt{t}} I_{d_x} & O \\
    O & - \lambda I_{d_u} \\
\end{bmatrix} \\
& \succeq  \begin{bmatrix}
  \frac{\epsilon t \sqrt{t}}{8C_K^2 t + 2\epsilon\sqrt{t}} I_{d_x} & O \\
  O & - \epsilon \sqrt{t} I_{d_u} \\
\end{bmatrix}.
\end{align*}

\medskip

\underline{\textbf{Step 2:}} Next, we find a lower bound on the smallest eigenvalue of the random matrix $\xi_2 \xi_2^\top$. Consider the event
$$
\mathcal{E}_{6,t} = \left\lbrace  \lambda_{\min}\left(\sum_{s=t(\delta)}^{t-1} \nu_s \nu_s^\top \right) \ge \frac{\sigma^2 \sqrt{d_x (t - t(\delta))} }{2} \right\rbrace.
$$
By Proposition \ref{prop:RM++}, the event $\mathcal{E}_{6,t}$ holds with probability at least $1-\delta$, provided that we have $t \ge c_4  (d_u + \log(e/\delta)) + t(\delta)$ for some universal positive constant $c_4> 0$. Note that we need to apply Proposition \ref{prop:RM++} to the normalized random matrix $\frac{1}{\sigma^2\sqrt{d_x t}}\sum_{s=t(\delta)}^t \nu_s \nu_s^\top$. Thus, provided that $t\ge 2 t(\delta)$, under the event $\mathcal{E}_{4,t}$, we have
$$
\xi_{1} \xi_{1}^\top \succeq  \begin{bmatrix}
  O & O \\
  O & \frac{\sigma^2 \sqrt{d_x t}}{2\sqrt{2}} I_{d_u}
\end{bmatrix}.
$$
\medskip

\underline{\textbf{Step 3:}} We now upper bound the norm of the self-normalized matrix process $\xi_1^\top \xi_2( \xi_2^\top \xi_2 +  I_{d})^{-1} \xi_2^\top \xi_1$. Consider the event
$$
\mathcal{E}_{7,\delta,t} = \bigg\lbrace   \left\Vert ( \xi_2^\top \xi_2 +  I_{d})^{-1/2} \xi_2^\top \xi_1  \right\Vert^2 \le 7 \sigma^2 \sqrt{d_x}  \log\left(\frac{e^d \det\left( \sum_{s=t(\delta)}^{t-1} (M_s \xi_{1,s})(M_s \xi_{1,s})^\top + I_d \right)}{\delta}\right)   \bigg\rbrace.
$$
By Proposition \ref{prop:SNP++}, the event $\mathcal{E}_{7,t,\delta}$ holds with probability at least $1-\delta$. Therefore, provided that $t\ge \log(e/\delta)$, under the event $\mathcal{E}_{7,t,\delta} \cap \mathcal{E}_{4,t}$ we have
\begin{align*}
  \det\left( \sum_{s=t(\delta)}^{t-1} (M_s \xi_{1,s})(M_s \xi_{1,s})^\top + I_d \right)^{1/d} & \le \sum_{s=t(\delta)}^{t-1} \Vert M_s \xi_{1,s} \Vert^2 +1 \\
  & \le \sum_{s=t(\delta)}^{t-1} 2\Vert \widetilde{K_s}\Vert^2 \Vert \eta_s \Vert^2 +1 \\
  & \le 8 C_K^2 t.
\end{align*}
Therefore, for $t\ge \log(e/\delta)$, we have, under $\mathcal{E}_{7,t,\delta} \cap \mathcal{E}_{4,t}$, that
\begin{align*}
  \left\Vert ( \xi_2^\top \xi_2 +  I_{d})^{-1/2} \xi_2^\top \xi_1  \right\Vert^2 \le C_4 \sigma^2 \sqrt{d_x} ( d \log(eC_Kt) +\log(e/\delta)),
\end{align*}
for some universal positive constant $C_4 > 0$.

\medskip

\underline{\textbf{Step 4:}} \emph{(Putting everything together)} From the first and second step, provided that $t\ge 2t(\delta)$ and $t\ge \frac{\sigma^4 d_x}{32}$, under the event $\mathcal{E}_{4,t} \cap \mathcal{E}_{5,t} \cap \mathcal{E}_{6,t}$,  we have
\begin{align*}
  \xi_2 \xi_2^\top + \frac{1}{2} \xi_1 \xi_1^\top & \succeq \begin{bmatrix} \frac{ \frac{\sigma^2\sqrt{d_x}}{8\sqrt{2}} t \sqrt{t}}{8C_K^2 t +  \frac{\sigma^2\sqrt{d_x}}{4\sqrt{2}}\sqrt{t}} I_{d_x} & O \\
  O &  \frac{\sigma^2\sqrt{d_x}}{8\sqrt{2}}\sqrt{t} I_{d_u} \end{bmatrix} \\
  &\succeq  \frac{\sigma^2\sqrt{d_xt}}{8\sqrt{2}}  \begin{bmatrix}  \frac{1}{8C_K^2  +  \frac{\sigma^2\sqrt{d_x}}{4\sqrt{2t}}} I_{d_x} & O \\
  O &   I_{d_u} \end{bmatrix} \\
  & \succeq  \frac{\sigma^2 \sqrt{d_x t} }{81\sqrt{2}C_K^2}  I_d
\end{align*}
where we chose $\epsilon = \frac{\sigma^2\sqrt{d_x}}{8\sqrt{2}}$. Therefore provided $t\ge 2t(\delta)$, $t\ge \frac{\sigma^4 d_x}{32}$ and $t\ge \log(e/\delta)$, we have under the event $\mathcal{E}_{4,t} \cap \mathcal{E}_{5,t} \cap \mathcal{E}_{6,t} \cap \mathcal{E}_{7,\delta,t}$ that
\begin{align*}
  \lambda_{\min}\left(\sum_{s=t(\delta)}^{t-1} (M_s \xi_s) (M_s\xi_s)^\top \right) \ge  \frac{\sigma^2 \sqrt{d_x t} }{81\sqrt{2}C_K^2} - C_4 \sigma^2 \sqrt{d_x} ( \log(eC_Kt) +\log(e/\delta)) - \frac{1}{2}
\end{align*}

Now, using Lemma \ref{lem:technical}, there exists an universal positive constant $c_6 > 0$, such that under the following condition
\begin{equation}\label{eq:Isc1}
  t^{1/2} \ge c_5 \max\left(t(\delta)^{1/2}, \sigma^2 C_K^2 (d \log(eC_k d_x d_u)   + \log(e/\delta)) \right)
\end{equation}
then the following conditions also hold
\begin{itemize}
  \itemsep0em
  \item $ \frac{\sigma^2 \sqrt{d_x t} }{81\sqrt{2}C_K^2} - C_4 \sigma^2 \sqrt{d_x} ( \log(eC_Kt) +\log(e/\delta)) - \frac{1}{2} \ge  \frac{\sigma^2 \sqrt{d_x t} }{100\sqrt{2}C_K^2} $
  \item $t\ge \log(e/\delta)$
  \item $t \ge 2 t(\delta)$
  \item $t \ge \frac{\sigma^4 d_x}{32}$
  \item $t\ge c_2 \sigma^2(d_x + \log(e/\delta))$
  \item $t\ge c_3 \sigma^2(d_x + \log(e /\delta))$
  \item $t \ge c_4 (d_u + \log(e/\delta))$
\end{itemize}
Hence, if condition \eqref{eq:Isc1} holds, we have under the event $\mathcal{E}_{4,t} \cap \mathcal{E}_{5,t} \cap \mathcal{E}_{6,t} \cap \mathcal{E}_{7,\delta,t}$ that
$$
  \lambda_{\min}\left(\sum_{s=t(\delta)}^{t-1} (M_s \xi_s) (M_s\xi_s)^\top \right) \ge  \frac{\sigma^2 \sqrt{d_x t} }{100\sqrt{2}C_K^2}.
$$

\paragraph{The concluding step.} To conclude, provided that condition \eqref{eq:Isc1} holds, we have under the event $\mathcal{A}_{\delta,t} \cap \mathcal{E}_{1,\delta} \cap \mathcal{E}_{2,\delta,t} \cap \mathcal{E}_{3,t} \cap \mathcal{E}_{4,t} \cap \mathcal{E}_{5,t} \cap \mathcal{E}_{6,t} \cap \mathcal{E}_{7,\delta,t}$,
$$
\lambda_{\min}\left(\sum_{s=0}^t y_s y_s^\top\right) \ge \frac{\sigma^2 \sqrt{d_x t} }{100\sqrt{2}C_K^2} - 1 -  35 C_K^2\sqrt{d_x}\sigma^2 ( d \gamma_\star \log(e\sigma C_\circ \mathcal{G}_\circ \Vert P_\star \Vert d_x t) + \log(e/\delta)).
$$
Using again Lemma \ref{lem:technical}, there exists an universal positive constant $c > 0$, such that if
\begin{equation}\label{eq:I:se:sc2}
  t^{1/2} \ge c \max(t(\delta)^{1/2}, C_K^4 (d\gamma_\star \log(e \sigma C_\circ \mathcal{G}_\circ \Vert P_\star \Vert d_x d_u \gamma_\star) + \log(e/\delta)) )
\end{equation}
then
\begin{itemize}
  \itemsep0em
  \item $\frac{\sigma^2 \sqrt{d_x t} }{100\sqrt{2}C_K^2} - 1 -  35 C_K^2\sqrt{d_x}\sigma^2 ( d \gamma_\star \log(e\sigma C_\circ \mathcal{G}_\circ \Vert P_\star \Vert d_x t) + \log(e/\delta)) \ge \frac{\sigma^2 \sqrt{d_x t} }{150 C_K^2}$,
  \item condition \eqref{eq:Isc1} holds.
\end{itemize}
Therefore provided condition \eqref{eq:I:se:sc2} holds, we have
\begin{align*}
  \mathbb{P}\left( \lambda_{\min}\left(  \sum_{s=0}^{t-1} y_s y_s^\top \right)  \ge \frac{\sigma^2 \sqrt{d_x t} } {150 C_K^2}  \right) &\ge 1 - \mathbb{P}(  \mathcal{A}_{\delta,t}^c \cup \mathcal{E}_{1,\delta}^c \cup \mathcal{E}_{2,\delta,t}^c \cup \mathcal{E}_{3,t}^c \cup \mathcal{E}_{4,t}^c \cup \mathcal{E}_{5,t}^c \cup \mathcal{E}_{6,t}^c \cup \mathcal{E}_{7,\delta,t}^c) \\
  & \ge 1 - 8\delta
\end{align*}
Hence reparametrizing by $\delta' = 8\delta$ gives the desired result with modified universal constants.
\end{proof}


\subsection{Scenario II -- ($A$ known)}

In this scenario, the cumulative covariates matrix is $\sum_{s=0}^{t-1} u_s u_s^\top$. We present two results about its smallest eigenvalue. In the first result, we show that this eigenvalue scales at least as $\sqrt{t}$. In the second result, we obtain a linear growth rate, when the certainty equivalence controller $K_t$ has become close to the true optimal controller $K_{\star}$.

\begin{proposition}[Sufficent exploration]\label{prop:IIb:we}
  Under Algorithm $\Alg$, we have for all $t\ge 1$, and $\delta \in (0,1)$,
  \begin{equation*}
    \lambda_{\min}\left(\sum_{s=0}^{t-1} u_s u_s^\top\right) \ge \sqrt{\frac{t}{2}}
  \end{equation*}
  holds with probability at least $1-\delta$, provided that
  $
  t \ge c (d_u \gamma_\star \log(   e \sigma C_\circ \mathcal{G}_\circ d_x \gamma_\star) + \log(e/\delta))
  $
  for some universal positive constant $c> 0$.
\end{proposition}

\begin{proof}[Proof of Proposition \ref{prop:IIb:we}]
Recall that under $\Alg$, we have
  \begin{equation*}
    u_t \gets \begin{cases}
       K_t  x_t  & \text{if } \ell_t = 1 \text{ and } \Vert K_t \Vert^2 \le h(t), \text{ and } \lambda_{\min} \left(\sum_{s=0}^{t-1} u_s u_s^{\top}  \right) \ge \sqrt{t} \\
       K_\circ   x_t + \zeta_{t}  & \text{otherwise.}
    \end{cases}
  \end{equation*}
For ease of notation, for all $s\ge 0$, we denote $z_s = K_\circ x_s$. Consider the event
  $$
  \mathcal{E}_{1,t} = \left\lbrace \exists i \in \lbrace t/2,\dots, t-1\rbrace: \quad \lambda_{\min}\left(\sum_{s=0}^{i-1} u_i u_i^\top\right) \ge \sqrt{i}  \right\rbrace.
  $$
  Under the event $\mathcal{E}_{1,t}^c$, for all $s \in \lbrace t/2, \dots, t-1\rbrace$, $u_s = K_\circ  x_s + \zeta_{s} = z_s + \zeta_s$. Thus, by Lemma \ref{lem:selb} (with $\lambda = 1$),
  \begin{align*}
    \sum_{s=0}^t u_s u_s^\top  \succeq \sum_{s=t/2}^t u_s u_s^\top  \succeq \sum_{s=t/2}^t \zeta_{s} \zeta_{s}^\top -  I_{d_u} -  \left\Vert \left(\sum_{s=t/2}^{t} z_s z_s^\top + I_{d_u} \right)^{-1/2} \left(\sum_{s=t/2}^{t} z_s \zeta_s^\top \right)\right\Vert^2 I_{d_u}.
  \end{align*}
  For all $\delta \in(0,1)$ and $t\ge 1$, define the following events
  \begin{align*}
    \mathcal{A}_{t,\delta}  & = \left\lbrace \sum_{s=0}^t \Vert x_s  \Vert \le C_1  \mathcal{G}_\circ^2 C_\circ^2 \sigma^2 (d_x t^{1+2\gamma} + \log(e/\delta)) \right\rbrace, \\
    \mathcal{E}_{2, t,\delta} & = \left\lbrace \left\Vert \left(\sum_{s=t/2}^{t-1} z_s z_s^\top + I_d \right)^{-1/2} \left(\sum_{s=t/2}^{t-1} z_s \zeta_s^\top \right)\right\Vert^2
    \le 7  \log\left(\frac{ e^{d_u}   \det\left(\sum_{s=t/2}^{t-1} z_s z_s^\top  + I_{d_u}  \right) }{\delta}  \right)      \right\rbrace, \\
    \mathcal{E}_{3, t} & = \left\lbrace   \lambda_{\min}\left(\sum_{s=t/2}^{t-1} \zeta_s \zeta_s^\top \right)  \ge \frac{t}{3}   \right\rbrace.
  \end{align*}
In view of Proposition \ref{prop:caution}, $\mathbb{P}(\mathcal{A}_{\rho,t}) \ge 1 - \delta$. By Proposition \ref{prop:SNP++}, we have $\mathbb{P}(\mathcal{E}_{2,t,\delta}) \ge 1 - \delta$, and by Proposition \ref{prop:RM++}, $\mathbb{P}(\mathcal{E}_{3,t}) \ge 1-\delta$ when $t \ge c_1(d_u + \log(e/\delta))$ for some universal positive constant $c_1 > 0$. Under the event $\mathcal{A}_{t,\delta} \cap \mathcal{E}_{1,t}^c\cap \mathcal{E}_{2,t,\delta}$, when $t\ge \log(e/\delta)$, we have
  \begin{align*}
    \det\left(\sum_{s=t/2}^{t-1} z_s z_s^\top  + I_{d_u}  \right) ^{1/d_u} & \le \sum_{s=t/2}^t \Vert z_s \Vert^2 +1 \\
    & \le \Vert K_\circ \Vert^2 \sum_{s=0}^t \Vert x_s \Vert^2 + 1 \\
    & \le 3 C \sigma^2 \Vert K_\circ \Vert^2 C_\circ^2  \mathcal{G}_\circ^2  d_x  t^{1+2\gamma} \\
    & \le 3 C \sigma^2 C_\circ^4 \mathcal{G}_\circ^2 d_x t^{3\gamma_\star}.
  \end{align*}
  Thus, under the event $ \mathcal{A}_{t,\delta} \cap \mathcal{E}_{1,t}^c\cap \mathcal{E}_{2,t,\delta}$, we have
  \begin{align*}
        \left\Vert \left(\sum_{s=t/2}^{t-1} z_s z_s^\top + I_{d_u} \right)^{-1/2} \left(\sum_{s=t/2}^{t-1} z_s \zeta_s^\top \right)\right\Vert^2 & \le  C_2 ( d_u\gamma_\star \log(\sigma C_\circ \mathcal{G}_\circ d_x t) + \log(e/\delta) ).
  \end{align*}
  Hence, under the event $ \mathcal{A}_{t,\delta} \cap \mathcal{E}_{1,t}^c\cap \mathcal{E}_{2,t,\delta} \cap \mathcal{E}_{3,t}$,
  \begin{align*}
        \lambda_{\min}\left(\sum_{s=0}^{t-1} u_s u_s^\top \right) \ge \frac{t}{3} - 1 - C_3 \left (d_u \gamma_\star \log(\sigma C_\circ \mathcal{G}_\circ d_x t)  + \log(e/\delta) \right),
  \end{align*}
when $t\ge \log(e/\delta)$. Now, Lemma \ref{lem:technical} ensures that there exists some universal constant $c_1> 0$ such that if
  \begin{align}\label{eq1:sc}
    t \ge c_1 \left( d_u \gamma_\star\log(\sigma C_\circ \mathcal{G}_\circ d_x d_u \gamma_\star)  + \log(e/\delta) \right)
  \end{align}
  then
  \begin{itemize}
    \itemsep0em
    \item $\frac{t}{3} - 1 - C_3 \left (d_u \gamma_\star\log(\sigma C_\circ \mathcal{G}_\circ d_x t)  + \log(e/\delta) \right) \ge \frac{t}{6} > \sqrt{t}$,
    \item $t \ge c_1 (d_u + \log(e/\delta))$,
    \item $t \ge \log(e/\delta)$.
  \end{itemize}
  Therefore when the condition \eqref{eq1:sc} holds, then under the event $ \mathcal{A}_{t,\delta} \cap \mathcal{E}_{1,t}^c\cap \mathcal{E}_{2,t,\delta} \cap \mathcal{E}_{3,t}$ it must hold that
  \begin{align*}
    \lambda_{\min}\left(\sum_{s=0}^{t-1} u_s u_s^\top \right) > \sqrt{t}
  \end{align*}
  but this cannot hold under $\mathcal{E}_{t,\delta}^c$, therefore it must be that $ \mathcal{A}_{t,\delta} \cap \mathcal{E}_{2,t,\delta} \cap \mathcal{E}_{3,t} \subseteq \mathcal{E}_{1,t}$ which in turn implies that
  $$
  \mathbb{P}(\mathcal{E}_{1,t}) \ge 1 - \mathbb{P}(\mathcal{A}_{t\delta}^c) -   \mathbb{P}(\mathcal{E}_{2,t,\delta}^c) - \mathbb{P}(\mathcal{E}_{3,t}^c) \ge 1 - 3 \delta.
  $$
 Reparametrizing $\delta' = 3 \delta$ yields the desired result with modified universal constants.
\end{proof}

\newpage

\begin{proposition}[Sufficient exploration with refined rates] \label{prop:IIb:se}
  Under $\Alg$, assume that $\mu_\star^2 = \min(\lambda_{\min}(K_\star K_\star^\top), 1) > 0$ and that for all $\delta \in (0,1)$ we have
  $$
  \mathbb{P}\left(\forall t \ge t(\delta), \ \ \Vert K_t - K_\star \Vert < \frac{\mu_\star}{2} \right) \ge 1- \delta
  $$
  for some $t(\delta) \ge 1$. Then for all $\delta \in (0,1)$,
  $$
  \lambda_{\min}\left(\sum_{s=0}^{t-1} u_s u_s^\top \right) \ge \frac{\mu_\star^2 t}{10}
  $$
 holds with probability at least $1-\delta$,  provided that
  $$
  t \ge c \max\left(t(\delta), \frac{ \sigma^4 C_K^2}{\mu_\star^2} \left((d_u + d_x) \gamma_\star \log(e \sigma C_\circ \mathcal{G}_\circ \Vert P_\star \Vert d_u d_x \gamma_\star  )  + \log(e/\delta) \right)\right),
  $$
  for some universal positive constants $C,c > 0$.
\end{proposition}

\begin{proof}[Proof of Proposition \ref{prop:IIb:se}]
  We start by defining the event
  $$
  \mathcal{E}_{1,\delta} =  \left\lbrace \forall t \ge t(\delta):  \Vert K_t - K_\star \Vert \le \frac{\mu_\star}{2}   \right\rbrace.
  $$
  We note that under the event $\mathcal{E}_{1,\delta}$, we have $ (2\Vert K_\star \Vert)^2 \succeq K_t K_t^\top \succeq \left(\frac{\mu_\star}{2}\right)^2 I_{d_u}$. Now, recall that under $\Alg$, we have $u_t = (1-\alpha_t)(K_t  x_t) + \alpha_t (K_\circ  x_t +\zeta_t)$ for all $t \ge 1$, where
  \begin{equation*}
    \forall t \ge 1 :\quad \alpha_t =  \begin{cases}
      0 & \text{if } \ell_t = 1 \text{ and } \Vert K_t \Vert^2 \le h(t), \text{ and } \lambda_{\min} \left(\sum_{s=0}^{t-1} u_s u_s^{\top}  \right) \ge \sqrt{t} \\
      1 & \text{otherwise.}
  \end{cases}
\end{equation*}
Define  $\widetilde{K}_t = (1-\alpha_t) K_t 1_{ \lbrace \Vert K_t - K_\star \Vert \le \mu_\star  \rbrace} + \alpha_t K_\circ$  for all $t\ge 1$, and
\begin{align*}
  z_t = \widetilde{K}_t (A x_t + B u_t),
  \quad
  M_t = \begin{bmatrix}
    \widetilde{K}_t  & \alpha_t I_{d_u}
  \end{bmatrix}, \quad
  \text{ and } \quad \xi_t = \begin{bmatrix}
    \eta_{t-1} \\
    \zeta_t
  \end{bmatrix}.
\end{align*}
Note that $u_t = z_t + M_t \xi_t$ for all $t\ge t(\delta)$ under the event $\mathcal{E}_{1,\delta}$. We may also use Lemma \ref{lem:selb}, and obtain always under $\mathcal{E}_{1,\delta}$,
$$
\sum_{s=t(\rho)}^t u_s u_s^\top \succeq \sum_{s=t(\rho)}^{t} (M_t\xi_t) (M_t\xi_t)^\top  - I_{d_u} - \left\Vert \left( \sum_{s=t(\delta)}^t z_s z_s^\top +  I_{d_u} \right)^{\!\!\!\!-1/2}\!\!\!\! \left(\sum_{s=t(\delta)}^t z_s (M_s\xi_s)^\top \right)\right\Vert^2 I_{d_u}.
$$
It is worth mentioning here that $\xi_t$ is independent of $(M_0, \dots, M_t)$ and $(\alpha_0, \dots,\alpha_t)$.  Furthermore, we can easily verify that $ \Vert M_t \Vert \le 2C_K$  and  $\xi_t$  is zero-mean and  $\sigma^2$-sub-subgaussian. Let us consider the events
\begin{align*}
  \mathcal{E}_{2,\delta,t} & = \left\lbrace \sum_{s=t(\delta)}^t  (M_s \xi_s) (M_s \xi_s)^\top \succeq \sum_{s=t(\delta)}^t M_s M_s^\top - \frac{\mu_\star^2}{8} (t-t(\delta)+1) I_{d_u}\right\rbrace, \\
  \mathcal{E}_{3,\delta,t} & = \Bigg\lbrace \left\Vert \left( \sum_{s=t(\delta)}^t \!\!z_s z_s^\top + I_{d_u}\right)^{\!\!\!\!-1/2}\!\!\!\! \left(\sum_{s=t(\delta)}^t \!\!z_s (M_s\xi_s)^\top \right)\right\Vert^2 \!\!\! \\
  & \qquad\qquad \le 14\sigma^2 C_K^2 \log\left( \frac{e^{d_u} \det\left(\sum_{s=t(\delta)}^t z_s z_s^\top  + I_{d_u}\right)}{\delta} \right)  \Bigg\rbrace, \\
  \mathcal{E}_{4,t} & = \left\lbrace \lambda_{\max} \left(\sum_{s=0}^t\eta_t \eta_t^\top\right) \le \frac{3t}{2}  \right\rbrace, \\
  \mathcal{A}_{\delta,t} & = \left\lbrace \sum_{s=0}^t \Vert x_s\Vert^2 \le C_1 \sigma^2 \mathcal{G}_\circ^2 C_\circ^2  (d_x t^{1+2\gamma} + \log(e/\delta)) \right\rbrace.
\end{align*}
We have, by Corollary \ref{cor:RM++}, that $\mathbb{P}\left( \mathcal{E}_{2,\delta,t}\right) \ge 1 - \delta$ provided that $t\ge \frac{c_1(\sigma C_K)^4} {\mu_\star^2}(d_{u} + \log(e/\delta))$ for some universal positive constant $c_1>0$. By Proposition \ref{prop:SNP++},  $\mathbb{P}( \mathcal{E}_{3,\delta,t}) \ge 1 - \delta$. Applying Corollary \ref{cor:RM++}, we get $\mathbb{P}(\mathcal{E}_{4,t}) \ge 1 - \delta$ provided that  $t \ge c_2 \sigma^4 (d_x + \log(e/\delta))$. Finally we have under $\Alg$, by Proposition \ref{prop:caution}, $\mathbb{P}(\mathcal{A}_{\delta,t}) \ge 1-\delta$ for some universal positive constant $C_1 > 1$ (used in the definition of $\mathcal{A}_{\delta,t}$).

Under the event $\mathcal{E}_{1,\delta,t}\cap \mathcal{E}_{2,\delta,t}$, we have
\begin{align*}
  \sum_{s=t(\delta)}^t  (M_s \xi_s) (M_s \xi_s)^\top & \succeq \sum_{s=t(\delta)}^t M_s M_s^\top - \frac{\mu_\star^2}{8} (t-t(\delta) + 1) I_{d_u} \\
  & \succeq  \left(\sum_{s=t(\delta)}^t (1-\alpha_t) 1_{\lbrace \Vert K_t - K_\star \Vert < \mu_\star\rbrace}K_s K_s^\top + \alpha_t K_\circ K_\circ^\top + \alpha_t I_{d_u}  \right) \\
  &\qquad\qquad \quad - \frac{\mu_\star^2}{8}(t-t(\delta) + 1) I_{d_u} \\
  & \succeq \left(\sum_{s=t(\delta)}^t (1-\alpha_t) \frac{\mu_\star^2}{4} I_{d_u} + \alpha_t I_{d_u}  \right)- \frac{\mu_\star^2}{8} (t-t(\delta) + 1) I_{d_u}  \\
  & \succeq \frac{\mu_\star^2}{8}  (t-t(\delta) + 1) I_{d_u}.
\end{align*}

When $t\ge \log(e/\delta)$, under the event $\mathcal{E}_{3,\delta,t} \cap \mathcal{E}_{4,t} \cap \mathcal{A}_{\delta,t}$, we obtain

\begin{align*}
  \det\left(\sum_{s=t(\delta)}^t z_s z_s^\top  + I_{d_u}\right)^{1/d_u} & \le \sum_{s=t(\delta)}^t \Vert z_s \Vert^2 + 1 \\
  & \le 4 C_K^2 \sum_{s=t(\delta)}^t \Vert x_{s+1} - \eta_s \Vert^2 +  + 1 \\
  & \le 4 C_K^2 \sum_{s=0}^t 2\Vert x_{s+1} \Vert^2 + 2\Vert \eta_{s} \Vert^2 +1 \\
  & \le 4 C_K^2 ( 4 C_1  \sigma^2 C_\circ^2 \mathcal{G}_\circ^2 d_x t^{1+2\gamma} + 3t)  + 1\\
  & \le C_2 \sigma^2 C_K^2 C_\circ^2 \mathcal{G}_\circ^2 d_x t^{3\gamma^\star},
\end{align*}
for some universal positive constant $C_2$ that is large enough for the last inequality to hold. Therefore,  when $t\ge \log(e/\delta)$, under the event $\mathcal{E}_{3,\delta,t} \cap \mathcal{E}_{4,t} \cap \mathcal{A}_{\delta,t}$ it holds that
\begin{align*}
 \left\Vert \left( \sum_{s=t(\delta)}^t \!\!z_s z_s^\top + I_{d_u}\right)^{\!\!\!\!-1/2}\!\!\!\! \left(\sum_{s=t(\delta)}^t \!\!z_s (M_s\xi_s)^\top \right)\right\Vert^2 \!\!\!\le C_3 \sigma^2 C_K^2 \left( d \log(e \sigma C_\circ \mathcal{G}_\circ \Vert P_\star \Vert d_x t) + \log(e/\delta)\right),
\end{align*}
for some universal positive constant $C_3 > 0$ where we used the crude upper bound $C_K \le C_\circ \Vert P_\star \Vert$ and denoted $d=d_x +d_u$. Therefore, when $t\ge \log(e/\delta)$, under the event $ \mathcal{A}_{\delta,t} \cap \mathcal{E}_{1,\delta,t} \cap \mathcal{E}_{2,\delta,t} \cap \mathcal{E}_{3,\delta,t} \cap \mathcal{E}_{4,t}$, we have
\begin{align*}
  \sum_{s=t(\delta)}^t u_s u_s^\top \succeq \frac{\mu_\star^2}{8} (t - t(\delta) +1) - 1 -  C_3 \sigma^2 C_K^2 \left( d \log(e \sigma C_\circ \mathcal{G}_\circ \Vert P_\star \Vert d_x t) + \log(e/\delta)\right).
\end{align*}

Using Lemma \ref{lem:technical}, there exists a constant $c_3 > 0$ such that if
\begin{equation}\label{eq:IIb:refined:sc}
  t \ge c_3  \left( t(\delta) + \frac{\sigma^4 C_K^2}{\mu_\star^2} ( d \gamma_\star \log( e \sigma C_\circ \mathcal{G}_\circ  d \gamma_\star) + \log(e/\delta)\right),
\end{equation}
then the following holds
\begin{itemize}
  \itemsep0em
  \item $\frac{\mu_\star^2}{8} (t - t(\delta) +1) - 1 -  C_3 \sigma^2 C_K^2 \left( d \log(e \sigma C_\circ \mathcal{G}_\circ \Vert P_\star \Vert d_x t) + \log(e/\delta)\right) \ge \frac{\mu_\star^2 t}{10} $,
  \item $t\ge \log(e/\delta)$,
  \item $t\ge \frac{c_1 (\sigma C_K)^2}{\mu_\star^2} ( d_u + \log(e/\delta))$,
  \item $t\ge c_2 \sigma^4 (d_x + \log(e/\delta))$.
\end{itemize}
Therefore, provided \eqref{eq:IIb:refined:sc} holds,
$$
\mathbb{P}\left(    \lambda_{\min}\left(\sum_{s=t(\delta)}^t u_s u_s^\top\right) \ge \frac{\mu_\star^2 t}{10}   \right)
 \ge 1 - \mathbb{P}(\mathcal{A}_{\delta,t}^c \cup \mathcal{E}_{1,\delta}^c \cup \mathcal{E}_{2,\delta,t}^c \cup \mathcal{E}_{3,\delta,t}^c \cup \mathcal{E}_{4,t}^c) \ge 1 - 5\delta.
$$
Reparametrizing by $\delta' = 5\delta$ gives the desired result with modified universal positive constant. This concludes the proof.
\end{proof}


\subsection{Scenario II -- ($B$ known)}

In this scenario, the cumulative covariates matrix is $\sum_{s=0}^{t} x_s x_s^\top$. We establish that its smallest eigenvalue scales linearly with time.

\begin{proposition} \label{prop:IIa:se}
  Under Algorithm $\Alg$, for all $\delta \in (0,1)$,
  $$
  \lambda_{\min}\left( \sum_{s=0}^t x_s x_s^\top \right ) \ge \frac{t}{4}
  $$
  holds with probability at least $1-\delta$, when
  $
  t \ge c \sigma^2 \left(  d_x \gamma_\star \log( \sigma \mathcal{G}_\circ  C_\circ   d_x \gamma_\star) + \log(e/\delta) \right)
  $
  for some universal positive constant $c > 1$.
\end{proposition}

\begin{proof}[Proof of Proposition \ref{prop:IIa:se}]
Recall that for all $s\ge 0$, we have $x_{s+1} = A x_s +  B u_s + \eta_s$. For ease of notation, we define for all $s\ge 0$,  $z_s = Ax_s + Bu_s$. Thus, we may write $x_{s+1} = z_s + \eta_s$ and note that $\eta_s$ is independent of $z_s$. Now, a direct application of Lemma \ref{lem:selb} (with the choice $\lambda = d_x  $) gives, for all $t \ge 1$,
\begin{equation}\label{se:case1:eq0}
\sum_{s=0}^t x_s x_s^\top \succeq \sum_{s=0}^{t-1} \eta_s \eta_s^\top -  I_{d_x} -  \left\Vert \left(\sum_{s=0}^{t-1} z_s z_s^\top +  I_{d_x} \right)^{-1/2} \left(\sum_{s=0}^{t-1} z_s \eta_s^\top \right)\right\Vert^2 I_{d_x}.
\end{equation}
We shall see that the terms appearing on the right hand side of the above inequality can be bounded adequatly when certain events hold. Let $\delta \in (0,1)$ and $t \ge 1$, and define the following events
\begin{align*}
  \mathcal{A}_{\delta, t}  & = \left\lbrace \sum_{s=0}^t \Vert x_s  \Vert \le C_1  \mathcal{G}_\circ^2 C_\circ^2 \sigma^2 (d_x h(t)g(t) + \log(e/\delta)) \right\rbrace, \\
  \mathcal{E}_{1, \delta, t} & = \left\lbrace \left\Vert \left(\sum_{s=0}^{t-1} z_s z_s^\top +  I_d \right)^{-1/2} \left(\sum_{s=0}^{t-1} z_s \eta_s^\top \right)\right\Vert^2
  \le 7 \sigma^2 \log \left(\frac{e^{d_x} \det\left(\sum_{s=0}^{t-1} z_s z_s^\top  + I_{d_x}  \right)}{\delta} \right)  \right\rbrace, \\
  \mathcal{E}_{2, \rho, t} & = \left\lbrace  \frac{2t}{3} \le \lambda_{\min}\left(\sum_{s=0}^{t-1} \eta_s \eta_s^\top  \right) \quad \text{and} \quad \lambda_{\min}\left(\sum_{s=0}^{t-1} \eta_s \eta_s^\top  \right) \le \frac{4 t}{3}    \right\rbrace.
\end{align*}
We have:
\begin{align}
  \mathbb{P}\left( \mathcal{A}_{\rho, t}\right) & \ge 1 - \delta, \label{se:case1:eq1}\\
  \mathbb{P}\left( \mathcal{E}_{1, \rho, t}\right) & \ge 1 - \delta, \label{se:case1:eq2}\\
 \mathbb{P}\left(\mathcal{E}_{2,\rho,t}\right) &\ge 1 -\delta, \ \ \hbox{ if } \ \ t \ge c_1\sigma^2 (d_x + \log(e/\delta)). \label{se:case1:eq3}
\end{align}
\eqref{se:case1:eq1} follows from Proposition \ref{prop:caution} with the constant $C_1$ defined in the statement of the proposition, \eqref{se:case1:eq2} follows from Proposition \ref{prop:SNP++}, and finally \eqref{se:case1:eq3} follows from Corollary \ref{cor:RM++} (with the choice $\varepsilon = 1/3$, and using $4\sigma^2 \ge 1$ by isotropy of noise).

\medskip

Under the event  $\mathcal{A}_{\rho,t} \cap \mathcal{E}_{2,\rho,t}$, we have
\begin{align*}
  \det\left(\lambda^{-1}\sum_{s=0}^{t-1} z_s z_s^\top  + I_{d_x}  \right)^{1/d_x} & \le \frac{1}{\lambda}\sum_{s=0}^{t-1} \Vert z_s \Vert^2 + 1 \\
  & \le  \sum_{s=0}^{t-1} \Vert x_{s+1} \Vert^2  + \Vert \eta_s \Vert^2 +1 \\
  & \le   6 C_1\sigma^2 C_\circ^2 \mathcal{G}_\circ^2 d_x t^{1+2\gamma} \\
  & \le   6 C_1\sigma^2 C_\circ^2 \mathcal{G}_\circ^2 d_x t^{3\gamma_\star}
\end{align*}
where we assumed that $t\ge \log(e/\delta)$. Thus under the event $\mathcal{A}_{t,\delta} \cap \mathcal{E}_{1,t, \delta} \cap \mathcal{E}_{2,t,\delta}$, we have
\begin{align*}
  \left\Vert \left(\sum_{s=0}^{t-1} z_s z_s^\top +  I_d \right)^{-1/2} \left(\sum_{s=0}^{t-1} z_s \eta_s^\top \right)\right\Vert^2 \le C_2 \sigma^2 ( d_x \gamma_\star (\log( e \sigma C_\circ \mathcal{G}_\circ d_x t) +  \log(e/\delta)).
\end{align*}

\medskip

Therefore, under the event $\mathcal{A}_{t,\delta} \cap \mathcal{E}_{1,t, \delta} \cap \mathcal{E}_{2,t,\delta}$, in view of \eqref{se:case1:eq0}, it must hold that
\begin{align*}
  \lambda_{\min}\left( \sum_{s=0}^{t} x_s x_s^\top \right) \ge \frac{2t}{3} - 1 - C_2 \sigma^2 ( d_x \gamma_\star (\log( e \sigma C_\circ \mathcal{G}_\circ d_x t) +  \log(e/\delta)).
\end{align*}
Using Lemma \ref{lem:technical}, we can find an universal positive constant $c> 0$ such that when the following condition holds
\begin{equation}\label{eq:IIa:se:sc}
  t \ge c_2 \sigma^2 (d_x \gamma_\star \log(   e\sigma C_\circ \mathcal{G}_\circ d_x \gamma_\star ) + \log(e/\delta)),
\end{equation}
then it must also hold that
$$
\frac{2t}{3} - 1 - C_2 \sigma^2 ( d_x \gamma_\star (\log( e \sigma C_\circ \mathcal{G}_\circ d_x t) +  \log(e/\delta)) \ge \frac{t}{4}, \; \; t \ge \log(e/\delta), \;\; \text{and} \;\; t \ge c_1 \sigma^2 (d_x +  \log(e/\delta)).
$$
Hence, when condition \eqref{eq:IIa:se:sc} holds,
$$
\mathbb{P}\left( \lambda_{\min}\left( \sum_{s=0}^t x_s x_s^\top \right) \ge \frac{t}{4} \right) \ge 1 - \mathbb{P}( \mathcal{A}_{t,\delta}^c \cup \mathcal{E}_{1,t,\delta}^c \cup \mathcal{E}_{2,t,\delta}^c) \ge  1 - 3 \delta.
$$
Reparametrizing $\delta' = 3\delta$ gives the desired bound with different universal constants.
\end{proof}



\subsection{Proofs of the main ingredients}\label{proofs:tools}

\begin{proof}[Proof of Lemma \ref{lem:selb}]
Let $\lambda, \varepsilon > 0$, $t \ge 1 $, and $u \in S^{d-1}$. We have
\begin{align*}
  \sum_{s=1}^t  & \vert u^\top y_s \vert^2  = \sum_{s=1}^t \vert u^\top \xi_s \vert ^2  + 2  (u^\top \xi_s) (u^\top z_s) +  \vert u^\top z_s\vert^2 \\
  & \ge  \sum_{s=1}^t \vert u^\top \xi_s \vert ^2  + (1-\varepsilon) \vert u^\top z_s\vert^2  - \varepsilon \lambda +  \inf_{v\in \mathbb{R}^d}  \varepsilon \lambda v^\top v + \sum_{s=1}^t  2 (u^\top \xi_s) (v^\top z_s) +   \varepsilon \vert v^\top z_s\vert^2  \\
  & \ge \sum_{s=1}^t \vert u^\top \xi_s \vert ^2  + (1-\varepsilon) \vert u^\top z_s\vert^2 - \varepsilon\lambda -  \sup_{v\in \mathbb{R}^d}  - \varepsilon \lambda v^\top v - \sum_{s=1}^t  2 (u^\top \xi_s) (v^\top z_s) +  \varepsilon \vert v^\top z_s\vert^2
\end{align*}
where the first inequality follows by adding then substracting $\lambda u^\top u = \lambda$, and by taking the infimum over $v \in \mathbb{R}^d$. Next, we can easily verify that
\begin{align*}
  \sup_{v \in \mathbb{R}^d } \; - 2 \,v^\top  & \left(\sum_{s=1}^t  z_s (u^\top \xi_s)\right)  - \varepsilon v^\top \left( \sum_{s=1}^t z_s z_s^\top  + \lambda I_d \right)v \\
  &  =  \frac{1}{\varepsilon}  \left\Vert \left( \sum_{s=1}^t z_s z_s^\top + \lambda I_d \right)^{-1/2} \left( \sum_{s=1}^t z_s (u^\top \xi_s)\right ) \right\Vert^2.
\end{align*}
Thus if follows that
\begin{equation*}
    \sum_{s=1}^t \vert u^\top y_s\vert^2 \ge \sum_{s=1}^{t} \vert u^\top \xi_s \vert^2  + (1-\varepsilon) \vert u^\top z_s \vert^2 - \varepsilon \lambda u^\top u - \frac{1}{\varepsilon}\left\Vert \left(\sum_{s=1}^t z_s z_s^\top + \lambda I_d \right)^{-1/2} \left( \sum_{s=1}^t z_s \xi_s^\top\right) u \right\Vert^2
\end{equation*}
which implies that
$$
\sum_{s=1}^t y_s y_s^\top \succeq  \sum_{s=1}^t \xi_s \xi_s^\top + (1-\varepsilon) \sum_{s=1}^t z_s z_s^\top  - \frac{1}{\varepsilon} \left(\sum_{s=1}^tz_s \xi_s \right)^\top \left(\sum_{s=1}^t z_s z_s^\top + \lambda I_d \right)^{-1} \left(\sum_{s=1}^t z_s \xi_s^\top \right) - \varepsilon \lambda I_d
$$
\end{proof}

\begin{proof}[Proof of Proposition \ref{prop:RM++}]
  The vectors $ M_s \xi_s$ are zero-mean, $(\sigma m_{s})^2$-sub-gaussian, conditionally on $\mathcal{F}_{s-1}$:
  $$
  \E\left[\exp(\theta^\top M_s \xi_s)\vert \mathcal{F}_{s-1}\right] \le \exp\left(\frac{\Vert M_s \theta \Vert^2 \sigma^2}{2}\right) \le  \exp\left(\frac{\Vert \theta \Vert^2 (\sigma m_{s})^2}{2}\right).
  $$
  Thus, for all $x \in S^{d-1}$ and $s \ge 1$, we have, by Lemma \ref{lem:subexp},  that the random variable  $(x^\top M_s\xi_s)^2 - \E[(x^\top M_s\xi_s)^2 \vert \mathcal{F}_{s-1}]$ is $(4\sigma m_{s})^2$-sub-exponential conditionally on $\mathcal{F}_{s-1}$. Therefore, fixing $x\in S^{d-1}$, with a peeling argument, we immediately obtain for all $\vert \lambda \vert  < \frac{1}{(4\sigma)^2 \Vert m_{1:t}\Vert_{\infty}^2}$,
  $$
 \E\left[ \exp\left(\lambda  \sum_{s=1}^t(x^\top M_s\xi_s)^2 - \E\left[(x^\top M_s\xi_s)^2  \vert  \mathcal{F}_{s-1}\right]  \right)\right] \le \exp\left( \frac{\lambda^2(4\sigma)^4 \Vert m_{1:t} \Vert^4_4 }{2}\right).
  $$
  Now, Markov inequality yields for all $\rho > 0$,  and all $\vert \lambda \vert < \frac{1}{(4\sigma)^2 \Vert m_{1:t} \Vert_\infty^2 },$
$$
\mathbb{P}\left( \sum_{s=1}^t (x^\top M_s \xi_s)^2 - \E\left[ (x^\top M_s \xi_s)^2 \vert \mathcal{F}_{s-1} \right] > \rho   \right) \le \exp\left( \frac{1}{2} \left( \lambda^2 (4 \sigma)^4 \Vert m_{1:t} \Vert_4^4  - 2\lambda\rho \right) \right).
$$
Using $\lambda = \min \left( \frac{\rho}{(4\sigma m)^4 \Vert m_{1:t} \Vert_4^4}, \frac{1}{(4\sigma )^2 \Vert m_{1:t}\Vert_\infty^2} \right)$ gives
 \begin{align*}
   \mathbb{P} & \left( \sum_{s=1}^t (x^\top M_s \xi_s)^2 - \E\left[ (x^\top M_s \xi_s)^2 \vert \mathcal{F}_{s-1} \right] >  \rho   \right) \\
   &\qquad\qquad \le \exp\left( -\frac{1}{2} \min\left( \frac{\rho^2}{(4\sigma )^4 \Vert m_{1:t}\Vert_4^4}, \frac{\rho}{(4\sigma )^2 \Vert m_{1:t}\Vert_\infty^2}\right) \right).
 \end{align*}
  In a similar way we can establish
 \begin{align*}
   \mathbb{P} & \left( \sum_{s=1}^t  \E\left[ (x^\top M_s \xi_s)^2 \vert \mathcal{F}_{s-1} \right] -  (x^\top M_s \xi_s)^2 >  \rho   \right)\\
   &\qquad\qquad \le \exp\left( -\frac{1}{2} \min\left( \frac{\rho^2}{(4\sigma )^4 \Vert m_{1:t}\Vert_4^4}, \frac{\rho}{(4\sigma )^2 \Vert m_{1:t}\Vert_\infty^2}\right) \right).
\end{align*}
  Therefore, by union bound we obtain
\begin{align*}
   \mathbb{P} &\left( \left\vert \sum_{s=1}^t (x^\top M_s \xi_s)^2 - \E\left[ (x^\top M_s \xi_s)^2 \vert \mathcal{F}_{s-1} \right] \right\vert >  \rho  \right) \\
   & \qquad\qquad \le 2\exp\left( -\frac{1}{2} \min\left( \frac{\rho^2}{(4\sigma )^4 \Vert m_{1:t}\Vert_4^4}, \frac{\rho}{(4\sigma )^2 \Vert m_{1:t}\Vert_\infty^2}\right) \right).
 \end{align*}
  By isotropy of the noise vectors $(\xi_t)_{t\ge1}$, we have $\sum_{s=1}^t \E[(x^\top M_s \xi_s) \vert \mathcal{F}_{s-1}] =  \sum_{s=1}^t  x^\top M_s M_s^\top x$. Now, applying an $\epsilon$-net argument with $\epsilon = 1/4$, we get, by Lemma \ref{lem:netarg}, that
 \begin{align*}
  \mathbb{P} &\left( \left\Vert  \sum_{s=1}^t (M_s \xi_s) (M_s \xi_s)^\top - \sum_{s=1}^t M_s M_s^\top \right\Vert > \rho\right) \\
  &\qquad\qquad \le 2 \cdot 9^d \exp\left( -\frac{1}{2} \min\left( \frac{\rho^2}{8^2(\sigma )^4 \Vert m_{1:t}\Vert_4^4}, \frac{\rho}{8^2(\sigma )^2 \Vert m_{1:t}\Vert_\infty^2}\right) \right).
 \end{align*}
  Exploiting the fact that $\Vert m_{1:t}\Vert_4^4 \le \Vert m_{1:t}\Vert_\infty^2 \Vert m_{1:t}\Vert_2^2$, then reparametrizing, we obtain
  $$
  \mathbb{P}\left( \left\Vert  \sum_{s=1}^t (M_s \xi_s) (M_s \xi_s)^\top - \sum_{s=1}^t M_s M_s^\top \right\Vert > 8\sigma^2 \Vert m_{1:t}\Vert^2_2 \max\left(\sqrt{ \frac{2\rho + 5 d}{r_t^2} }, \frac{2\rho + 5 d}{r_t^2} \right)\right) \le 2 e^{-\rho},
  $$
  where $r_t = \Vert m_{1:t}\Vert /\Vert m_{1:t}\Vert_\infty $
  \end{proof}

  \begin{proof} [Proof of Corollary \ref{cor:RM++}]
    Applying Proposition \ref{prop:RM++},we get for all $\rho > 0$,
    $$
    \mathbb{P}\left( \frac{1}{t}\left\Vert  \sum_{s=1}^t (M_s \xi_s) (M_s \xi_s)^\top - \sum_{s=1}^t M_s M_s^\top \right\Vert > 8\sigma^2 m^2  \max\left(\sqrt{\frac{2\rho + 5 d}{t} }, \frac{2\rho + 5 d}{t} \right)\right) \le 2 e^{-\rho}
    $$
    where we see that $\Vert \frac{1}{t} M_s\Vert\le \frac{m}{t}$ almost surely for all $1 \le s\le t$. Then, we can verify that, under the condition $t \ge \min\left(\frac{8^2 (\sigma m)^2}{\varepsilon^2}, \frac{8(\sigma m)^2}{\varepsilon}\right) (5d + 2\rho)$,
    $$
    \mathbb{P}\left( \frac{1}{t}\left\Vert  \sum_{s=1}^t (M_s \xi_s) (M_s \xi_s)^\top - \sum_{s=1}^t M_s M_s^\top \right\Vert >  \varepsilon \right) \le 2 e^{-\rho}.
    $$
    This concludes the proof after noting that
    $$
    \left\Vert  \sum_{s=1}^t (M_s \xi_s) (M_s \xi_s)^\top - \sum_{s=1}^t M_s M_s^\top \right\Vert < \varepsilon t
    $$
    implies that
    $$
    \sum_{s=1}^t M_s M_s^\top - \varepsilon t I_d \preceq \sum_{s=1}^t (M_s \xi_s) (M_s \xi_s)^\top  \preceq \sum_{s=1}^t M_s M_s^\top + \varepsilon t I_d.
    $$
  \end{proof}

  \begin{proof}[Proof of Proposition \ref{prop:SNP++}] The proof follows immediately from Theorem 1 of \cite{abbasi2011improved} together with an $\epsilon$-net argument via Lemma \ref{lem:netarg}. Start by fixing $t \ge 1$. First, since $S_t(z, M \xi)$ is positive semidefinite matrix, we may express
  $\Vert S_t(z,M\xi) \Vert =  \sup_{x \in S^{d-1}}  x^\top S_t(z, M\xi) x$. Next, we note that the vectors $(x^\top M_s \xi_s)_{ 1\le s\le t}$ are zero-mean, $(\sigma \Vert m_{1:t} \Vert_\infty)^2$-sub-gaussian, conditionally on $\mathcal{F}_{t-1}$. Thus, Theorem 1 of \cite{abbasi2011improved} applies and we have for all $\rho> 0$,
  $$
  \mathbb{P}\left( x^\top S_t(z, M\xi) x > 2 \sigma^2 \Vert m_{1:t}\Vert^2_\infty \left(\frac{1}{2}   \log\det \left( V^{-1} \sum_{s=1}^t z_s z_s^\top  + I_d \right) + \rho   \right) \right) \le e^{-\rho}.
  $$
  Now, an $\epsilon$-net arguemnt with $\epsilon = 1/2$ (Lemma \ref{lem:netarg}) yields
  $$
  \mathbb{P}\left( \Vert S_t(z, M\xi) \Vert  >  \sigma^2 \Vert m_{1:t}\Vert^2_\infty \left( 2\log\det \left( V^{-1} \sum_{s=1}^t z_s z_s^\top  + I_d \right) + 7d  + 4\rho   \right) \right) \le e^{-\rho},
  $$
  which concludes the proof.
  \end{proof}

  \subsection{Additional lemmas}


  \begin{lemma}\label{lem:psd:ineq1}
    For all $K \in \mathbb{R}^{d_u \times d_x}$, there exists an orthogonal matrix $Q$ that depends on $K$, such that for all $\beta > 0$ and $\alpha > 0$, we have
    \begin{equation}
      \begin{bmatrix}
      I_{d_x} & K^\top \\
      K       & KK^\top + \beta I_{d_u}
    \end{bmatrix} \succeq  Q^\top \begin{bmatrix}
       I_{d_x-r} & O & O \\
      O &  \frac{\alpha\beta}{\Vert K \Vert^2 + \alpha\beta}I_{r} & O \\
      O & O & (1-\alpha)\beta I_{d_u}
    \end{bmatrix} Q,
  \end{equation}
    where $r = \mathrm{rank}(K) \le \min(d_x, d_u)$.
  \end{lemma}

  \begin{proof}[Proof of Lemma \ref{lem:psd:ineq1}]
  Let $r=\mathrm{rank}(K)$. Consider the singular value decomposition of $K$: $K = U^\top \Sigma V$, $U \in \mathbb{R}^{d_x \times d_x}$ is orthogonal, $V \in \mathbb{R}^{d_x \times d_x}$ is orthogonal, $\Sigma \in \mathbb{R}^{d_u \times d_x}$. Then
  \begin{align*}
    \begin{bmatrix}
    I_{d_x} & K^\top \\
    K       & KK^\top + \beta I_{d_u}
    \end{bmatrix} & =     \begin{bmatrix}
        V^\top V & V\Sigma^\top U  \\
        U^\top \Sigma V       & U\Sigma\Sigma^\top U + \beta U^\top U
      \end{bmatrix} \\
      & = Q^\top \begin{bmatrix}
      I_{d_x} & \Sigma^\top \\
      \Sigma & \Sigma \Sigma^\top + \beta I_{d_u}
      \end{bmatrix}  Q \\
      & \succeq Q^\top \begin{bmatrix}
      I_{d_x} - \Sigma^\top (\Sigma \Sigma^\top + \alpha \beta I_{d_u})^{-1} \Sigma & O \\
      O & (1-\alpha)\beta I_{d_u}
      \end{bmatrix}  Q \\
      & \succeq Q^\top \begin{bmatrix}
      I_{d_x -r} & O & O \\
      O & \frac{\alpha \beta}{\Vert K \Vert^2 + \alpha \beta}I_r & O \\
      O & O & (1-\alpha)\beta I_{d_u}
      \end{bmatrix}  Q
  \end{align*}
  where (i) we set $Q = \mathrm{diag}(V, U)$, (ii) we assumed that $U$ and $V$ are properly chosen so that the non zero elements of $\Sigma^\top (\Sigma \Sigma^\top + \alpha \beta I_{d_u})^{-1} \Sigma$ lie at the bottom right block of the resulting matrix, and (iii) we noted that $\lambda_{\max}(\Sigma^\top (\Sigma \Sigma^\top + \alpha \beta I_{d_u})^{-1} \Sigma) = \frac{\Vert K\Vert^2}{\Vert K\Vert^2 + \alpha\beta}$.
  \end{proof}

  \begin{lemma}\label{lem:selb2}
    Let $E, X$ be two tall matrices, then
    $$
    \begin{bmatrix}
      E^\top E & E^\top X \\
      X^\top E & X^\top X
    \end{bmatrix} \succeq \begin{bmatrix}
      \frac{\lambda}{\Vert X \Vert^2 + \lambda}E^\top E & O \\
      O & - \lambda I_d
  \end{bmatrix}.
    $$
  \end{lemma}

  \begin{proof}[Proof of Lemma \ref{lem:selb2}]
    We have
    \begin{align*}
      \begin{bmatrix}
        E^\top E & E^\top X \\
        X^\top E & X^\top X
      \end{bmatrix} & \succeq
      \begin{bmatrix}
        E^\top E - E^\top X \left(X^\top X + \lambda I_{d}\right)^{-1} X^\top E  & O \\
        O & - \lambda I_{d}
      \end{bmatrix} \\
      & \succeq
      \begin{bmatrix}
        E^\top \left(I_{p} - X \left(X^\top X + \lambda I_{d}\right)^{-1} X^\top \right) E  & O \\
        O & - \lambda I_{d}
      \end{bmatrix} \\
      & \succeq   \begin{bmatrix}
          \frac{\lambda}{\Vert X \Vert^2 + \lambda} E^\top E  & O \\
          O & - \lambda I_{d}
        \end{bmatrix},
    \end{align*}
where we used the fact that, using an SVD,
    $$
    I_{t d_x} - X^\top (X^\top X + \lambda I_{d_u})^{-1} X \succeq  \frac{\lambda}{\Vert X \Vert^2 + \lambda }I_{t d_x }.
    $$
  \end{proof}

\newpage
\section{ Polynomial Growth}\label{app:polygrowth}

In this section, we establish that under $\Alg$, the growth rate of $\sum_{s=0}^t \Vert x_s \Vert^2$ is not larger than $g(t)h(t)$ with high probability. This controlled growth rate is a consequence of the hysteresis switching mechanism of $\Alg$. The proof of the result relies on further results on the stability of time-varying linear systems presented in Appendix \ref{app:lds}.





\begin{proposition} \label{prop:caution} Under $\Alg$, and assuming that for all $t\ge1$, that $h(t) \ge 1$ and $g(t) \ge t$, we have, for all $\delta \in (0,1 )$,
\begin{equation}
 \mathbb{P}\left( \sum_{s=0}^t \Vert x_s \Vert^2 \ge C\sigma^2C_\circ^2\mathcal{G}_\circ^2  (d_x h(t)g(t) + \log(e/\delta)) \right) \le \delta,
\end{equation}
for some universal positive constant $C > 0$.
\end{proposition}

\begin{proof}
Let $t\ge 0$. We start by defining the following events
\begin{align*}
 \mathcal{E}_t & = \left \lbrace  \sum_{s=0}^t \Vert x_s \Vert^2 \le \sigma^2 d_x g(t) \right\rbrace, \\
 \forall i \in \lbrace 0, \dots, t-1 \rbrace, \quad  \mathcal{E}_i & = \left\lbrace \sum_{s=0}^i \Vert x_s \Vert^2 \le \sigma^2 d_x g(i) \right\rbrace \bigcap \left(\bigcap_{j=i+1}^t \left\lbrace \sum_{s=0}^j \Vert x_s \Vert^2 > \sigma^2 d_x g(j)\right\rbrace \right).
\end{align*}
Note that the events $\mathcal{E}_0, \dots, \mathcal{E}_t$ form a partition of the underlying probability space. Furthermore, for each $i \in \lbrace 0, \dots, t-1\rbrace$, if the event $\mathcal{E}_i$ holds then the following also holds
\begin{itemize}
 \item[\it(i)] $ \sum_{s=0}^{i} \Vert x_s \Vert^2 \le \sigma^2 d_x \;g(i)  $,
 \item[\it(ii)] $x_{s+1} = (A+BK_\circ)x_{s} + B \nu_s + \eta_s$ for all $i < s \le t$,
 \item[\it(iii)] $ \Vert x_{i+1} \Vert \le 2 C_\circ \sigma \sqrt{d_x \; h(i)g(i)} + \Vert B \nu_i + \eta_i \Vert$,
\end{itemize}
where we recall $C_\circ = \max(\Vert A\Vert, \Vert B \Vert, \Vert B K_\circ \Vert, \Vert K_\circ\Vert, 1)$. In view of the above, we will show that for all $i \in \lbrace 0, \dots, t\rbrace$, for all $\rho > 0$ we have
\begin{equation}\label{caution:eq1}
 \mathbb{P}\left(\mathcal{E}_i \cap \left\lbrace \sum_{s=0}^t \Vert x_s \Vert^2 > 44  \mathcal{G}_\circ^2 C_\circ^2 \sigma^2 (3 d_x h(t)g(t) + \rho)\right\rbrace\right) \le 3e^{-\rho},
\end{equation}
where $\mathcal{G}_\circ = \limsup_{t \to \infty} \sum_{s=0}^t \Vert (A +BK_\circ)^s\Vert$.
This in turn will allow us to conclude that for all $\rho > 0$,
\begin{align*}
& \mathbb{P}\left( \sum_{s=0}^t \Vert x_s \Vert^2 > 44  \mathcal{G}_\circ^2 C_\circ^2 \sigma^2 (3 d_x h(t)g(t) + \rho)\right)  \\
 &\qquad =  \sum_{i=0}^t \mathbb{P}\left(\mathcal{E}_i \cap \left\lbrace \sum_{s=0}^t \Vert x_s \Vert^2 > 44  \mathcal{G}_\circ^2 C_\circ^2 \sigma^2 (3 d_x h(t)g(t) + \rho) \right\rbrace\right)\\
 & \qquad \le 3 (t+1) e^{-\rho}.
\end{align*}
Finally, reparametrizing $\rho = \rho' + \log(t+1)$, gives for all $\rho'> 0$,
$$
\mathbb{P}\left( \sum_{s=0}^t \Vert x_s \Vert^2 > 44  \mathcal{G}_\circ^2 C_\circ^2 \sigma^2 (4 d_x h(t)g(t) + \rho') \right) \le 3e^{-\rho'}.
$$
Now, it remains to show that \eqref{caution:eq1} indeed holds. Let $i \in \lbrace 0, \dots, t \rbrace$. We have
$$
\sum_{s=0}^t \Vert x_s \Vert^2 = \sum_{s=0}^i \Vert x_s \Vert^2 + \sum_{s=i+1}^t \Vert x_s \Vert^2.
$$
If $(i)$ holds then it holds that $\sum_{s=0}^i \Vert x_s \Vert^2 \le \sigma^2 d_x g(i)$. Next we bound the sum $\sum_{s=i+1}^t \Vert x_s \Vert^2$ with high probability. To do so, consider the following dynamical system:
$$
\forall k \ge 0, \quad y_{k+1} = ( A + B K_\circ ) y_{k} + B \nu_{i+1+k} + \eta_{i+1+k} \quad \text{and} \quad y_{0} = x_{i+1},
$$
where we note that $B\nu_i + \eta_i$ is zero-mean, sub-gaussian with variance proxy $\Vert B\Vert^2 \sigma_i^2 + \sigma^2 \le 2C_\circ^2 \sigma^2 $. We further note that if $(ii)$ holds then $x_s = y_{s-i-1}$ for all $i < s \le t$. Thus, applying Lemma \ref{lem:lti} (see Appendix \ref{app:lds}), we obtain
\begin{equation}\label{caution:eq2}
\mathbb{P}\left( (ii) \text{ and } \sum_{s=i+1}^{t} \Vert x_s \Vert^2  >   2 \mathcal{G}_\circ^2\left( \Vert x_{i+1} \Vert^2 + 2C_\circ^2 \sigma^2  \cdot \left(2 d_x (t-i-1) + 3 \rho \right )\right)  \right) \le e^{-\rho}.
\end{equation}
Note that if $(iii)$ holds bounding $\Vert x_{i+1}\Vert$ amounts to bounding $\Vert B\nu_i + \eta_i \Vert$.  Standard concentration bounds lead to
$$
\mathbb{P}\left(  \Vert B \nu_i + \eta_i \Vert^2 > 8 C_\circ^2 \sigma^2 (2d_x + \rho)  \right) \le 2 e^{-\rho},
$$
which implies
\begin{equation}\label{caution:eq3}
\mathbb{P}\left( (iii) \text{ and } \Vert x_{i+1} \Vert^2 > 8 C_\circ^2 \sigma^2 d_x \; h(i)g(i) + 16 C_\circ^2 \sigma^2 (2d_x + \rho)      \right) \le 2e^{-\rho}.
\end{equation}
Combing the high probability bounds \eqref{caution:eq2} and \eqref{caution:eq3} using a union bound and considering the fact that $h(t) \ge 1$ and $g(t) \ge t$ for all $t \ge 1$ yields
$$
\mathbb{P}\left((ii) \text{ and } (iii) \text{ and } \sum_{s=i+1}^t \Vert x_{s} \Vert^2  > 44 \mathcal{G}_\circ^2 C_\circ^2 \sigma^2 (2 d_x h(t)g(t) + \rho)  \right) < 3e^{-\rho}.
$$
Finally putting everything together, after simplifications, gives
$$
\mathbb{P}\left((i) \text{ and } (ii) \text{ and } (iii) \text{ and } \sum_{s=0}^t \Vert x_s \Vert^2 > 44  \mathcal{G}_\circ^2 C_\circ^2 \sigma^2 (3 d_x h(t)g(t) + \rho)\right) \le 3 e^{-\rho}.
$$
Recalling that $(i)$, $(ii)$ and $(iii)$ are implied by the event $\mathcal{E}_i$, the desired high probablity bound \eqref{caution:eq1} follows immediately. This concludes the proof. Hiding the universal constants and reparametrizing by $\delta = 3e^{-\rho}$, we may simply write:
\begin{equation}
 \mathbb{P}\left( \sum_{s=0}^t \Vert x_s \Vert^2 \lesssim \sigma^2 \mathcal{G}_\circ^2 C_\circ^2 (d_x h(t)g(t) + \rho)  \right) \ge 1 - 3e^{-\rho}.
\end{equation}
\end{proof}

\newpage

\section{Control Theory}\label{app:control}

This section provides basic notions and results in control theory, and more importantly results quantifying the sensitivity of the solution of Riccati equation to small perturbations of the state transition and state-action transition matrices.

\subsection{Lyapunov equation}

\paragraph{The equation and its solution.} Let $M, N \in \mathbb{R}^{d\times d}$ where $N$ is a symmetric positive definite matrix. The  \emph{discrete Lyapunov equation} equation corresponding to the pair $(M,N)$ is defined as: $X = M^\top X M + N$ where $X$ is the matrix variable. If $\rho(M)<1$, then the discrete Lyapunov equation admits a unique positive definite matrix that we shall denote by $\mathcal{L}(M,N)$. In this case, the explicit value of the solution is: $\mathcal{L}(M, N) = \sum_{k=0}^\infty (M^k)^\top N (M^k)$.

\paragraph{Quantifying stability.} The fact that $\mathcal{L}(M,N)$ is well defined when $\rho(M)<1$ follows from Gelfand's formula which ensures that $\sup_{k\ge 0}\Vert M^k \Vert/\rho(M)^k < \infty$. For our purposes, we wish to quantify the transiant behaviour of $\Vert M^k \Vert$ in terms $\mathcal{L}(M,N)$. The following standard Lemma allows us to do so. We provide its proof for completeness.

\begin{lemma}[Stabiliy quantified] \label{lem:stability}
Let $M \in \mathbb{R}^{d\times d}$ be a stable matrix. Then, the corresponding Lyapunov equation to the pair $(M,I_d)$ admits a unique positive definite solution, $L = \mathcal{L}(M, I_d)$. Furthermore, we have
\begin{equation}
  \forall k \ge 0, \quad  \Vert M^k \Vert \le \left(1- \frac{1}{\Vert L \Vert}\right)^{k/2} \Vert L \Vert^{1/2}.
\end{equation}
\end{lemma}

\begin{proof} The existence of $\mathcal{L}(M, I_d)$ is a standard fact of Lyapunov theory. Using the Lyapunov quation, we write $ M = M^\top L M + I_d$, which we can rearrange as $M^\top L M = L^{1/2} (I_d - L^{-1})L^{1/2}$ using the fact that $L \succ 0$. In fact $L \succ I_d$ in view of the closed form of $\mathcal{L}(M,I_d)$, from which we  obtain
\begin{equation}
    M^\top L M \preceq \left(1 - \frac{1}{\Vert L \Vert}\right) L.
\end{equation}
Multiplying both sides of the above inequality by $(M^{k-1})^\top$ from left, by $M^{k-1}$ from right, and reiterating the above inequality yields  $(M^k)^\top M^k \preceq (M^\top)^k L M^k \preceq (1- \Vert L\Vert^{-1})^{k} L$. Thus, it immediately follows that $\Vert M^k \Vert^2 \le \left(  1- \Vert L \Vert^{-1}\right)^k \Vert L \Vert$, which concludes the proof.
\end{proof}

\subsection{Riccati equation}
\paragraph{The equation and its solution.} The \emph{Discrete Algebraic Riccati Equation} (DARE) refers to the matrix equation $P = A^\top P A - A^\top P B(R + B^\top B)^{-1}B^\top P A + Q
$ in the matrix variable $P$. When the pair $(A,B)$ is stabilizable\footnote{The pair of matrices $(A,B)$ is stabilizable if there exists a matrix $K \in \mathbb{R}^{d_x\times d_u}$ such that $\rho(A+BK)<1$.} and $Q \succ 0 $\footnote{Actually, the matrix $Q$ may be positive semi-definite but in this case, the pair $(C, A)$ needs to be detectable for the Riccati equation to admit a unique solution where $Q = C^\top C$ (see Theorem 8 by \cite{kuvcera1972discrete})},
the DARE admits a unique positive definite solution, that we shall denote $P(A,B)$.
Furthermore, the optimal gain matrix is $K(A,B) = - (R + B^\top P(A,B) B)^{-1}B^\top P(A,B) A$ and verifies $\rho(A+BK(A,B))< 1$ .

\paragraph{Relation to Lyapunov equation.} Note that it can be easily verified that $P(A,B) = \mathcal{L}(A+BK(A,B), Q + K(A,B)^\top R K(A,B))$. This motivates the definition $P(A,B,K) = \mathcal{L}(A+BK, Q + K^\top R K)$ whenever $\rho(A+BK)<1$.

\paragraph{Perturbation bounds.} To simplify the notations, we adopt the following shorthands: (i) for the true parameter $(A,B)$, we shall refer to $K_\star = K(A,B)$, $P_\star = P(A,B) = P(A,B, K_\star)$, and $P_\star(K) = P(A,B,K)$, (ii) for an alternate parameter $(A',B')$, we shall denote $K' = K(A',B')$, and $P' = P(A',B') = P(A',B',K')$. Now, a key observation behind existing regret analysis of the online LQR is to obtain the bound $\Vert P_\star(K') - P_\star \Vert \lesssim \max \left\lbrace \Vert A'-A \Vert^2, \Vert B' - B\Vert^2 \right\rbrace$.  The first step towards this bound is to note that $\Vert P_\star(K') - P_\star \Vert \lesssim \Vert K'- K_\star\Vert^2$.  This is ensured by the so-called cost difference lemma which is due to \cite{fazel2018global}.

\begin{lemma}\label{lem:diffcosts}
  For any matrix $K \in \mathbb{R}^{d_u\times d_x}$ such that $\rho(A + B K) < 1$. We have
  \begin{equation}
    P_\star(K) - P_\star = \mathcal{L}\left(A +B K, (K - K_\star)^\top  (R + B^\top P_\star B) (K - K_\star) \right).
  \end{equation}
\end{lemma}

Perturbation bounds were first rigorously estabilished by \cite{konstantinov1993perturbation}, showing locally the order of perturbation $\Vert P' - P_\star\Vert \lesssim  \max\left\lbrace \Vert A' - A_\star \Vert, \Vert B' - B_\star  \Vert \right \rbrace$ which in turn can be used to show that $\Vert K' - K_\star \Vert \lesssim \max\left\lbrace \Vert A' - A_\star \Vert, \Vert B' - B_\star\Vert \right \rbrace$. Combining that with Lemma \ref{lem:diffcosts} allows us to derive the desired inequality. Recently the constants were refined and made explicit in \cite{mania2019certainty} and \cite{simchowitz2020naive}. In this paper, we use the following result (see Theorem 5 and Proposition 6 in Appendix B of \cite{simchowitz2020naive} where we computed some constants for simplicity):

\begin{proposition}\label{prop:DARE:bounds} Assume that $R = I_{d_u}$ and $Q \succeq I_{d_x}$.   Let $(A, B)$ be a stailizable system. For all alternative pairs $(A',B') \in \mathbb{R}^{d_x\times d_x} \times \mathbb{R}^{d_x\times d_u}$, if
  \begin{equation}
    \max\left \lbrace \Vert A' - A \Vert, \Vert B' - B \Vert \right\rbrace  <  \frac{1}{54 \Vert P_{\star} \Vert^5},
  \end{equation}
  then the following holds:
  \begin{itemize}
    \item[(i)] the system $(A', B')$ is stabilizable, and consequently its corresponding DARE admits a unique positive definite solution $P' = P(A',B')$ with a corresponding gain matrix $K'=K(A',B')$;
    \item[(ii)] the optimal gain $K'$ corresponding to the system $(A',B')$ satisifies
    \begin{align}
        \Vert P' \Vert & \le 1.09 \Vert P_\star \Vert,  \\
        \Vert B (K' - K_\star)\Vert & \le 32 \Vert P_\star \Vert^{7/2} \max \left\lbrace \Vert A' - A \Vert, \Vert B' - B\Vert  \right\rbrace,  \\
        \Vert R^{1/2} (K' - K_\star)\Vert & \le 28 \Vert P_\star \Vert^{7/2} \max \left\lbrace \Vert A' - A \Vert, \Vert B' - B \Vert  \right\rbrace, \\
        \Vert P_\star(K') - P_\star \Vert & \le 142 \Vert P_\star \Vert^{8}  \max \left\lbrace \Vert A' - A \Vert^2, \Vert B' - B \Vert^2  \right\rbrace,  \\
        \Vert P_\star(K') \Vert & \le 1,05 \Vert P_\star \Vert.
    \end{align}
  \end{itemize}
\end{proposition}

\newpage

\section{Stability of Perturbed Linear Dynamical Systems} \label{app:lds}


This appendix presents an analysis of the stability of time-varying linear systems. The analysis is instrumental to understand the behavior of the system under $\Alg$, and in particular to show that the system does not grow faster than polynomially in time, see Appendix \ref{app:polygrowth}. We start this appendix by stating results about the stability of generic time-varying linear systems. We then explain how to apply these results to our system under $\Alg$. The proofs are postponed to the end of the appendix.

\subsection{Generic time-varying linear systems and their stability}

Consider the stochastic process $(y_t)_{t\ge 0}$ taking values in $\mathbb{R}^d$, such that:
\begin{equation}\label{eq:ld}
  \forall t \ge 0, \quad y_{t+1} = M_t y_t + \xi_t.
\end{equation}
The initial state $y_0$ may be random,  $(\xi_t)_{t\ge 0}$ is a sequence of zero-mean, $\sigma^2$-sub-gaussian random vectors taking values in $\mathbb{R}^d$ that are independent of $y_0$, and finally, $(M_t)_{t\ge 0}$ is a sequence of matrices taking values in $\mathbb{R}^{d\times d}$ that are possibly random, or even adversarially chosen. By {\it stability} of the process $(y_t)_{t\ge 0}$, we mean that the growth of $\sum_{s=0}^t \Vert y_s\Vert^2$  as $t$ increases is no more than $t$ with high probability. We will make this definition precise in the upcoming lemmas.

\medskip

We investigate two classes of systems:
\begin{itemize}
\item[(i)] Time-invariant systems where $M_t=M$ for all $t\ge 0$;
\item[(ii)] Adversarially time-varying systems where the matrices $M_t$ for $t\ge 0$ are random and may be adversarially selected in a small neighborhood of a stable deterministic matrix $M$ with $\rho(M)<1$.
\end{itemize}

\paragraph{(i) Time-invariant systems.} For time invariant systems, we present Lemma \ref{lem:lti} whose proof relies on Hanson-Wright inequality, and on the specific structure of some truncated block Toeplitz matrices that arise naturally in the analysis. The specific structure of these matrices stems from the causal nature of the dynamical system, and manifests itself in the following constant:
\begin{equation}\label{eq:G_M}
  \mathcal{G}_M = \limsup_{t\to \infty} \sum_{s=0}^t \Vert M^s \Vert,
\end{equation}
which is well defined as long as $\rho(M)<1$.

\begin{lemma}[Time-invariant systems]\label{lem:lti}
Consider the linear system $y_{t+1}=My_t+\xi_t$ as defined in \eqref{eq:ld}. Assume that $M$ is deterministic and satisfies $\rho(M)<1$. Then:
  \begin{equation}
    \forall t \ge 1, \forall \rho > 0, \quad  \mathbb{P}\left( \sum_{s=0}^t \left\Vert y_s\right\Vert^2  \le  2\mathcal{G}_M^2 \left( \Vert y_0 \Vert^2 +  \sigma^2 \left(d t + 2\sqrt{d t\rho} + 2 \rho \right )  \right)\right) \ge 1-  e^{-\rho}.
  \end{equation}
\end{lemma}
The proof of Lemma \ref{lem:lti} is presented in Appendix \ref{sec:lds:proofs}.

\paragraph{(ii) Time-varying systems.} For time-varying systems, the stability results are presented in Lemma \ref{lem:lds:perturbed}. Again they rely on Hanson-Wright inequality, and the specific causal structure of the system. $M_t$ varies around the matrix $M$ and we assume that $\|M_t-M\|\le \varepsilon$ for all $t\ge 0$. The analysis requires us to establish properties of {\it perturbed} truncated block Toeplitz matrices. In this analysis, the constant $\mathcal{G}_M$ is replaced by:
\begin{equation}\label{eq:G_M(epsilon)}
  \mathcal{G}_M(\varepsilon) = \sup \left\lbrace \limsup_{t \to \infty} (1 + \sum_{s=0}^t \big\Vert \prod_{k=0}^s  N_k \big\Vert): (N_t)_{t\ge0} \quad\text{where}\quad  \sup_{t\ge 0} \Vert N_t - M \Vert \le \varepsilon \right\rbrace.
\end{equation}
In the above definition, the supremum is taken over all possible deterministic sequence of matrices $(N_t)_{t\ge0}$. Clearly $\mathcal{G}_M(0) = \mathcal{G}_M$. However, it is not  obvious to determine under which condition on $\varepsilon$, the constant $\mathcal{G}_M(\varepsilon) < \infty$. Lemma \ref{lem:perturbed} provides an answer to this issue. As it turns out, we may express this condition in terms of the solution of the discrete Lyapunov equation corresponding to the pair $(M,I_d)$, which we denote as
\begin{equation}
  L = \mathcal{L}(M,I_d).
\end{equation}

\begin{lemma}[Stability under perturbation]\label{lem:perturbed}
Let $M \in \mathbb{R}^{d\times d}$ such that $\rho(M) < 1$, and $\varepsilon > 0$. For any sequence of matrices $ (\Delta_t)_{t\ge 0}) $ taking values in $\mathbb{R}^{d\times d}$, and such that $\sup_{t\ge 0}\Vert \Delta_t \Vert < \varepsilon$, it holds that
\begin{equation}
\forall k \ge 0: \qquad \left\Vert \prod_{i=1}^k(M + \Delta_i) \right\Vert \le \Vert L \Vert^{1/2} \left(\Vert L\Vert^{1/2} \varepsilon  + \left(1 - \frac{1}{\Vert L\Vert} \right)^{1/2}\right)^k,
\end{equation}
where $L = \mathcal{L}(M,I_d)$ is the positive definite solution to the Lyapunov equation corresponding to the pair $(M,I_d)$. Furthermore, for all $x \in (0,1)$, if $\varepsilon <\frac{x}{2 \Vert L\Vert^{3/2}}$, then $\mathcal{G}_M(\varepsilon) \le \frac{ 2\Vert L \Vert^{3/2}}{(1-x)}$, and $\rho(M+\Delta_t) < 1$ for all $t \ge 0$.
\end{lemma}


The proof of the previous lemma follows the same steps as those of Lemma 5 in \cite{mania2019certainty}, or Lemma 4 in \cite{dean2019sample}) with the slight difference that $(\Delta_t)_{t\ge1}$ are not fixed  and using the constants that we get from Lemma \ref{lem:stability}. We omit the proof here. We are now ready to state the result on the stability of the system with perturbed dynamics.

\begin{lemma}[Time-varying systems] \label{lem:lds:perturbed} Consider a linear dynamical system $(y_t)_{t \ge 0}$ described as in \eqref{eq:ld}. Furthermore, assume that the sequence of matrices $(M_t)_{t\ge 0}$ is such that there exists a stable matrix $M \in \mathbb{R}^{d\times d}$, and a postive $\varepsilon < \frac{1}{2 \Vert \mathcal{L}(M,I)\Vert^{3/2}}$ with $\sup_{t\ge 0} \Vert M_t - M \Vert \le \varepsilon$. Then, for any non increasing scalar sequence $(a_t)_{t\ge 0}$, we have, for all $t \ge 1$, and for all $\rho > 0$,
  \begin{equation}
   \mathbb{P}\left( \sum_{s=0}^t a_t^2 \Vert y_t \Vert^2  >  2\mathcal{G}_M(\varepsilon)^2 \left(\Vert a_{0:t} \Vert^2_\infty \Vert y_0 \Vert^2 + \sigma^2\left(d \Vert a_{0:t-1}\Vert_2^2 + 2\sqrt{d\Vert a_{0:t-1}\Vert_2^2\rho}+ 2 \rho\right) \right) \right) \le e^{-\rho}
  \end{equation}
\end{lemma}
The proof of Lemma \ref{lem:lds:perturbed} is presented in Appendix \ref{sec:lds:proofs}. It is worth mentioning that Lemma \ref{lem:lti} is in fact a consequence of Lemma \ref{lem:lds:perturbed}, but we keep the two lemmas as well as their proofs separate for clarity of the exposition.

\subsection{Application to $\Alg$}

The results presented above will often be used in the analysis of the behaviour of the states $(x_t)_{t\ge 0}$ under $\Alg$ when the controller used is fixed over a period of time, say between $s$ and $t$. In such cases, we may express the dynamics as follows:
$$
\forall s\le \tau < t:\qquad   x_{\tau+1} = (A + B \widetilde{K}_\tau) x_\tau + \xi_\tau \qquad \text{with} \qquad x_0 = 0,
$$
where either (i) $\widetilde{K}_\tau = K_\circ$  for all $ s\le \tau < t$ or (ii) $\widetilde{K}_\tau = K_\tau$ for all $ s\le \tau < t$. We now specify noise sequence $(\xi_\tau)_{s \le \tau < t}$ in the three envisioned scenarios.

\begin{itemize}
  \item In scenario I, under $\Alg$, we have $x_{t+1} = (A + B\widetilde{K}_t)x_t + B \nu_t + \eta_t$ for all $t\ge 1$,  with $x_0 = 0$. We may write $\xi_t = B \nu_t + \eta_t$ for all $t\ge 0$, thus $(\xi_t)_{t\ge 1}$ is a sequence of independent, zero-mean, sub-gaussian random vectors where for each $t \ge 1$, $\xi_t$ has variance proxy $\Vert B\Vert^2 \sigma_t^2 + \sigma^2$ where we recall that $\sigma_t \le \sigma$ for all $t \ge 1$. Hence we may simply use $\widetilde{\sigma}^2 = \sigma^2( \Vert B \Vert^2 + 1)$.
  \item In scenario II - $B$ known,  under $\Alg$, we have $x_{t+1} = (A + B \widetilde{K}_t) x_t + \eta_t$ for all $t\ge 1$, with $x_0 = 0$. We may write $\xi_t = \eta_t$ for all $t\ge 0$, thus having $(\xi_t)_{t\ge 1}$ is a sequence of i.i.d. zero-mean, $\sigma^2$-sub-gaussian random vectors.
  \item In scenario II - $A$ known, under $\Alg$, $x_{t+1} = (A + B \widetilde{K}_t) x_t + 1_{\lbrace \widetilde{K}_t = K_\circ \rbrace} B \zeta_t + \eta_t$ for all $t\ge 0$, with $x_0 = 0$, ruling out the pathological case that $K_t = K_\circ$ which may only happen with probability zero. Hence $(\xi_\tau)_{s \le \tau < t}$ coincide with $(\eta_t)_{ s\le \tau < t}$ provided $\widetilde{K}_\tau = K_\tau$ for all $s\le \tau < t$ and we can use $\widetilde{\sigma}^2 = \sigma^2$. Similarly $(\xi_\tau)_{s \le \tau < t}$ coincide with $(B\eta_t +\eta_t)_{ s\le \tau < t}$ provided $\widetilde{K}_\tau = K_\circ$ for all $s\le \tau < t$ and we may use $\widetilde{\sigma}^2 = \sigma^2(\Vert B\Vert^2 + 1)$.
  \end{itemize}
  Note that in all cases $\widetilde{\sigma}^2 \le \sigma^2(\Vert B\Vert^2 + 1 )$

For the case where $\Alg$ uses the stabilizing controller between rounds $s$ and $t$, we apply Lemma \ref{lem:lti} and obtain:

\begin{proposition}\label{prop:stable:K_circ}
Refer to the condition $ \lbrace \forall s\le \tau < t: \ \widetilde{K}_t = K_\circ \rbrace$ as $(SC)$. Then, for all $\rho > 0$,
\begin{equation}
  \mathbb{P}\left( \sum_{\tau = s}^t \Vert x_\tau \Vert^2 > 2 \mathcal{G}_\circ^2 (\Vert x_s \Vert^2 + \sigma^2 (2d_x t + 3 \rho)),\text{ and condition } (SC) \text{ holds}   \right) \le e^{-\rho}.
\end{equation}
\end{proposition}

\begin{proof}[Proof]
  Consider the events
  \begin{align*}
    {E}_{1} &= \left\lbrace \sum_{\tau = s}^t \Vert x_\tau \Vert^2 > 2 \mathcal{G}_\circ^2 (\Vert x_s \Vert^2 + \sigma^2 (2d_x t + 3 \rho)),\text{ and condition } (SC) \text{ holds} \right\rbrace \\
    {E}_{2} &= \left\lbrace \sum_{\tau = 0}^{t-s} \Vert y_\tau \Vert^2 > 2 \mathcal{G}_\circ^2 (\Vert x_s \Vert^2 + \sigma^2 (2d_x t + 3 \rho)) \right\rbrace,
  \end{align*}
  where the dynamical system $(y_t)_{t\ge 0}$ is defined as
  $$
  \forall \tau \ge 0, \quad y_{\tau+1} = M y_\tau + B \zeta_{\tau+s} + \eta_{\tau+s}, \quad  \quad y_0  = x_{s},  \quad \text{and} \quad  M = A+BK_\circ.
  $$
Observe that ${E}_1 \subseteq {E}_2$. We apply Lemma \ref{lem:lti} to conclude
\end{proof}

\medskip

For the case where $\Alg$ uses the certainty equivalence controller between $s$ and $t$, we are mainly interested in scenarios when $\max_{s\le \tau < t} \Vert K_\tau - K_\star \Vert \le \frac{1}{4 \Vert P_\star \Vert^2}$. In particular, we note that for $\varepsilon < \frac{1}{4 \Vert P_\star \Vert^2}$, by Lemma \ref{lem:perturbed}, we have
\begin{equation}
  \mathcal{G}_\star(\varepsilon) = \mathcal{G}_{A+BK_\star}(\varepsilon) \le 4\Vert P_\star\Vert^{3/2}.
\end{equation}
Now, we may state and prove the following result.

\begin{proposition}\label{prop:stable:K_t}
Refer to the condition $ \left\lbrace \forall s\le \tau < t: \quad \widetilde{K}_t = K_t  \text{ and }  \Vert B(\widetilde{K}_\tau - K_\star)\Vert \le \frac{1}{4\Vert P_\star \Vert^{3/2}}  \right\rbrace$ by $(CE)$. Then, for all $\rho > 0$,
  \begin{align}
    \mathbb{P}\bigg( \sum_{\tau = s}^t a_\tau^2  \Vert x_\tau \Vert^2 > 8 \Vert P_\star \Vert^{3/2} (\Vert a_{s:t}\Vert^2_\infty \Vert x_s \Vert^2 &+ \widetilde{\sigma}^2 (2d_x \Vert a_{s:t-1}\Vert^2_2 + 3 \rho)),\nonumber\\
    &\text{ and condition } (CE) \text{ holds.}   \bigg) \le e^{-\rho}.
  \end{align}
\end{proposition}

 Proposition \ref{prop:stable:K_t}, it is an immediate consequence of Lemma \ref{lem:lds:perturbed}.

\begin{proof}[Proof]
  Consider the events
  \begin{align*}
    {E}_{1} &= \left\lbrace \sum_{\tau = s}^t a_\tau^2  \Vert x_\tau \Vert^2 > 8 \Vert P_\star \Vert^{3/2} (\Vert a_{s:t}\Vert^2_\infty \Vert x_s \Vert^2 + \widetilde{\sigma}^2 (2d_x \Vert a_{s:t-1}\Vert^2_2 + 3 \rho)),\text{ and } (CE) \text{ holds} \right\rbrace \\
    {E}_{2} &= \left\lbrace \sum_{\tau = 0}^{t-s} \Vert y_\tau \Vert^2 > 8 \Vert P_\star \Vert^{3/2} (\Vert x_s \Vert^2 + \sigma^2 (2d_x t + 3 \rho)) \right\rbrace,
  \end{align*}
  where the dynamical system $(y_t)_{t\ge 0}$ is defined as
  \begin{align*}
  \forall \tau \ge 0, \quad y_{\tau+1} = M_\tau y_\tau + B \zeta_{\tau+s} + \eta_{\tau+s}, \quad  \quad y_0  = x_{s}
  \end{align*}
  with
  \begin{align*}
   \quad  M_{\tau} = A+B\left((\widetilde{K}_{\tau +s } - K_\star)1_{\left\lbrace \Vert B(\widetilde{K}_{\tau+s} - K_\star) \Vert \le \frac{1}{4 \Vert P_\star  \Vert^{3/2}} \right\rbrace} + K_\star\right).
  \end{align*}
Observe that ${E}_1 \subseteq {E}_2$. We apply Lemma \ref{lem:lds:perturbed} to conclude.
\end{proof}

\subsection{Proofs}\label{sec:lds:proofs}

To establish Lemmas \ref{lem:lti} and \ref{lem:lds:perturbed}, we first give intermediate results about block Toeplitz matrices.

\subsubsection{Block Toeplitz matrices}

\begin{lemma}[Norms of truncated block Toeplitz matrices]\label{lem:toeplitz}
  Let $M \in \mathbb{R}^{d\times d}$. Consider the following matrices
  \begin{equation}
    \Gamma(M,t) = \begin{bmatrix} I_d \\
    M \\
    M^2 \\
    \vdots \\
    M^t
    \end{bmatrix}
    \qquad \text{and} \qquad \mathcal{T}(M, t) = \begin{bmatrix}
      I_d    &          &        &         &     \\
      M      & I_d      &        &   O     &     \\
      M^2    & M        & I_d    &         &     \\
      \vdots &  \ddots  & \ddots & \ddots  &     \\
      M^t    &   \dots  &  M^2   &   M     & I_d
  \end{bmatrix}.
 \end{equation}
 If $\rho(M)<1$, then $\mathcal{G}_M$ given in \eqref{eq:G_M} is well defined and for all $t\ge 0$, the following holds:
 \begin{itemize}
   \item[\it (i)] $\Vert \Gamma(M,t) \Vert \le \mathcal{G}_M$,
   \item[\it (ii)] $\Vert \mathcal{T}(M,t) \Vert \le \mathcal{G}_M$,
   \item[\it (iii)] $\Vert \mathcal{T}(M,t) \Vert_F \le  \mathcal{G}_M \sqrt{d (t+1)}$.
 \end{itemize}
\end{lemma}
\begin{proof}
First, $\mathcal{G}_M<\infty$ whenever $\rho(M)< 1$. Now, it is clear that $\Vert \Gamma(M, t) \Vert \le \sum_{s=0}^t \Vert M^{s} \Vert \le \mathcal{G}_M$. Next, we can also immediately bound the norm of the truncated block Toeplitz matrix  $ \Vert \mathcal{T}(M, t) \Vert \le \sum_{s=0}^t \Vert M^{s} \Vert \le \mathcal{G}_M$. Finally, we have $\Vert \mathcal{T}(M, t) \Vert_F \le \sqrt{d(t+1)} \Vert  \mathcal{T}(M, t) \Vert \le \sqrt{d(t+1)}  \mathcal{G}_M$, which concludes the proof.
\end{proof}

\begin{lemma}[Norms of perturbed truncated block Toeplitz matrices]\label{lem:perturbed toeplitz}
  Let $M \in \mathbb{R}^{d \times d}$, and $(\Delta_t)_{t\ge0}$ be a sequence of matrices in $\mathbb{R}^{d\times d}$ such that $\sup_{t\ge 0} \Vert \Delta_t \Vert < \varepsilon$ for some $\varepsilon > 0$. Consider the following matrices
  \begin{align*}
    \Gamma(M, \Delta_{0:t}) = \begin{bmatrix} I_{d} \\
      m_{0, 0} \\
      m_{1, 0} \\
      \vdots \\
      m_{t-1,0}
    \end{bmatrix}, \ \text{ and }
   \end{align*}
    \begin{align*}
    \mathcal{T}(M, \Delta_{0:t})  = \begin{bmatrix}
      \kappa_{0,1}I_d    &          &        &         &     \\
      \kappa_{1,1}m_{1,1}      & \kappa_{1,2}I_d      &        &   O     &     \\
      \kappa_{2,1}m_{2,1}    & \kappa_{2,2}m_{2,2}        & \kappa_{2,3}I_d    &         &     \\
      \vdots &  \ddots  & \ddots & \ddots  &     \\
      \kappa_{t,1}m_{t,1}    &   \dots  &  \kappa_{t,t-1}m_{t,t-1}   &   \kappa_{t,t}m_{t, t}     & \kappa_{t,t+1}I_d
    \end{bmatrix},
  \end{align*}
  where for all $s \le t$, $m_{t, s} = \prod_{k=s}^{t} (M + \Delta_k)$ and for all $s \le t+1$, $\kappa_{t,s} \le 1$. If $\rho(M)<1$ and $\varepsilon < \frac{1}{2\Vert L\Vert^{3/2}}$ where $L$ denotes the solution of the Lyapunov equation corresponding to the pair $(M,I_d)$. Then, $\mathcal{G}_M(\varepsilon)$, as defined \eqref{eq:G_M(epsilon)}, is finite and for all $t\ge 0$ the following holds:
 \begin{itemize}
   \item[\it (i)] $\Vert \Gamma(M,\Delta_{0:t}) \Vert \le \mathcal{G}_M(\varepsilon)$,
   \item[\it (ii)] $\Vert \mathcal{T}(M,\Delta_{0:t}) \Vert \le \mathcal{G}_M(\varepsilon)$.
 \end{itemize}
\end{lemma}
\begin{proof}
We start by noting that Lemma \ref{lem:perturbed} immediately applies since $\rho(M)< 1$, and $\varepsilon < \frac{1}{2 \Vert L \Vert^{3/2}}$. This ensures that $\mathcal{G}_M(\varepsilon) < \infty$. Next, we have
\begin{align*}
  \Vert \Gamma(M, \Delta_{0:t}) \Vert \le 1 + \sum_{s=0}^{t-1} \left\Vert m_{s,0} \right \Vert \le 1 + \sum_{s=0}^{t-1} \left\Vert \prod_{k=0}^{s}  (M + \Delta_k)\right \Vert
\end{align*}
and
\begin{align*}
  \Vert \mathcal{T}(M, \Delta_{0:t}) \Vert & \le  \max_{0\le s \le t}\kappa_{s, s+1} + \sum_{s=0}^{t-1} \max_{1\le k \le t-s} \kappa_{s,t} \max_{1\le k \le t-s} \Vert m_{k+s, k} \Vert \\
  & \le   1 + \sum_{s=0}^{t-1} \max_{1\le k \le t-s} \left\Vert \prod_{i = k}^{k+s} (M + \Delta_i) \right\Vert \\
  & \le \mathcal{G}_M(\varepsilon)
\end{align*}
where in the second inequality used the fact that $\kappa_{t,s} \le 1$ for all $t,s$. This concludes the proof.
\end{proof}

\subsubsection{Proof of Lemma \ref{lem:lti}}
\begin{proof}\label{lem:lti:proof}
  First, let us note that for all $s \ge 1$, we may expand the dynamics and write
   \begin{align*}
     y_{s} = M^{s} y_0 + \sum_{k=0}^{s-1} M^{s-1-k} \xi_k,
   \end{align*}
  We deduce that
  \begin{align}\label{lti:eq1}
    \sum_{s=0}^t \Vert y_s \Vert^2 & \le 2 \sum_{s=0}^t \left\Vert M^s y_0 \right\Vert^2 + 2 \sum_{s=0}^t \left\Vert \sum_{k=0}^{s-1} M^{s-k-1} \varepsilon_{k}\right\Vert^2.
  \end{align}
  Introducing the following matrices
  \begin{align*}
    \Gamma(M,t)  = \begin{bmatrix} I \\
     M \\
      M^{2} \\
      \dots \\
      M^t
  \end{bmatrix}, \quad
    \mathcal{T}(M, t) = \begin{bmatrix}
      I_d    &          &        &         &     \\
      M      & I_d      &        &   O     &     \\
      M^2    & M        & I_d    &         &     \\
      \vdots &  \ddots  & \ddots & \ddots  &     \\
      M^t    &   \dots  &  M^2   &   M     & I_d
  \end{bmatrix} \quad \text{and} \quad  \xi_{0:t} = \begin{bmatrix}
    \xi_{0} \\
    \xi_{1} \\
    \xi_{2} \\
    \vdots \\
    \xi_{t}
  \end{bmatrix},
  \end{align*}
  we may rewrite \eqref{lti:eq1} in the following convenient form
  \begin{equation}\label{lti:eq2}
    \sum_{s=0}^t \Vert y_s \Vert^2  \le  2 \Vert \Gamma(M,t) y_0\Vert^2 + 2 \left\Vert \mathcal{T}(M, t-1) \xi_{0:t-1} \right\Vert^2.
  \end{equation}
   We can now upper bound the term $\left\Vert \mathcal{T}(M, t-1) \xi_{0:t-1} \right\Vert^2$ with high probability using Hanson-Wright inequality (See Proposition \ref{prop:HW}). To this aim, we need upper bounds of $\Vert \mathcal{T}(M, t-1)\Vert_F^2$, $\Vert \mathcal{T}(M, t-1)^\top \mathcal{T}(M, t-1)\Vert_F$, and $\Vert \mathcal{T}(M, t-1)^\top \mathcal{T}(M, t-1) \Vert$. By a direct application of Lemma \ref{lem:toeplitz}, we obtain
   \begin{align*}
     \Vert \mathcal{T}(M, t-1)\Vert_F^2 & \le \mathcal{G}_M^2 d t \\
     \Vert \mathcal{T}(M, t-1)^\top \mathcal{T}(M, t-1)\Vert_F & \le \Vert \mathcal{T}(M, t-1)\Vert \, \Vert \mathcal{T}(M, t-1)\Vert_F \le \mathcal{G}_M^2 \sqrt{d t} \\
     \Vert \mathcal{T}(M, t-1)^\top \mathcal{T}(M, t-1)\Vert & \le \mathcal{G}_M^2
   \end{align*}
   Applying Hanson-Wright inequality yields
   \begin{equation}\label{lti:eq3}
    \forall \rho > 0, \quad  \mathbb{P}\left( \left\Vert \mathcal{T}(M, t-1) \xi_{0:t-1} \right\Vert^2  \le    \sigma^2 \mathcal{G}_M^2 \left(d t + 2\sqrt{d t\rho} + 2 \rho \right )  \right) \ge 1-  e^{-\rho}.
  \end{equation}
   Again by a direct application of Lemma \ref{lem:toeplitz}, we have $\Vert \Gamma(M,t) \Vert< \mathcal{G}_M$, which leads to
   \begin{equation}\label{lti:eq4}
   \Vert \Gamma(M, t) y_0 \Vert^2 \le \mathcal{G}_M^2 \Vert y_0 \Vert^2.
   \end{equation}
   Finally, considering the inequality \eqref{lti:eq2}, the high probability upper bound \eqref{lti:eq3} and the deterministic upper bound \eqref{lti:eq4}, we obtain
   \begin{equation*}
     \forall \rho > 0, \quad  \mathbb{P}\left( \sum_{s=0}^t \left\Vert y_s\right\Vert^2  \le  2\mathcal{G}_M^2 \left( \Vert y_0 \Vert^2 +  \sigma^2 \left(d t + 2\sqrt{d t\rho} + 2 \rho \right )  \right)\right) \ge 1-  e^{-\rho}.
   \end{equation*}
\end{proof}

\subsubsection{Proof of Lemma \ref{lem:lds:perturbed}}\label{proof:lds:perturbed}

\begin{proof}
Denote for all $s\ge 1$, $\Delta_s = M_s - M$. Now, for all $s\ge 1$, we have
  $$
  y_{s} = \left(\prod_{k=0}^{s-1} (M + \Delta_k) \right) y_0 + \sum_{k=0}^{s-1} \left(\prod_{i=k+1}^{s-1} (M + \Delta_i)   \right) \xi_k.
  $$
Hence
  \begin{equation}\label{tvls:eq1}
  \sum_{s=0}^t a_s^2  \Vert y_s \Vert^2 \le 2 \Vert a_{0:t} \Vert_\infty^2 \sum_{s=0}^t \left\Vert \left(\prod_{k=0}^{s-1} (M + \Delta_k) \right) y_0 \right\Vert^2 + 2 \sum_{s=0}^t   \left\Vert \sum_{k=0}^{s-1} \frac{a_s}{a_k}\left(\prod_{i=k+1}^{s-1} (M + \Delta_i)   \right) a_k \xi_k \right\Vert^2.
  \end{equation}
  Introducing the following matrices
  \begin{align*}
    \Gamma(M, \Delta_{0:t}) & = \begin{bmatrix} I_{d} \\
      m_{0, 0} \\
      m_{1, 0} \\
      \vdots \\
      m_{t,0}
    \end{bmatrix}, \ \ \
  \mathcal{T}(M, \Delta_{0:t})  = \begin{bmatrix}
    \kappa_{0,1}I_d    &          &        &         &     \\
    \kappa_{1,1}m_{1,1}      & \kappa_{1,2}I_d      &        &   O     &     \\
    \kappa_{2,1}m_{2,1}    & \kappa_{2,2}m_{2,2}        & \kappa_{2,3}I_d    &         &     \\
    \vdots &  \ddots  & \ddots & \ddots  &     \\
    \kappa_{t,1}m_{t,1}    &   \dots  &  \kappa_{t,t-1}m_{t,t-1}   &   \kappa_{t,t}m_{t, t}     & \kappa_{t,t+1}I_d
  \end{bmatrix}, \\
  \text{and} \quad  \xi_{0:t} & = \begin{bmatrix}
  a_0\xi_{0} \\
  a_1\xi_{1} \\
  a_2\xi_{2} \\
  \vdots \\
  a_t \xi_{t}
\end{bmatrix} \quad \text{where}  \ \  \forall s \le t, \ \   m_{t, s} = \prod_{k=s}^{t} (M + \Delta_k) \ \ \text{ and } \ \  \forall s \le t+1, \ \  \kappa_{t,s} = \frac{a_t}{a_s}.
  \end{align*}
  We may rewrite \eqref{tvls:eq1} in the following convenient form
  \begin{align*}
    \sum_{s=0}^t \Vert y_s \Vert^2 & \le 2 \Vert \Gamma(M,\Delta_{0:{t-1}}) y_0\Vert^2 + 2 \Vert \mathcal{T}(M, \Delta_{0:t-1}) \xi_{0:t-1} \Vert^2 \\
    & \le  2 \Vert \Gamma(M,\Delta_{0:{t-1}})\Vert^2 \Vert y_0 \Vert^2 + 2 \Vert \mathcal{T}(M, \Delta_{0:t-1})\Vert^2  \Vert\xi_{0:t-1} \Vert^2 \\
    & \le  2 \Vert \Gamma(M,\Delta_{0:{t-1}})\Vert^2 \Vert y_0 \Vert^2 + 2 \Vert \mathcal{T}(M, \Delta_{0:t-1})\Vert^2  \sum_{s=0}^{t-1} a_s^2 \Vert \xi_s\Vert^2.
  \end{align*}
  By Hanson-Wright inequality, we have for all $\rho > 0$,
  \begin{equation}
    \mathbb{P}\left( \sum_{s=0}^{t-1} a_s^2 \Vert \xi_s \Vert^2 \le \sigma^2 ( d \Vert a_{0:t-1}\Vert^2_2 + 2\sqrt{d\Vert a_{0:t-1}\Vert^2_2\rho} + 2 \rho)\right) \ge 1 - e^{-\rho}.
  \end{equation}
  Next by Lemma \ref{lem:perturbed toeplitz}, we have
  \begin{align*}
    \Vert \Gamma(M, \Delta_{0:t}) \Vert \le \mathcal{G}_M(\varepsilon)\ \ \hbox{ and } \ \
    \Vert \mathcal{T}(M, \Delta_{0:t})\Vert  \le \mathcal{G}_M(\varepsilon).
  \end{align*}
  It follows that for all $\rho> 0$,
  $$
  \mathbb{P}\left( \sum_{s=0}^t \Vert y_s \Vert^2 \le 2 \mathcal{G}_M(\varepsilon) (\Vert a_{0:t}\Vert_\infty^2 \Vert y_0 \Vert^2 + \sigma^2 ( d \Vert a_{0:t-1}\Vert^2_2 + 2\sqrt{d\Vert a_{0:t-1}\Vert^2_2\rho} + 2 \rho)\right) \ge 1 - e^{-\rho}.
  $$
\end{proof}

\newpage
\section{Probabilistic Tools}

\subsection*{J.1\ \ Sub-gaussian vectors}

\begin{definition} A random vector $\xi$ taking values in $\mathbb{R}^d$ is said to be zero-mean, $\sigma^2$-sub-gaussian if $\E[\xi] = 0$ and
  \begin{equation*}
    \forall \theta \in \mathbb{R}^d : \qquad \E[\exp(\theta^\top \xi )] \le \exp\left(\frac{\Vert \theta \Vert^2 \sigma^2}{2}\right)
  \end{equation*}
\end{definition}
\begin{definition} A random variable $X$ taking values in $\mathbb{R}$ is said to be zero-mean, and $\lambda$-sub-exponential if $\E[X] = 0$ and
  $$
  \forall \vert s \vert \le \frac{1}{\lambda}: \quad \E[\exp(s (X^2 - \E[X^2]) )] \le \exp\left(\frac{s^2 \lambda^2}{2} \right)
  $$
\end{definition}

\begin{lemma}\label{lem:subexp}
Let $X$ be a zero-mean, $\sigma^2$-sub-gaussian random variable taking values in $\mathbb{R}$, then $X^2 - \E[X^2]$ is a zero-mean $(4\sigma)^2$-sub-exponential random variables taking values in $\mathbb{R}$.
\end{lemma}

\subsection*{J.2 \ \ Hanson-Wright inequality and $\epsilon$-net arguments}
We use the following version of Hanson-Wright inequality due to \cite{hsu2012}. This result result does not require strong independence assumptions with the caveat that it is only a one sided high probability bound. However this is sufficent for our purposes.

\begin{proposition}[Hanson-Wright inequality]\label{prop:HW}
  Let  $M \in \mathbb{R}^{m \times n}$ be a matrix, and $\xi$ be a zero-mean, $\sigma^2$-sub-gaussian random vector in $\mathbb{R}^d$. We have
  \begin{align}
   \forall \rho> 0, \quad   \mathbb{P}\left(  \Vert M \xi \Vert^2  > \sigma^2  (\Vert M \Vert_F^2  + 2 \Vert M^\top M \Vert_F \sqrt{\rho} + 2 \Vert M^\top M \Vert \rho)   \right) \le e^{-\rho}.
  \end{align}
\end{proposition}

The following lemma can be found for instance in \cite{vershynin2018high}.

\begin{lemma}[An $\epsilon$-net argument]\label{lem:netarg} Let $W$ be an $d \times d$ symmetric random matrix, and $\epsilon \in (0, 1/2)$. Furthermore, let $\mathcal{N}$ be $\epsilon$-net of $S^{d-1}$ with minimal cardinality. Then, for all $\rho > 0$, we have
  $$
  \mathbb{P}\left(\Vert W\Vert > \rho \right) \le \left(\frac{2}{\epsilon} + 1\right)^{d} \max_{x \in \mathcal{N}} \; \mathbb{P}\left( \vert x^\top W x \vert > (1-2\epsilon)\rho \right).
  $$
\end{lemma}

\medskip

\subsection*{J.3 \ \ Miscellaneous lemmas}

\begin{lemma}\label{lem:technical} For all $\alpha, a > 0$ and $b\in\mathbb{R}$, if $t^\alpha \ge (2a/\alpha)\log(2a/\alpha) + 2b$, then $t^\alpha \ge a \log(t) + b$.
\end{lemma}

\medskip

\begin{lemma}\label{lem:uniform_over_time}
Let $({\cal E}_t)_{t\ge 0}$ denote a sequence of events (defined on a probability space). Assume that for some $\alpha > 0$, we have
  $$
  \forall \delta \in (0,1),\ \ \forall t : t^\alpha \ge c_1 + c_2 \log(1/\delta), \ \ \ \mathbb{P}\left( \mathcal{E}_t \right) \ge 1 - \delta.
  $$
  Then,
  $$
  \forall \delta \in(0,1), \quad \mathbb{P}\left(\bigcap_{t\ge c_1' + c_2' \log(1/\delta) } \mathcal{E}_t \right) \ge 1-\delta,
  $$
  for some constants $c_1'$ and $c_2'$ with $c_1' \lesssim c_1 + (c_2/\alpha)\log(e c_2/\alpha)$, and $c_2' \lesssim c_2$.
\end{lemma}

\begin{proof}[Proof of Lemma \ref{lem:uniform_over_time}] The result stems from  the union bound. By assumption, we have
  $$
  \forall \delta \in (0,1), \ \ \forall t^\alpha  \ge c_1 + c_2 \log(t^2/\delta), \quad  \mathbb{P}(\mathcal{E}_t) \ge 1-\delta/t^2.
  $$
  By Lemma \ref{lem:technical}, we have
  $$
  t^\alpha  \ge 2c_1 + (4c_2/\alpha)\log(4c_2/\alpha)+ 2c_2 \log(1/\delta) \implies t \ge c_1 + c_2 \log(t^2/\delta).
  $$
  Thus
  $$
  \forall \delta \in (0,1), \forall t \ge 2c_1 + 4c_2 \log(4c_2)+ 2c_2 \log(1/\delta): \qquad  \mathbb{P}(\mathcal{E}_t) \ge 1-\delta/t^2.
  $$
 Using the union bound, we get
  $$
  \forall \delta \in (0,1), \quad \mathbb{P}\left(\bigcap_{t\ge c_1' + c_2' \log(1/\delta) } \mathcal{E}_t \right) \ge 1- \pi^2\delta/6
  $$
  where $c_1' = 2c_1 + (4c_2/\alpha) \log(4c_2/\alpha)$ and $c_2' = 2c_2$. Reparametrizing $\delta' = \pi^2\delta/6$ gives the desired result.
\end{proof}

\medskip

\end{document}